\documentclass[12pt,oneside,letterpaper]{report}

\newcommand{\draftfinal}[2]{\ifdefined\draftversion#1\else#2\fi}

\newcommand{\finalonly}[1]{\draftfinal{}{#1}}

\newcommand{\thesistitle}{Theory and Algorithms for Learning with\\ Multi-Class Abstention and Multi-Expert Deferral}
\newcommand{\thesisauthor}{Anqi Mao}
\newcommand{\thesisadvisor}{Professor Mehryar Mohri}

\newcommand{\thesisdept}{Mathematics}
\newcommand{\gradmonth}{January}
\newcommand{\gradyear}{2025}
\newcommand{\thesisdedication}{}

\RequirePackage[margin=1in, includefoot, letterpaper]{geometry}

\RequirePackage[prologue]{xcolor}
\definecolor[named]{ThesisBlue}{cmyk}{1,0.1,0,0.1}
\definecolor[named]{ThesisYellow}{cmyk}{0,0.16,1,0}
\definecolor[named]{ThesisOrange}{cmyk}{0,0.42,1,0.01}
\definecolor[named]{ThesisRed}{cmyk}{0,0.90,0.86,0}
\definecolor[named]{ThesisLightBlue}{cmyk}{0.49,0.01,0,0}
\definecolor[named]{ThesisGreen}{cmyk}{0.20,0,1,0.19}
\definecolor[named]{ThesisPurple}{cmyk}{0.55,1,0,0.15}
\definecolor[named]{ThesisDarkBlue}{cmyk}{1,0.58,0,0.21}

\definecolor{SchoolColor}{rgb}{0.3412, 0.0235, 0.5490} 
\definecolor{chaptergrey}{rgb}{0.2600, 0.0200, 0.4600} 
\definecolor{midgrey}{rgb}{0.4, 0.4, 0.4}

\usepackage{chemformula}
\usepackage{array}

\RequirePackage[labelfont={bf,sf,small,singlespacing},
                textfont={sf,small,singlespacing},
                margin=0pt,
                figurewithin=chapter,
                tablewithin=chapter]{caption}

\usepackage{fix-cm}
\RequirePackage[raggedright,sc]{titlesec}
\definecolor{gray75}{gray}{0.75}
\newcommand{\hsp}{\hspace{20pt}}

\titleformat{\chapter}[hang]
{\Huge\sc}
{\textcolor{SchoolColor}{\thechapter}\hsp\textcolor{gray75}{|}\hsp}
{0pt}{\Huge\sc\raggedright}

\usepackage{setspace}

\finalonly{
  \doublespacing 
}

\usepackage{lipsum}

\input{defs}

\usepackage{comment}

\usepackage{hyperref}
\hypersetup{colorlinks,
  linkcolor=ThesisDarkBlue,
  citecolor=ThesisPurple,
  urlcolor=ThesisDarkBlue,
  filecolor=ThesisDarkBlue}
 
\usepackage[
    backend=biber,
    style=authoryear-comp,
    natbib=true,
    url=false, 
    doi=true,
    eprint=false
]{biblatex}
\newtheorem{lemma}[theorem]{Lemma}
 
 \usepackage{lscape}
 
\begin{document}

\pagenumbering{roman}
%
\thispagestyle{empty}
%

\vspace*{25pt}
\begin{center}

  {\Large
    \begin{doublespace}
      {\textcolor{SchoolColor}{\textsc{\thesistitle}}}
    \end{doublespace}
  }
  \vspace{.7in}

  by
  \vspace{.7in}

  \thesisauthor
  \vfill

  \begin{doublespace}
    \textsc{
    A dissertation submitted in partial fulfillment\\
    of the requirements for the degree of\\
    Doctor of Philosophy\\
    Department of \thesisdept\\
    New York University\\
    \gradmonth, \gradyear}
  \end{doublespace}
\end{center}
\vfill

\noindent\makebox[\textwidth]{\hfill\makebox[2.5in]{\hrulefill}}\\
\makebox[\textwidth]{\hfill\makebox[2.5in]{\hfill\thesisadvisor}}

\newpage

\thispagestyle{empty}
\vspace*{25pt}
\begin{center}
  \scshape \noindent \small \copyright \  \small  \thesisauthor \\
  All rights reserved, \gradyear
\end{center}
\vspace*{0in}
\newpage

\cleardoublepage
\phantomsection
\chapter*{Dedication}
\addcontentsline{toc}{chapter}{Dedication}
\vspace*{\fill}
\begin{center}
  \thesisdedication
  To my family.
\end{center}
\vfill
\newpage

\chapter*{Acknowledgements}
\addcontentsline{toc}{chapter}{Acknowledgements}

 First and foremost, I would like to express my sincere gratitude to Prof. Mehryar Mohri for introducing me to the world
of machine learning and for his unwavering support over the years. Prof. Mohri has contributed
numerous profound insights to the field, including the foundation of this thesis, and I am deeply
grateful for the opportunity to learn directly from him.

I would also like to thank my committee members. I had the honor of collaborating with Corinna Cortes
and learned a lot from her, particularly on the effective empirical evaluation of different
algorithms. I thank Pranjal Awasthi for hosting my internships at Google Research and for our
collaboration on adversarial robustness. I appreciate Prof. Esteban Tabak for being my first-year faculty
mentor and for his support as the director of graduate studies. Lastly, I thank Prof. Yanjun Han for his
time and effort on the thesis and defense committee.

Finally, I would like to thank Prof. Georg Stadler for introducing me to the Courant Institute and for helping me prepare for my qualifying exams. I am grateful to my peers and friends at the Courant Institute for their support and friendship over the past five years. I also thank Michelle Shin and Gehan Abreu De Colon for their invaluable assistance and support as part of the administrative staff.


\chapter*{Abstract}
\addcontentsline{toc}{chapter}{Abstract}

Large language models (LLMs) have achieved remarkable performance on
diverse tasks across multiple domains. However, their practical
application faces two critical challenges: the occurrence of
\emph{hallucinations}, that is the generation of incorrect or
misleading content, and an inefficient inference.  Leveraging multiple
experts can address both issues.  To reduce hallucinations, one can
refrain from using the original predictor in uncertain instances and
defer to one of the more complex and more accurate experts. To enhance
efficiency, one can derive models of different sizes distilled from
the original complex model and use one of these more streamlined
versions, while deferring to the more complex and less efficient ones
for suitable contexts.  Both problems require assigning each instance
to the most suitable expert. This motivates the problem of
\emph{learning with multiple-expert deferral}.  Another problem closely related to
deferral is that of \emph{learning with multi-class abstention}, where incorrect multi-class predictions can be costly and it is
then preferable to abstain from making predictions for some input
instances, since the cost of abstention is typically less significant. This can be considered a special case of multi-expert deferral, characterized by a single expert and a constant cost.

This thesis presents a comprehensive study of learning
with multi-class abstention and multi-expert deferral, all supported by strong consistency guarantees. 

Learning with abstention is a framework where the learner can abstain from making a multi-class prediction with some pre-defined cost. We analyze the score-based formulation and the predictor-rejector formulation of learning with abstention in the multi-class classification setting.
We introduce several new families of surrogate losses for which we
prove strong non-asymptotic and hypothesis set-specific consistency
guarantees, thereby resolving positively two existing open questions. These guarantees provide upper bounds on the estimation error
of the abstention loss function in terms of that of the surrogate
loss and can help compare different surrogates. We analyze both a single-stage setting where the predictor and
rejector are learned simultaneously and a two-stage setting crucial in
applications, where the predictor is learned in a first stage using a
standard surrogate loss such as cross-entropy.
These guarantees suggest new multi-class abstention algorithms based
on minimizing these surrogate losses.
We experimentally evaluate our new algorithms on
CIFAR-10, CIFAR-100, and SVHN datasets. Our results demonstrate empirically the benefit of our new surrogate losses and show the remarkable performance of our broadly applicable two-stage abstention algorithm.

Learning with multi-expert deferral is a key scenario where the learner can choose to defer the prediction to several experts. We present a study of surrogate losses and algorithms for this general problem in the classification setting for both the single-stage and two-stage scenarios. We first
introduce a new family of surrogate losses specifically tailored for the multiple-expert setting, where the prediction and deferral functions are learned simultaneously. We then prove that these surrogate losses benefit from strong $\sH$-consistency bounds. We illustrate the application of our analysis through several examples of
practical surrogate losses, for which we give explicit guarantees.
The two-stage scenario is crucial in practice for many applications. Here, a predictor is derived in a first stage by training with a
common loss function such as cross-entropy. In the second stage, a
deferral function is learned to assign the most suitable expert to each input.
We design a new family of surrogate loss functions for
this scenario both in the score-based and the predictor-rejector
settings and prove that they are supported by $\sH$-consistency
bounds, which implies their Bayes-consistency. Moreover, we show that, for a constant cost function, our two-stage surrogate losses are realizable $\sH$-consistent.
These loss functions readily lead to the design of new learning to
defer algorithms based on their minimization. 

While learning with multi-expert deferral has received significant attention in classification
contexts, it presents unique challenges in regression due to the
infinite and continuous nature of the label space. We further introduce a novel framework of \emph{regression with deferral}, propose new surrogate loss functions and prove that they are supported by
$\sH$-consistency bounds. Our framework is versatile, applying to
multiple experts, accommodating any bounded regression losses,
addressing both instance-dependent and label-dependent costs, and
supporting both single-stage and two-stage methods.  Our single-stage
formulation subsumes as a special case the recent \emph{regression
with abstention} \citep{cheng2023regression} framework, where only a
single expert is considered, specifically for the squared loss and a
label-independent cost.  Minimizing our proposed loss functions
directly leads to novel algorithms for regression with deferral. We
report the results of extensive experiments showing the effectiveness
of our proposed algorithms.

\newpage

\tableofcontents

\cleardoublepage
\phantomsection
\addcontentsline{toc}{chapter}{List of Figures}
\listoffigures
\newpage

\cleardoublepage
\phantomsection
\addcontentsline{toc}{chapter}{List of Tables}
\listoftables
\newpage

\cleardoublepage
\phantomsection
\addcontentsline{toc}{chapter}{List of Appendices}
\listofappendices
\newpage

\pagenumbering{arabic} 




\chapter*{Introduction} \label{ch0}
\addcontentsline{toc}{chapter}{Introduction}

The accuracy of learning algorithms can be greatly enhanced by
redirecting uncertain predictions to experts or advanced pre-trained
models. Experts can be individuals with specialized domain knowledge
or more sophisticated, albeit costly, pre-trained models. The cost of
an expert is important to consider, as it may capture the
computational resources it requires or the quality of its performance.
The cost can further be instance-dependent and label-dependent.

How can we effectively assign each input instance to the most suitable
expert among a pool of several, considering both accuracy and cost?
This is the challenge of \emph{learning with multi-expert deferral}, which is prevalent in various domains, including
natural language generation tasks, speech recognition,
image annotation and classification, medical diagnosis, financial
forecasting, natural language processing, computer vision, and many
others.  For instance, contemporary language models and
dialog-based text generation systems have exhibited susceptibility to
generating erroneous information, often referred to as
\emph{hallucinations}.  Thus, their response quality can be
substantially improved by deferring uncertain predictions to more
advanced or domain-specific pre-trained models. This particular issue
has been recognized as a central challenge for large language models
(LLMs) \citep{WeiEtAl2022,bubeck2023sparks}. 

The problem of \emph{learning to defer (L2D)}, or the special case of \emph{learning with
abstention} characterized by a single expert and constant cost, has 
received much attention in classification tasks. Previous work on
this topic mainly includes the following formulations or methods:
\emph{confidence-based methods}, which consist of
abstaining when the score returned by a pre-trained model falls below
some threshold
\citep{Chow1957,chow1970optimum,bartlett2008classification,yuan2010classification,WegkampYuan2011,ramaswamy2018consistent,NiCHS19}; \emph{selective
classification}, which analyzes a set-up with a \emph{predictor} and a
\emph{selector} and defines a selection risk or loss normalized by the expected selection or coverage
\citep{el2010foundations,wiener2011agnostic,el2012active,wiener2015agnostic,geifman2017selective,geifman2019selectivenet}; a
\emph{predictor-rejector formulation}, which is based on learning both
a \emph{predictor} and a \emph{rejector}, each from a different family
of functions, and that takes into account explicitly the 
abstention cost $c$
\citep{CortesDeSalvoMohri2016,CortesDeSalvoMohri2016bis,CortesDeSalvoMohri2023,cheng2023regression,MohriAndorChoiCollinsMaoZhong2024learning,li2024no}; and a more
recent \emph{score-based formulation} that consists of augmenting the
multi-class categories with a rejection label and of abstaining when
the score assigned to the rejection label is the highest
\citep{mozannar2020consistent,caogeneralizing,MaoMohriZhong2024score}.

The study of confidence-based methods was initiated by \citet{Chow1957,chow1970optimum} who explored the
trade-off between error rate and rejection rate, and also presented an
analysis of the Bayes optimal decision in this context. Later,
\citet{FumeraRoliGiacinto2000} proposed a multiple thresholds rule for
situations where a posteriori probabilities were impacted by
errors. \citet{Tortorella2001} introduced an optimal rejection rule
for binary classifiers, relying on the Receiver Operating
Characteristic (ROC) curve. Additionally, \citet{SantosPires2005}
compared their methodology with that of \citet{chow1970optimum}. Numerous publications have proposed various rejection techniques to
reduce the misclassification rate, though without theoretical analysis
\citep{FumeraRoli2002, Pietraszek2005, BounsiarGrallBeauseroy2007,
  LandgrebeTaxPaclikDuin2005, Melvin2008}. \citet{HerbeiWegkamp2005}
examined classification with a rejection option involving a cost and
provided excess error bounds for these ternary functions.  \citet{bartlett2008classification}
developed a loss function for this scenario that takes into account
the abstention cost $c$. They proposed learning a predictor using a
\emph{double hinge loss} and demonstrated its consistency
benefits. This approach has been further explored in several
subsequent publications \citep{GrandvaletKeshetRakotomamonjyCanu2008,
  yuan2010classification,WegkampYuan2011}. \citet{ramaswamy2018consistent} further examined
confidence-based abstention in multi-class classification, showing
that certain multi-class hinge loss formulations and a newly
constructed polyhedral binary-encoded predictions (BEP) surrogate loss
are Bayes-consistent. \citet{charoenphakdee2021classification}
suggested a cost-sensitive approach for multi-class abstention by
breaking down the multi-class problem into multiple binary
cost-sensitive classification problems
\citep{elkan2001foundations}. They introduced a family of
cost-sensitive one-versus-all surrogate losses, which are
Bayes-consistent in that context. \citet{narasimhan2023learning} investigated the connection between learning with abstention and out-of-distribution detection. They developed a plug-in method aimed at approximating the Bayes-optimal classifier, and demonstrated its application in the context of learning an out-of-distribution (OOD) aware classifier.

Selective classification methods were introduced by \citet{el2010foundations} who
investigated the trade-off between classifier coverage and
accuracy. In a follow-up study, \citet{wiener2011agnostic} developed a
strategy for learning a specific kind of selective classification
called weakly optimal, which has a diminishing rejection rate under
certain Bernstein-type conditions. Many successful connections to
selective classification have been established, including active learning
\citep{el2012active,wiener2015compression,wiener2015agnostic,puchkin2021exponential,denis2022active,zhu2022efficient},
multi-class rejection
\citep{tax2008growing,dubuisson1993statistical,le2010optimum},
reinforcement learning \citep{li2008knows}, online learning
\citep{zhang2016extended}, modern confidence-based rejection methods
\citep{geifman2017selective}, neural network architectures  \citep{geifman2019selectivenet}, loss functions based on
gambling's doubling rate \citep{ziyin2019deep}, disparity-free
approaches \citep{schreuder2021classification}, and the abstention
problem in a ``confidence set'' framework
\citep{gangrade2021selective,chzhen2021set}.

The predictor-rejector formulation was advocated by \citet*{CortesDeSalvoMohri2016} who contended that
confidence-based abstention is generally suboptimal, unless the
learned predictor is the Bayes classifier. They demonstrated that, in
most cases, no threshold-based abstention can achieve the desired
outcome. They proposed a new abstention framework that involves
learning both a predictor $h$ and a rejector $r$
\emph{simultaneously}, which can generally differ from a
threshold-based function. They defined a predictor-rejector
formulation loss function for the pair $(h, r)$, considering the
abstention cost $c$. The authors provided Rademacher complexity-based
generalization bounds for this learning problem and proposed various
surrogate loss functions for the binary classification abstention
loss. They demonstrated that these surrogate losses offered
consistency guarantees and developed algorithms based on these
surrogate losses, which empirically outperformed confidence-based
abstention benchmarks. This work led to several follow-up studies,
including a theoretical and algorithmic investigation of boosting with
abstention \citep{CortesDeSalvoMohri2016bis} and an analysis of
extending the results to a multi-class setting \citep{NiCHS19}. These
authors acknowledged the difficulty in designing calibrated or
Bayes-consistent surrogate losses based on the predictor-rejector
abstention by \citet{CortesDeSalvoMohri2016} and left it as an open
question. Furthermore, \cite{cheng2023regression} applied this framework in the context of regression with abstention, introducing Bayes-consistent surrogate losses.
\citet{MohriAndorChoiCollinsMaoZhong2024learning} examined the framework in the scenario of learning with a fixed predictor, where they proposed novel algorithms for decontextualization tasks. Additionally, \citet{li2024no} studied the Bayes-consistency of
no-rejection learning for regression with
abstention. 

\citet{mozannar2020consistent} introduced an alternative
\emph{score-based formulation} for multi-class abstention. In this
approach, besides the standard scoring functions associated with each
label, a new scoring function is linked to a new rejection
label. Rejection occurs when the score assigned to the rejection label
exceeds other scores, implicitly defining the rejector through this
specific rule. The authors proposed a surrogate loss for their method
based on cross-entropy (logistic loss with softmax applied to neural
network outputs), which they demonstrated to be Bayes-consistent. Building upon this work,
\citet{caogeneralizing} presented a more comprehensive collection of
Bayes-consistent surrogate losses for the score-based
formulation. These surrogate losses can be constructed using any
consistent loss function for the standard multi-class classification
problem. \citet{verma2022calibrated}
proposed an alternative Bayes-consistent surrogate loss, the
one-versus-all loss, which was later examined within a broader family
of loss functions \citep{charusaie2022sample}. \citet{cao2023defense}
proposed an asymmetric softmax function, which can induce a valid
probability estimator for learning to
defer. \cite{pmlr-v206-mozannar23a} showed that the surrogate losses
in \citep{mozannar2020consistent,verma2022calibrated} are not
realizable $\sH$-consistent. They proposed an alternative surrogate
loss that is realizable $\sH$-consistent, but they were unable to
prove or disprove whether the proposed surrogate loss is
Bayes-consistent. In particular, they left open the problem of finding
surrogate losses that are both Bayes-consistent and realizable
$\sH$-consistent when the cost function for the expert is its
classification error. The problem becomes even more challenging when
considering more general and realistic cost functions.

Additional studies have focused on post-hoc methods, with
\citet{okati2021differentiable} suggesting an alternative optimization
technique between the predictor and rejector, and
\citet{narasimhanpost} offering corrections for underfitting surrogate
losses \citep{liu2024mitigating}.  The L2D framework or variations
thereof have found applications in diverse scenarios, spanning
regression, reinforcement learning, and human-in-the-loop systems,
among others
\citep{de2020regression,de2021classification,straitouri2021reinforcement,zhao2021directing,joshi2021pre,gao2021human,mozannar2022teaching,hemmer2023learning,chen2024learning}.

All the studies mentioned so far mainly focused on learning to defer
with a single expert, or the special case of learning with abstention characterized by a constant cost. Most recently, \citet{verma2023learning}
highlighted the significance of learning to defer with multiple
experts
\citep{hemmer2022forming,keswani2021towards,kerrigan2021combining,straitouri2022provably,benz2022counterfactual,tailor2024learning} and extended the
surrogate loss in \citep{verma2022calibrated,mozannar2020consistent}
to accommodate the multiple-expert setting, which is the first work to propose Bayes-consistent surrogate losses in this
scenario. They further showed that a mixture of experts (MoE) approach
to multi-expert L2D proposed in \citep{hemmer2022forming} is not
consistent.

Meanwhile, recent work by \citet{awasthi2022Hconsistency,AwasthiMaoMohriZhong2022multi} introduced new consistency
guarantees, called $\sH$-consistency bounds, which they argued are
more relevant to learning than Bayes-consistency since they are
hypothesis set-specific and non-asymptotic. $\sH$-consistency bounds
are also stronger guarantees than Bayes-consistency. They established
$\sH$-consistent bounds for common surrogate losses in standard
classification (see also \citep{MaoMohriZhong2023cross,zheng2023revisiting,MaoMohriZhong2023characterization}). This naturally
raises the question: can we design deferral and abstention surrogate losses that
benefit from these more significant consistency guarantees?

This thesis presents a comprehensive study of learning
with multi-class abstention and multi-expert deferral, supported by strong consistency guarantees. We introduce novel families of surrogate losses for the abstention loss
function and the more general deferral loss function across different settings. We prove strong non-asymptotic and hypothesis
set-specific consistency guarantees for these surrogate losses, which
upper-bound the estimation error of the abstention/deferral loss function in
terms of the estimation error of the surrogate loss. We further experimentally evaluate our new algorithms on benchmark datasets, highlighting the practical significance of our new surrogate losses and algorithms for deferral and abstention.


In Chapter~\ref{ch2}, we present a series of new theoretical and algorithmic
results for multi-class classification for the score-based abstention
formulation. We first formalize the
setting and define explicitly the underlying abstention loss. We
then show how the general family of surrogate losses introduced by
\citet{caogeneralizing} can be naturally derived from that expression. More importantly, we prove \emph{$\sH$-consistency bounds} for these
surrogate losses, which are
non-asymptotic and hypothesis set-specific guarantees upper-bounding
the estimation error of the abstention loss function in terms of the
estimation error of the surrogate loss \citep{AwasthiMaoMohriZhong2022multi}. These
provide stronger guarantees than Bayes-consistency guarantees, which
only provide an asymptotic guarantee and hold only for the full family
of measurable functions. We first derive our guarantees for a broad
family of score-based abstention surrogates, which we name
\emph{\compsum\ score-based surrogate losses}. These include the
surrogate losses in \citep{mozannar2020consistent,caogeneralizing},
for which our guarantees admit their Bayes-consistency as a special
case. Our theory can also help compare different surrogate losses. To
make it more explicit, we give an explicit analysis of the
\emph{minimizability gaps} appearing in our bounds. We further prove a
general result showing that an $\sH$-consistency bound in standard
classification yields immediately an $\sH$-consistency bound for
score-based abstention losses. Minimization of these new surrogate
losses directly leads to new algorithm for multi-class abstention.

In Chapter~\ref{ch3}, we present a series of new
theoretical and algorithmic results for multi-class
learning with abstention in predictor-rejector formulation and, in
particular, resolve an open question proposed by \citet{NiCHS19} in a strongly positive way. We first show that in some instances
the optimal solution cannot be derived in the score-based
formulation, unless we resort to more complex scoring functions. In
contrast, the solution can be straightforwardly derived in the
predictor-rejector formulation. We then present and analyze a new family of
surrogate loss functions for multi-class abstention in the
predictor-rejector formulation. We first give a negative result, ruling out abstention surrogate
losses that do not
verify a technical condition. Next, we present several
positive results for abstention surrogate
losses verifying that condition, for which we prove non-asymptotic \emph{$(\sH,\sR)$-consistency bounds}
that are stronger than Bayes-consistency. We further discuss in detail the difference between the
  predictor-rejector formulation and the score-based formulation,
  which underscores our work's innovation and significant
  contribution.

In Chapter~\ref{ch4}, we study the general framework of learning with multi-expert deferral. We first introduce a new family of surrogate losses specifically
tailored for the multiple-expert setting, where the prediction and
deferral functions are learned simultaneously. Next, we prove that
these surrogate losses benefit from $\sH$-consistency bounds.
This implies, in particular, their Bayes-consistency.
We illustrate the application of our analysis through several examples
of practical surrogate losses, for which we give explicit guarantees.
These loss functions readily lead to the design of new learning to
defer algorithms based on their minimization. Our $\sH$-consistency bounds incorporate a crucial term known as the
\emph{minimizability gap}.  We show that this makes them more
advantageous guarantees than bounds based on the approximation error. We further demonstrate that our $\sH$-consistency bounds can be used
to derive generalization bounds for the minimizer of a surrogate loss
expressed in terms of the minimizability gaps. While the main focus of this chapter is a theoretical analysis, we also report the results of several experiments with SVHN and CIFAR-10 datasets.

In Chapter~\ref{ch5}, we study a two-stage scenario for
learning with multi-expert deferral that is crucial in practice
for many applications.  In this scenario, a predictor is derived in a
first stage by training with a common loss function such as
cross-entropy.  In the second stage, a deferral function is learned to
assign the most suitable expert to each input.  We design a new family
of surrogate loss functions for this scenario both in the
\emph{score-based setting} and the
\emph{predictor-rejector setting}
and prove that they
are supported by $\sH$-consistency bounds, which implies their
Bayes-consistency. While the main focus of this chapter is a theoretical
analysis, we also report the results of several experiments on
CIFAR-10 and SVHN datasets.

In Chapter~\ref{ch6}, we deal with the problem of learning with multi-expert deferral in the regression setting. While this problem has received
significant attention in classification contexts, it presents unique challenges in regression due
to the infinite and continuous nature of the label space. In
particular, the \emph{score-based formulation} commonly used in
classification is inapplicable here, since regression problems cannot
be represented using multi-class scoring functions, with auxiliary
labels corresponding to each expert. Our approach involves defining prediction and deferral functions,
extending the predictor-rejector formulation in classification to the regression setting. We present a
comprehensive analysis for both the single-stage scenario
(simultaneous learning of predictor and deferral functions), and the two-stage scenario
(pre-trained predictor with learned deferral function). We introduce new surrogate loss
functions for both scenarios and prove that they are supported by
$\sH$-consistency bounds. These are consistency guarantees that are
stronger than Bayes consistency, as they are non-asymptotic and
hypothesis set-specific. Our framework is versatile, applying to
multiple experts, accommodating any bounded regression losses,
addressing both instance-dependent and label-dependent costs, and
supporting both single-stage and two-stage methods. We also
instantiate our formulations in the special case of a single expert, and demonstrate that our
single-stage formulation includes the recent \emph{regression with
abstention} framework \citep{cheng2023regression} as a special case,
where only a single expert, the squared loss and a label-independent
cost are considered. We further report the
results of extensive experiments showing the effectiveness of our
proposed algorithms.

This thesis is based on \citet*{MaoMohriZhong2024score,MaoMohriZhong2024predictor,MaoMohriZhong2024deferral,MaoMohriMohriZhong2023two,mao2024regression}.


\chapter{Score-Based Multi-Class Abstention} \label{ch2}
In this chapter, we present a series of new theoretical and algorithmic
results for multi-class classification with score-based abstention
formulation. In Section~\ref{sec:preliminary}, we formalize the
setting and first define explicitly the underlying abstention loss. We
then show how the general family of surrogate losses introduced by
\citet{caogeneralizing} can be naturally derived from that expression
in Section~\ref{sec:score-general}.

More importantly, we prove \emph{$\sH$-consistency bounds} for these
surrogate losses (Section~\ref{sec:score-bounds}), which are
non-asymptotic and hypothesis set-specific guarantees upper-bounding
the estimation error $\paren{\sE_{\labsc}( h) - \sE_{\labsc}^*( \sH)}$ of the abstention loss function in terms of the
estimation error of the surrogate loss \citep{AwasthiMaoMohriZhong2022multi}. These
provide stronger guarantees than Bayes-consistency guarantees, which
only provide an asymptotic guarantee and hold only for the full family
of measurable functions. We first derive our guarantees for a broad
family of score-based abstention surrogates, which we name
\emph{\compsum\ score-based surrogate losses}. These include the
surrogate losses in \citep{mozannar2020consistent,caogeneralizing},
for which our guarantees admit their Bayes-consistency as a special
case. Our theory can also help compare different surrogate losses. To
make it more explicit, we give an analysis of the
minimizability gaps appearing in our bounds. We further prove a
general result showing that an $\sH$-consistency bound in standard
classification yields immediately an $\sH$-consistency bound for
score-based abstention losses. Minimization of these new surrogate
losses directly leads to new algorithm for multi-class abstention.

In Section~\ref{sec:two-stage-mabsc}, we analyze a two-stage algorithmic
scheme often more relevant in practice, for which we give surrogate
losses that we prove to benefit from $\sH$-consistency bounds. These
are also non-asymptotic and hypothesis set-specific guarantees
upper-bounding the estimation error of the abstention loss function in
terms of the estimation error of the first-stage surrogate loss and
second-stage one. Minimizing these new surrogate losses directly leads
to new algorithm for multi-class abstention.

In Section~\ref{sec:realizable-mabsc}, we demonstrate that our proposed
two-stage score-based surrogate losses are not only Bayes-consistent,
but also realizable $\sH$-consistent. This effectively addresses the
open question posed by \citet{pmlr-v206-mozannar23a} and highlights
the benefits of the two-stage formulation.

In Section~\ref{sec:finite-sample}, we show that our $\sH$-consistency
bounds can be directly used to derive finite sample estimation bounds
for a surrogate loss minimizer of the abstention loss. These are more
favorable and more relevant guarantee than a similar finite sample
guarantee that could be derived from an excess error bound.

In Section~\ref{sec:experiments-mabsc}, we report the results of several
experiments comparing these algorithms and discuss them in light of
our theoretical guarantees.  Our empirical results show, in
particular, that the two-stage score-based abstention surrogate loss
consistently outperforms the state-of-the-art cross-entropy
scored-based abstention surrogate losses on CIFAR-10, CIFAR-100 and
SVHN, while highlighting that the relative performance of the
state-of-the-art \compsum\ scored-based abstention losses varies by
the datasets. We present a summary of our main contribution as follows
and start with a formal description of the problem formulations.
\begin{itemize}

    \item Derivation of the cross-entropy score-based surrogate loss
      from first principles, which include the state-of-the-art
      surrogate losses as special cases.

    \item $\sH$-consistency bounds for cross-entropy score-based
      surrogate losses, which can help theoretically compare different
      cross-entropy score-based surrogate losses and guide the design
      of a multi-class abstention algorithm in comparison to the
      existing asymptotic consistency guarantees.

    \item A novel family of surrogate loss functions in the two-stage
      setting and their strong $\sH$-consistency bound guarantees.

    \item Realizable $\sH$-consistency guarantees of proposed
      two-stage score-based surrogate loss, which effectively
      addresses the open question posed by
      \citet{pmlr-v206-mozannar23a} and highlights the benefits of the
      two-stage formulation.

    \item Extensive experiments demonstrating the practical
      significance of our new surrogate losses and the varying
      relative performance of the state-of-the-art cross-entropy
      score-based surrogate losses across datasets.
\end{itemize}
The presentation in this chapter is based on \citep{MaoMohriZhong2024score}.

\section{Preliminaries}
\label{sec:preliminary}

We consider the standard multi-class classification setting with an
input space $\sX$ and a set of $n \geq 2$ classes or labels $\sY =
\curl*{1, \ldots, n}$. We will denote by $\sD$ a distribution over
$\sX \times \sY$ and by $\sfp(y \!\mid\! x)$, the conditional probability of $Y =
y$ given $X = x$, that is $\sfp(y \!\mid\! x) = \sD(Y = y \!\mid\! X = x)$. We
will also use $p(x) = \paren*{\sfp(1 \!\mid\! x), \ldots, \sfp(n \!\mid\! x)}$ to denote the
vectors of these probabilities for a given $x$.

We study the learning scenario of multi-class classification with
abstention in the
\emph{score-based formulation} proposed by
\citet{mozannar2020consistent} and recently studied by
\citet{caogeneralizing}.

\paragraph{Score-Based Abstention Formulation}

In this formulation of the abstention problem, the label set $\sY$ is
augmented with an additional category $(n + 1)$ corresponding to
abstention. We denote by $\sY \cup \curl*{n+1} = \curl*{1, \ldots, n,
  n + 1}$ the augmented set and consider a hypothesis set $ \sH$ of
functions mapping from $\sX \times (\sY \cup \curl*{n + 1})$ to
$\Rset$.
The label associated by $ h \in \sH$ to an input $x \in \sX$ is
denoted by $ \hh(x)$ and defined by $ \hh(x) = n + 1$ if $ h(x, n + 1)
\geq \max_{y \in \sY} h(x, y)$; otherwise, $ \hh(x)$ is defined as an
element in $\sY$ with the highest score, $ \hh(x) = \argmax_{y \in
  \sY} h(x, y)$, with an arbitrary but fixed deterministic strategy
for breaking ties.  When $ \hh(x) = n + 1$, the learner abstains from
making a prediction for $x$. Otherwise, it
predicts the label $y = \hh(x)$. The \emph{score-based abstention
loss} $\labsc$ for this formulation is defined as follows for any $ h
\in \sH$ and $(x, y) \in \sX \times \sY$:
\begin{equation}
\label{eq:abs-score}
\labsc( h, x, y)
= \1_{ \hh(x)\neq y}\1_{ \hh(x)\neq n + 1} + c(x) \1_{ \hh(x) = n + 1}.
\end{equation}
Thus, when it does not abstain, $ \hh(x) \neq n + 1$, the learner
incurs the familiar zero-one classification loss and when it abstains,
$ \hh(x) = n + 1$, the cost $c(x)$.  Given a finite sample drawn
i.i.d.\ from $\sD$, the learning problem consists of selecting a
hypothesis $ h$ in $ \sH$ with small expected score-based abstention
loss, $\E_{(x, y) \sim \sD}[\labsc( h, x, y)]$. Note that the cost $c$
implicitly controls the rejection rate when minimizing the abstention
loss.

Optimizing the score-based abstention loss is intractable for most
hypothesis sets. Thus, instead, learning algorithms for this scenario
must resort to a surrogate loss $\lsc$ for $\labsc$. In the next
sections, we will define score-based surrogate losses and analyze
their properties. Given a loss function $ \sfL$, we denote by
$\sE_{\lsc}( h) = \E_{(x, y) \sim \sD}\bracket*{\lsc( h, x, y)}$ the
generalization error or expected loss of $ h$ and by $\sE_{\lsc}^*(
\sH) = \inf_{ h \in \sH} \sE_{\lsc}( h)$ the minimal generalization
error. In the following, to simplify the presentation, we assume that
the cost function $c\in (0,1)$ is constant. However, many of our
results extend straightforwardly to the general case.

\paragraph{$\sH$-Consistency Bounds}

We will seek to derive \emph{$\sH$-consistency bounds} for
$\lsc$. These are strong guarantees that take the form of inequalities
establishing a relationship between the abstention loss $\labsc$ of
any hypothesis $h \in \sH$ and the surrogate loss $\lsc$ associated
with it
\citep{awasthi2021calibration,awasthi2021finer,awasthi2022Hconsistency,
  AwasthiMaoMohriZhong2022multi,AwasthiMaoMohriZhong2023theoretically,awasthi2024dc,
  MaoMohriZhong2023cross,MaoMohriZhong2023ranking,
  MaoMohriZhong2023rankingabs,zheng2023revisiting,
  MaoMohriZhong2023characterization,MaoMohriZhong2023structured}. These
are bounds of the form $\sE_{\labsc}( h) - \sE_{\labsc}^*( \sH) \leq
f\paren*{\sE_{\lsc}( h) - \sE_{\lsc}^*( \sH)}$, for some
non-decreasing function $f$, that upper-bounds the estimation error $\paren{\sE_{\labsc}( h) - \sE_{\labsc}^*( \sH)}$ of
the loss $\labsc$ in terms of that of $\lsc$ for a given hypothesis
set $ \sH$. Thus, they show that if we can reduce the surrogate
estimation error $(\sE_{\lsc}( h) - \sE_{\lsc}^*( \sH))$ to a small value $\e > 0$,
then the estimation error of $\labsc$ is guaranteed to be at most
$f(\e)$. These guarantees are non-asymptotic and take into
consideration the specific hypothesis set $\sH$ used.

\paragraph{Minimizability Gaps} 

A key quantity appearing in these bounds
is the \emph{minimizability gap}, denoted by $\sM_{\lsc}(\sH)$ and
defined by $\sM_{\lsc}(\sH) = \sE^*_{\lsc}( \sH) - \E_x
\bracket[\big]{\inf_{ h \in \sH} \E_y\bracket*{\lsc( h, X, y) \mid X =
    x}}$ for a given hypothesis set $\sH$. Thus, the minimizability
gap for a hypothesis set $\sH$ and loss function $\lsc$ measures the
difference of the best-in-class expected loss and the expected
pointwise infimum of the loss.  Since the infimum is super-additive,
it follows that the minimizability gap is always non-negative.
When the loss function $\lsc$ depends only on $h(x, \cdot)$ for all
$h$, $x$, and $y\in \sY$, that is, $\sfL(h, x, y) = \Psi(h(x, 1),
\ldots, h(x, n + 1), y)$ for some function $\Psi$, it can be shown
that the minimizability gap vanishes for the family of all measurable
functions: $\sM( \sH_{\rm{all}}) = 0$
\citep[lemma~2.5]{steinwart2007compare}. However, in general, the
minimizability gap is non-zero for restricted hypothesis sets $\sH$
and is therefore essential to analyze. It is worth noting that the
minimizability gap can be upper-bounded by the approximation error
$\sA_{\lsc}(\sH) = \sE^*_{\lsc}(\sH) - \E_x\bracket[\big]{\inf_{ h \in
    \sH_{\rm{all}}} \E_y \bracket{\lsc( h, X, y) \mid X =
    x}}$.
    However, the minimizability gap is a more refined quantity
than the approximation error and can lead to more favorable
guarantees (see Appendix~\ref{app:better-bounds}).
Note that in the approximation error $\sA_{\lsc}(\sH)$, the second term takes the infimum over all measurable functions $\sH_{\rm{all}}$, unlike in the minimizability gap $\sM_{\lsc}(\sH)$, where it is restricted to the hypothesis set $\sH$.

\section{Single-Stage Score-Based Formulation}
\label{sec:score}

In this section, we first derive the general form of a family of
surrogate loss functions $\lsc$ for $\labsc$ by analyzing the
abstention loss $\labsc$. Next, we give $ \sH$-consistency bounds for
these surrogate losses, which provide non-asymptotic hypothesis
set-specific guarantees upper-bounding the estimation error of the
loss function $\labsc$ in terms of estimation error of $\lsc$.

\subsection{General Surrogate Losses}
\label{sec:score-general}
Consider a hypothesis $ h$ in the score-based setting.  Note that
for any $(x, y) \in \sX \times \sY$, $ \hh(x) = n + 1$ implies $
\hh(x) \neq y$, therefore, we have: $\1_{ \hh(x)\neq y}\1_{
  \hh(x) = n + 1}= \1_{ \hh(x) = n + 1}$.  Thus,
$\labsc( h, x, y)$ can be rewritten as follows:
\begin{align*}
\labsc( h, x, y)
& = \1_{ \hh(x)\neq y}\paren*{1 - \1_{ \hh(x) = n + 1}} + c \1_{ \hh(x) = n + 1}\\
& = \1_{ \hh(x)\neq y} - \1_{ \hh(x)\neq y} \1_{ \hh(x) = n + 1} + c \1_{ \hh(x) = n + 1}\\
& = \1_{ \hh(x)\neq y} - \1_{ \hh(x) = n + 1} + c \1_{ \hh(x) = n + 1}\\
& = \1_{ \hh(x)\neq y} + (c - 1) \1_{ \hh(x) = n + 1}\\
& = \1_{ \hh(x)\neq y} + (1 - c) \1_{ \hh(x)\neq n + 1} + c - 1.
\end{align*}
In view of this expression, since the last term $(c - 1)$ is a
constant, if $ \ell$ is a surrogate loss for the zero-one
multi-class classification loss over the set of labels $ \sY$, then
$\lsc$ defined as follows is a natural surrogate loss for $\labsc$:
for all $(x, y) \in \sX \times \sY$,
\begin{equation}
\label{eq:sur-score-mabsc}
\lsc \paren*{ h, x, y}
=  \ell \paren*{ h, x, y} + (1 - c) \,  \ell\paren*{ h, x, n + 1}.
\end{equation}
This is precisely the form of the surrogate losses proposed by
\citet{mozannar2020consistent}, for which the analysis just presented
gives a natural derivation. This is also the form of the surrogate
losses adopted by \citet{caogeneralizing}.

\subsection{\texorpdfstring{$\sH$}{H}-Consistency Bound Guarantees}
\label{sec:score-bounds}

\citet{caogeneralizing} presented a nice study of the surrogate loss
$\lsc$ for a specific family of zero-one loss surrogates $ \ell$.
The authors showed that the surrogate loss $\lsc$ is Bayes-consistent
with respect to the score-based abstention loss $\labsc$ when $
\ell$ is Bayes-consistent with respect to the multi-class zero-one
classification loss $\ell_{0-1}$.  Bayes-consistency guarantees that,
asymptotically, a nearly optimal minimizer of $\lsc$ over the family
of all measurable functions is also a nearly optimal minimizer of
$\labsc$: $\lim_{n \to +\infty}
\sE_{\lsc}(h_n) - \sE_{\lsc}^*\paren*{\sH_{\mathrm{all}}} = 0$
implies $\lim_{n \to +\infty} \sE_{\labsc}(h_n) -
\sE_{\labsc}^*\paren*{\sH_{\mathrm{all}}} = 0$ for any distribution and any sequence
$\{h_n\}_{n\in \Nset} \subset \sH_{\rm{all}}$. However, this does not provide any guarantee for a restricted
subset $ \sH$ of the family of all measurable functions.  It also
provides no guarantee for approximate minimizers since convergence
could be arbitrarily slow and the result is only asymptotic.

In the following, we will prove $\sH$-consistency bound guarantees,
which are stronger results that are non-asymptotic and that hold for a
restricted hypothesis set $ \sH$. The specific instance of our results
where $\sH$ is the family of all measurable functions directly implies
the Bayes-consistency results of \citet{caogeneralizing}.

\paragraph{$\sH$-Consistency for Cross-Entropy Abstention Losses} We first prove $\sH$-consistency for a broad family of
score-based abstention surrogate losses $\lsc_{\mu}$, that we will
refer to as \emph{\compsum\ score-based surrogate losses}. These are
loss functions defined by
\begin{equation}
\label{eq:L-mu}
\lsc_{\mu} \paren*{ h, x, y}
=  \ell_{\mu} \paren*{ h, x, y} + (1 - c) \, \ell_{\mu}\paren*{ h, x, n + 1},    
\end{equation}
where, for any $ h\in  \sH$, $x\in \sX$, $y\in\sY$ and $\mu\geq 0$,
\begin{align*}
\ell_{\mu}(h,x, y) =\begin{cases}
\frac{1}{1 - \mu} \paren*{\bracket*{\sum_{y'\in\sY \cup \curl*{n+1}} e^{{ h(x, y') -  h(x, y)}}}^{1 - \mu} - 1} & \mu\neq 1  \\
\log\paren*{\sum_{y'\in \sY\cup \curl*{n+1}} e^{ h(x, y') -  h(x, y)}} & \mu = 1.
\end{cases}    
\end{align*}
The loss function $\ell_{\mu}$ coincides with the (multinomial)
logistic loss
\citep{Verhulst1838,Verhulst1845,Berkson1944,Berkson1951} when
$\mu=1$, matches the generalized cross-entropy loss
\citep{zhang2018generalized} when $\mu\in (1,2)$, and the mean
absolute loss \citep{ghosh2017robust} when $\mu=2$. Thus, the
\compsum\ score-based surrogate losses $ \sfL_{\mu}$ include the
abstention surrogate losses proposed in \citep{mozannar2020consistent}
which correspond to the special case of $\mu=1$ and the abstention
surrogate losses adopted in \citep{caogeneralizing}, which correspond
to the special case of $\mu\in [1, 2]$.

We say that a hypothesis set $\sH$ is \emph{symmetric} when the
scoring functions it induces does not depend on any particular
ordering of the labels, that is when there exists a family $\sF$ of
functions $f$ mapping from $\sX$ to $\Rset$ such that, for any $x \in
\sX$, $\curl*{\bracket*{h(x,1),\ldots,h(x,n),h(x,n+1)}\colon h\in \sH}
= \curl*{\bracket*{f_1(x),\ldots, f_n(x),f_{n+1}(x)}\colon f_1,
  \ldots, f_{n+1}\in \sF}$. We say that a hypothesis set $\sH$ is
\emph{complete} if the set of scores it generates spans $\Rset$, that
is, $\curl*{h(x, y)\colon h\in \sH} = \Rset$, for any $(x, y)\in \sX
\times \sY \cup \curl*{n+1}$. Common hypothesis sets used in practice,
such as the family of linear models, that of neural networks and of
course that of all measurable functions are all symmetric and
complete.  The guarantees given in the following result are thus
general and widely applicable.

\begin{restatable}[\textbf{$\sH$-consistency bounds for \compsum\
      score-based surrogates}]
  {theorem}{BoundCompSum}
\label{Thm:bound_comp_sum}
Assume that $\sH$ is symmetric and complete. Then, for any hypothesis
$h \in \sH$ and any distribution $\sD$, the following inequality holds:
\begin{equation*}
\sE_{\labsc}( h) - \sE_{\labsc}^*( \sH) + \sM_{\labsc}( \sH)
\leq \Gamma_{\mu}
  \paren*{\sE_{\lsc_{\mu}}( h) - \sE_{\lsc_{\mu}}^*(\sH)
    + \sM_{\lsc_{\mu}}(\sH)},
\end{equation*}
where $\Gamma_{\mu}(t)=\begin{cases}
\sqrt{(2-c)2^{\mu}(2-\mu) t} & \mu\in [0,1)\\
\sqrt{2(2-c)(n+1)^{\mu-1}t } & \mu\in [1,2) \\
(\mu - 1)(n+1)^{\mu - 1} t & \mu \in [2,+ \infty).
\end{cases}$
\end{restatable}
The proof is given in Appendix~\ref{app:bound_comp_sum}. It consists
of analyzing the
calibration gap (refer to the beginning of Appendix~\ref{app:bound_comp_sum} for its definition) of the score-based abstention loss $\labsc$ and that
of $\lsc_{\mu}$, and of finding a concave function $\Gamma_\mu$
relating these two quantities. Note that our proofs and results are
distinct, original, and more sophisticated than those in the standard
setting \citep{MaoMohriZhong2023cross}, where the standard loss
$\ell_{\mu}$ is analyzed. Establishing $\sH$-consistency bounds for
$\lsc_{\mu}$ is more intricate compared to $\ell_{\mu}$. This is
because the target loss in the score-based multi-class abstention is
inherently different from that of the standard multi-class scenario
(the multi-class zero-one loss). Thus, we need to tackle a more
complex calibration gap, integrating both the conditional probability
vector and the cost function. This complexity presents an added layer
of challenge when attempting to establish a lower bound for the
calibration gap of the surrogate loss in relation to the target loss
in the score-based abstention setting.

To understand the result, consider first the case where the
minimizability gaps are zero. As mentioned earlier, this would be the
case, for example, when $ \sH$ is the family of all measurable
functions or when $ \sH$ contains the Bayes classifier.
In that case, the theorem shows that if the estimation loss
$(\sE_{\lsc_{\mu}}( h) - \sE_{\lsc_{\mu}}^*(\sH))$ is reduced to $\e$,
then, for $\mu \in [0, 2)$, in particular for the logistic score-based
  surrogate ($\mu = 1$) and the generalized cross-entropy score-based
  surrogate ($\mu \in (1, 2)$), modulo a multiplicative constant, the
  score-based abstention estimation loss $(\sE_{\labsc}( h) -
  \sE_{\labsc}^*( \sH))$ is bounded by $\sqrt{\e}$. The bound is even
  more favorable for the mean absolute error score-based surrogate
  ($\mu = 2$) or for \compsum\ score-based surrogate $\lsc_{\mu}$ with
  $\mu \in (2, +\infty)$ since in that case, modulo a multiplicative
  constant, the score-based abstention estimation loss $(\sE_{\labsc}(
  h) - \sE_{\labsc}^*( \sH))$ is bounded by $\e$.
  
These are strong results since they are not asymptotic and are
hypothesis set-specific. In particular,
Theorem~\ref{Thm:bound_comp_sum} provides stronger guarantees than the
Bayes-consistency results of \cite{mozannar2020consistent} or
\cite{caogeneralizing} for cross-entropy abstention surrogate losses
\eqref{eq:L-mu} with the logistic loss ($\mu = 1$), generalized
cross-entropy loss ($\mu\in (1,2)$) and mean absolute error loss
($\mu=2$) adopted for $\ell$. These Bayes-consistency results can be
obtained by considering the special case of $ \sH$ being the family of
all measurable functions and taking the limit.
  
Moreover, Theorem~\ref{Thm:bound_comp_sum} also provides similar
guarantees for other types of \compsum\ score-based surrogate losses,
such as $\mu\in [0,1)$ and $\mu\in [2,+\infty)$, which are new
    surrogate losses for score-based multi-class abstention that, to
    the best of our knowledge, have not been previously studied in the
    literature. In particular, our $\sH$-consistency bounds can help
    theoretically compare different \compsum\ score-based surrogate
    losses and guide the design of a multi-class abstention
    algorithm. In contrast, asymptotic consistency guarantees given
    for a subset of \compsum\ score-based surrogate losses in
    \citep{mozannar2020consistent,caogeneralizing} do not provide any
    such comparative information.
    
    Recall that the minimizability gap is always upper-bounded by the
    approximation error. By Lemma~\ref{lemma:calibration_gap_score-mabsc} in
    Appendix~\ref{app:score-mabsc}, the minimizability gap for the
    abstention loss $\sM_{\labs}(\sH)$ coincides with the
    approximation error $\sA_{\labs}(\sH)$ when the labels generated
    by the hypothesis set encompass all possible outcomes, which
    naturally holds true for typical hypothesis sets. However, for a
    surrogate loss, the minimizability gap is in general a more
    refined quantity than the approximation error and can lead to more
    favorable guarantees. More precisely, $\sH$-consistency bounds
    expressed in terms of minimizability gaps are better and more
    significant than the excess error bounds expressed in terms of
    approximation errors (See Appendix~\ref{app:better-bounds} for a
    more detailed discussion).

\subsection{Analysis of Minimizability Gaps}
\label{sec:min-gaps}

In general, the minimizability gaps do not vanish and their magnitude,
$\sM_{\lsc_{\mu}}(\sH)$, is important to take into account when
comparing \compsum\ score-based surrogate losses, in addition to the
functional form of $\Gamma_\mu$. Thus, we will specifically analyze
them below.  Note that the dependency of the multiplicative constant
on the number of classes in some of these bounds ($\mu \in (1,
+\infty)$) makes them less favorable, while for $\mu \in [0, 1]$, the
bounds do not depend on the number of classes.

\ignore{The minimizability gap $\sM_{\lsc}( \sH) = \sE^*_{\lsc}( \sH)
  - \E_x \bracket[\big]{\inf_{ h \in \sH} \E_y\bracket*{\lsc( h, X, y)
      \mid X = x}}$ measures the difference between the best-in-class
  expected loss and the expected infimum of the pointwise expected
  loss. }
In the deterministic cases where for any $x\in \sX$ and $y\in \sY$,
either $\sfp(y \!\mid\! x) = 0$ or $1$, the pointwise expected loss admits an
explicit form. Thus, the following result characterizes the
minimizability gaps directly in those cases.  \ignore{We will consider
  the deterministic case where for any $x\in \sX$ and $y\in \sY$,
  either $\sfp(y \!\mid\! x) = 0$ or $1$.  We will specifically study the
  $\Lambda$-bounded hypothesis sets $ \sH_{\Lambda}$, that is, for
  fixed $x \in \sX$, we have $\curl*{\paren*{ h(x, 1), \ldots, h(x,
      n), h(x, n+1)}\colon h \in \sH}$ = $[-\Lambda, +\Lambda]^{n+1}$
  where $\Lambda \in [0, + \infty]$. The family of all measurable
  functions is a special $\Lambda$-bounded hypothesis set which
  corresponds to $\Lambda= + \infty$. The following theorem
  characterizes the minimizability gaps in those cases.
\begin{restatable}[\textbf{Characterization of minimizability gaps}]
  {theorem}{GapUpperBoundDetermi}
\label{Thm:gap-upper-bound-determi}
Assume that $ \sH$ is $\Lambda$-bounded. Then, for the
\compsum\ score-based surrogate losses $\lsc_{\mu}$ and any
deterministic distribution, the minimizability gaps can be upper
bounded as $\sM_{\lsc_{\mu}}( \sH) \leq
\sM_{\lsc_{\mu}}(\sH)(\Lambda)$, where
\begin{align*}
 \sM_{\lsc_{\mu}}(\sH)(\Lambda)=
 \min\curl*{\sT_{\mu}\paren*{\sE^*_{\lsc_{0}}( \sH)},c}
 - \min\curl*{\sT_{\mu}\paren*{\uv \sE^*(\sH)},c}
\end{align*}
with $\uv \sE^*(\sH)=e^{-2 \Lambda}\,n\leq \sE^*_{\lsc_{0}}( \sH)$ and
$\sT_{\mu}(t)$ defined as
\begin{equation*}
\sT_{\mu}(t)
=
\begin{cases}
\frac{1}{1 - \mu} \paren*{(1 + t)^{1 - \mu} - 1} & \mu \geq 0, \mu \neq 1 \\
\log(1 + t) & \mu = 1.
\end{cases}
\end{equation*}
\end{restatable}
See Appendix~\ref{app:score-mabsc} for the proof.
}

\begin{restatable}[\textbf{Characterization of minimizability gaps}]
  {theorem}{GapUpperBoundDetermi}
\label{Thm:gap-upper-bound-determi}
Assume that $ \sH$ is symmetric and complete. Then, for the \compsum\ score-based surrogate losses $\lsc_{\mu}$ and any deterministic distribution, the minimizability gaps can be characterized as follows:
\begin{align*}
\sM_{\lsc_{\mu}}( \sH) = \sE_{\lsc_{\mu}}^*(\sH) -  \begin{cases}
\frac{1}{1 - \mu} \bracket*{\bracket*{1+\paren*{1-c}^{\frac{1}{2-\mu}}}^{2 - \mu} - (2-c)} & \mu \notin \curl*{1,2}\\
-\log \paren*{\frac{1}{2-c}}-(1-c)\log \paren*{\frac{1-c}{2-c}}
 & \mu=1\\
1-c & \mu =2.
\end{cases}
\end{align*}
\end{restatable}
See Appendix~\ref{app:gap-upper-bound-determi} for the proof. By
l’H\^opital's rule, $\sE_{\lsc_{\mu}}^*(\sH)-\sM_{\lsc_{\mu}}( \sH)$
is continuous as a function of $\mu$ at $\mu = 1$. In light of the
equality $\lim_{x\to 0^{+}}\paren[\big]{1 + u^{\frac1x}}^x =
\max\curl*{1, u} = 1$, for $u \in [0, 1]$,
$\sE_{\lsc_{\mu}}^*(\sH)-\sM_{\lsc_{\mu}}( \sH)$ is continuous as a
function of $\mu$ at $\mu = 2$. Moreover, for any $c\in (0, 1)$,
$\sE_{\lsc_{\mu}}^*(\sH)-\sM_{\lsc_{\mu}}( \sH)$ is decreasing with
respect to $\mu$. On the other hand, since the function $\mu \mapsto
\frac{1}{1-\mu}\paren*{t^{1-\mu}-1} \1_{\mu \neq 1} + \log(t) \1_{\mu
  = 1}$ is decreasing for any $t > 0$, we obtain that $ \ell_{\mu}$ is
decreasing with respect to $\mu$, which implies that $\lsc_{\mu}$ is
decreasing and then $\sE_{\lsc_{\mu}}^*(\sH)$ is decreasing with
respect to $\mu$ as well. For a specific problem, a favorable
$\mu\in[0, \infty)$ is one that minimizes $\sM_{\lsc_{\mu}}( \sH)$,
  which, in practice, can be selected via cross-validation.

\subsection{General Transformation}

More generally, we prove the following result, which shows that an
$\sH$-consistency bound for $ \ell$ with respect to the
zero-one loss, yields immediately an $\sH$-consistency bound
for $\lsc$ with respect to $\labsc$.

\begin{restatable}{theorem}{BoundScoreMabsc}
\label{Thm:bound-score-mabsc}
Assume that $ \ell$ admits an $ \sH$-consistency bound with respect to
the multi-class zero-one classification loss $ \ell_{0-1}$ with a
concave function $\Gamma$, that is, for all $ h \in \sH$, the
following inequality holds:
\begin{equation*}
\sE_{\ell_{0-1}}( h) - \sE_{\ell_{0-1}}^*( \sH) + \sM_{\ell_{0-1}}( \sH)
\leq \Gamma\paren*{\sE_{ \ell}( h)-\sE_{ \ell}^*( \sH) +\sM_{ \ell}( \sH)}.
\end{equation*}
Then, $\lsc$ defined by \eqref{eq:sur-score-mabsc} admits an
$\sH$-consistency bound with respect to $\labsc$ with the functional
form $(2 - c)\Gamma\paren{\frac{t}{2 - c}}$, that is, for all $ h\in
\sH$, we have 
\begin{equation*}
\sE_{\labsc}( h) - \sE_{\labsc}^*( \sH) + \sM_{\labsc}( \sH)
\leq (2 - c) \Gamma\paren*{\frac{\sE_{\lsc}( h) - \sE_{\lsc}^*( \sH) +\sM_{\lsc}( \sH)}{2 - c}}.
\end{equation*}
\end{restatable}
The proof is given in Appendix~\ref{app:bound-score}.
\cite{AwasthiMaoMohriZhong2022multi} recently presented a series of results
providing $\sH$-consistency bounds for common surrogate losses in the
standard multi-class classification, including max losses such as
those of \citet{crammer2001algorithmic}, sum losses such as those of
\citet{weston1998multi} and constrained losses such as the loss
functions adopted by \citet{lee2004multicategory}. Thus, plugging in
any of those $\sH$-consistency bounds in Theorem~\ref{Thm:bound-score}
yields immediately a new $ \sH$-consistency bound for the
corresponding score-based abstention surrogate losses.

\section{Two-Stage Score-Based Formulation}
\label{sec:two-stage-mabsc}

In the single-stage scenario discussed in Section~\ref{sec:score}, the
learner simultaneously learns when to abstain and how to make
predictions otherwise. However, in practice often there is already a
predictor available and retraining can be very costly. A two-stage
solution is thus much more relevant for those critical applications,
where the learner only learns when to abstain in the second stage
based on the predictor trained in the first stage. With the two stage
solution, we can improve the performance of a large pre-trained model
by teaching it the option of abstaining without having to retrain the
model. In this section, we analyze a two-stage algorithmic scheme, for
which we propose surrogate losses that we prove to benefit from
$\sH$-consistency bounds.

Given a hypothesis set $\sH$ of functions mapping from $\sX \times
(\sY \cup \curl*{n + 1})$ to $\Rset$, it can be decomposed into $ \sH=
\sH_{\sY}\times \sH_{n+1}$, where $\sH_{\sY}$ denotes the hypothesis
set spanned by the first $n$ scores corresponding to the labels, and
$\sH_{n+1}$ represents the hypothesis set spanned by the last score
corresponding to the additional category.\ignore{The corresponding
  hypotheses are denoted by $h_{\sY} \in \sH_{\sY}$ and $h_{n+1} \in
  \sH_{n+1}$, respectively.}  We consider the following two-stage
algorithmic scheme: in the first stage, we learn a hypothesis $h_{\sY}
\in \sH_{\sY}$ by optimizing a surrogate loss $\ell$ for standard
multi-class classification; in the second stage, we fix the $h_{\sY}$
learned in the first stage and then learn a hypothesis $h_{n+1} \in
\sH_{n+1}$ by optimizing a surrogate loss function $\ell_{h_{\sY}}$
defined for any $h_{n+1} \in \sH_{n+1}$ and $(x, y) \in \sX \times
\sY$ by
\begin{equation}
\label{eq:ell-Phi-h}
\ell_{h_{\sY}}\paren*{h_{n+1}, x, y} = \1_{\hh_{\sY}(x) \neq y} \Phi\paren*{h_{n+1}(x) - \max_{y\in \sY}h_{\sY}(x, y)} + c \Phi\paren*{\max_{y\in \sY}h_{\sY}(x, y) - h_{n+1}(x)},
\end{equation}
where $\Phi$ is a decreasing function.  The learned hypothesis $h \in
\sH$ corresponding to those two stages can be expressed as $h =
(h_{\sY}, h_{n+1})$.  We note that the first stage consists of the
familiar task of finding a predictor using a standard surrogate loss
such as the logistic loss $\ell(h,x, y)=\log\paren*{\sum_{y'\in
    \sY}e^{h(x, y')-h(x, y)}}$ (or cross-entropy combined with the
softmax). Recall that the learner abstains from making a prediction
for $x$ and incurs a cost $c$ when $ h_{n+1}(x) \geq \max_{y \in \sY}
h_{\sY}(x, y)$.  In the second stage, the first term of
\eqref{eq:ell-Phi-h} encourages abstention for an input instance whose
prediction made by the pre-trained predictor $h_{\sY}$ is incorrect,
while the second term penalizes abstention according to the cost
$c$. The function $\Phi$ can be chosen as any margin-based loss
function in binary classification, including the exponential loss or
the logistic loss.

Let $\ell_{0-1}^{\rm{binary}}$ be the binary zero-one classification loss. Then, the two-stage surrogate losses benefit from the $\sH$-consistency bounds shown in Theorem~\ref{Thm:bound-general-two-step}. For a fixed parameter $\tau$, we define the $\tau$-translated hypothesis set of $\sH_{n+1}$ by $\sH_{n+1}^{\tau} =\curl*{ h_{n+1} - \tau : h_{n+1} \in \sH_{n+1}}$.
\begin{restatable}[\textbf{$\sH$-consistency bounds for
      two-stage surrogates}]{theorem}{BoundGenralTwoStep}
\label{Thm:bound-general-two-step}
Given a hypothesis set $\sH=\sH_{\sY}\times \sH_{n+1}$. Assume that
$\ell$ admits an $\sH_{\sY}$-consistency bound with respect to the
multi-class zero-one classification loss $ \ell_{0-1}$ and that $\Phi$
admits an $\sH_{n+1}^{\tau}$-consistency bound with respect to the
binary zero-one classification loss $\ell_{0-1}^{\rm{binary}}$ for any
$\tau \in \Rset$.  Thus, there are non-decreasing concave functions
$\Gamma_1$ and $\Gamma_2$ such that, for all $h_{\sY}\in \sH_{\sY}$,
$h_{n+1}^{\tau} \in \sH_{n+1}^{\tau}$ and $\tau \in \Rset$, we have
\begin{align*}
\sE_{\ell_{0-1}}(h_{\sY}) - \sE_{\ell_{0-1}}^*(\sH_{\sY}) + \sM_{\ell_{0-1}}(\sH_{\sY}) &\leq \Gamma_1\paren*{\sE_{\ell}(h_{\sY})-\sE_{\ell}^*(\sH_{\sY}) +\sM_{\ell}(\sH_{\sY})}\\
\sE_{\ell_{0-1}^{\rm{binary}}}(h_{n+1}^{\tau}) - \sE_{\ell_{0-1}^{\rm{binary}}}^*(\sH_{n+1}^{\tau}) + \sM_{\ell_{0-1}^{\rm{binary}}}(\sH_{n+1}^{\tau}) &\leq \Gamma_2\paren*{\sE_{\Phi}(h_{n+1}^{\tau})-\sE_{\Phi}^*(\sH_{n+1}^{\tau}) +\sM_{\Phi}(\sH_{n+1}^{\tau})}.
\end{align*}
Then, the following holds for all $ h=(h_{\sY},h_{n+1})\in \sH$:
\begin{align*}
\sE_{\labs}(h) - \sE_{\labs}^*(\sH) + \sM_{\labs}(\sH) & \leq \Gamma_1\paren*{\sE_{\ell}(h_{\sY})-\sE_{\ell}^*(\sH_{\sY}) +\sM_{\ell}(\sH_{\sY})}\\
&\qquad  + (1+c)\Gamma_2\paren[\bigg]{\frac{\sE_{\ell_{h_{\sY}}}(h_{n+1})-\sE_{\ell_{h_{\sY}}}^*(\sH_{n+1}) +\sM_{\ell_{h_{\sY}}}(\sH_{n+1})}{c}},
\end{align*}
where the
constant factors $(1 + c)$ and $\frac{1}{c}$ can be removed 
when $\Gamma_2$ is linear.
\end{restatable}
The proof is given in Appendix~\ref{app:bound-general-two-step}. The
assumptions in Theorem~\ref{Thm:bound-general-two-step} are mild and
hold for common hypothesis sets such as linear models and neural
networks with common surrogate losses in the binary and multi-class
classification, as shown by
\citep{awasthi2022Hconsistency,AwasthiMaoMohriZhong2022multi}. Recall that the
minimizability gaps vanish when $\sH_{\sY}$ and $\sH_{n+1}$ are the
family of all measurable functions or when $\sH_{\sY}$ and $\sH_{n+1}$
contain the Bayes predictors.  In their absence, the theorem shows
that if the estimation loss
$(\sE_{\ell}(h_{\sY})-\sE_{\ell}^*(\sH_{\sY}))$ is reduced to $\e_1$
and the estimation loss
$(\sE_{\ell_{h_{\sY}}}(h_{n+1})-\sE_{\ell_{h_{\sY}}}^*(\sH_{n+1}))$ to
$\e_2$, then, modulo constant factors, the score-based abstention
estimation loss $(\sE_{\labs}(h) - \sE_{\labs}^*(\sH))$ is bounded by
$\Gamma_1(\e_1) + \Gamma_2(\e_2)$.
Thus, this gives a strong guarantee for the surrogate losses described
in this two-stage setting.

\section{Realizable \texorpdfstring{$\sH$}{H}-Consistency and Benefits of Two-Stage Surrogate Losses}
\label{sec:realizable-mabsc}

\citet{pmlr-v206-mozannar23a} recently showed that cross-entropy
score-based surrogate losses are not realizable $\sH$-consistent, as
defined by \citet{long2013consistency,zhang2020bayes}, in relation to
abstention loss. Instead, the authors proposed a novel surrogate
loss that is proved to be realizable $\sH$-consistent when $\sH$ is
\emph{closed under scaling}, although its Bayes-consistency remains
unclear. Devising a surrogate loss that exhibits both
Bayes-consistency and realizable $\sH$-consistency remains an open
problem. A hypothesis set $\sH$ is said to be \emph{closed under
scaling} if, for any hypothesis $h$ belonging to $\sH$, the scaled
hypothesis $\alpha h$ also belongs to $\sH$ for all $\alpha \in \Rset$.

We prove in Theorem~\ref{Thm:bound-general-two-step-realizable} of
Appendix~\ref{app:bound-general-two-step-realizable}, that for any
realizable distribution, when both the first-stage surrogate
estimation loss $\sE_{\ell}(h_{\sY}) - \sE_{\ell}^*(\sH_{\sY})$ and
the second-stage surrogate estimation loss
$\sE_{\ell_{h_{\sY}}}(h_{n+1}) - \sE_{\ell_{h_{\sY}}}^*(\sH_{n+1})$
converge to zero, the abstention estimation loss $\sE_{\labs}(h) -
\sE_{\labs}^*(\sH)$ also approaches zero. This implies that the
two-stage score-based surrogate loss is realizable $\sH$-consistent
with respect to $\labsc$, which provides a significant advantage over
the single-stage cross-entropy score-based surrogate loss. It is
important to note that Theorem~\ref{Thm:bound-general-two-step} shows
that the two-stage formulation is also Bayes-consistent. This
addresses the open problem in \citep{pmlr-v206-mozannar23a} and
highlights the benefits of the two-stage formulation. In the following
section, our empirical results further demonstrate that the two-stage
score-based surrogate loss outperforms the state-of-the-art
cross-entropy score-based surrogate loss.

\section{Finite Sample Guarantees}
\label{sec:finite-sample}

Our $\sH$-consistency bounds enable the direct derivation of
finite-sample estimation bounds for a surrogate loss minimizer. These
are expressed in terms of the Rademacher complexity of the hypothesis
set $\sH$, the loss function, and the minimizability gaps. Here, we
provide a simple illustration based on
Theorem~\ref{Thm:bound_comp_sum}.

Let $\h h_S$ be the empirical minimizer of the surrogate loss
$\lsc_{\mu}$: $ \h h_S = \argmin_{h \in \sH} \frac{1}{m} \sum_{i =
  1}^m \lsc_{\mu}(h, x_i, y_i)$, for an i.i.d sample $S =
\paren*{(x_1, y_1), \ldots, (x_m, y_m)}$ of size $m$. Let
$\Rad_m^{\lsc_{\mu}}(\sH)$ be the Rademacher complexity of the set
$\sH_{\lsc_{\mu}} = \curl*{(x, y) \mapsto \lsc_{\mu}(h, x, y) \colon h
  \in \sH}$ and $B_{\lsc_{\mu}}$ an upper bound on the surrogate loss
$\lsc_{\mu}$. By using the standard Rademacher complexity bounds
\citep{MohriRostamizadehTalwalkar2018}, for any $\delta>0$, with
probability at least $1 - \delta$, the following holds for all $h \in
\sH$:
\[\abs*{\sE_{\lsc_{\mu}}(h) - \h \sE_{\lsc_{\mu}, S}(h)}
\leq 2 \Rad_m^{\lsc_{\mu}}(\sH) +
B_{\lsc_{\mu}} \sqrt{\tfrac{\log (2/\delta)}{2m}}.\]
Fix $\e > 0$. By the definition of the infimum, there exists $h^* \in
\sH$ such that $\sE_{\lsc_{\mu}}(h^*) \leq
\sE_{\lsc_{\mu}}^*(\sH) + \e$. By definition of
$\h h_S$, we have
\begin{align*}
& \sE_{\lsc_{\mu}}(\h h_S) - \sE_{\lsc_{\mu}}^*(\sH)
\\
& = \sE_{\lsc_{\mu}}(\h h_S) - \h\sE_{\lsc_{\mu}, S}(\h h_S) + \h\sE_{\lsc_{\mu}, S}(\h h_S) - \sE_{\lsc_{\mu}}^*(\sH)\\
& \leq \sE_{\lsc_{\mu}}(\h h_S) - \h\sE_{\lsc_{\mu}, S}(\h h_S) + \h\sE_{\lsc_{\mu}, S}(h^*) - \sE_{\lsc_{\mu}}^*(\sH)\\
& \leq \sE_{\lsc_{\mu}}(\h h_S) - \h\sE_{\lsc_{\mu}, S}(\h h_S) + \h\sE_{\lsc_{\mu}, S}(h^*) - \sE_{\lsc_{\mu}}^*(h^*) + \e\\
& \leq
  2 \bracket*{2 \Rad_m^{\lsc_{\mu}}(\sH) +
B_{\lsc_{\mu}} \sqrt{\tfrac{\log (2/\delta)}{2m}}} + \e.    
\end{align*}
Since the inequality holds for all $\e > 0$, it implies:
\[
\sE_{\lsc_{\mu}}(\h h_S) - \sE_{\lsc_{\mu}}^*(\sH)
\leq 
4 \Rad_m^{\lsc_{\mu}}(\sH) +
2 B_{\lsc_{\mu}} \sqrt{\tfrac{\log (2/\delta)}{2m}}.
\]
Plugging in this inequality in the bound of
Theorem~\ref{Thm:bound_comp_sum}, we obtain that for any $\delta > 0$,
with probability at least $1 - \delta$ over the draw of an i.i.d
sample $S$ of size $m$, the following finite sample guarantee holds
for $\h h_S$:
\begin{align*}
\sE_{\labsc}(\h h_S) - \sE_{\labsc}^*( \sH)\leq \Gamma_{\mu}
  \paren[\Big]{4 \Rad_m^{\lsc_{\mu}}(\sH)
  +
2 B_{\lsc_{\mu}} \textstyle \sqrt{\tfrac{\log \frac{2}{\delta}}{2m}}
    + \sM_{\lsc_{\mu}}(\sH)} - \sM_{\labsc}( \sH).    
\end{align*}
To our knowledge, these are the first abstention estimation loss
guarantees for empirical minimizers of a cross-entropy score-based
surrogate loss. Our comments about the properties of $\Gamma_{\mu}$
below Theorem~\ref{Thm:bound_comp_sum}, in particular its functional
form or its dependency on the number of classes $n$, similarly apply
here. Similar finite sample guarantees can also be derived based on
Theorems~\ref{Thm:bound-score-mabsc} and \ref{Thm:bound-general-two-step}.

As commented before Section~\ref{sec:min-gaps}, for a surrogate loss,
the minimizability gap is in general a more refined quantity than the
approximation error, while for the abstention loss, these two
quantities coincide for typical hypothesis sets (See
Appendix~\ref{app:better-bounds}). Thus, our bound can be rewritten as
follows for typical hypothesis sets:
\begin{align*}
\sE_{\labsc}(\h h_S) - \sE_{\labsc}^*( \sH_{\rm{all}})  \leq \Gamma_{\mu}
  \paren*{4 \Rad_m^{\lsc_{\mu}}(\sH) +
    2 B_{\lsc_{\mu}}
    \sqrt{\frac{\log \frac{2}{\delta}}{2m}}
    + \sM_{\lsc_{\mu}}(\sH)}.
\end{align*}
Our guarantee is thus more favorable and more relevant than a similar
finite sample guarantee where $\sM_{\lsc_{\mu}}( \sH)$ is replaced
with $\sA_{\lsc_{\mu}}(\sH)$, which could be derived from an excess
error bound.

\section{Experiments}
\label{sec:experiments-mabsc}

\begin{table}[t]
\caption{Abstention Loss for Models Obtained with Different Surrogate
  Losses; Mean $\pm$ Standard Deviation\ignore{ over three runs} for
  Both Two-Stage Score-Based Abstention Surrogate Loss and the
  State-Of-The-Art Cross-Entropy Score-Based Surrogate Losses in
  \citep{mozannar2020consistent} ($\mu = 1.0$) and
  \citep{caogeneralizing} ($\mu=1.7$).}
    \label{tab:comparison}
\begin{center}
    \begin{tabular}{@{\hspace{0pt}}lll@{\hspace{0pt}}}
      METHOD & DATASET & ABSTENTION LOSS \\
    \midrule
    Cross-entropy score-based ($\mu=1.0$) & \multirow{3}{*}{CIFAR-10} & 4.48\% $\pm$ 0.10\% \\
    Cross-entropy score-based ($\mu=1.7$)  & & 3.62\% $\pm$ 0.07\%  \\
    \textbf{Two-stage score-based}  & & \textbf{3.22\% \!$\pm$ 0.04\%}    \\
    \midrule
    Cross-entropy score-based ($\mu=1.0$) & \multirow{3}{*}{CIFAR-100} & 10.40\% $\pm$ 0.10\% \\
    Cross-entropy score-based ($\mu=1.7$)  & & 14.99\% $\pm$ 0.01\%\\
    \textbf{Two-stage score-based} & & \textbf{\phantom{0}9.54\% \!$\pm$ 0.07\%}   \\
    \midrule
    Cross-entropy score-based ($\mu=1.0$) & \multirow{3}{*}{SVHN} & 1.61\% $\pm$ 0.06\% \\
    Cross-entropy score-based ($\mu=1.7$) & & 2.16\% $\pm$ 0.04\%\\
    \textbf{Two-stage score-based}  & & \textbf{0.93\% \!$\pm$ 0.02\%}  \\
    \end{tabular}
\end{center}
\end{table}

In this section, we report the results of experiments comparing the
single-stage and two-stage score-based abstention surrogate losses,
for three widely used datasets CIFAR-10, CIFAR-100
\citep{Krizhevsky09learningmultiple} and SVHN \citep{Netzer2011}.

\paragraph{Experimental Settings} 

As with \citep{mozannar2020consistent,caogeneralizing}, we use ResNet
\citep{he2016deep} and WideResNet (WRN) \citep{zagoruyko2016wide} with
ReLU activations. Here, ResNet-$n$ denotes a residual network with $n$
convolutional layers and WRN-$n$-$k$ denotes a residual network with
$n$ convolutional layers and a widening factor $k$. We trained
ResNet-$34$ for CIFAR-10 and SVHN, and WRN-$28$-$10$ for CIFAR-100.
We applied standard data augmentations, 4-pixel padding with $32
\times 32$ random crops and random horizontal flips for CIFAR-10 and
CIFAR-100. We used Stochastic Gradient Descent (SGD) with Nesterov
momentum \citep{nesterov1983method} and set batch size $1\mathord,024$
and weight decay $1\times 10^{-4}$ in the training. We trained for
$200$ epochs using the cosine decay learning rate schedule
\citep{loshchilov2016sgdr} with the initial learning rate of $0.1$.

For each dataset, the cost value $c$ was selected to be close to the
best-in-class zero-one classification loss, which are $\curl*{0.05,
  0.15, 0.03}$ for CIFAR-10, CIFAR-100 and SVHN respectively, since a
too small value leads to abstention on almost all points and a too
large one leads to almost no abstention. Other neighboring values for
$c$ lead to similar results.

The abstention surrogate loss proposed in
\citep{mozannar2020consistent} corresponds to the special case of
\compsum\ score-based surrogate losses $ \sfL_{\mu}$ with $\mu = 1$,
and meanwhile the abstention surrogate loss adopted in
\citep{caogeneralizing} corresponds to the special case of
\compsum\ score-based surrogate losses $ \sfL_{\mu}$ with $\mu =
1.7$. Note that the simple confidence-based approach by thresholding
estimators of conditional probability typically does not perform as
well as these state-of-the-art surrogate losses
\citep{caogeneralizing}. For our two-stage score-based abstention
surrogate loss, we adopted the logistic loss in the first stage and
the exponential loss $\Phi(t) = \exp(-t)$ in the second stage.

\paragraph{Evaluation} 

We evaluated all the models based on the abstention loss $\labsc$, and
reported the mean and standard deviation over three trials.

\paragraph{Results}

Table~\ref{tab:comparison} shows that the two-stage score-based
surrogate losses consistently outperform the cross-entropy score-based
surrogates used in the state-of-the-art algorithms
\citep{mozannar2020consistent,caogeneralizing} for all the
datasets. Table~\ref{tab:comparison} also shows the relative
performance of the cross-entropy surrogate \eqref{eq:L-mu} with
$\ell_{\mu}$ adopted as the generalized cross-entropy loss ($\mu=1.7$)
and that with $\ell_{\mu}$ adopted as the logistic loss ($\mu=1.0$)
varies by the datasets.

As show in Section~\ref{sec:two-stage-mabsc} and
Section~\ref{sec:realizable-mabsc}, the two-stage surrogate losses benefit
from the guarantees of both realizable $\sH$-consistency and
Bayes-consistency while the cross-entropy surrogate loss does not
exhibit realizable $\sH$-consistency, as shown by
\citet{pmlr-v206-mozannar23a}. This explains the superior performance
of two-stage surrogate losses over the cross-entropy surrogate
loss. It is worth noting that the hypothesis set we used for each
dataset is sufficiently rich, and the experimental setup closely
resembles a realizable scenario.

As our theoretical analysis (Theorem~\ref{Thm:bound_comp_sum} and
Theorem~\ref{Thm:gap-upper-bound-determi}) suggests, the relative
performance variation between the cross-entropy surrogate loss with
$\mu = 1.0$ used in \citep{mozannar2020consistent} and the
cross-entropy surrogate loss with $\mu = 1.7$ used in
\citep{caogeneralizing} can be explained by the functional forms of
their $\sH$-consistency bounds and the magnitude of their
minimizability gaps. Specifically, the dependency of the
multiplicative constant on the number of classes in $\sH$-consistency
bounds (Theorem~\ref{Thm:bound_comp_sum}) for the cross-entropy
surrogate loss with $\mu = 1.7$ makes it less favorable when dealing
with a large number of classes, such as in the case of CIFAR-100.
This suggests that the recent observation made in
\citep{caogeneralizing} that the cross-entropy surrogate with $\mu =
1.7$ outperforms the one with $\mu = 1.0$ does not apply to the
scenario where the evaluation involves datasets like CIFAR-100.  For a
more comprehensive discussion of our experimental results, please
refer to Appendix~\ref{app:experimemts}.
\ignore{ This agrees with our theoretical analysis based on
  $\sH$-consistency bounds in Theorem~\ref{Thm:bound_comp_sum} and
  Theorem~\ref{Thm:gap-upper-bound-determi}, since both losses have
  the same square-root functional form while on CIFAR-10, the
  magnitude of the minimizability gap decreases with $\mu$ in light of
  the fact that $\sE_{\lsc_{\mu}}^*(\sH)$ is close for both losses,
  and since on SVHN and CIFAR-100, the dependency of the
  multiplicative constant on the number of classes appears for
  $\mu=1.7$, which makes it less favorable, particularly clear when
  $n$ is large (CIFAR-100).
}

\section{Conclusion}

Our comprehensive study of score-based multi-class abstention
introduced novel surrogate loss families with strong hypothesis
set-specific and non-asymptotic theoretical guarantees. Empirical
results demonstrate the practical advantage of these surrogate losses
and their derived algorithms. This work establishes a powerful
framework for designing novel, more reliable abstention-aware algorithms
applicable across diverse domains.

\ignore{
We presented a detailed study of score-based multi-class
abstention. We introduced new families of surrogate losses within the
framework and provided strong theoretical guarantees for them, which
are specific to the hypothesis set and non-asymptotic. Our empirical
results further illustrate the practical significance of these
surrogate losses and the new algorithms based on them. We believe that
our analysis can be leveraged in a wide range of scenarios to design
new algorithms.
}

\chapter{Predictor-Rejector Multi-Class Abstention} \label{ch3}
In this chapter, we will be particularly interested in the predictor-rejector
formulation, which explicitly models the cost of abstention. The
selective classification of \citet{el2010foundations} is also
interesting, but it does not explicitly factor in the cost $c$ and is
based on a distinct objective. Confidence-based methods are also very
natural and straightforward, but they may fail when the predictor is
not calibrated, a property that often does not hold. Additionally,
they have been shown to be suboptimal when the predictor differs from
the Bayes classifier \citep{CortesDeSalvoMohri2016}. The score-based
formulation \citep{mozannar2020consistent} admits very common
properties and also explicitly takes into account the rejection cost
$c$.  We will compare the predictor-rejector formulation with the
score-based one. We will show via an example that the
predictor-rejector is more natural in some instances and will also
compare the two formulations in our experiments. We further elaborate on the difference
between the two formulations in Appendix~\ref{app:difference}.

How should the problem of multi-class classification with abstention
be formulated and when is it appropriate to abstain?
The extension of the results of \citet{CortesDeSalvoMohri2016} to
\emph{multi-class classification} was found to be very challenging by
\citet{NiCHS19}. In fact, these authors left the following as an open question: \emph{can we define Bayes-consistent surrogate losses for the
predictor-rejector abstention formulation in the multi-class setting?}
This paper deals precisely with this topic: we present a series of new
theoretical, algorithmic, and empirical results for multi-class
learning with abstention in predictor-rejector formulation and, in
particular, resolve this open question in a strongly positive way.

For the score-based formulation, a surrogate loss function based on
cross-entropy was introduced by \citet{mozannar2020consistent}, which
was proven to be Bayes-consistent. Building upon this work,
\citet{caogeneralizing} presented a more comprehensive collection of
Bayes-consistent surrogate losses for the score-based
formulation. These surrogate losses can be constructed using any
consistent loss function for the standard multi-class classification
problem. 
In Chapter~\ref{ch2}, we gave an extensive analysis of surrogate losses for the score-based formulation supported by \emph{$\sH$-consistency bounds}. In a recent study, \citet{pmlr-v206-mozannar23a}
demonstrated that existing score-based surrogate losses for abstention
are not \emph{realizable consistent} with respect to the abstention
loss, as defined by \citet{long2013consistency}. Instead, they
proposed a novel surrogate loss that achieves realizable $(\sH,
\sR)$-consistency, provided that the sets of predictors $\sH$ and
rejectors $\sR$ are \emph{closed under scaling}.  However, the authors
expressed uncertainty regarding the Bayes-consistency of their
proposed surrogate losses and left open the question of
\emph{identifying abstention surrogate losses that are both consistent
and realizable $(\sH, \sR)$-consistent when $\sH$ and $\sR$ satisfy
the scaling closure property}.  We address this open question by
demonstrating that our newly proposed surrogate losses benefit from
both Bayes-consistency and realizable consistency.

We show in Section~\ref{sec:score-example} that in some instances
the optimal solution cannot be derived in the score-based
formulation, unless we resort to more complex scoring functions. In
contrast, the solution can be straightforwardly derived in the
predictor-rejector formulation.
In Section~\ref{sec:general}, we present and analyze a new family of
surrogate loss functions for multi-class abstention in the
predictor-rejector formulation, first in the \emph{single-stage
setting}, where the predictor $h$ and the rejector $r$ are selected
simultaneously, next in a \emph{two-stage setting}, where first the
predictor $h$ is chosen and fixed and subsequently the rejector $r$ is
determined. The two-stage setting is crucial in many applications
since the predictor $h$ is often already learned after a costly
training of several hours or days. Re-training to ensure a
simultaneous learning of $h$ and $r$ is then inconceivable due to its
prohibitive cost.

In the \emph{single-stage setting} (Section~\ref{sec:single-stage-mabs}),
we first give a negative result, ruling out abstention surrogate
losses that do not
verify a technical condition. Next, we present several
positive results for abstention surrogate
losses verifying that condition, for which we prove non-asymptotic \emph{$(\sH,\sR)$-consistency bounds} \citep{awasthi2022Hconsistency,AwasthiMaoMohriZhong2022multi}
that are stronger than Bayes-consistency.
Next, in Section~\ref{sec:two-stage}, we also prove \emph{$(\sH,
\sR)$-consistency bounds} for the \emph{two-stage setting}. Minimizing
these new surrogate losses directly result in new algorithms for
multi-class abstention.

In Section~\ref{sec:general-realizable}, we prove realizable
consistency guarantees for both single-stage and two-stage
predictor-rejector surrogate losses.\ignore{This property underscores
  the advantages of our predictor-rejector formulation and is
  supported by the empirical success of the newly proposed surrogate
  losses. We further discuss in detail the difference between the
  predictor-rejector formulation and the score-based formulation,
  which underscores our work's innovation and significant
  contribution.}
In Section~\ref{sec:experiments-mabs}, we empirically show that our two-stage predictor-rejector surrogate loss
consistently outperforms state-of-the-art
scored-based surrogate losses\ignore{on CIFAR-10, CIFAR-100 and
SVHN datasets}, while our single-stage one achieves comparable results. Our main contributions are summarized below: 
\begin{itemize}
\itemsep-0.1em 
\item Counterexample for score-based abstention formulation.

\item Negative results for single-stage predictor-rejector surrogate losses.

\item New families of single-stage predictor-rejector surrogate
      losses for which we prove strong non-asymptotic and hypothesis
      set-specific consistency guarantees, thereby resolving
      positively an open question mentioned by \citet{NiCHS19}.

\item Two-stage predictor-rejector formulations and their
      $\sH$-consistency bound guarantees.

\item Realizable consistency guarantees for both single-stage and
      two-stage surrogate losses, which resolve positively the recent
      open question posed by \citet{pmlr-v206-mozannar23a}.

\item Experiments on CIFAR-10, CIFAR-100 and SVHN
datasets empirically demonstrating the usefulness of our
      proposed surrogate losses.
 
\end{itemize}

The presentation in this chapter is based on \citep{MaoMohriZhong2024predictor}.

\section{Preliminaries}
\label{sec:preliminaries}

We first introduce some preliminary concepts and definitions,
including the description of the predictor-rejector formulation and
background on $\sH$-consistency bounds.
We examine the standard multi-class classification scenario with an
input space $\sX$ and a set of $n \geq 2$ classes or labels $\sY =
\curl*{1, \ldots, n}$. We will denote by $\sD$ a distribution over
$\sX \times \sY$ and by $\sfp(y \!\mid\! x) = \sD(Y = y \mid X = x)$ the
conditional probability of $Y = y$ given $X = x$. We will also adopt
the shorthand $p(x) = \paren*{\sfp(1 \!\mid\! x), \ldots, \sfp(n \!\mid\! x)}$ to denote the
vector of conditional probabilities, given $x \in \sX$. We study more
specifically the learning scenario of multi-class classification with
abstention, within the predictor-rejector formulation introduced by
\citet{CortesDeSalvoMohri2016bis}.

\textbf{Predictor-rejector formulation.}
In this formulation of the abstention problem, we have a hypothesis
set $\sH$ of prediction functions mapping from $\sX \times \sY$ to
$\Rset$, along with a family $\sR$ of abstention functions, or
\emph{rejectors}, which map from $\sX$ to $\Rset$. For an input $x \in
\sX$ and a hypothesis $h \in \sH$, the predicted label $\hh(x)$ is
defined as the one with the highest score, $\hh(x) = \argmax_{y\in
  \sY}h(x, y)$, using an arbitrary yet fixed deterministic strategy to
break ties. A rejector $r \in \sR$ is used to abstain from predicting
on input $x$ if $r(x)$ is non-positive, $r(x) \leq 0$, at a cost $c(x)
\in [0, 1]$. The \emph{predictor-rejector abstention loss} $\labs$ for
this formulation is thus defined for any $(h, r) \in \sH \times \sR$
and $(x, y) \in \sX \times \sY$ as
\begin{equation}
\label{eq:abs}
\labs(h, r, x, y)
= \1_{\hh(x) \neq y} \1_{r(x)> 0} + c(x) \1_{r(x) \leq 0}.
\end{equation}
When the learner does not abstain ($r(x) > 0$), it incurs the familiar
zero-one classification loss. Otherwise, it abstains ($r(x) \leq 0$)
at the expense of a cost $c(x)$. The learning problem consists of
using a finite sample of size $m$ drawn i.i.d.\ from $\sD$ to select a
predictor $h$ and rejector $r$ with small expected $\labs$ loss,
$\E_{(x, y) \sim \sD}[\labs(h, r, x, y)]$.
For simplification, in the following, we assume a
constant cost function $c \in (0, 1)$. This assumption is not
necessary however, and all $(\sH,
\sR)$-consistency bounds results in Sections~\ref{sec:single-stage-mabs} and \ref{sec:two-stage-mabs} extend straightforwardly to
general cost functions.

Optimizing the predictor-rejector abstention loss is intractable for
most hypothesis sets. Thus, learning algorithms in this context must
rely on a surrogate loss $\sfL$ for $\labs$. In the subsequent
sections, we study the consistency properties of several surrogate
losses. Given a surrogate loss function $\sfL$ in the
predictor-rejector framework, we denote by $\sE_{\sfL}(h, r) = \E_{(x,
  y)\sim \sD}\bracket*{\sfL(h, r, x, y)}$ the $\sfL$-expected loss of
a pair $(h, r) \in \sH \times \sR$ and by $\sE_{\sfL}^*(\sH, \sR) =
\inf_{h \in \sH, r \in \sR} \sE_{\sfL}(h, r)$ its infimum over $\sH
\times \sR$.

\textbf{$(\sH, \sR)$-consistency bounds.} We will prove $(\sH,
\sR)$-consistency bounds for several surrogate losses in the
predictor-rejector framework, which extend to the predictor-rejector
framework the $\sH$-consistency bounds of
\citet{awasthi2021calibration,awasthi2021finer,awasthi2022Hconsistency,AwasthiMaoMohriZhong2022multi,AwasthiMaoMohriZhong2023theoretically,awasthi2024dc,MaoMohriZhong2023cross,MaoMohriZhong2023ranking,MaoMohriZhong2023rankingabs,zheng2023revisiting,MaoMohriZhong2023characterization,MaoMohriZhong2023structured,mao2024h,mao2024regression,mao2024universal,mao2025enhanced,mao2024multi,mao2024realizable,MohriAndorChoiCollinsMaoZhong2024learning,cortes2024cardinality,zhong2025fundamental,MaoMohriZhong2025principled,MaoMohriZhong2025mastering,cortes2025balancing,cortes2025improved,desalvo2025budgeted,mohri2025beyond}.
These are inequalities upper-bounding the predictor-rejector
abstention estimation loss $\labs$ of a hypothesis $h \in \sH$ and
rejector $r \in \sR$ with respect to their surrogate estimation
loss. They admit the following form:
\[
\sE_{\labs}(h, r) - \sE_{\labs}^*(\sH, \sR) \leq f\paren{\sE_{\sfL}(h,
  r) - \sE_{\sfL}^*(\sH, \sR)},
\]
where $f$ is a non-decreasing
function. Thus, the estimation error $(\sE_{\labs}(h, r) -
\sE_{\labs}^*(\sH,\sR))$ is then bounded by $f(\e)$ if the surrogate
estimation loss $(\sE_{\sfL}(h, r) - \sE_{\sfL}^*(\sH,\sR))$ is reduced
to $\e$.
An important term that appears in these bounds is the
\emph{minimizability gap}, which is defined by
$\sM_\sfL(\sH, \sR) = \sE_{\sfL}^*(\sH, \sR) -
\E_x\bracket*{\inf_{h \in \sH, r\in \sR} \E_y \bracket{\sfL(h, r,
    X, y) \mid X = x}}$.
When the loss function
$\sfL$ depends solely on $h(x, \cdot)$ and $r(x)$, which holds for
most loss functions used in applications, and when $\sH$ and $\sR$
include all measurable functions, the minimizability gap is null
\citep[lemma~2.5]{steinwart2007compare}. However, it is generally
non-zero for restricted hypothesis sets $\sH$ and $\sR$. The
minimizability gap can be upper-bounded by the approximation error
$\sA_\sfL(\sH, \sR) = \sE_{\sfL}^*(\sH, \sR) -
\E_x\bracket[\big]{\inf_{h, r} \E_y \bracket{\sfL(h, r, X, y) \mid X =
    x}}$, where the infimum is taken over all measurable
functions. But, the minimizability gap is a finer quantity and leads
to more favorable guarantees.

\section{Counterexample for score-based abstention losses}
\label{sec:score-example}

We first discuss a natural example here that can be tackled
straightforwardly in the predictor-rejector formulation but for which
the same solution cannot be derived in the score-based formulation
setting, unless we resort to more complex functions. This natural
example motivates our study of surrogate losses for the
predictor-rejector formulation in Section~\ref{sec:general}. We begin by
discussing the score-based formulation and subsequently
highlight its relationship with the predictor-rejector formulation.

\textbf{Score-based abstention formulation.}
In this version of the abstention problem, the label set $\sY$ is
expanded by adding an extra category $(n + 1)$, which represents
abstention. We indicate the augmented set as $\wt \sY = \curl*{1, \ldots, n,
  n + 1}$ and consider a hypothesis set $\wt \sH$ comprising functions
that map from $\sX \times \wt \sY$ to $\Rset$.
The label assigned to an input $x \in \sX$ by $\wt h \in \wt \sH$ is
denoted as $\wt \hh(x)$. It is defined as $\wt \hh(x) = n + 1$ if $\wt
h(x, n + 1) \geq \max_{y \in \sY} \wt h(x, y)$; otherwise, $\wt
\hh(x)$ is determined as an element in $\sY$ with the highest score,
$\wt \hh(x) = \argmax_{y \in \sY} \wt h(x, y)$, using an arbitrary yet
fixed deterministic strategy to break ties. When $\wt \hh(x) = n + 1$,
the learner chooses to abstain from predicting for $x$ and incurs a
cost $c$. In contrast, it predicts the label $y = \wt \hh(x)$ if
otherwise. The \emph{score-based abstention loss} $\labsc$ for this
formulation is defined for any $\wt h \in \wt \sH$ and $(x, y) \in \sX
\times \sY$ as follows:
\begin{equation}
\label{eq:abs-score-mabs}
\labsc(\wt h, x, y)
= \1_{\wt \hh(x)\neq y}\1_{\wt \hh(x)\neq n + 1} + c \, \1_{\wt \hh(x) = n + 1}.
\end{equation}
Thus, as in the predictor-rejector context, when the learner does not
abstain ($\wt \hh(x) \neq n + 1$), it reduces to the familiar
zero-one classification loss. Conversely, when it abstains ($\wt
\hh(x) = n + 1$), it incurs the cost $c$. With a finite sample
drawn i.i.d.\ from $\sD$, the learning problem involves choosing a
hypothesis $\wt h$ within $\wt \sH$ that yields a minimal expected
score-based abstention loss, $\E_{(x, y) \sim \sD}[\labsc(\wt h, x,
  y)]$.

\textbf{Relationship between the two formulations.}
In the score-based formulation, rejection is defined by the condition
$\wt h(x, n + 1) - \max_{y \in \sY} \wt h(x, y) \geq 0$. Thus, a
predictor-rejector formulation with $(h, r) \in \sH \times \sR$ can be
equivalently formulated as a score-based problem with $\wt h$ defined
by $\wt h(x, y) = h(x, y)$ for $y \in \sY$ and $\wt h(x, n + 1) =
\max_{y \in \sY} h(x, y) - r(x)$: $\labsc(\wt h, x, y) = \labs(h, r,
x, y)$ for all $(x, y) \times \sX \times \sY$. Note, however, that
function $\wt h(\cdot, n + 1)$ is in general more complex.  As an
example, while $h$ and $r$ may both be in a family of linear
functions, in general, $\wt h(\cdot, n + 1)$ defined in this way is no
more linear. Thus, a score-based formulation might require working
with more complex functions. We further elaborate on the difference
between the two formulations in Appendix~\ref{app:difference}.
\ignore{Similarly, a score-based formulation
with $\wt h$ can be equivalently viewed as a predictor-rejector formulation
with $h$ being the restriction of $\wt h$ to $\sX \times \sY$ and
$r(x) = \max_{y \in \sY} \wt h(x, y) - \wt h(x, n + 1)$.}

\begin{figure}[t]
\begin{center}
  \includegraphics[scale=0.4]{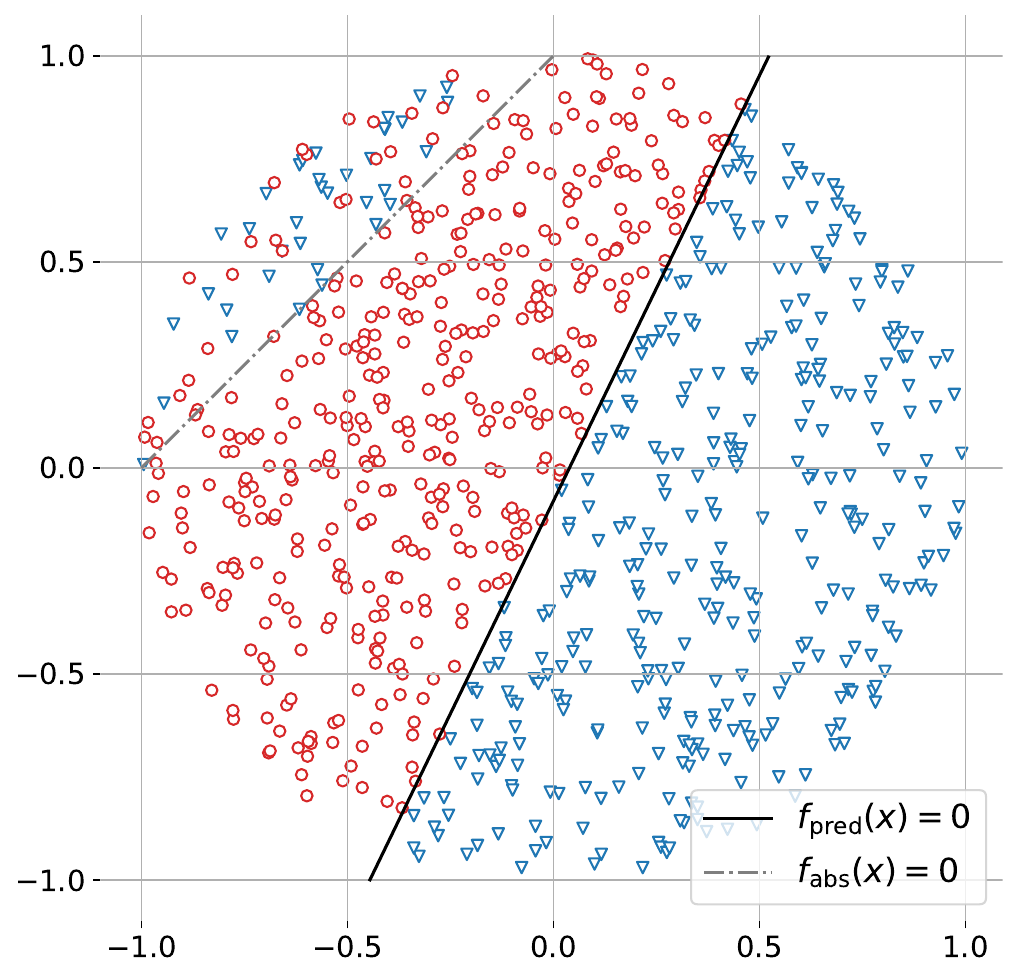}
\caption{Counterexample for score-based abstention losses.}
\label{fig:example}
\end{center}
\end{figure}

\textbf{Counterexample setting.} Let $\sY = \curl*{1, 2}$ and let $x$
follow the uniform distribution on the unit ball $\sfB_2(1) =
\curl*{x\colon \norm*{x}_2 \leq 1}$. We will consider the linear
models $f\in\sF_{\mathrm{lin}} = \{x \rightarrow \bw\cdot x + \bb
  \mid \|\bw\|_2 = 1\}$.  We set the label of a point $x$ as follows:
fix two linear functions $f_{\rm{abs}}(x) = \bw_{\rm{abs}}\cdot
x+\bb_{\rm{abs}}$ and $f_{\rm{pred}}(x) = \bw_{\rm{pred}}\cdot
x + \bb_{\rm{pred}}$ in $\sF_{\mathrm{lin}}$ , if $f_{\rm{abs}}(x)\leq
0$, then set $y = 1$ with probability $\frac12$ and $y = 2$ with
probability $\frac12$; if $f_{\rm{abs}}(x)> 0 \text{ and
}f_{\rm{pred}}(x) > 0$, then set $y = 1$; if $f_{\rm{abs}}(x) > 0 \text{
  and }f_{\rm{pred}}(x)\leq 0$, then set $y = 2$; see
Figure~\ref{fig:example}.

We denote by $\sH_{\mathrm{lin}}$ the hypothesis set of linear scoring
functions $h$ with two labels: $h(\cdot, 1)$ and $h(\cdot, 2)$
are in $\sF_{\rm{lin}}$ with the natural constraint $h(\cdot,1) +
h(\cdot,2) = 0$ as in \citet{lee2004multicategory}. We also denote by
$\wt \sH_{\mathrm{lin}}$ the linear hypothesis set of functions $\wt
h$ with three scores $\wt h(\cdot,1)$, $\wt h(\cdot, 2)$, $\wt
h(\cdot, 3)$ in $\sF_{\rm{lin}}$ with the same constraint $\wt
h(\cdot, 1) + \wt h(\cdot, 2)=0$ while $\wt h(\cdot, 3)$ is
independent of $\wt h(\cdot,1)$ and $\wt h(\cdot, 2)$. Note that here the
constraint is imposed only to simplify the analysis and is not
necessary for the counter-example to hold.  Thus, for any cost $c \in
\left[0,\frac12\right)$, the Bayes solution in this setting consists
  of abstaining on $\curl*{x\in \sfB_2(1):f_{\rm{abs}}(x)\leq 0}$ and
  otherwise making a prediction according to the decision surface
  $f_{\rm{pred}}(x) = 0$.

In the predictor-rejector formulation, the learner seeks to select a hypothesis $h$ in $\sH_{\rm{lin}}$
and a rejector $r$ in $\sF_{\rm{lin}}$ with small expected predictor-rejector
loss, $\E_{(x, y) \sim \sD}[\labs(h, r, x, y)]$. In the score-based abstention formulation, the learner seeks to select a hypothesis $\wt h$ in $\wt
\sH_{\rm{lin}}$ with small expected score-based abstention loss, $\E_{(x, y) \sim
  \sD}[\labsc(\wt h, x, y)]$. We will show that, in the predictor-rejector formulation, it is straightforward to find the Bayes solution but, in the score-based formulation, the same solution cannot be achieved, unless a more complex family of functions is adopted for $\wt h(\cdot, 3)$.

\textbf{Predictor-rejector formulation succeeds.} For the predictor-rejector abstention loss $\labs$, it is straightforward to see that for any cost $c\in \left[0,\frac12\right)$, the best-in-class predictor and rejector $h^*_{\sH_{\rm{lin}}}$ and $r^*_{\sF_{\rm{lin}}}$ can be expressed as follows:
$
h^*_{\sH_{\rm{lin}}}(\cdot, 1) = f_{\rm{pred}}(\cdot), \,h^*_{\sH_{\rm{lin}}}(\cdot, 2)  = -f_{\rm{pred}}(\cdot), \,r^*_{\sF_{\rm{lin}}}  = f_{\rm{abs}}
$.
Moreover, it is clear that $h^*_{\sH_{\rm{lin}}}$ and
$r^*_{\sF_{\rm{lin}}}$ match the Bayes solution.

\textbf{Score-based abstention formulation fails.} 
For the score-based abstention loss $\labsc$, for any cost $c\in \left[0,\frac12\right)$, the best-in-class classifier $\wt h^*_{\wt \sH_{\rm{lin}}}\!\!$ has the following form:
\[
\wt h^*_{\wt \sH_{\rm{lin}}}(\cdot, 1) =  f_1(\cdot),
\quad
\wt h^*_{\wt \sH_{\rm{lin}}}(\cdot, 2) = -f_1(\cdot),
\quad
\wt h^*_{\wt \sH_{\rm{lin}}}(\cdot,3) = f_2(\cdot),
\]
for some $f_1,f_2\in \sF_{\rm{lin}}$. Thus, $\wt h^*_{\wt
  \sH_{\rm{lin}}}\!\!$ abstains from making prediction on
$\curl*{x\in \sfB_2(1):f_2(x)\geq \abs*{f_1(x)}}$ and otherwise predicts
according to the decision surface $f_1(x) = 0$. To match the Bayes
solution, $f_1$ must equal $f_{\rm{pred}}$ and $f_2$ must
satisfy the following condition:
$
  f_2(x) \geq \abs*{f_{\rm{pred}}(x)}
  \Leftrightarrow f_{\rm{abs}}(x) \leq 0,
$
which does not hold. Therefore,
unless we resort to more complex functions for $\wt h(\cdot,3)$, the
score-based formulation cannot result in the Bayes solution.

\section{Predictor-rejector surrogate losses}
\label{sec:general}

In this section, we present and analyze a new family of
surrogate loss functions $\sfL$ for $\labs$, first in 
the \emph{single-stage setting}, where the predictor $h$ and
the rejector $r$ are selected simultaneously, next in 
a \emph{two-stage setting}, where first the predictor $h$ 
is chosen and fixed and subsequently the rejector $r$ is determined. We give $(\sH,\sR)$-consistency bounds and guarantees for both settings. 

\subsection{Single-stage predictor-rejector surrogate losses}
\label{sec:single-stage-mabs}

In view of the expression of the predictor-rejector abstention loss \[\labs(h, r, x, y)
= \1_{\hh(x) \neq y} \1_{r(x)> 0} + c\1_{r(x) \leq 0},\] if $\ell$ is a surrogate loss for the zero-one
multi-class classification loss over the set of labels $\sY$, then, $\sfL$ defined as follows is a natural surrogate loss for $\labs$: for all $(x, y) \in \sX \times \sY$,
\begin{equation}
\label{eq:sur-general}
\sfL(h, r, x, y)
= \ell(h, x, y)\Phi\paren*{-\alpha r(x)} + \Psi(c) \Phi\paren*{\beta r(x)},
\end{equation}
where $\Psi$ is a non-decreasing
function, $\Phi$ is a non-increasing auxiliary
function upper-bounding $t \mapsto \1_{t \leq 0}$ and $\alpha$, $\beta$ are positive constants. The formulation \eqref{eq:sur-general} of
$\sfL$ is a multi-class generalization of the binary abstention
surrogate loss proposed in \citep{CortesDeSalvoMohri2016bis,CortesDeSalvoMohri2016}, where the binary margin-based loss $\wt\Phi(yh(x))$ and $\Psi(t) =t$ are used instead:
\begin{equation}
\label{eq:sur-general-binary}
\sfL_{\mathrm{bin}}(h, r, x, y) =
\wt\Phi(yh(x))\Phi\paren*{-\alpha r(x)} + c \Phi\paren*{\beta r(x)}.
\end{equation}
Minimizing $\sfL_{\mathrm{bin}}$ with a regularization term was shown to achieve state-of-the-art results in the binary case with margin-based losses $\wt\Phi$ such as the exponential loss $\wt\Phi_{\rm{exp}}(t) = \exp(-t)$ and the hinge loss $\wt\Phi_{\rm{hinge}}(t) = \max\curl*{1 - t,0}$ \citep{CortesDeSalvoMohri2016bis,CortesDeSalvoMohri2016}. However, we will show below that its multi-class generalization $\sfL$ imposes a more stringent condition on the choice of the surrogate loss $\ell$, which rules out for example the multi-class exponential loss. We
will show, however, that several other loss
functions do satisfy that condition, for 
example the multi-class hinge loss.
In the following, for simplicity, we consider $\Phi(t) = \exp(-t)$ as in \citep{CortesDeSalvoMohri2016bis} for the main analysis, though a similar analysis can be given for other functions $\Phi$. We first present a negative result, ruling out 
surrogate losses $\sfL$ that are based on a loss $\ell$ that does not verify a certain
condition. 

As with \citep{AwasthiMaoMohriZhong2022multi}, we assume that the hypothesis sets are \emph{symmetric and complete}.
We say that a hypothesis set $\sG$ is \emph{symmetric} if there exists a family
$\sF$ of functions $f$ mapping from $\sX$ to $\Rset$ such that
$\curl*{\bracket*{g(x,1),\ldots,g(x,n)}\colon g\in
  \sG} = \curl*{\bracket*{f_1(x),\ldots, f_n(x)}\colon f_1, \ldots, f_n\in
  \sF}$, for any $x \in \sX$. We say that a hypothesis set $\sH$ is \emph{complete} if the set
of scores it generates spans $\Rset$, that is, $\curl*{g(x,y)\colon
  g\in \sG} = \Rset$, for any $(x, y)\in \sX \times \sY$. The hypothesis sets widely used in practice including linear models and multilayer feedforward neural networks are all symmetric and complete.

\begin{restatable}[\textbf{Negative result for single-stage surrogates}]{theorem}{NegativeBound}
\label{Thm:negative-bound}
Assume that $\sH$ is symmetric and complete, and that $\sR$ is
complete. If there exists $x \in \sX$ such that $\inf_{h \in \sH}
\E_y\bracket*{\ell(h,X, y) \mid X = x}\neq \frac{\beta \Psi \paren*{1
    - \max_{y\in \sY}\sfp(y \!\mid\! x)}}{\alpha}$, then, there does not exist a
non-decreasing function $\Gamma\colon \Rset_{+} \to \Rset_{+}$ with the property
$\lim_{t\to 0^{+}}\Gamma(t) = 0$ such that the following $(\sH,
\sR)$-consistency bound holds: for all $h \in \sH$, $r \in \sR$, and
any distribution,
\begin{equation*}
\sE_{\labs}(h, r) - \sE_{\labs}^*(\sH, \sR) + \sM_{\labs}(\sH, \sR)
\leq \Gamma\paren*{\sE_{\sfL}(h, r) - \sE_{\sfL}^*(\sH, \sR) +
  \sM_{\sfL}(\sH, \sR)}.
\end{equation*}
\end{restatable}
The proof (Appendix~\ref{app:general-negativ}) proceeds by
contradiction. Assuming that the $(\sH, \sR)$-consistency bound is
valid would entail that the pointwise best-in-class predictor and
best-in-class rejector for the single-stage surrogate loss align with
those of the abstention loss, which can be characterized by
Lemma~\ref{lemma:calibration_gap_general} in
Appendix~\ref{app:general}. Incorporating those explicit forms into
the analysis of the conditional risk of the surrogate loss leads to a
contradiction upon examination of its derivatives.

In view of Theorem~\ref{Thm:negative-bound}, to find a surrogate loss $\sfL$
that admits a meaningful $(\sH, \sR)$-consistency bound, we need to
consider multi-class surrogate losses $\ell$ for which the following
condition holds for any $x\in \sX$, for some $\Psi$ and pair
$(\alpha, \beta) \in \Rset_{+}^2$:
\begin{align*}
  \inf_{h \in \sH} \E_y\bracket*{\ell(h,X, y) \mid X = x}
  = \frac{\beta\Psi\paren*{1 - \max_{y\in \sY}\sfp(y \!\mid\! x)}}{\alpha}.
\end{align*}
In the binary case, this condition is easily verifiable, as $\max_{y\in \sY}\sfp(y \!\mid\! x)$ uniquely determines the other probabilities. However, in the multi-class scenario, a fixed $\max_{y\in \sY}\sfp(y \!\mid\! x)$ still allows for variation in the other probabilities within $\inf_{h \in \sH} \E_y\left[\ell(h,X, y) \mid X = x\right]$. This leads to the difficulties in extending the binary framework to the multi-class classification.

Nevertheless, we will show that this necessary condition is satisfied by three common
multi-class surrogate losses $\ell$. Furthermore, we will prove $(\sH,
\sR)$-consistency bounds for the predictor-rejector surrogate losses
$\sfL$ based on any of these three choices of $\ell$ defined for all
$h \in \sH$ and $(x, y)$ as follows:

(i) The \emph{mean absolute error loss} \citep{ghosh2017robust}: $\ell_{\rm{mae}}(h, x, y) = 1
- \frac{e^{h(x, y)}}{\sum_{y'\in \sY}e^{h(x, y')}}$;

(ii) The \emph{constrained $\rho$-hinge loss}:
$\ell_{\rho-\mathrm{hinge}}(h, x, y) = \sum_{y'\neq
  y}\Phi_{\rho-\rm{hinge}}\paren*{-h(x, y')}$, $\rho > 0$, with
$\Phi_{\rho-\rm{hinge}}(t) = \max\curl[\big]{0, 1 - \frac{t}{\rho}}$
the $\rho$-hinge loss, and the constraint $\sum_{y\in \sY} h(x, y) =
0$.

(iii) The \emph{$\rho$-Margin loss}:
$\ell_{\rho}(h, x, y) = \Phi_{\rho}\paren*{\rho_h(x, y)}$, with
$\rho_h(x, y) = h(x, y) - \max_{y' \neq y} h(x, y')$ the confidence
margin and $\Phi_{\rho}(t) = \min\curl[\big]{\max\curl[\big]{0,1 -
    \frac{t}{\rho}},1}, \rho>0$ the $\rho$-margin loss.
    
\begin{restatable}[\textbf{$(\sH, \sR)$-consistency bounds for
      single-stage surrogates}]{theorem}{SpecificLossBound}
\label{Thm:spcific-loss-bound}
Assume that $\sH$ is symmetric and complete and $\sR$ is
complete. Then, for $\alpha=\beta$, and $\ell = \ell_{\rm{mae}}$, or
$\ell = \ell_{\rho}$ with $\Psi(t) = t$, or $\ell =
\ell_{\rho-\mathrm{hinge}}$ with $\Psi(t) = nt$, the following $(\sH,
\sR)$-consistency bound holds for all $h \in \sH$, $r \in \sR$ and any distribution:
\begin{equation*}
\sE_{\labs}(h, r) - \sE_{\labs}^*(\sH, \sR) + \sM_{\labs}(\sH, \sR)
\leq \Gamma\paren*{\sE_{\sfL}(h, r) - \sE_{\sfL}^*(\sH, \sR) +
  \sM_{\sfL}(\sH, \sR)},
\end{equation*}
where $\Gamma(t) = \max\curl*{2n\sqrt{t}, n\,t}$ for $\ell=
\ell_{\rm{mae}}$; $\Gamma(t) = \max\curl*{2\sqrt{t}, t}$ for $\ell=
\ell_{\rho}$; and $\Gamma(t) = \max\curl*{2\sqrt{n t}, t}$ for $\ell=
\ell_{\rho-\rm{hinge}}$.
\end{restatable}
The theorem provides
strong guarantees for the predictor-rejector surrogate losses we
described in the single-stage setting. The technique used in the proof (Appendix~\ref{app:general-positive-single-stage})
is novel and requires careful analysis of various cases involving the pointwise
best-in-class predictor and rejector. This analysis is
challenging and needs to take into account the \emph{conditional risk} and \emph{calibration gap} (see Appendix~\ref{app:general}) of specific loss functions. The approach is
substantially different from the standard scenarios examined in
\citep{AwasthiMaoMohriZhong2022multi}, due to the simultaneous minimization of both
the predictor and rejector in the abstention setting. Discussions on Theorem~\ref{Thm:spcific-loss-bound} are given in Remark~\ref{remark:spcific-loss-bound} in Appendix~\ref{app:remark}. The following is a direct
consequence of Theorem~\ref{Thm:spcific-loss-bound} when $\sH$ and $\sR$ include all measurable functions,
since the minimizability gaps $\sM_{\labs}$ and $\sM_{\sfL}$ are
then zero.

\begin{restatable}[\textbf{Excess error bounds for single-stage surrogates}]{corollary}{SpecificLossBoundCor}
\label{cor:spcific-loss-bound}
For $\alpha=\beta$, and $\ell= \ell_{\rm{mae}}$, or $\ell = \ell_{\rho}$
with $\Psi(t)=t$, or
$\ell = \ell_{\rho-\mathrm{hinge}}$ with $\Psi(t)=nt$,
the following excess error bound holds for all $h\in \sH_{\rm{all}}$,
$r\in \sR_{\rm{all}}$ and any distribution:
\begin{equation*}
\sE_{\labs}(h, r) - \sE_{\labs}^*(\sH_{\rm{all}}, \sR_{\rm{all}}) 
 \leq \Gamma\paren*{\sE_{\sfL}(h, r)-\sE_{\sfL}^*(\sH_{\rm{all}}, \sR_{\rm{all}})},
\end{equation*}
where $\Gamma$ has the same form as in Theorem~\ref{Thm:spcific-loss-bound}.
\end{restatable}

The corollary resolves in a positive way the open question mentioned
by \citet{NiCHS19}. In fact it provides a stronger result since it
gives non-asymptotic excess error bounds for the three abstention
surrogate losses previously described. These are stronger guarantees
than Bayes-consistency of these loss functions, which follow
immediately by taking the limit.

It is noted that our novel single-stage predictor-rejector
surrogate losses might present some challenges for optimization. This
is due to several factors: the difficulty of optimizing the mean
absolute error loss \citep{zhang2018generalized} (also see
Section~\ref{sec:experiments-mabs}), the restriction imposed by the
constrained hinge loss, which is incompatible with the standard use of the softmax function in neural network hypotheses, and the
non-convexity of the $\rho$-margin loss. However, our primary
objective has been a theoretical analysis and the significance of
these surrogate losses lies in their novelty and strong guarantees. As
shown in Corollary~\ref{cor:spcific-loss-bound}, they are the first
Bayes-consistent surrogate losses within the predictor-rejector
formulation for multi-class abstention, addressing an open question in
the literature \citep{NiCHS19}.

\subsection{Two-stage predictor-rejector surrogate losses}
\label{sec:two-stage-mabs}

Here, we explore a two-stage algorithmic approach, for which we
introduce surrogate losses with more flexible choices of $\ell$ that
admit better optimization properties, and establish $(\sH,
\sR)$-consistency bounds for them.  This is a key scenario since, in
practice, often a large pre-trained prediction model is already
available (first stage), and retraining it would be prohibitively
expensive. The problem then consists of leaving the first stage
prediction model unchanged and of subsequently learning a useful
rejection model (second stage).

Let $\ell$ be a surrogate loss for standard multi-class classification
and $\Phi$ a function like the exponential function that determines a
margin-based loss $\Phi(yr(x))$ in binary classification, with $y \in
\curl*{-1, +1}$ for a function $r$. We propose a two-stage algorithmic
approach and a surrogate loss to minimize in the second stage: first,
find a predictor $h$ by minimizing $\ell$; second, with $h$ fixed,
find $r$ by minimizing $\ell_{\Phi, h}$, a surrogate loss function of
$r$ for all $(x, y)$, defined by
\begin{align}
\label{eq:ell-Phi-h-mabs}
\ell_{\Phi,h}\paren*{r, x, y}  
= \1_{\hh(x) \neq y} \Phi\paren*{-r(x)} + c \Phi\paren*{r(x)},
\end{align}
where $t \mapsto \Phi(t)$ is a non-increasing auxiliary function upper-bounding $\1_{t \leq 0}$.
This algorithmic approach is straightforward since the first stage
involves the familiar task of finding a predictor using a standard
surrogate loss, such as logistic loss (or cross-entropy with softmax). The second stage is also relatively simple as $h$ remains
fixed and the form of $\ell_{\Phi, h}$ is uncomplicated, with $\Phi$
possibly being the logistic or exponential loss.
It is important to underscore that the judicious selection of the
indicator function in the initial term of \eqref{eq:ell-Phi-h-mabs} plays a
crucial role in guaranteeing that the two-stage surrogate loss
benefits from $(\sH, \sR)$-consistency bounds.  If a surrogate loss is
used in the first term, this may not necessarily hold.

Note that in \eqref{eq:ell-Phi-h-mabs}, $h$ is fixed and only $r$ is learned by minimizing the surrogate loss corresponding to that $h$, while in contrast both $h$ and $r$ are jointly learned in the abstention loss \eqref{eq:abs}. We denote by $\ell_{\mathrm{abs}, h}$ the two-stage version of the abstention loss \eqref{eq:abs} with a fixed predictor $h$, defined as: for any $r \in \sR$, $x \in \sX$ and $y \in \sY$:
\begin{equation}
\label{eq:two-stage-abstention-loss}
\ell_{\mathrm{abs}, h} (r, x, y) = \1_{\hh(x) \neq y} \1_{r(x)> 0} + c \1_{r(x) \leq 0}.
\end{equation}
In other words, both $\ell_{\Phi,h}$ and $\ell_{\mathrm{abs}, h}$ are loss functions of the abstention function $r$, while $\labs$ is a loss function of the pair $(h, r) \in (\sH, \sR)$.

Define the binary zero-one classification loss as
$\ell_{0-1}^{\rm{binary}}(r, x, y) = \1_{y\neq \sign(r(x))}$, where
$\sign(t) = \1_{t>0} - \1_{t\leq 0}$.  As with the single-stage
surrogate losses, the two-stage surrogate losses benefit from strong
consistency guarantees as well. We first show that in the second stage
where a predictor $h$ is fixed, the surrogate loss function
$\ell_{\Phi,h}$ benefits from $\sR$-consistency bounds with respect to
$\ell_{\mathrm{abs}, h}$ if, $\Phi$ admits an $\sR$-consistency bound
with respect to the binary zero-one loss $\ell_{0-1}^{\rm{binary}}$.
\begin{restatable}[\textbf{$\sR$-consistency bounds for second-stage surrogates}]{theorem}{BoundGenralSecondStep}
\label{Thm:bound-general-second-step}
Fix a predictor $h$. Assume that $\Phi$ admits an $\sR$-consistency
bound with respect to $\ell_{0-1}^{\rm{binary}}$.  Thus, there exists
a non-decreasing concave function $\Gamma$ such that, for all $r \in
\sR$,
\begin{align*}
\sE_{\ell_{0-1}^{\rm{binary}}}(r) - \sE_{\ell_{0-1}^{\rm{binary}}}^*(\sR) + \sM_{\ell_{0-1}^{\rm{binary}}}(\sR)
& \leq \Gamma\paren*{\sE_{\Phi}(r)-\sE_{\Phi}^*(\sR) +\sM_{\Phi}(\sR)}.
\end{align*}
Then, the following $\sR$-consistency bound holds for all $r\in \sR$ and any distribution:
\begin{equation*}
\sE_{\ell_{\mathrm{abs}, h} }(r)-\sE_{\ell_{\mathrm{abs}, h} }^*(\sR) +\sM_{\ell_{\mathrm{abs}, h} }(\sR) \leq \Gamma\paren*{\paren*{\sE_{\ell_{\Phi, h} }(r)-\sE_{\ell_{\Phi, h} }^*(\sR) +\sM_{\ell_{\Phi, h}}(\sR)}/c}.
\end{equation*}
\end{restatable}
The proof (Appendix~\ref{app:general-positive-second-stage}) consists
of analyzing the calibration gap of the abstention loss and
second-stage surrogate loss, for a fixed predictor $h$.  The
calibration gap here is more complex than that in the standard setting
as it takes into account the conditional probability, the error of
that fixed predictor and the cost, and thus requires a completely
different analysis. To establish $\sR$-consistency bounds, we need to
upper-bound the calibration gap of the abstention loss by that of the
surrogate loss. However, directly working with them is rather
difficult due to their complex forms. Instead, a key observation is
that both forms share structural similarities with the calibration
gaps in the standard classification.  Motivated by the above
observation, we construct an appropriate conditional distribution to
transform the two calibration gaps into standard ones. Then, by
applying Lemma~\ref{lemma:aux-mabs} in Appendix~\ref{app:general}, we
manage to leverage the $\sR$-consistency bound of $\Phi$ with respect
to the binary zero-one classification loss to upper-bound the target
calibration gap by that of the surrogate calibration gap.

We further discuss Theorem~\ref{Thm:bound-general-second-step} in
Remark~\ref{remark:bound-general-second-step}
(Appendix~\ref{app:remark}). In the special case where $\sH$ and $\sR$
are the family of all measurable functions, all the minimizability gap
terms in Theorem~\ref{Thm:bound-general-second-step} vanish. Thus, we
obtain the following corollary.
\begin{corollary}
\label{cor:tsr-mabs}
Fix a predictor $h$. Assume that
$\Phi$ admits an excess error bound with respect to  $\ell_{0-1}^{\rm{binary}}$.  Thus,
there exists a non-decreasing concave functions $\Gamma$
such that, for all $r \in \sR_{\rm{all}}$,
\begin{align*}
\sE_{\ell_{0-1}^{\rm{binary}}}(r) - \sE_{\ell_{0-1}^{\rm{binary}}}^*(\sR_{\rm{all}})
& \leq \Gamma\paren*{\sE_{\Phi}(r)-\sE_{\Phi}^*(\sR_{\rm{all}})}.
\end{align*}
Then, the following excess error bound holds for all $r\in \sR_{\rm{all}}$ and any distribution:
\begin{equation*}
\sE_{\ell_{\mathrm{abs}, h} }(r)-\sE_{\ell_{\mathrm{abs}, h} }^*(\sR_{\rm{all}}) \leq \Gamma\paren*{\paren*{\sE_{\ell_{\Phi, h} }(r)-\sE_{\ell_{\Phi, h} }^*(\sR_{\rm{all}})} / c}.
\end{equation*}
\end{corollary}
See Remark~\ref{remark:tsr} (Appendix~\ref{app:remark}) for a
brief discussion of Corollary~\ref{cor:tsr-mabs}.

We now present $(\sH, \sR)$-consistency bounds for the whole two-stage
approach with respect to the abstention loss function $\labs$. Let
$\ell_{0-1}$ be the multi-class zero-one loss: $\ell_{0-1}(h, x, y) =
1_{\hh(x) \neq y}$. We will consider hypothesis sets $\sR$ that are
\emph{regular for abstention}, that is such that for any $x\in
\sX$, there exist $f, g \in \sR$ with $f(x)>0$ and $g(x)\leq 0$.
If $\sR$ is regular for abstention, then, for any $x$, there is an
option to accept and an option to reject.  \ignore{The next result
  shows that for those margin-based loss functions $\Phi$, their
  corresponding two-stage abstention surrogate losses admit
  $(\sH,\sR)$-consistency bounds with respect to the abstention loss
  $\labs$ as well.}
\begin{restatable}[\textbf{$(\sH, \sR)$-consistency bounds for two-stage
      approach}]{theorem}{BoundGenralTwoStepMabs}
\label{Thm:bound-general-two-step-mabs} 
Suppose that $\sR$ is regular.  Assume that $\ell$ admits an
$\sH$-consistency bound with respect to $ \ell_{0-1}$ and that $\Phi$
admits an $\sR$-consistency bound with respect to
$\ell_{0-1}^{\rm{binary}}$.  Thus, there are non-decreasing concave
functions $\Gamma_1$ and $\Gamma_2$ such that, for all $h\in \sH$ and
$r \in \sR$,
\begin{align*}
\sE_{\ell_{0-1}}(h) - \sE_{\ell_{0-1}}^*(\sH) + \sM_{\ell_{0-1}}(\sH)
& \leq \Gamma_1\paren*{\sE_{\ell}(h)-\sE_{\ell}^*(\sH) +\sM_{\ell}(\sH)}\\
\sE_{\ell_{0-1}^{\rm{binary}}}(r) - \sE_{\ell_{0-1}^{\rm{binary}}}^*(\sR) + \sM_{\ell_{0-1}^{\rm{binary}}}(\sR)
& \leq \Gamma_2\paren*{\sE_{\Phi}(r)-\sE_{\Phi}^*(\sR) +\sM_{\Phi}(\sR)}.
\end{align*}
Then, the following $(\sH,\sR)$-consistency bound holds for all $ h\in \sH$, $r\in \sR$ and any distribution:
\begin{align*}
\sE_{\labs}(h, r) - \sE_{\labs}^*(\sH, \sR) + \sM_{\labs}(\sH, \sR)
& \leq \Gamma_1\paren*{\sE_{\ell}(h)-\sE_{\ell}^*(\sH) +\sM_{\ell}(\sH)}\\
& \quad + (1+c)\Gamma_2\paren*{\paren*{\sE_{\ell_{\Phi,h}}(r)-\sE_{\ell_{\Phi,h}}^*(\sR) +\sM_{\ell_{\Phi,h}}(\sR)}/c},
\end{align*}
where the
constant factors $(1 + c)$ and $\frac{1}{c}$ can be removed 
when $\Gamma_2$ is linear.
\end{restatable}
In the proof (Appendix~\ref{app:general-positive-two-stage}), we
express the pointwise estimation error term for the target abstention
loss as the sum of two terms. The first term represents the pointwise
estimation error of the abstention loss with a fixed $h$, while the
second term denotes that with a fixed $r^*$. This proof is entirely
novel and distinct from the approach used for a standard loss without
abstention in \citep{AwasthiMaoMohriZhong2022multi}. Discussions on
Theorem~\ref{Thm:bound-general-two-step-mabs} are given in
Remark~\ref{remark:bound-general-two-step} in
Appendix~\ref{app:remark}. As before, when $\sH$ and $\sR$ are the
family of all measurable functions, the following result on excess
error bounds holds.
\begin{corollary}
\label{cor:tshr-mabs}
Assume that $\ell$ admits an excess error bound with respect to $
\ell_{0-1}$ and that $\Phi$ admits an excess error bound with respect
to $\ell_{0-1}^{\rm{binary}}$.  Thus, there are non-decreasing concave
functions $\Gamma_1$ and $\Gamma_2$ such that, for all $h\in
\sH_{\rm{all}}$ and $r \in \sR_{\rm{all}}$,
\begin{align*}
\sE_{\ell_{0-1}}(h) - \sE_{\ell_{0-1}}^*(\sH_{\rm{all}})
& \leq \Gamma_1\paren*{\sE_{\ell}(h)-\sE_{\ell}^*(\sH_{\rm{all}})}\\
\sE_{\ell_{0-1}^{\rm{binary}}}(r) - \sE_{\ell_{0-1}^{\rm{binary}}}^*(\sR_{\rm{all}})
& \leq \Gamma_2\paren*{\sE_{\Phi}(r)-\sE_{\Phi}^*(\sR_{\rm{all}})}.
\end{align*}
Then, the following excess error bound holds for all $ h\in \sH_{\rm{all}}$ and $r\in \sR_{\rm{all}}$ and any distribution:
\begin{align*}
\sE_{\labs}(h, r) - \sE_{\labs}^*(\sH_{\rm{all}}, \sR_{\rm{all}})
& \leq \Gamma_1\paren*{\sE_{\ell}(h)-\sE_{\ell}^*(\sH_{\rm{all}})} + (1 + c)\Gamma_2\paren*{\paren*{\sE_{\ell_{\Phi,h}}(r)-\sE_{\ell_{\Phi,h}}^*(\sR_{\rm{all}})}/c},
\end{align*}
where the
constant factors $(1 + c)$ and $\frac{1}{c}$ can be removed 
when $\Gamma_2$ is linear.
\end{corollary}
See Remark~\ref{remark:tshr} (Appendix~\ref{app:remark}) for a
brief discussion of Corollary~\ref{cor:tshr-mabs}.

These results provide a strong guarantee for surrogate losses in the
two-stage setting. Additionally, while the choice of $\ell$ in the
single-stage setting was subject to certain conditions, here, the
multi-class surrogate loss $\ell$ can be chosen more
flexibly. Specifically, it can be selected as the logistic loss (or
cross-entropy with softmax), which is not only easier to optimize but
is also better tailored for complex neural networks. In the second
stage, the formulation is straightforward, and the choice of function
$\Phi$ is flexible, resulting in a simple smooth convex optimization
problem with respect to the rejector function $r$. Moreover, the
second stage simplifies the process as $h$ remains constant and only
the rejector is optimized. This approach can enhance optimization
efficiency. In Appendix~\ref{app:two-stage}, we further highlight the
significance of our findings regarding the two-stage formulation in
comparison with single-stage surrogate losses.

\subsection{Other advantages of the predictor-rejector formulation}
\label{sec:general-realizable}

In this section, we prove another advantage of our predictor-rejector
surrogate losses, that is \emph{realizable consistency}
\citep{long2013consistency, zhang2020bayes}. The property involves
\emph{$(\sH,\sR)$-realizable} distributions, \ignore{that are
  distributions}under which there exist $h^*\in \sH$ and $r^*\in \sR$
such that $\sE_{\labs}(h^*,r^*)=0$. The realizable distribution allows the optimal solution to abstain on points where the cost $c(x)$ is zero.
\begin{definition}
$\sfL$ is \emph{realizable $(\sH,
\sR)$-consistent} with respect to $\labs$ if, for any $(\sH,
\sR)$-realizable distribution, $\lim_{n \rightarrow +\infty}\sE_{\sfL}(h_n,r_n) =  \sE^*_{\sfL}(\sH,\sR) \implies \lim_{n \rightarrow +\infty}\sE_{\labs}(h_n,r_n) =  \sE^*_{\labs}(\sH,\sR)$.
\end{definition}  
In the following, we will establish that our predictor-rejector
surrogate losses, both the single-stage and two-stage variants, are
realizable $(\sH, \sR)$-consistent when $\sH$ and $\sR$ are
\emph{closed under scaling}. A hypothesis set $\sG$ is closed under
scaling if, $g \in \sG$ implies that $\nu g$ is also in $\sG$ for any
$\nu \in \Rset$.\ignore{ This property underscores the advantages of our
predictor-rejector formulation and is supported by the empirical
success of our proposed surrogate losses, as demonstrated in
Section~\ref{sec:experiments-mabs}.}
We will adopt the following mild assumption for the auxiliary function
$\Phi$ in the \ignore{both the single-stage and
  two-stage}predictor-rejector surrogate losses.
\begin{assumption}
\label{assumption:phi}
For any $t \in \Rset$, $\Phi(t) \geq \1_{t \leq 0}$ and $\lim_{t\to
  + \infty}\Phi(t) = 0$.
\end{assumption}

In other words, $\Phi$ upper-bounds the indicator function and
approaches zero as $t$ goes to infinity.  \ignore{This is satisfied by
  common margin-based loss functions, e.g. the hinge function $t
  \mapsto \max(0, 1 - t)$ used in support vector machines, the
  exponential function $t \mapsto e^{-t}$ in Adaboost, and the
  logistic function $t \mapsto \log (1 + e^{-t})$ in logistic
  regression, etc.}We first prove a general result showing that 
single-stage predictor-rejector surrogate losses are realizable
$(\sH,\sR)$-consistent with respect to $\sfL_{\rm{abs}}$ if the
adopted $\ell$ satisfies Assumption~\ref{assumption:ell}.
\begin{assumption}
\label{assumption:ell}
When $h(x, y) - \max_{y' \neq y}h (x, y') > 0$, $\lim_{\nu \to +
  \infty}\ell(\nu h, x,y ) = 0$ and $\ell \geq \ell_{0-1}$.
\end{assumption}
The assumption implies that for a sample $(x, y)$ for which a predictor
$h$ achieves zero error, the infimum value of the loss function $\ell$
is zero for any hypothesis set $\sH$ that includes the
predictor $h$, provided that $\sH$ is closed under scaling.

\begin{restatable}{theorem}{SpecificLossBoundRealizable}
\label{Thm:spcific-loss-bound-realizable}
Assume that $\sH$ and $\sR$ are closed under scaling. Let $\Psi(0)=0$
and $\Phi$ satisfy Assumption~\ref{assumption:phi}. Then, for any $\ell$ that satisfies Assumption~\ref{assumption:ell},
the following $(\sH, \sR)$-consistency
bound holds for any $(\sH,\sR)$-realizable distribution, $h \in \sH$ and 
$r \in \sR$:
\begin{equation*}
\sE_{\labs}(h, r) - \sE_{\labs}^*(\sH, \sR)
\leq \sE_{\sfL}(h, r)-\sE_{\sfL}^*(\sH, \sR).
\end{equation*}
\end{restatable}
\ignore{
\begin{proof}
It is straightforward to see that $\sfL$ serves as an upper bound for $\labs$ when $\ell$ serves as an upper bound for $\ell_{0-1}$ under Assumption~\ref{assumption:phi}.
By definition, for any $(\sH,\sR)$-realizable distribution, there exists $h^*\in \sH$ and $r^*\in \sR$ such that $\sE_{\labs}(h^*,r^*) = \sE_{\labs}^*(\sH, \sR) = 0$. Then, by the assumption that $\sH$ and $\sR$ are closed under scaling, for any $\nu>0$,
\begin{align*}
\sE^*_{\sfL}(\sH,\sR)
&\leq\sE_{\sfL}(\nu h^*,\nu r^*)\\
&=\mathbb{E}\bracket*{\sfL(\nu h^*,\nu r^*,x,y)\mid r^*< 0}\mathbb{P}(r^*< 0) + \mathbb{E}\bracket*{\sfL(\nu h^*,\nu r^*,x,y)\mid r^*>0}\mathbb{P}(r^*> 0)
\end{align*}
Next, we investigate the two terms.
The first term is when $r^*< 0$, then we must have $c=0$ since the data is realizable. By taking the limit, we obtain:
\begin{align*}
&\lim_{\nu\to +\infty}\mathbb{E}\bracket*{\sfL(\nu h^*,\nu r^*,x,y)\mid r^*< 0}\mathbb{P}(r^*< 0)\\
&=\lim_{\nu\to +\infty}\mathbb{E}\bracket*{\ell(\nu h^*, x, y)\Phi\paren*{-\alpha \nu r^*(x)} + \Psi(c) \Phi\paren*{\beta \nu r^*(x)}\mid r^*< 0}\mathbb{P}(r^*< 0)\\
&=\lim_{\nu\to +\infty}\mathbb{E}\bracket*{\ell(\nu h^*, x, y)\Phi\paren*{-\alpha \nu r^*(x)}\mid r^*< 0}\mathbb{P}(r^*< 0) \tag{$c=0$ and $\Psi(0)=0$}\\
&=0. \tag{by the Lebesgue dominated convergence theorem and $\lim_{t\to + \infty}\Phi(t)=0$}
\end{align*}
The second term is when $r^*> 0$, then we must have $h^*(x,y) - \max_{y'\neq y}h^*(x,y') > 0$ since the data is realizable. Thus, using the fact that $\lim_{\nu\to +\infty}\ell(\nu h^*,x,y)=0$ and taking the limit, we obtain
\begin{align*}
&\lim_{\nu\to +\infty}\mathbb{E}\bracket*{\sfL(\nu h^*,\nu r^*,x,y)\mid r^*< 0}\mathbb{P}(r^*< 0)\\
&=\lim_{\nu\to +\infty}\mathbb{E}\bracket*{\ell(\nu h^*, x, y)\Phi\paren*{-\alpha \nu r^*(x)} + \Psi(c) \Phi\paren*{\beta \nu r^*(x)}\mid r^*< 0}\mathbb{P}(r^*< 0)\\
&=0. \tag{by the Lebesgue dominated convergence theorem, $\lim_{t\to + \infty}\Phi(t)=0$, $\lim_{\nu\to +\infty}\ell(\nu h^*,x,y)=0$}
\end{align*}
Therefore, by combining the above two analysis, we obtain
\begin{align*}
\sE^*_{\sfL}(\sH,\sR)\leq \lim_{\nu\to +\infty}\sE_{\sfL}(\nu h^*,\nu r^*)=0.
\end{align*}
By using the fact that $\sfL$ serves as an upper bound for $\labs$ and $\sE_{\labs}^*(\sH, \sR)=0$, we conclude that
\begin{equation*}
\sE_{\labs}(h, r) - \sE_{\labs}^*(\sH, \sR)
\leq \sE_{\sfL}(h, r)-\sE_{\sfL}^*(\sH, \sR).
\end{equation*}
\end{proof}}
The proof is included in Appendix~\ref{app:general-positive-single-stage-realizable}. We first establish upper bounds for $\sE_{\sfL}^*(\sH, \sR)$ using the optimal predictor and rejector for the abstention loss, subsequently expressing the upper bound as the sum of two terms.  By applying the Lebesgue dominated convergence theorem, we show that both terms vanish, and thus $\sE_{\sfL}^*(\sH, \sR) = 0$. It is worth noting that $(\sH, \sR)$-consistency bounds in Theorem~\ref{Thm:spcific-loss-bound} imply the realizable-consistency for considered loss functions. This is because under the realizable assumption, the minimizability gaps vanish for these loss functions. Nevertheless, Theorem~\ref{Thm:spcific-loss-bound-realizable} proves that a more general family of loss functions can actually achieve realizable consistency.

According to
Theorem~\ref{Thm:spcific-loss-bound-realizable}, for any distribution
that is $(\sH, \sR)$-realizable, minimizing single-stage surrogate
estimation loss $\sE_{\sfL}(h, r)-\sE_{\sfL}^*(\sH, \sR)$ results in
the minimization of the abstention estimation loss $\sE_{\labs}(h, r)
- \sE_{\labs}^*(\sH, \sR)$. This suggests that the single-stage
predictor-rejector surrogate loss functions are realizable
$(\sH, \sR)$-consistent. In particular, when $\ell$ is chosen as $2\ell_{\rm{mae}}$, $\ell = \ell_{\rho}$ and $\ell =
\ell_{\rho-\mathrm{hinge}}$ as suggested in Section~\ref{sec:single-stage-mabs}, Assumption~\ref{assumption:phi} is satisfied. Thus, we obtain the following corollary.
\begin{corollary}
Under the same assumption as in Theorem~\ref{Thm:spcific-loss-bound-realizable}, for
$\ell= 2\ell_{\rm{mae}}$, $\ell = \ell_{\rho}$ and $\ell =
\ell_{\rho-\mathrm{hinge}}$,  the single-stage
predictor-rejector surrogate $\sfL$ is realizable $(\sH,
\sR)$-consistent with respect to $\labs$.

\ignore{the following $(\sH, \sR)$-consistency
bound holds for any $(\sH,\sR)$-realizable distribution, $h \in \sH$ and 
$r \in \sR$:
\begin{equation*}
\sE_{\labs}(h, r) - \sE_{\labs}^*(\sH, \sR)
\leq \sE_{\sfL}(h, r)-\sE_{\sfL}^*(\sH, \sR).
\end{equation*}}
\end{corollary}

Next, we prove a similar result showing that the two-stage predictor-rejector surrogate losses are realizable $(\sH,\sR)$-consistent with respect to $\sfL_{\rm{abs}}$ if the multi-class surrogate loss $\ell$ is realizable $\sH$-consistent with respect to the multi-class zero-one loss $\ell_{0-1}$ when $\sH$ is closed under scaling.
\begin{definition}
We say that $\ell$ is realizable $\sH$-consistent with respect to $\ell_{0-1}$ if, for any distribution such that $\sE^*_{\ell_{0-1}}(\sH) = 0$, $\lim_{n \rightarrow +\infty}\sE_{\ell}(h_n) = \sE^*_{\ell}(\sH) \implies \lim_{n \rightarrow +\infty}\sE_{\ell_{0-1}}(h_n) =  \sE^*_{\ell_{0-1}}(\sH) = 0 $.
\end{definition}

\begin{restatable}{theorem}{BoundGenralTwoStepRealizableMabs}
\label{Thm:bound-general-two-step-realizable-mabs}
Assume that $\sH$ and $\sR$ are closed under scaling.  Let $\ell$ be
any multi-class surrogate loss that is realizable $\sH$-consistent with respect to $\ell_{0-1}$ when $\sH$ is closed under scaling and $\Phi$ satisfies Assumption~\ref{assumption:phi}.  Let $\hat h$ be the minimizer of $\sE_{\ell}$ and $\hat r$ be the minimizer of $\sE_{\ell_{\Phi, \hat h}}$.
  Then, for any
$(\sH,\sR)$-realizable distribution, $\sE_{\labs}(\hat h, \hat r) = 0$.
\end{restatable}
The proof is included in
Appendix~\ref{app:general-positive-two-stage-realizable}.  We first
establish the upper bound $\sE_{\labs}(\hat h, \hat r) \leq
\sE_{\ell_{\Phi, \hat h}}(\hat r)$. Next, we analyze two cases:
whether abstention occurs or not. By applying the Lebesgue dominated
convergence theorem, we show that $\sE_{\ell_{\Phi, \hat h}}(\hat r) =
0$ in both cases, consequently leading to $\sE_{\labs}(\hat h, \hat r)
= 0$.  By Theorem~\ref{Thm:bound-general-two-step-realizable-mabs}, under
the realizability assumption, minimizing a two-stage
predictor-rejector surrogate loss leads to zero abstention loss. This
implies that the two-stage predictor-rejector surrogate loss functions
are also realizable $(\sH,
\sR)$-consistent. \citet{KuznetsovMohriSyed2014} prove the realizable
$\sH$-consistency of a broad family of multi-class surrogate losses
including the logistic loss commonly used in practice. Thus, we obtain
the following corollary.
\begin{corollary}
\label{cor:bound-general-two-step-realizable}
Under the same assumption as in
Theorem~\ref{Thm:bound-general-two-step-realizable-mabs}, for $\ell$ being
the logistic loss, the two-stage predictor-rejector surrogate loss is
realizable $(\sH, \sR)$-consistent with respect to $\labs$.
\end{corollary}

Note that existing score-based abstention surrogate losses were shown
to be not realizable consistent in \citep{pmlr-v206-mozannar23a}.
Recall that in Section~\ref{sec:single-stage-mabs} and
Section~\ref{sec:two-stage-mabs}, the $(\sH, \sR)$-consistency bounds
guarantees (applicable to all distributions without any assumptions)
indicate that both our single-stage and two-stage predictor-rejector
surrogate losses are also Bayes-consistent, while it is unknown if the
surrogate loss proposed by \citet{pmlr-v206-mozannar23a} is.
By combining the results from Section~\ref{sec:single-stage-mabs},
Section~\ref{sec:two-stage-mabs} and Section~\ref{sec:general-realizable},
we demonstrate the advantages of the predictor-rejector
formulation. As a by-product of our results, we address an open
question in the literature \citep{NiCHS19}.

\section{Experiments}
\label{sec:experiments-mabs}

\begin{table}[t]
\caption{Abstention loss of our predictor-rejector surrogate losses against baselines:
the state-of-the-art score-based abstention surrogate losses in
\citep{mozannar2020consistent,caogeneralizing}.}
    \label{tab:comparison-mabs}
\begin{center}
    \begin{tabular}{@{\hspace{0pt}}lll@{\hspace{0pt}}}
    \toprule
      Dataset & Method & Abstention loss \\
    \toprule
    \multirow{4}{*}{SVHN} & \citep{mozannar2020consistent} &  1.61\% $\pm$ 0.06\% \\
     & \citep{caogeneralizing} & 2.16\% $\pm$ 0.04\%\\
     & single-stage predictor-rejector ($\ell_{\rm{mae}}$) &  2.22\% $\pm$ 0.01\% \\
     & \textbf{two-stage predictor-rejector}  & \textbf{0.94\% \!$\pm$ 0.02\%} \\
    \midrule
    \multirow{4}{*}{CIFAR-10} & \citep{mozannar2020consistent} &  4.48\% $\pm$ 0.10\% \\
     & \citep{caogeneralizing}  & 3.62\% $\pm$ 0.07\%  \\
     & single-stage predictor-rejector ($\ell_{\rm{mae}}$) & 3.64\% $\pm$ 0.05\% \\
     & \textbf{two-stage predictor-rejector}  & \textbf{3.31\% \!$\pm$ 0.02\%}   \\
    \midrule
    \multirow{4}{*}{CIFAR-100} & \citep{mozannar2020consistent} &  10.40\% $\pm$ 0.10\% \\
     & \citep{caogeneralizing}  & 14.99\% $\pm$ 0.01\%\\
     & single-stage predictor-rejector ($\ell_{\rm{mae}}$) & 14.99\% $\pm$ 0.01\% \\
     & \textbf{two-stage predictor-rejector} & \textbf{\phantom{0}9.23\% \!$\pm$ 0.03\%}  \\
    \bottomrule
    \end{tabular}
\end{center}
\end{table}
In this section, we present experimental results for the single-stage
and two-stage predictor-rejector surrogate losses, as well as for the
state-of-the-art score-based abstention surrogate losses
\citep{mozannar2020consistent,caogeneralizing} on three popular
datasets: SVHN \citep{Netzer2011}, CIFAR-10 and CIFAR-100
\citep{Krizhevsky09learningmultiple}. Note that the basic
confidence-based approach has already been shown in
\citep{caogeneralizing} to be empirically inferior to state-of-the-art
score-based abstention surrogate losses. More details on the
experiments are included in Appendix~\ref{app:setup}.

\textbf{Results.}  In Table~\ref{tab:comparison-mabs}, we report the mean
and standard deviation of the abstention loss over three runs for our
algorithms and the baselines.
Table~\ref{tab:comparison-mabs} shows that our two-stage predictor-rejector
surrogate loss consistently outperforms the state-of-the-art
score-based abstention surrogate losses in
\citep{mozannar2020consistent,caogeneralizing} across all cases. The
single-stage predictor-rejector surrogate loss with $\ell$ set as the
mean absolute error loss achieves comparable results. Our
predictor-rejector surrogate losses, both the single-stage and
two-stage variants, benefit from $(\sH, \sR)$-consistency bounds and
realizable $(\sH, \sR)$-consistency guarantees. 
While the optimization of mean absolute error loss is known to be
challenging, as highlighted in the study by Zhang et al. (2018), our
two-stage algorithm sidesteps this hurdle since it can use the
more tractable logistic loss.

\section{Conclusion}

We presented a series of theoretical, algorithmic, and empirical
results for multi-class classification with abstention. Our
theoretical analysis, including proofs of $(\sH, \sR)$-consistency
bounds and realizable $(\sH, \sR)$-consistency, covers both single-stage
and two-stage predictor-rejector surrogate losses.
These results further provide valuable tools applicable to the
analysis of other loss functions in learning with abstention.

Our two-stage algorithmic approach provides practical and efficient
solutions for multi-class abstention across various tasks. This
approach proves particularly advantageous in scenarios where a large
pre-trained prediction model is readily available, and the expense
associated with retraining is prohibitive. Our empirical findings
corroborate the efficacy of these algorithms, further reinforcing
their practical usefulness.
Additionally, our work reveals some limitations of the 
score-based abstention formulation, such as its inability to
consistently yield optimal solutions in certain cases. In contrast, we
present a collection of positive outcomes for various families of
predictor-rejector surrogate loss functions. Importantly, our findings
also provide resolutions to two open questions within the literature.

We believe that our analysis and the novel loss functions we
introduced can guide the design of algorithms across a
broad spectrum of scenarios beyond classification with abstention.

\chapter{Single-Stage Multi-Expert Deferral} \label{ch4}
In this chapter, we study the general framework of learning with multi-expert deferral.
We first introduce a new family of surrogate losses specifically
tailored for the multiple-expert setting, where the prediction and
deferral functions are learned simultaneously
(Section~\ref{sec:general-surrogate-losses}). Next, we prove that
these surrogate losses benefit from $\sH$-consistency bounds
(Section~\ref{sec:H-consistency-bounds}).
This implies, in particular, their Bayes-consistency.
We illustrate the application of our analysis through several examples
of practical surrogate losses, for which we give explicit guarantees.
These loss functions readily lead to the design of new learning to
defer algorithms based on their minimization. 
Our $\sH$-consistency bounds incorporate a crucial term known as the
\emph{minimizability gap}.  We show that this makes them more
advantageous guarantees than bounds based on the approximation error
(Section~\ref{sec:minimizability_gaps}).
We further demonstrate that our $\sH$-consistency bounds can be used
to derive generalization bounds for the minimizer of a surrogate loss
expressed in terms of the minimizability gaps
(Section~\ref{sec:learning-bound}). While the main focus of this chapter is a theoretical analysis, we also report the results of several experiments with SVHN and CIFAR-10 datasets (Section~\ref{sec:experiments}).
\ignore{Finally, we present the results of several experiments with multiple
datasets demonstrating the effectiveness of these algorithms
(Section~\ref{sec:experiments}).}

We start with the introduction of
preliminary definitions and notation needed for our discussion of the
problem of learning with multiple-expert deferral.

The presentation in this chapter is based on \citep{MaoMohriZhong2024deferral}.

\section{Preliminaries}
\label{sec:pre-tdef}

\begin{figure}[t]
    \centering
    \resizebox{0.9\textwidth}{!}{
    \begin{tikzpicture}[node distance=2cm]
    \node (x) [io] {input $x \in \sX$};
    \node (h) [process, right of=x, xshift=0.4cm] {hypothesis $h \in \sH$};
    
    \node (h3) [decision1, right of=h, xshift=2.5cm] {predict 3};
    \node (h2) [decision1, above of=h3, yshift=-0.4cm] {predict 2};
    \node (h1) [decision1, above of=h2, yshift=-0.4cm] {predict 1};
    \node (h4) [decision2, below of=h3, yshift=0.4cm] {expert $\expert_1$};
    \node (h5) [decision2, below of=h4, yshift=0.4cm] {expert $\expert_2$};
    
    \node (c1) [startstop1, right of=h1, xshift=1cm] {$\1_{1\neq y}$};
    \node (c2) [startstop1, right of=h2, xshift=1cm] {$\1_{2\neq y}$};
    \node (c3) [startstop1, right of=h3, xshift=1cm] {$\1_{3\neq y}$};
    \node (c4) [startstop2, right of=h4, xshift=1cm] {$c_1(x, y)$};
    \node (c5) [startstop2, right of=h5, xshift=1cm] {$c_2(x, y)$};
    \node (loss) [above of=c1, yshift=-1.2cm] {incur loss};
    
    \node (y) [io, right of=c3, xshift=1cm] {label $y \in \sY$};
    
    \draw [arrow] (x) -- (h);
    
    \draw [arrow] (h) -- node[sloped, anchor=center, above, yshift=-0.1cm] {$\hh(x) = 1$} (h1);
    \draw [arrow] (h) -- node[sloped, anchor=center, above, yshift=-0.1cm] {$\hh(x) = 2$} (h2);
    \draw [arrow] (h) -- node[sloped, anchor=center, above, yshift=-0.1cm] {$\hh(x) = 3$} (h3);
    \draw [arrow] (h) -- node[sloped, anchor=center, above, yshift=-0.1cm] {$\hh(x) = 4$} (h4);
    \draw [arrow] (h) -- node[sloped, anchor=center, above, yshift=-0.1cm] {$\hh(x) = 5$} (h5);
    
    \draw [arrow] (h1) -- (c1);
    \draw [arrow] (h2) -- (c2);
    \draw [arrow] (h3) -- (c3);
    \draw [arrow] (h4) -- (c4);
    \draw [arrow] (h5) -- (c5);
    
    \draw [arrow] (y) -- (c1);
    \draw [arrow] (y) -- (c2);
    \draw [arrow] (y) -- (c3);
    \draw [arrow] (y) -- (c4);
    \draw [arrow] (y) -- (c5);
    \end{tikzpicture}
     }
    \caption{Illustration of the scenario of learning with multiple-expert deferral ($n=3$ and $\num=2$).}
    \label{fig:deferral}
\end{figure}
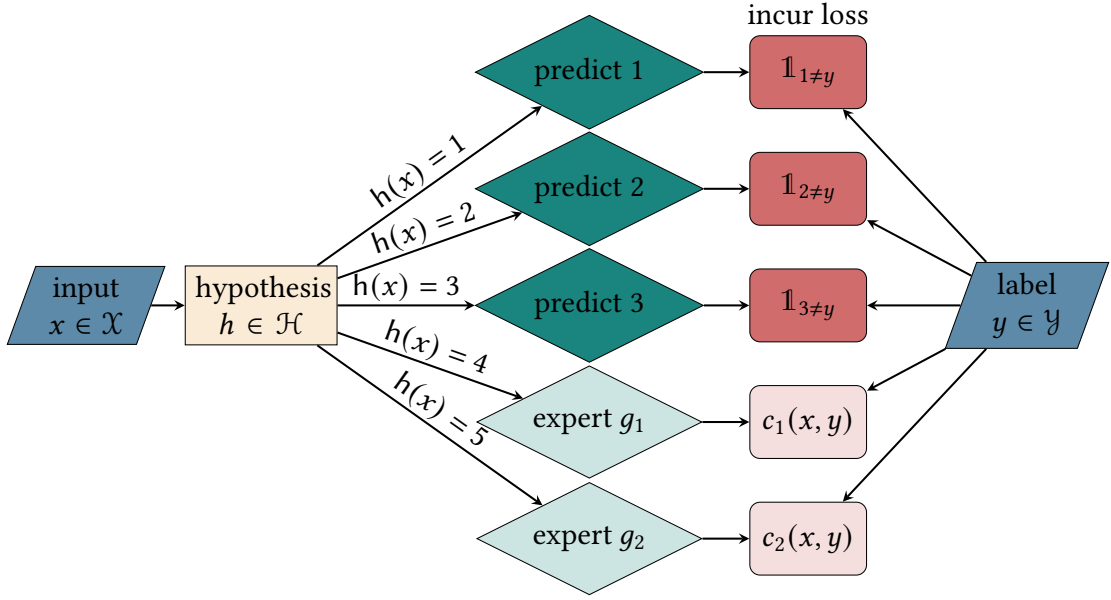

We consider the standard multi-class classification setting with an
input space $\sX$ and a set of $n \geq 2$ labels $\sY =[n]$, where we
use the notation $[n]$ to denote the set $\curl*{1, \ldots, n}$.
We study the scenario of \emph{learning with multi-expert deferral}, where the label set
$\sY$ is augmented with $\num$ additional labels $\curl*{n + 1,
  \ldots, n + \num}$ corresponding to $\num$ pre-defined experts $\expert_1,
\ldots, \expert_\num$, which are a series of functions mapping from $\sX\times \sY$ to $\Rset$.  In this scenario, the learner has the option of
returning a label $y \in \sY$, which represents the category
predicted, or a label $y = n + j$, $1 \leq j \leq \num$, in which case it is
\emph{deferring} to expert $\expert_j$.

We denote by $\ov \sY = [n + \num]$ the augmented label set and
consider a hypothesis set $ \sH$ of functions mapping from $\sX \times
\ov \sY$ to $\Rset$. The prediction associated by $h \in \sH$ to an
input $x \in \sX$ is denoted by $\hh(x)$ and defined as the element in
$\ov \sY$ with the highest score, $\hh(x) = \argmax_{y \in [n + \num]}
h(x, y)$, with an arbitrary but fixed deterministic strategy for
breaking ties. We denote by $\sH_{\rm{all}}$ the family of all
measurable functions.

The \emph{deferral loss function}
$\ldefsc$ is defined as follows for any $h \in \sH$ and $(x, y) \in
\sX \times \sY$:
\begin{equation}
\label{eq:def-score-tdef}
\ldefsc(h, x, y)
= \1_{\hh(x)\neq y} \1_{\hh(x)\in [n]}
+ \sum_{j = 1}^{\num} c_j(x, y) \1_{\hh(x) = n + j}
\end{equation}
Thus, the loss incurred coincides with the standard zero-one
classification loss when $\hh(x)$, the label predicted, is in $\sY$.
Otherwise, when $\hh(x)$ is equal to $n + j$, the loss incurred is
$c_j(x, y)$, the cost of deferring to expert $\expert_j$. We give an illustration of the scenario of learning to defer with three classes and two experts ($n=3$ and $\num=2$) in Figure~\ref{fig:deferral}. We will denote by
$\uv c_j \geq 0$ and $\ov c_j \leq 1$ finite lower and upper bounds on
the cost $c_j$, that is $c_j(x, y)\in [\uv c_j,\ov c_j]$ for all $(x,
y) \in \sX \times \sY$.
There are many possible choices for these costs. Our analysis is
general and requires no assumption other than their boundedness.
One natural choice is to define cost $c_j$ as a function relying on expert
$\expert_j$'s accuracy, for example
$c_j(x, y) = \alpha_j \1_{\expertexpert_{j}(x) \neq y} +
\beta_j$, with $\alpha_j, \beta_j > 0$, where $\expertexpert_{j}(x)=\argmax_{y\in [n]}\expert_j(x,y)$ is the
prediction made by expert $\expert_j$ for input $x$.

Given a distribution $\sD$ over $\sX \times \sY$, we will denote by
$\sE_{\ldefsc}(h)$ the expected deferral loss of a hypothesis $h \in
\sH$,
\begin{equation}
\sE_{\ldefsc}(h) = \E_{(x, y) \sim \sD}[\ldefsc(h, x, y)],
\end{equation}
and
by $\sE^*_{\ldefsc}(\sH) = \inf_{h \in \sH} \sE_{\ldefsc}(h)$ its
infimum or best-in-class expected loss. We will adopt similar
definitions for any surrogate loss function $\lsc$: 
\begin{equation}
\sE_{\lsc}(h) = \E_{(x, y) \sim \sD}[\lsc(h, x, y)],\quad\sE^*_{\lsc}(\sH) = \inf_{h \in \sH} \sE_{\lsc}(h).
\end{equation}

\section{General surrogate losses}
\label{sec:general-surrogate-losses}

In this section, we introduce a new family of surrogate losses
specifically tailored for the multiple-expert setting starting from
first principles.

The scenario we consider is one where the
prediction (first $n$ scores) and deferral functions (last $\num$
scores) are learned simultaneously.
Consider a hypothesis $h \in \sH$. Note that, for any $(x, y) \in \sX
\times \sY$, if the learner chooses to defer to an expert, $\hh(x) \in \curl*{n + 1, \ldots, n + \num}$, then it does not make a prediction of the category, and thus
$\hh(x) \neq y$\ignore{and $\hh(x) = n + j$, for some $j \in [\num]$}. This implies that
the following identity holds: 
\[\1_{\hh(x)\neq y} \1_{\hh(x) \in
  \curl*{n + 1,\ldots, n + \num}} = \1_{\hh(x) \in \curl*{n +
    1,\ldots, n + \num}}.\] 
Using this identity and $\1_{\hh(x)\in [n]} = 1 - \1_{\hh(x)
  \in \curl*{n + 1,\ldots, n + \num}}$, we can write the first term of \eqref{eq:def-score-tdef} as $\1_{\hh(x)\neq y} - \1_{\hh(x) \in \curl*{n +
    1,\ldots, n + \num}}$. Note that deferring occurs if and only if one of the experts is selected, that is $\1_{\hh(x) \in \curl*{n + 1,\ldots, n + \num}}
  = \sum_{j = 1}^{\num}\1_{\hh(x) = n + j}$. Therefore, the deferral loss function can be written in the following form for any $h \in \sH$ and $(x, y) \in \sX \times \sY$:
\begin{align*}
&\ldef(h, x, y)\\
& = \1_{\hh(x)\neq y} - \sum_{j = 1}^{\num}\1_{\hh(x) = n + j} + \sum_{j = 1}^{\num} c_j(x, y) \1_{\hh(x) = n + j}\\
& = \1_{\hh(x)\neq y} + \sum_{j = 1}^{\num} \paren*{c_j(x, y) - 1} \1_{\hh(x) = n + j}\\
& =  \1_{\hh(x)\neq y}
+ \sum_{j = 1}^{\num} \paren*{1 - c_j(x, y)} \1_{\hh(x) \neq n + j}
+ \sum_{j = 1}^{\num} \paren*{c_j(x, y) - 1}.
\end{align*}
In light of this expression, since the last term $\sum_{j = 1}^{\num}
\paren*{c_j(x, y) - 1}$ does not depend on $h$, if $\ell$ is a
surrogate loss for the zero-one multi-class classification loss over the augmented label set $\ov \sY$, then $\lsc$, defined as follows for any $h \in \sH$ and $(x, y) \in \sX \times \sY$,
is a natural surrogate loss for $\ldefsc$:
\begin{equation}
\label{eq:sur-score}
\lsc(h, x, y)
=  \ell \paren*{h, x, y} + \sum_{j = 1}^{\num} \paren*{1 - c_j(x, y)} \,  \ell\paren*{h, x, n + j}.
\end{equation}
We will study the properties of the general family of surrogate losses
$\lsc$ thereby defined. Note that in the special case where $\ell$ is
the logistic loss and $\num = 1$, that is where there is only one
pre-defined expert, $\lsc$ coincides with the surrogate loss proposed
in \citep{mozannar2020consistent,caogeneralizing,MaoMohriZhong2024score}. However, even for
that special case, our derivation of the surrogate loss from first
principle is new and it is this analysis that enables us to define a
surrogate loss for the more general case of multiple experts and other
$\ell$ loss functions. Our formulation also recovers the softmax
surrogate loss in \citep{verma2023learning} when
$\ell=\ell_{\rm{log}}$ and $c_j(x,y)=1_{\expertexpert_j(x)\neq y}$.

\section{\texorpdfstring{$\sH$}{H}-consistency bounds for surrogate losses}
\label{sec:H-consistency-bounds}

Here, we prove strong consistency guarantees for a surrogate deferral
loss $\lsc$ of the form described in the previous section, provided
that the loss function $\ell$ it is based upon admits a similar
consistency guarantee with respect to the standard zero-one
classification loss.

\textbf{$\sH$-consistency bounds.} To do so, we will adopt the notion
of \emph{$\sH$-consistency bounds} recently introduced by
\citet*{awasthi2022Hconsistency,AwasthiMaoMohriZhong2022multi} and also studied in \citep{awasthi2021calibration,awasthi2021finer,AwasthiMaoMohriZhong2023theoretically,awasthi2024dc,MaoMohriZhong2023cross,MaoMohriZhong2023ranking,MaoMohriZhong2023rankingabs,zheng2023revisiting,MaoMohriZhong2023characterization,MaoMohriZhong2023structured,mao2024h,mao2024regression,mao2024universal,mao2025enhanced,mao2024multi,mao2024realizable,MohriAndorChoiCollinsMaoZhong2024learning,cortes2024cardinality,zhong2025fundamental,MaoMohriZhong2025principled,MaoMohriZhong2025mastering,cortes2025balancing,cortes2025improved,desalvo2025budgeted,mohri2025beyond}. These
are guarantees that, unlike Bayes-consistency or excess error bound,
take into account the specific hypothesis set $\sH$ and do not assume
$\sH$ to be the family of all measurable functions. Moreover, in
contrast with Bayes-consistency, they are not just asymptotic
guarantees. In this context, they have the following form:
$\sE_{\ldefsc}(h) - \sE_{\ldefsc}^*(\sH) \leq f\paren*{\sE_{\lsc}(h) -
  \sE_{\lsc}^*(\sH)}$, where $f$ is a non-decreasing function,
typically concave.  Thus, when the surrogate estimation loss
$\paren*{\sE_{\lsc}(h) - \sE_{\lsc}^*(\sH)}$ is reduced to $\e$, the
deferral estimation loss $\paren*{\sE_{\ldefsc}(h) -
  \sE_{\ldefsc}^*(\sH)}$ is guaranteed to be at most $f(\e)$.

\textbf{Minimizability gaps.} A key quantity appearing in these bounds
is the \emph{minimizability gap} $\sM_{\ell}(\sH)$ which, for a loss
function $\ell$ and hypothesis set $\sH$, measures the difference of
the best-in-class expected loss and the expected pointwise infimum of
the loss:
\[
\sM_{\ell}(\sH) = \sE^*_{\ell}(\sH) -
\E_x\bracket[\big]{\inf_{h \in \sH} \E_{y | x} \bracket*{\ell(h, x,
    y)}}.
\]
By the super-additivity of the infimum, since $\sE^*_{\ell}(\sH) =
\inf_{h \in \sH} \E_x\bracket[\big]{\E_{y | x} \bracket*{\ell(h, x,
    y)}}$, the minimizability gap is always non-negative.

When a loss function $\ell$ only depends on $h(x,\cdot)$ for all
$h$, $x$, and $y$, that is \[\ell(h, x, y) = \Psi(h(x,1), \ldots,
h(x,n), y),\] for some function $\Psi$, then it is not hard to show
that the minimizability gap vanishes for the family of all measurable
functions: $\sM_{\ell}(\sH_{\rm{all}}) = 0$
\citep{steinwart2007compare}[lemma~2.5]. It is also null when
$\sE^*_{\ell}(\sH) = \sE^*_{\ell}(\sH_{\rm{all}})$, that is when
the Bayes-error coincides with the best-in-class error. In general,
however, the minimizabiliy gap is non-zero for a restricted hypothesis
set $\sH$ and is therefore important to analyze.  
In Section~\ref{sec:minimizability_gaps}, we will discuss in more
detail minimizability gaps for a relatively broad case and demonstrate
that $\sH$-consistency bounds with minimizability gaps can often be
more favorable than excess error bounds based on the approximation
error.

The following theorem is the main result of this section.

\begin{restatable}[\textbf{$\sH$-consistency bounds for
      score-based surrogates}]{theorem}{BoundScore}
\label{Thm:bound-score}
Assume that $\ell$ admits an $\sH$-consistency bound with respect to
the multi-class zero-one classification loss $\ell_{0-1}$. Thus, there
exists a non-decreasing concave function $\Gamma$ with $\Gamma(0)=0$
such that, for any distribution $\sD$ and for all $h \in \sH$, we have
\begin{equation*}
\sE_{\ell_{0-1}}(h) - \sE_{\ell_{0-1}}^*( \sH) + \sM_{\ell_{0-1}}( \sH)
\leq \Gamma\paren*{\sE_{\ell}(h)-\sE_{\ell}^*( \sH) + \sM_{\ell}(\sH)}.
\end{equation*}
Then, $\lsc$ admits the following $ \sH$-consistency bound with
respect to $\ldefsc$: for all $h\in \sH$,
\begin{equation}
\label{eq:H-consistency-bounds}
\sE_{\ldefsc}(h) - \sE_{\ldefsc}^*( \sH) + \sM_{\ldefsc}( \sH)
\leq \paren[\bigg]{\num + 1 - \sum_{j = 1}^{\num}\uv c_j} \Gamma\paren*{\frac{\sE_{\lsc}(h) - \sE_{\lsc}^*( \sH) + \sM_{\lsc}( \sH)}{\num + 1-\sum_{j = 1}^{\num}\ov c_j}}.
\end{equation}
Furthermore, constant factors $\paren*{\num + 1 - \sum_{j = 1}^{\num}\uv
  c_j}$ and $\frac{1}{\num + 1 - \sum_{j = 1}^{\num}\ov c_j}$ can be
removed when $\Gamma$ is linear.
\end{restatable}
The proof is given in Appendix~\ref{app:score}. It consists of first
analyzing the conditional regret of the deferral loss and that of a
surrogate loss. Next, we show how the former can be upper-bounded in
terms of the latter by leveraging the $\sH$-consistency bound of
$\ell$ with respect to the zero-one loss with an appropriate
conditional distribution that we construct.  This, combined with the
results of \citet{AwasthiMaoMohriZhong2022multi}, proves our
$\sH$-consistency bounds.

Let us emphasize that the theorem is broadly applicable and that there
are many choices for the surrogate loss $\ell$ meeting the assumption
of the theorem: \citet{AwasthiMaoMohriZhong2022multi} showed that a
variety of surrogate loss functions $\ell$ admit an $\sH$-consistency
bound with respect to the zero-one loss for common hypothesis sets
such as linear models and multi-layer neural networks, including
\emph{sum losses} \citep{weston1998multi}, \emph{constrained losses}
\citep{lee2004multicategory}, and, as shown more recently by
\citet{MaoMohriZhong2023cross} (see also \citep{zheng2023revisiting,MaoMohriZhong2023characterization}),
\emph{comp-sum losses}, which include the logistic loss
\citep{Verhulst1838,Verhulst1845,Berkson1944,Berkson1951}, the
\emph{sum-exponential loss} and many other loss functions.

Thus, the theorem gives a strong guarantee for a broad family of
surrogate losses $\lsc$ based upon such loss functions $\ell$.  The
presence of the minimizability gaps in these bounds is important.  In
particular, while the minimizability gap can be upper-bounded by the
approximation error $\sA_{\ell}(\sH)= \sE^*_{\ell}(\sH) -
\E_x\bracket[\big]{\inf_{h \in \sH_{\rm{all}}} \E_{y | x}
  \bracket*{\ell(h, x, y)}} =
\sE^*_{\ell}(\sH)-\sE^*_{\ell}(\sH_{\rm{all}})$, it is a finer
quantity than the approximation error and can lead to more favorable
guarantees.

Note that when the Bayes-error coincides with the best-in-class error,
$\sE^*_{\sfL}(\sH) = \sE^*_{\sfL}(\sH_{\rm{all}})$, we have
$\sM_{\lsc}( \sH)\leq \sA_{\lsc}( \sH) = 0$.  This leads to the
following corollary, using the non-negativity property of the
minimizability gap.
\begin{corollary}
\label{cor:bound-score}
Assume that $\ell$ admits an $\sH$-consistency bound with respect to
the multi-class zero-one classification loss $\ell_{0-1}$. Then, for
all $h\in \sH$ and any distribution such that
$\sE^*_{\sfL}(\sH)=\sE_{\sfL}^*(\sH_{\rm{all}})$, the following bound holds:
\begin{equation*}
\sE_{\ldefsc}(h) - \sE_{\ldefsc}^*( \sH)
\leq \paren[\bigg]{\num + 1 - \sum_{j = 1}^{\num}\uv c_j} \Gamma\paren*{\frac{\sE_{\lsc}(h) - \sE_{\lsc}^*( \sH)}{\num + 1-\sum_{j = 1}^{\num}\ov c_j}},
\end{equation*}
Furthermore, constant factors $\paren*{\num + 1 - \sum_{j = 1}^{\num}\uv c_j}$ and $\frac{1}{\num + 1-\sum_{j = 1}^{\num}\ov c_j}$ can be removed 
when $\Gamma$ is linear.
\end{corollary}
Thus, when the estimation error of the surrogate loss, $\sE_{\lsc}(h)
- \sE_{\lsc}^*(\sH)$, is reduced to $\e$, the estimation error of the
deferral loss, $\sE_{\ldefsc}(h) - \sE_{\ldefsc}^*(\sH)$, is upper
bounded by \[\paren*{\num + 1 - \sum_{j = 1}^{\num}\uv c_j}
\Gamma\paren*{\e/\paren*{\num + 1 - \sum_{j = 1}^{\num}\ov c_j}}.\]
Moreover, $\sH$-consistency holds since $\sE_{\lsc}(h) -
\sE_{\lsc}^*(\sH) \to 0$ implies $\sE_{\ldefsc}(h) -
\sE_{\ldefsc}^*(\sH)\to 0$.

Table~\ref{tab:sur-score-comp} shows several examples of surrogate
deferral losses and their corresponding $\sH$-consistency bounds,
using the multi-class $\sH$-consistency bounds known for comp-sum
losses $\ell$ with respect to the zero-one loss
\citep[Theorem~1]{MaoMohriZhong2023cross}.  The bounds have been simplified here
using the inequalities $1 \leq \num + 1 - \sum_{j = 1}^{\num}\ov
c_j\leq \num + 1 - \sum_{j = 1}^{\num}\uv c_j\leq \num + 1$. See
Appendix~\ref{app:sur-score-example-comp} for a more detailed
derivation.

Similarly, Table~\ref{tab:sur-score-sum} and
Table~\ref{tab:sur-score-cstnd} show several examples of surrogate
deferral losses with sum losses or constrained losses adopted for
$\ell$ and their corresponding $\sH$-consistency bounds, using the
multi-class $\sH$-consistency bounds in
\citep[Table~2]{AwasthiMaoMohriZhong2022multi} and
\citep[Table~3]{AwasthiMaoMohriZhong2022multi} respectively. Here too,
we present the simplified bounds by using the inequalities $1 \leq
\num + 1 - \sum_{j = 1}^{\num}\ov c_j\leq \num + 1 - \sum_{j =
  1}^{\num}\uv c_j\leq \num + 1$.  See
Appendix~\ref{app:sur-score-example-sum} and
Appendix~\ref{app:sur-score-example-cstnd} for a more detailed
derivation.

\begin{table}[t]
  \centering
  \resizebox{\textwidth}{!}{
  \begin{tabular}{@{\hspace{0cm}}lll@{\hspace{0cm}}}
    \toprule
    $\ell$  & $\sfL$   & $\sH$-consistency bounds\\
    \midrule
    $\ell_{\rm{exp}}$ & $\sum_{y'\neq y} e^{h(x, y') - h(x, y)}+\sum_{j=1}^{\num}(1-c_j(x,y))\sum_{y'\neq n+j} e^{h(x, y') - h(x, n+j)}$ & $\sqrt{2}(\num+1)\paren*{\sE_{\lsc}(h) - \sE_{\lsc}^*( \sH)}^{\frac12}$\\
    $\ell_{\rm{log}}$   & $-\log\paren*{\frac{e^{h(x,y)}}{\sum_{y'\in \ov \sY}e^{h(x,y')}}}-\sum_{j=1}^{\num}(1-c_j(x,y))\log\paren*{\frac{e^{h(x,n+j)}}{\sum_{y'\in \ov \sY}e^{h(x,y')}}}$ & $\sqrt{2}(\num+1)\paren*{\sE_{\lsc}(h) - \sE_{\lsc}^*( \sH)}^{\frac12}$   \\
    $\ell_{\rm{gce}}$    & $\frac{1}{\alpha}\bracket*{1 - \bracket*{\frac{e^{h(x,y)}}
    {\sum_{y'\in \ov \sY} e^{h(x,y')}}}^{\alpha}}+\frac{1}{\alpha}\sum_{j=1}^{\num}(1-c_j(x,y))\bracket*{1 - \bracket*{\frac{e^{h(x,n+j)}}
    {\sum_{y'\in \ov \sY} e^{h(x,y')}}}^{\alpha}}$  & $\sqrt{2n^{\alpha}}(\num+1)\paren*{\sE_{\lsc}(h) - \sE_{\lsc}^*( \sH)}^{\frac12}$  \\
    $\ell_{\rm{mae}}$ & $ 1 - \frac{e^{h(x,y)}}{\sum_{y'\in \ov \sY} e^{h(x, y')}}+\sum_{j=1}^{\num}(1-c_j(x,y))\paren*{1 - \frac{e^{h(x,n+j)}}{\sum_{y'\in \ov \sY} e^{h(x, y')}}}$ &    $n \paren*{\sE_{\lsc}(h) - \sE_{\lsc}^*( \sH)}$    \\
    \bottomrule
  \end{tabular}
  }
  \caption{Examples of the deferral surrogate loss \eqref{eq:sur-score} with comp-sum losses adopted for $\ell$
  and their associated $\sH$-consistency bounds provided by
  Corollary~\ref{cor:bound-score} (with only the surrogate portion
  displayed).}
\label{tab:sur-score-comp}
\end{table}

\begin{table}[t]
  \centering
  \resizebox{\textwidth}{!}{
  \begin{tabular}{@{\hspace{0cm}}lll@{\hspace{0cm}}}
    \toprule
    $\ell$  & $\sfL$   & $\sH$-consistency bounds\\
    \midrule
    $\Phi_{\mathrm{sq}}^{\mathrm{sum}}$ & 
    $\sum_{y'\neq y} \Phi_{\rm{sq}}\paren*{\Delta_h(x,y,y')}+\sum_{j=1}^{\num}(1-c_j(x,y))\sum_{y'\neq n+j} \Phi_{\rm{sq}}\paren*{\Delta_h(x,n+j,y')}$  & $(\num+1)\paren*{\sE_{\lsc}(h) - \sE_{\lsc}^*( \sH)}^{\frac12}$\\
    $\Phi_{\mathrm{exp}}^{\mathrm{sum}}$   & $\sum_{y'\neq y} \Phi_{\rm{exp}}\paren*{\Delta_h(x,y,y')}+\sum_{j=1}^{\num}(1-c_j(x,y))\sum_{y'\neq n+j} \Phi_{\rm{exp}}\paren*{\Delta_h(x,n+j,y')}$ & $\sqrt{2}(\num+1)\paren*{\sE_{\lsc}(h) - \sE_{\lsc}^*( \sH)}^{\frac12}$   \\
    $\Phi_{\rho}^{\mathrm{sum}}$    &  $\sum_{y'\neq y} \Phi_{\rho}\paren*{\Delta_h(x,y,y')}+\sum_{j=1}^{\num}(1-c_j(x,y))\sum_{y'\neq n+j} \Phi_{\rho}\paren*{\Delta_h(x,n+j,y')}$  & $\sE_{\lsc}(h) - \sE_{\lsc}^*( \sH)$  \\
    \bottomrule
  \end{tabular}
  }
 \caption{Examples of the deferral surrogate loss \eqref{eq:sur-score} with sum losses adopted for $\ell$
  and their associated $\sH$-consistency bounds provided by
  Corollary~\ref{cor:bound-score} (with only the surrogate portion
  displayed), where $\Delta_h(x,y,y')=h(x, y) - h(x, y')$, and $\Phi_{\mathrm{sq}}(t)=\max\curl*{0, 1 - t}^2$, $\Phi_{\mathrm{exp}}(t)=e^{-t}$,
and $\Phi_{\rho}(t)=\min\curl*{\max\curl*{0,1 - t/\rho},1}$.}
\label{tab:sur-score-sum}
\end{table}

\begin{table}[t]
  \centering
  \resizebox{\textwidth}{!}{
  \begin{tabular}{@{\hspace{0cm}}lll@{\hspace{0cm}}}
    \toprule
    $\ell$  & $\sfL$   & $\sH$-consistency bounds\\
    \midrule
    $\Phi_{\mathrm{hinge}}^{\mathrm{cstnd}}$ & $\sum_{y'\neq y}\Phi_{\mathrm{hinge}}\paren*{-h(x, y')}+\sum_{j=1}^{\num}(1-c_j(x,y))\sum_{y'\neq n+j}\Phi_{\mathrm{hinge}}\paren*{-h(x, y')}$ & $\sE_{\lsc}(h) - \sE_{\lsc}^*( \sH)$\\
    $\Phi_{\mathrm{sq}}^{\mathrm{cstnd}}$   & $\sum_{y'\neq y}\Phi_{\mathrm{sq}}\paren*{-h(x, y')}+\sum_{j=1}^{\num}(1-c_j(x,y))\sum_{y'\neq n+j}\Phi_{\mathrm{sq}}\paren*{-h(x, y')}$ & $(\num+1)\paren*{\sE_{\lsc}(h) - \sE_{\lsc}^*( \sH)}^{\frac12}$   \\
    $\Phi_{\mathrm{exp}}^{\mathrm{cstnd}}$    & $\sum_{y'\neq y}\Phi_{\mathrm{exp}}\paren*{-h(x, y')}+\sum_{j=1}^{\num}(1-c_j(x,y))\sum_{y'\neq n+j}\Phi_{\mathrm{exp}}\paren*{-h(x, y')}$  & $\sqrt{2}(\num+1)\paren*{\sE_{\lsc}(h) - \sE_{\lsc}^*( \sH)}^{\frac12}$  \\
    $\Phi_{\rho}^{\mathrm{cstnd}}$ & $\sum_{y'\neq y}\Phi_{\rho}\paren*{-h(x, y')}+\sum_{j=1}^{\num}(1-c_j(x,y))\sum_{y'\neq n+j}\Phi_{\rho}\paren*{-h(x, y')}$ &    $\sE_{\lsc}(h) - \sE_{\lsc}^*( \sH)$    \\
    \bottomrule
  \end{tabular}
  }
 \caption{Examples of the deferral surrogate loss \eqref{eq:sur-score} with constrained losses adopted for $\ell$
  and their associated $\sH$-consistency bounds provided by
  Corollary~\ref{cor:bound-score} (with only the surrogate portion
  displayed), where $\Phi_{\mathrm{hinge}}(t) = \max\curl*{0,1 - t}$, $\Phi_{\mathrm{sq}}(t)=\max\curl*{0, 1 - t}^2$,
  $\Phi_{\mathrm{exp}}(t)=e^{-t}$,
and
$\Phi_{\rho}(t)=\min\curl*{\max\curl*{0,1 - t/\rho},1}$ with the constraint that $\sum_{y\in \sY}h(x,y)=0$.}
\label{tab:sur-score-cstnd}
\end{table}

\section{Benefits of minimizability gaps}
\label{sec:minimizability_gaps}

As already pointed out, the minimizabiliy gap can be upper-bounded by
the approximation error $\sA_{\ell}(\sH)= \sE^*_{\ell}(\sH) -
\E_x\bracket[\big]{\inf_{h \in \sH_{\rm{all}}} \E_{y | x}
  \bracket*{\ell(h, x,
    y)}}=\sE^*_{\ell}(\sH)-\sE^*_{\ell}(\sH_{\rm{all}})$.  It is
however a finer quantity than the approximation error and can thus
lead to more favorable guarantees.  More precisely, as shown by
\citep{awasthi2022Hconsistency,AwasthiMaoMohriZhong2022multi}, for a
target loss function $\ell_2$ and a surrogate loss function $\ell_1$,
the excess error bound can be rewritten as
\begin{equation*}
\sE_{\ell_2} (h) - \sE^*_{\ell_2}(\sH) +\sA_{\ell_2}(\sH)
\leq \Gamma\paren*{ \sE_{\ell_1} (h) - \sE^*_{\ell_1}(\sH)+\sA_{\ell_1}(\sH)},
\end{equation*}
where $\Gamma$ is typically linear or the square-root function modulo
constants.  On the other hand, an $\sH$-consistency bound can be
expressed as follows:
\begin{equation*}
\sE_{\ell_2} (h) - \sE^*_{\ell_2}(\sH) +  \sM_{\ell_2}(\sH)  \leq \Gamma\paren*{ \sE_{\ell_1} (h) - \sE^*_{\ell_1}(\sH) + \sM_{\ell_1}(\sH)}.
\end{equation*}
For a target loss function $\ell_2$ with discrete outputs, such as the
zero-one loss or the deferral loss, we have
$\E_{x}\bracket[\big]{\inf_{h \in\sH}\E_{y | x}\bracket*{\ell_2(h, x,
    y)}}=\E_x\bracket[\big]{\inf_{h \in \sH_{\rm{all}}} \E_{y | x}
  \bracket*{\ell_2(h,x, y)}}$ when the hypothesis set generates labels
that cover all possible outcomes for each input (See
\citep[Lemma~3]{AwasthiMaoMohriZhong2022multi},
Lemma~\ref{lemma:calibration_gap_score} in
Appendix~\ref{sec:app_def_cond}). Consequently, we have
$\sM_{\ell_2}(\sH) = \sA_{\ell_2}(\sH)$. For a surrogate loss function
$\ell_1$, the minimizability gap is upper-bounded by the approximation
error, $\sM_{\ell_1}(\sH)\leq \sA_{\ell_1}(\sH)$, and is generally
finer.

Consider a simple binary classification example with the conditional
distribution denoted as $\eta(x)=D(Y=1 | X=x)$. Let $\sH$ be a family
of functions $h$ such that $|h(x)| \leq \Lambda$ for all $x \in \sX$,
for some $\Lambda > 0$, and such that all values in the range
$[-\Lambda, +\Lambda]$ can be achieved. For the exponential-based
margin loss, defined as $\ell(h, x, y) = e^{-yh(x)}$, we have
\begin{equation*}
\E_{y | x}[\ell(h, x, y)] = \eta(x)
e^{-h(x)} + (1 - \eta(x)) e^{h(x)}.  
\end{equation*}
It can be observed that the infimum over all measurable functions can
be written as follows, for all $x$:
\begin{equation*}
\inf_{h
  \in \sH_{\mathrm{all}}}\E_{y | x}[\ell(h, x, y)] =
2\sqrt{\eta(x)(1-\eta(x))},
\end{equation*}
while the infimum over $\sH$, $\inf_{h \in \sH}\E_{y
  | x}[\ell(h, x, y)]$, depends on $\Lambda$. That infimum over $\sH$ is achieved by
  \begin{equation*}
   h(x)=\begin{cases}
   \min\curl*{\frac{1}{2} \log \frac{\eta(x)}{1
    -\eta(x)}, \Lambda} & \eta(x) \geq 1/2\\
     \max\curl*{\frac{1}{2}\log \frac{\eta(x)}{1 - \eta(x)}, -\Lambda} & \text{otherwise}.
   \end{cases} 
  \end{equation*}
Thus, in the deterministic case, we can explicitly compute the
difference between the approximation error and the minimizability gap:
\begin{align*}
\sA_{\ell}(\sH)-\sM_{\ell}(\sH) = \E_{x}\bracket[\big]{\inf_{h \in\sH}\E_{y |
    x}\bracket*{\ell(h, x, y)}-\inf_{h \in
    \sH_{\rm{all}}} \E_{y | x} \bracket*{\ell(h,x, y)}}
    = e^{-\Lambda}.   
\end{align*}
As the parameter $\Lambda$ decreases, the hypothesis set $\sH$ becomes
more restricted and the difference between the approximation error and
the minimizability gap increases. In summary, an $\sH$-consistency
bound can be more favorable than the excess error bound as
$\sM_{\ell_2}(\sH) = \sA_{\ell_2}(\sH)$ when $\ell_2$ represents the
zero-one loss or deferral loss, and $\sM_{\ell_1}(\sH) \leq
\sA_{\ell_1}(\sH)$. Moreover, we will show in the next section that
our $\sH$-consistency bounds can lead to learning bounds for the
deferral loss and a hypothesis set $\sH$ with finite samples.

\section{Learning bounds}
\label{sec:learning-bound}
For a sample $S=\paren*{(x_1,y_1),\ldots,(x_m,y_m)}$ drawn from
$\sD^m$, we will denote by $\h h_S$ the empirical minimizer of the
empirical loss within $\sH$ with respect to the surrogate loss
function $\sfL$:
$
\h h_S=\argmin_{h\in \sH}\frac{1}{m}\sum_{i=1}^m \sfL(h,x_i,y_i).
$
Given an $\sH$-consistency bound in the form of
\eqref{eq:H-consistency-bounds}, we can further use it to derive a
learning bound for the deferral loss by upper-bounding the surrogate
estimation error $\sE_{\lsc}(\h h_S) - \sE_{\lsc}^*(\sH)$ with the
complexity (e.g. the Rademacher complexity) of the family of functions
associated with $\sfL$ and $\sH$: $\sH_{\sfL}=\curl*{(x, y) \mapsto
  \sfL(h, x, y) \colon h \in \sH}$.

We denote by $\Rad_m^{\sfL}(\sH)$ the Rademacher complexity of
$\sH_{\sfL}$ and by $B_{\sfL}$ an upper bound of the surrogate loss
$\sfL$. Then, we obtain the following learning bound for the deferral
loss based on \eqref{eq:H-consistency-bounds}.

\begin{restatable}[\textbf{Learning bound}]{theorem}{GBoundScore}
\label{Thm:Gbound-score}
Under the same assumptions as Theorem~\ref{Thm:bound-score}, for any
$\delta > 0$, with probability at least $1-\delta$ over the draw of an
i.i.d sample $S$ of size $m$, the following deferral loss estimation
bound holds for $\h h_S$:
\begin{equation*}
\sE_{\ldefsc}(\h h_S) - \sE_{\ldefsc}^*( \sH) + \sM_{\lsc}( \sH) \leq \paren[\bigg]{\num
  + 1 - \sum_{j = 1}^{\num}\uv c_j} \Gamma\paren*{\frac{4
    \Rad_m^{\sfL}(\sH) + 2 B_{\sfL} \sqrt{\tfrac{\log
        \frac{2}{\delta}}{2m}} + \sM_{\lsc}( \sH)}{\num + 1 - \sum_{j
      = 1}^{\num}\ov c_j}}.
\end{equation*}
\end{restatable}
The proof is presented in Appendix~\ref{app:Gbound-score}. To the best
of our knowledge, Theorem~\ref{Thm:Gbound-score} provides the first
finite-sample guarantee for the estimation error of the minimizer of a
surrogate deferral loss $\lsc$ defined for multiple experts.  The
proof exploits our $\sH$-consistency bounds with respect to the
deferral loss, as well as standard Rademacher complexity guarantees.

When $\uv c_j=0$ and $\ov c_j=1$ for any $j\in [\num]$, the right-hand
side of the bound admits the following simpler form:
\begin{align*}
 \paren*{\num + 1} \, \Gamma\paren*{4 \Rad_m^{\sfL}(\sH) +
   2 B_{\sfL} \sqrt{\tfrac{\log \frac{2}{\delta}}{2m}}
   + \sM_{\lsc}( \sH)}.   
\end{align*}
The dependency on the number of experts $\num$ makes this bound less
favorable. There is a trade-off however since, on the other hand, more
experts can help us achieve a better accuracy overall and reduce the
best-in-class deferral loss.  These learning bounds take into account
the minimizability gap, which varies as a function of the upper bound
$\Lambda$ on the magnitude of the scoring functions. Thus, both the
minimizability gaps and the Rademacher complexity term suggest a
regularization controlling the complexity of the hypothesis set
and the magnitude of the scores.
 
Adopting different loss functions $\ell$ in the definition of our
deferral surrogate loss \eqref{eq:sur-score} will lead to a different
functional form $\Gamma$, which can make the bound more or less
favorable. For example, a linear form of $\Gamma$ is in general more
favorable than a square-root form modulo a constant. But, the
dependency on the number of classes $n$ appearing in $\Gamma$ (e.g.,
$\ell = \ell_{\rm{gce}}$ or $\ell = \ell_{\rm{mae}}$) is also
important to take into account since a larger value of $n$ tends to
negatively impact the guarantees. We already discussed the dependency
on the number of experts $\num$ in $\Gamma$ (e.g., $\ell =
\ell_{\rm{gce}}$ or $\ell = \ell_{\rm{exp}}$) and the associated
trade-off, which is also important to consider.

Note that the bound of Theorem~\ref{Thm:Gbound-score} is expressed in
terms of the global complexity of the prediction and deferral scoring
functions $\sH$. One can however derive a finer bound distinguishing
the complexity of the deferral scoring functions and that of the
prediction scoring functions following a similar proof and analysis.

Recall that for a surrogate loss $\sfL$, the minimizability gap
$\sM_{\sfL}(\sH)$ is in general finer than the approximation error
$\sA_{\sfL}(\sH)$, while for the deferral loss, for common hypothesis
sets, these two quantities coincide. Thus, our bound can be rewritten
as follows for common hypothesis sets:
\begin{equation*}
\sE_{\ldefsc}(\h h_S) - \sE_{\ldefsc}^*(\sH_{\rm{all}})
\leq \paren[\bigg]{\num + 1 - \sum_{j = 1}^{\num}\uv c_j} \Gamma\paren*{\frac{4 \Rad_m^{\sfL}(\sH) +
2 B_{\sfL} \sqrt{\tfrac{\log \frac{2}{\delta}}{2m}} + \sM_{\lsc}( \sH)}{\num + 1-\sum_{j = 1}^{\num}\ov c_j}}.
\end{equation*}
This is more favorable and more relevant than a similar excess loss
bound where $\sM_{\lsc}( \sH)$ is replaced with $\sA_{\lsc}(\sH)$,
which could be derived from a generalization bound for the surrogate
loss.

\ignore{
as follows
and is a more meaningful learning guarantee:
\begin{align*}
\sE_{\ldefsc}(\h h_S) - \sE_{\ldefsc}^*(\sH_{\rm{all}})
\leq \paren[\bigg]{\num + 1 - \sum_{j = 1}^{\num}\uv c_j} \Gamma\paren*{\frac{4 \Rad_m^{\sfL}(\sH) +
2 B_{\sfL} \sqrt{\tfrac{\log \frac{2}{\delta}}{2m}} + \sA_{\lsc}(\sH)}{\num + 1-\sum_{j = 1}^{\num}\ov c_j}}.    
\end{align*}
}

\section{Experiments}
\label{sec:experiments-tdef}

In this section, we examine the empirical performance of our proposed
surrogate loss in the scenario of learning with multiple-expert deferral. More specifically, we aim to compare the overall system
accuracy for the learned predictor and deferral pairs, considering
varying numbers of experts. This comparison provides valuable
insights into the performance of our algorithm under different expert
configurations. We explore three different scenarios:
\begin{itemize}
  
    \item Only a single expert is available, specifically where a
      larger model than the base model is chosen as the deferral
      option.

    \item Two experts are available, consisting of one small model and
      one large model as the deferral options.

    \item Three experts are available, including one small model, one
      medium model, and one large model as the deferral options.
\end{itemize}
By comparing these scenarios, we evaluate the impact of varying the
number and type of experts on the overall system accuracy. 

\textbf{Type of cost.}  We carried out experiments with two types of
cost functions. For the first type, we selected the cost function to
be exactly the misclassification error of the expert: $c_j(x, y) =
\1_{\expertexpert_{j}(x) \neq y}$, where $\expertexpert_{j}(x) =
\argmax_{y \in [n]}\expert_j(x, y)$ is the prediction made by expert
$\expert_j$ for input $x$. In this scenario, the cost incurred for
deferring is determined solely based on the expert's accuracy. For the
second type, we chose a cost function admitting the form $c_j(x, y) =
\1_{\expertexpert_{j}(x) \neq y} + \beta_j$, where an additional
non-zero base cost $\beta_j$ is assigned to each expert.  Deferring to
a larger model then tends to incur a higher inference cost and hence,
the corresponding $\beta_j$ value for a larger model is higher as
well. In addition to the base cost, each expert also incurs a
misclassification error, as with the first type. Experimental setup and additional experiments (see Table~\ref{tab:additional}) are included in Appendix~\ref{app:experiments}.

\textbf{Experimental Results.} In Table~\ref{tab:first-type} and
Table~\ref{tab:second-type}, we report the mean and standard deviation
of the system accuracy over three runs with different random seeds. We
noticed a positive correlation between the number of experts and the
overall system accuracy. Specifically, as the number of experts
increases, the performance of the system in terms of accuracy
improves. This observation suggests that incorporating multiple
experts in the learning to defer framework can lead to better
predictions and decision-making. The results also demonstrate the
effectiveness of our proposed surrogate loss for deferral with
multiple experts.

\begin{table}[t]
  \centering
  \begin{tabular}{@{\hspace{0cm}}llll@{\hspace{0cm}}}
    \toprule
    & Single expert   & Two experts & Three experts\\
    \midrule
    SVHN  & 92.08 $\pm$ 0.15\% & 93.18 $\pm$ 0.18\%  &  93.46 $\pm$ 0.12\% \\
    CIFAR-10 & 73.31 $\pm$ 0.21\% & 77.12 $\pm$ 0.34\% & 78.71 $\pm$ 0.43\%\\
    \bottomrule
  \end{tabular}
 \caption{Overall system accuracy with the first type of cost functions.}
 \label{tab:first-type}
\end{table}

\begin{table}[t]
  \centering
  \begin{tabular}{@{\hspace{0cm}}llll@{\hspace{0cm}}}
    \toprule
    & Single expert   & Two experts & Three experts\\
    \midrule
    SVHN  & 92.36 $\pm$ 0.22\% & 93.23 $\pm$ 0.21\%  &  93.36 $\pm$ 0.11\% \\
    CIFAR-10 & 73.70 $\pm$ 0.40\% & 76.29 $\pm$ 0.41\% & 76.43 $\pm$ 0.55\%\\
    \bottomrule
  \end{tabular}
  \caption{Overall system accuracy with the second type of cost functions.}
  \label{tab:second-type}
\end{table}

\section{Conclusion}

We presented a comprehensive study of surrogate losses for the core
challenge of learning with multi-expert deferral. Through our
study, we established theoretical guarantees, strongly endorsing the adoption of the loss function family
we introduced. This versatile family of loss functions can effectively
facilitate the learning to defer algorithms across a wide range of
applications. Our analysis offers great flexibility by accommodating
diverse cost functions, encouraging exploration and evaluation of
various options in real-world scenarios. We encourage further research
into the theoretical properties of different choices and their impact
on the overall performance to gain deeper insights into their
effectiveness.

\chapter{Two-Stage Multi-Expert Deferral} \label{ch5}
In this chapter, we study a two-stage scenario for
learning with multi-expert deferral that is crucial in practice
for many applications.  In this scenario, a predictor is derived in a
first stage by training with a common loss function such as
cross-entropy.  In the second stage, a deferral function is learned to
assign the most suitable expert to each input.  We design a new family
of surrogate loss functions for this scenario both in the
\emph{score-based setting} (Section~\ref{sec:two-stage}) and the
\emph{predictor-rejector} setting
(Section~\ref{sec:two-stage-predictor-rejector}) and prove that they
are supported by $\sH$-consistency bounds, which implies their
Bayes-consistency.  Moreover, we show that, for a constant cost
function, our two-stage surrogate losses are realizable
$\sH$-consistent. While the main focus of this chapter is a theoretical
analysis, we also report the results of several experiments on
CIFAR-10 and SVHN datasets (Section~\ref{sec:experiments}). We begin by providing some basic
definitions and notation (Section~\ref{sec:pre}).

The presentation in this chapter is based on \citep{MaoMohriMohriZhong2023two}.

\section{Preliminaries}
\label{sec:pre}

We consider the standard multi-class classification setting with an
input space $\sX$ and a set of $n \geq 2$ labels $\sY =[n]$, where we
use the notation $[n]$ to denote the set $\curl*{1, \ldots, n}$.
We study the scenario of \emph{learning with multiple-expert deferral}, where the label set $\sY$ is augmented with $\num$
additional labels $\curl*{n + 1, \ldots, n + \num}$ corresponding to
$\num$ pre-defined experts $h_1, \ldots, h_\num$. In this scenario,
the learner has the option of returning a label $y \in \sY$, which
represents the category predicted, or a label $y = n + j$, $j \geq 1$,
in which case it is \emph{deferring} to expert $h_j$. This setting is
referred to as the \emph{score-based setting}
\citep{mozannar2020consistent,caogeneralizing,MaoMohriZhong2024score},
since the deferral corresponds to extra $\num$ scoring functions. An
alternative setting is the \emph{predictor-rejector setting}
\citep{CortesDeSalvoMohri2016,CortesDeSalvoMohri2023,
  MohriAndorChoiCollinsMaoZhong2024learning,MaoMohriZhong2024predictor},
where the deferral function is selected from a separate family of
functions $\sR$. We introduce that setting and include the
corresponding results in
Section~\ref{sec:two-stage-predictor-rejector} for completeness.

We denote by $\ov \sY = [n + \num]$ the augmented label set and
consider a hypothesis set $ \sH$ of functions mapping from $\sX \times
\ov \sY$ to $\Rset$. The prediction associated by $h \in \sH$ to an
input $x \in \sX$ is denoted by $\hh(x)$ and defined as the element in
$\ov \sY$ with the highest score, $\hh(x) = \argmax_{y \in [n + \num]}
h(x, y)$, with an arbitrary but fixed deterministic strategy for
breaking ties. We denote by $\sH_{\rm{all}}$ the family of all
measurable functions.

The \emph{deferral loss function}
$\ldefsc$ is defined as follows for any $h \in \sH$ and $(x, y) \in
\sX \times \sY$:
\begin{equation}
\label{eq:def-score}
\begin{aligned}
\ldefsc(h, x, y)
& = \1_{\hh(x)\neq y} \1_{\hh(x)\in [n]}
+ \sum_{j = 1}^{\num} c_j(x, y) \1_{\hh(x) = n + j}
\end{aligned}
\end{equation}
Thus, the loss incurred coincides with the standard zero-one
classification loss when $\hh(x)$, the label predicted, is in $\sY$.
Otherwise, when $\hh(x)$ is equal to $n + j$, the loss incurred is
$c_j(x, y)$, the cost of deferring to expert $h_j$. Let $\bar c_j(x, y) = 1 - c_j(x, y)$. We will denote by $\uv c_j \geq 0$ and $\ov c_j \leq 1$ finite lower and upper bounds on the cost $ \bar c_j$, that is $\bar c_j(x, y)\in [\uv c_j,\ov c_j]$ for all $(x,
y) \in \sX \times \sY$.
There are many possible choices for these costs.  Our analysis for
Theorem~\ref{Thm:bound-general-two-stage-score},
Corollary~\ref{cor:bound-general-two-stage-score},
Theorem~\ref{Thm:bound-general-two-step-multi} is general and requires
no assumption other than their boundedness. One natural choice is to
define cost $c_j$ as a function of expert $h_j$'s inaccuracy, for
example $c_j(x, y) = \alpha_j \1_{\hh_{j}(x) \neq y} + \beta_j$, with
$\alpha_j, \beta_j > 0$, where $\hh_{j}(x)$ is the prediction made by
$h_j$th for input $x$. Typically, the hyperparameter $\alpha_j$ has
two potential values: zero or one. When $\alpha_j$ is set to one, the
first term of the formulation pertains to the inaccuracy of expert
expert $h_j$. Conversely, with $\alpha_j$ set to zero, the first term
vanishes, focusing solely on the inference
cost. Theorems~\ref{Thm:bound-general-two-stage-score-realizable} and
\ref{Thm:bound-general-two-stage-general-realizable} are analyzed
under this assumption. The $\beta_j$ in the second term corresponds to
the inference cost incurred by expert $h_j$.

Given a distribution $\sD$ over $\sX \times \sY$, we will denote by
$\sE_{\ldefsc}(h)$ the expected deferral loss of a hypothesis $h \in
\sH$, $\sE_{\ldefsc}(h) = \E_{(x, y) \sim \sD}[\ldefsc(h, x, y)]$, and
by $\sE^*_{\ldefsc}(\sH) = \inf_{h \in \sH} \sE_{\ldefsc}(h)$ its
infimum or best-in-class expected loss. We will adopt similar
definitions for other loss functions.

Given a hypothesis set $\sH$, an \emph{$\sH$-consistency bound}
\citep{awasthi2021calibration, awasthi2021finer,
  awasthi2022Hconsistency, AwasthiMaoMohriZhong2022multi,
  AwasthiMaoMohriZhong2023theoretically, awasthi2024dc, MaoMohriZhong2023cross,
  MaoMohriZhong2023ranking, MaoMohriZhong2023rankingabs,
  zheng2023revisiting, MaoMohriZhong2023characterization,
  MaoMohriZhong2023structured} for a surrogate loss $\ell_1$ of a
target loss function $\ell_2$ is an inequality of the form
\begin{align}
\label{eq:est-bound}
\forall h \in \sH, \
\sE_{\ell_2}(h) - \sE^*_{\ell_2}(\sH)+\sM_{\ell_1}(\sH)
\leq \Gamma\paren*{\sE_{\ell_1}(h) - \sE^*_{\ell_1}(\sH)+\sM_{\ell_1}(\sH}),
\end{align}
where $\Gamma\colon \Rset_+ \to \Rset_+$ is a non-decreasing function
with $\Gamma(0) = 0$ and where $\sM_{\ell}(\sH)$ is \emph{the
minimizability gap} for the hypothesis set $\sH$ and loss function
$\ell$. $\sM_{\ell}(\sH)$ is defined as the difference of the
best-in-class expected loss and that of the expected pointwise infimum
loss: $\sM_{\ell}(\sH) = \sE^*_{\ell}(\sH) - \mathbb{E}_{x} \bracket*
{\inf_{h\in \sH}\E_{y|x}\bracket*{\ell(h, x, y)}}$. By the
super-additivity of the infimum, the minimizability gap is always
non-negative. The minimizability gap vanishes when the best-in-class
error $\sE^*_{\ell}(\sH)$ coincides with the Bayes error
$\sE^*_{\ell}(\sH_{\rm{all}})$, in particular when $\sH =
\sH_{\rm{all}}$
\citep{AwasthiMaoMohriZhong2022multi,awasthi2022Hconsistency}.

Thus, the $\sH$-consistency bound \eqref{eq:est-bound} relates the
minimization of the estimation error for the surrogate loss $\ell_1$
to that of the target loss $\ell_2$ in a quantitative way. It is a
stronger and more informative guarantee than Bayes-consistency which
implies Bayes-consistency, as can be seen by setting $\sH =
\sH_{\rm{all}}$.

\section{Two-stage \texorpdfstring{$\sH$}{H}-consistent surrogate loss}
\label{sec:two-stage}

In this section, we consider an important \emph{two-stage} scenario
for learning with multi-expert deferral. This is a critical
scenario in practice for many applications where a predictor is
already available, as a result of training with a loss function $\ell$
supported by $\sH$-consistency bounds, such as the logistic loss
(first stage). The logistic loss coincides with the cross-entropy loss
when a softmax activation is applied to the output of a neural
network.  The problem then consists of learning a deferral function
(second stage) to assign the most suitable expert to each input
instance.

We first design a new family of surrogate losses for this
\emph{two-stage} scenario
(Section~\ref{sec:two-stage-expression}). Next, we show that our
surrogate losses benefit from $\sH$-consistency bounds
(Section~\ref{sec:two-stage-bounds}). As a by-product, we prove
$\ov\sH$-consistency bounds in standard multi-class classification,
where $\ov\sH$ denotes hypothesis sets with a fixed scoring function
(Section~\ref{sec:fixed-score-bound}). These bounds have not been
studied before and can be of independent interest in other consistency
studies. Moreover, we show that, for a constant cost function, our
two-stage surrogate losses are realizable $\sH$-consistent
(Section~\ref{sec:realizable}).

\begin{table}[t]
\caption{Common surrogate losses in standard multi-class classification.}
 \label{tab:sur}
 \centering
 \begin{tabular}{@{\hspace{0cm}}ll@{\hspace{0cm}}}
  \toprule
  Name & Formulation\\
  \midrule
  Sum exponential loss  & $\ell_{\rm{exp}}(\ov h, x, y)
  = \sum_{y'\neq y}e^{\ov h(x, y') - \ov h(x, y)}$.  \\
  Multinomial logistic loss & $\ell_{\rm{log}}(\ov h, x, y)
  = \log\paren*{\sum_{y'\in \sY \cup \curl{0}}e^{\ov h(x, y') - \ov h(x, y)}}$. \\
  Generalized cross-entropy loss  & $\ell_{\rm{gce}}(\ov h, x, y)
  =\frac{1}{\alpha}\bracket*{1 - \bracket*{\frac{e^{\ov h(x, y)}}
  {\sum_{y'\in \sY \cup \curl{0}} e^{\ov h(x, y')}}}^{\alpha}},\alpha\in (0,1)$. 
   \\
  Mean absolute error loss & $\ell_{\rm{mae}}(\ov h, x, y)
 = 1 - \frac{e^{\ov h(x, y)}}{\sum_{y'\in \sY \cup \curl{0}} e^{\ov h(x, y')}}$. \\
  \bottomrule
 \end{tabular}
\end{table}

\subsection{General surrogate losses}
\label{sec:two-stage-expression}

A hypothesis set $\sH$ of functions mapping from $\sX \times [n +
  \num]$ to $\Rset$ can be decomposed as $\sH = \sH_p \times \sH_d$,
where $\sH_p$ denotes the hypothesis set spanned by the first $n$
scores, used for prediction, and $\sH_d$ the hypothesis set spanned by
the final $\num$ scores, used for deferral. Thus, any $h \in \sH$ can
be written as a pair $h = (\hp, \hd)$ with $\hp \in \sH_p$ and $\hd
\in \sH_d$.

Let $\ell_1$ be a surrogate loss for standard multi-class classification
with $n$ classes. We consider the following two-stage scenario: in the
first stage, $\hp$ is learned using the surrogate loss $\ell_1$; in the
second stage, $\hd$ is learned using a surrogate loss $\lsc_{\hp}$
that depends on the prediction function $\hp$ learned in the first
stage.

To any $\hd \in \sH_d$, we associate a hypothesis $\ov h_d$ defined
over $(\num + 1)$ classes $\curl*{0, 1, \ldots, \num}$ by $\ov h_d(x, 0)
= \max_{y \in \sY} \hp(x, y)$, that is the maximal score assigned by $\hp$ to
its predicted label, and $\ov h_d(x, j) = h_d(x, j)$ for $j \in
[\num]$. We can then define our suggested surrogate loss for the
second stage as follows:
\begin{equation}
\label{eq:ell-Phi-h-score}
\begin{aligned}
\lsc_{\hp} \paren*{\hd, x, y}
= \1_{\hhp (x) = y} \, \ell_2(\ov h_d, x, 0)
+ \sum_{j = 1}^{\num} \bar c_j(x, y) \ell_2(\ov h_d, x, j),
\end{aligned}
\end{equation}
where $\ell_2(\ov h_d, x, j)$ is a surrogate loss for standard
multi-class classification with $(\num + 1)$ categories $\curl*{0, 1,
  \ldots, \num}$.  Intuitively, the indicator term $\1_{h(x) \neq n + j}$ in
the deferral loss \eqref{eq:def-score} penalizes $h_d(x,j)$ when it
has a small value. Similarly, for a standard surrogate loss $\ell_2(\ov h_d,
x, j)$ such as the logistic loss, it penalizes $\ov h_d(x,j)$ when it has a small
value as well.
In Table~\ref{tab:sur-score-two-stage}, we present a summary of
examples of such second-stage surrogate losses, where $\ell_2$ is
selected from common surrogate losses in standard multi-class
classification defined in Table~\ref{tab:sur}. A detailed derivation
is presented in Appendix~\ref{app:sur-score-example-two-stage}.

From the point of view of the second stage, $x \mapsto \ov h_d(x, 0) =
\max_{y \in \sY} h_p(x, y)$ is a fixed function. We will denote by
$\ov \sH_d$ the family of hypotheses $\ov h_d \colon \sX \times
\curl*{0, 1, \ldots, \num} \to \Rset$ whose first scoring function,
$\ov h_d(\cdot, 0)$, is fixed and not to be learned in the second
stage.

Our formulation bears some similarity with the design of a surrogate
loss function for rejectors in
\citep{CortesDeSalvoMohri2016,CortesDeSalvoMohri2023} for learning
with rejection in binary classification, where the cost is a
constant. However, our surrogate loss is tailored to accommodate a
general cost function depending on both $x$ and $y$ for deferral, in
contrast with a constant one, and it allows for multiple deferral
options, as opposed to only one rejection option.

\subsection{\texorpdfstring{$\sH$}{H}-consistency bounds for two-stage surrogate losses}
\label{sec:two-stage-bounds}

In this section, we provide strong guarantees for two-stage surrogate
losses, provided that the first-stage loss function $\ell_1$ admits an
$\sH_p$-consistency bound, and the second-stage surrogate $\ell_2$
admits an $\ov \sH_d$-consistency bound.

\begin{restatable}[\textbf{$ \sH$-consistency bounds for score-based
   two-stage surrogates}]{theorem}{BoundGenralTwoStepScore}
\label{Thm:bound-general-two-stage-score}
Assume that $\ell_1$ admits an $\sH_p$-consistency bound and
$\ell_2$ admits an $\ov \sH_d$-consistency bound with
respect to the multi-class zero-one classification loss $ \ell_{0-1}$
respectively. Thus, there are non-decreasing concave functions
$\Gamma_1$ and $\Gamma_2$ such that, for all $\hp\in \sH_p$ and
$\ov h_d\in \ov \sH_d$,
we have
\begin{align*}
\sE_{\ell_{0-1}}(\hp) - \sE_{\ell_{0-1}}^*( \sH_p) + \sM_{\ell_{0-1}}( \sH_p)
& \leq \Gamma_1\paren*{\sE_{\ell_1}(\hp) - \sE_{\ell_1}^*(\sH_p)
  + \sM_{\ell_1}(\sH_p)}\\
\sE_{\ell_{0-1}}(\ov h_d) - \sE_{\ell_{0-1}}^*(\ov \sH_d) + \sM_{\ell_{0-1}}(\ov \sH_d)
& \leq \Gamma_2\paren*{\sE_{\ell_2}(\ov h_d) - \sE_{\ell_2}^*(\ov \sH_d)
  + \sM_{\ell_2}(\ov \sH_d)}.
\end{align*}
Then, the following holds for all $h\in \sH$:
\begin{align*}
& \sE_{\ldefsc}(h) - \sE_{\ldefsc}^*( \sH) + \sM_{\ldefsc}( \sH)\\
  &  \leq \Gamma_1\paren*{\sE_{\ell_1}(\hp) - \sE_{\ell_1}^*(\sH_p)
    + \sM_{\ell_1}(\sH_p)}
  + \paren[\bigg]{1 + \sum_{j = 1}^{\num}\ov c_j}\Gamma_2\paren[\bigg]{ \frac{\sE_{\lsc_{\hp}}(h_d)
      - \sE_{\lsc_{\hp}}^*(\sH_d)
      + \sM_{\lsc_{\hp}}(\sH_d)}{\sum_{j = 1}^{\num}\uv c_j} }.
\end{align*}
Furthermore, constant factors $\paren*{1+ \sum_{j = 1}^{\num}\ov c_j}$
and $\frac{1}{\sum_{j = 1}^{\num}\uv c_j}$ can be removed when
$\Gamma_2$ is linear.
\end{restatable}
The proof is given in
Appendix~\ref{app:bound-general-two-stage-score}. It consists of
expressing the conditional regret of the deferral loss as the sum of
two regrets, first by minimizing $h_d$ for a fixed $\hp$ and then
by minimizing $\hp$. Subsequently, we show how each regret
can be upper-bounded in terms of the conditional regret of each
stage's surrogate loss, leveraging the $\sH_{p}$-consistency bound
of $\ell_1$ and $\ov \sH_{d}$-consistency bound of $\ell_2$ with
respect to the zero-one loss. This, in conjunction with the concavity
of functions $\Gamma_1$ and $\Gamma_2$, establishes our
$\sH$-consistency bounds.

Thus, the theorem provides a strong guarantee for the two-stage
surrogate losses. A specific instance of
Theorem~\ref{Thm:bound-general-two-stage-score} holds for the case
where $\sE^*_{\ell_1}(\sH_p) = \sE^*_{\ell_1}(\sH_{\rm{all}})$ and
$\sE^*_{\sfL_{\hp}}(\sH_d) = \sE^*_{\sfL_{\hp}}(\sH_{\rm{all}})$,
ensuring that the Bayes-error coincides with the best-in-class error
and, consequently, $\sM_{\ell_1}(\sH_p) = \sM_{\lsc_{\hp}}(\sH_d) =
0$. Given Theorem~\ref{Thm:bound-general-two-stage-score} and the
non-negativity property of $\sM_{\ldefsc}(\sH)$, we can derive the
following corollary.
\begin{corollary}
\label{cor:bound-general-two-stage-score}
Assume that $\ell$ satisfies the same assumption as in
Theorem~\ref{Thm:bound-general-two-stage-score}. Then, for all $h\in
\sH$ and any distribution such that
$\sE^*_{\ell_1}(\sH_p) = \sE^*_{\ell_1}(\sH_{\rm{all}})$ and
$\sE^*_{\sfL_{\hp}}(\sH_d) = \sE^*_{\sfL_{\hp}}(\sH_{\rm{all}})$,
we have
\begin{equation*}
 \sE_{\ldefsc}(h) - \sE_{\ldefsc}^*( \sH)
 \leq \Gamma_1\paren*{\sE_{\ell_1}(\hp) - \sE_{\ell_1}^*(\sH_p)}
 + \paren*{1+ \sum_{j = 1}^{\num} \ov c_j}
 \Gamma_2\paren*{\frac{\sE_{\lsc_{\hp}}(h_d)
     - \sE_{\lsc_{\hp}}^*(\sH_d)}{\sum_{j = 1}^{\num}\uv c_j}},
\end{equation*}
where the constant factors $\paren*{1 + \sum_{j = 1}^{\num}\ov c_j}$
and $\frac{1}{\sum_{j = 1}^{\num}\uv c_j}$ can be removed when
$\Gamma_2$ is linear.
\end{corollary}
Corollary~\ref{cor:bound-general-two-stage-score} implies that when
the estimation error of the first-stage surrogate loss,
$\sE_{\ell_1}(\hp) - \sE_{\ell_1}^*(\sH_p)$, is reduced to
$\e_1$, and the estimation error of the second-stage surrogate loss,
$\sE_{\sfL_{\hp}}(h_d) - \sE_{\sfL_{\hp}}^*(\sH_d)$,
is reduced to $\e_2$, the estimation error of the deferral loss,
$\sE_{\ldefsc}(h) - \sE_{\ldefsc}^*(\sH)$, is upper-bound by
\begin{equation*}
  \Gamma_1\paren*{\e_1}
  + \paren*{1 + \sum_{j = 1}^{\num}\ov c_j}
  \Gamma_2\paren*{\frac{\e_2}{\sum_{j = 1}^{\num}\uv c_j}}.
\end{equation*}
The common surrogate losses mentioned earlier all satisfy the
first-stage requirement; however, it was unclear if they would meet
the second-stage criterion since the $\ov \sH_d$-consistency bound is
for hypothesis sets $\ov\sH_d$ with a fixed first scoring function.
This has not been previously studied in the literature. In the next
section, we prove for the first time that common multi-class surrogate
losses, such as the logistic loss, satisfy this requirement and can be
incorporated into both the first and second stage. Hence, based on \citep[Theorem~1]{MaoMohriZhong2023cross} and
Theorem~\ref{Thm:bound_combine} in
Section~\ref{sec:fixed-score-bound}, when using logistic loss in both
stages, the concave functions are $\Gamma_1(t) =\Gamma_2(t)
=\sqrt{2t}$, and thus
Corollary~\ref{cor:bound-general-two-stage-score} yields the following
$\sH$-consistency bound:
\begin{equation*}
 \sE_{\ldefsc}(h) - \sE_{\ldefsc}^*( \sH)
 \leq \sqrt{2} \, \bracket*{\sE_{\ell_1}(\hp) - \sE_{\ell_1}^*(\sH_p)}^{\frac12} + \sqrt{2} \, \bracket*{1+ \sum_{j = 1}^{\num}\ov c_j}\bracket*{\frac{\sE_{\lsc_{\hp}}(h_d) - \sE_{\lsc_{\hp}}^*(\sH_d)}{\sum_{j = 1}^{\num}\uv c_j}}^{\frac12}.  
\end{equation*}
In particular, the bound implies the Bayes-consistency of the
two-stage surrogate loss when $\ell_1 = \ell_2 =
\ell_{\log}$. Similarly, for other choices of $\ell_1$ and $\ell_2$
defined in Table~\ref{tab:sur}, the two-stage surrogate loss benefits
from an $\sH$-consistency bound and is also Bayes-consistent.

\begin{table}[t]
\caption{Examples for score-based second-stage surrogate losses
 \eqref{eq:ell-Phi-h-score}.}
 \label{tab:sur-score-two-stage}
 \centering
 \resizebox{\textwidth}{!}{
 \begin{tabular}{@{\hspace{0cm}}ll@{\hspace{0cm}}}
  \toprule
  $\ell_2$ & $\lsc_{\hp}$ \\
  \midrule
  $\ell_{\rm{exp}}$ & $\1_{\hhp(x) = y} \sum_{i=1}^{\num}e^{h(x,n+i) - \max_{y\in \sY}h(x, y)} + \sum_{j = 1}^{\num}\ov c_j(x, y)\bracket*{\sum_{i=1,i\neq j }^{\num}e^{h(x,n+i) -h(x,n+j)}+e^{\max_{y\in \sY}h(x, y) - h(x,n+j)}}$ \\
  $\ell_{\rm{log}}$  & $- \1_{\hhp(x) = y}\log\paren*{\frac{e^{\max_{y\in \sY} h(x, y)}}{e^{\max_{y\in \sY} h(x, y)}+ \sum_{i=1}^{\num}e^{h(x,n+i)}}} - \sum_{j = 1}^{\num}\ov c_j(x, y)\log\paren*{\frac{e^{h(x,n+j)}}{e^{\max_{y\in \sY} h(x, y)} + \sum_{i=1}^{\num}e^{h(x,n+i)}}}$  \\
  $\ell_{\rm{gce}}$  & $\1_{\hhp(x) = y}\frac{1}{\alpha}\bracket*{1- \bracket*{\frac{e^{\max_{y\in \sY} h(x, y)}}{e^{\max_{y\in \sY} h(x, y)}+ \sum_{i=1}^{\num}e^{h(x,n+i)}}}^{\alpha}} + \sum_{j = 1}^{\num}\ov c_j(x, y)\frac{1}{\alpha}\bracket*{1- \bracket*{\frac{e^{h(x,n+j)}}{e^{\max_{y\in \sY} h(x, y)} + \sum_{i=1}^{\num}e^{h(x,n+i)}}}^{\alpha}}$ 
   \\
  $\ell_{\rm{mae}}$ & $\1_{\hhp(x) = y}\bracket*{1- \frac{e^{\max_{y\in \sY} h(x, y)}}{e^{\max_{y\in \sY} h(x, y)}+ \sum_{i=1}^{\num}e^{h(x,n+i)}}} + \sum_{j = 1}^{\num}\ov c_j(x, y)\bracket*{1- \frac{e^{h(x,n+j)}}{e^{\max_{y\in \sY} h(x, y)} + \sum_{i=1}^{\num}e^{h(x,n+i)}}}$ \\
  \bottomrule
 \end{tabular}
 }
\end{table}

\subsection{\texorpdfstring{$\ov \sH$}{H}-consistency bounds for
  standard surrogate loss functions}
\label{sec:fixed-score-bound}

In this section, we seek to derive $\ov \sH$-consistency bounds for
common surrogate losses defined in Table~\ref{tab:sur} in the standard
multi-class classification scenario.  Recall that the first scoring
function of hypotheses in $\ov \sH_d$ is the function $\max_{y \in
  \sY} h_p(\cdot, y)$. Here, for any given function $\lambda$ mapping
from $\sX$ to $\Rset$, we define the hypothesis set $\ov \sH$
augmented by $\lambda$ in a similar way, that is to any $h \in \sH$ we
associate a hypothesis $\ov h \in \ov \sH$ defined by $\ov h(x, 0) =
\lambda(x)$ and $\ov h(x, j) = h(x, j)$ for $j \geq 1$.
These $\ov \sH$-consistency bounds offer strong guarantees when the
loss functions in Table~\ref{tab:sur} are used in the second stage of
the two-stage learning to defer surrogate losses
\eqref{eq:ell-Phi-h-score} instantiated in
Table~\ref{tab:sur-score-two-stage}. We believe that these results are
of independent interest and can admit other applications in the study
of $\sH$-consistency bounds. As with
\citep{MaoMohriZhong2023cross}, we assume that the hypothesis set $\sH$ is
\emph{symmetric and complete}. A hypothesis set is said to be
\emph{symmetric} if there exists a family $\sF$ of functions $f$
mapping from $\sX$ to $\Rset$ such that
$\curl*{\bracket*{h(x,1),\ldots,h(x,n)}\colon h\in \sH} =
\curl*{\bracket*{f_1(x),\ldots, f_n(x)}\colon f_1, \ldots, f_n\in
  \sF}$, for any $x \in \sX$. A hypothesis set $\sH$ is said to be
\emph{complete} if the set of scores it generates spans $\Rset$, that
is, $\curl*{h(x, y)\colon h\in \sH} = \Rset$, for any $(x, y)\in \sX
\times \sY$.
 
Note that for a symmetric and complete $\sH$, the associated $\ov \sH$
is not symmetric and complete. Therefore, the proof of
\citet{MaoMohriZhong2023cross} cannot be generalized to our setting. Our proofs
are presented in Appendix~\ref{app:bound_comp-sum}. We give a new
method for upper-bounding the conditional regret of the zero-one loss
by that of a surrogate loss. To achieve this, we upper-bound the
minimal conditional surrogate loss by the conditional loss of a
carefully constructed hypothesis in $\ov \sH$ denoted by $\ov
h_{\mu}$. The resulting softmax $\sS_{\mu}$ of this hypothesis only
differs from the original softmax $\sS$ corresponding to $\ov h$ on
exactly two of the labels.

\begin{restatable}[\textbf{$\ov \sH$-consistency bounds}]
 {theorem}{BoundCombine}
\label{Thm:bound_combine}
Assume that $\sH$ is symmetric and complete. Then, for any function
$\lambda$ mapping from $\sX$ to $\Rset$, hypothesis $\ov h$ in the
associated hypothesis set $\ov \sH$ and any distribution, the
following inequality holds:
\begin{equation*}
\sE_{\ell_{0-1}}\paren*{\ov h}- \sE_{\ell_{0-1}}^*\paren*{\ov \sH}
\leq \Gamma
\paren*{\sE_{\ell}\paren*{\ov h}- \sE_{\ell}^*\paren*{\ov \sH}+ \sM_{\ell}\paren*{\ov \sH}}
- \sM_{\ell_{0-1}}\paren*{\ov \sH},
\end{equation*}
where $\Gamma(t) =\sqrt{2t}$ for $\ell=\ell_{\rm{log}}\text{ or
}\ell_{\rm{exp}}$; $\Gamma(t) =\sqrt{2(n+1)^{\alpha}\,t}$ for
$\ell=\ell_{\rm{gce}}$; and $\Gamma(t) =(n+1)\,t$ for
$\ell=\ell_{\rm{mae}}$.
\end{restatable}
Let us underscore that our proof technique is novel and distinct from
the approach used in \citep{MaoMohriZhong2023cross}. Their method is tailored
for hypothesis sets where each score can span across
$\mathbb{R}$. This is not applicable in our context where the
hypothesis set adheres to a predefined scoring function. In their
proof, to set an upper bound on the estimation error of the zero-one
loss using that of the surrogate loss, they select an auxiliary
function $\overline{h}_{\mu}$ for any hypothesis $h$. This function is
contingent on the distinct scores of $h$. Subsequently, the authors
choose an optimal $\mu$ to set these bounds. Nevertheless, if any of
$h$'s scores are fixed, an optimal $\mu$ does not exist, preventing
the establishment of a meaningful bound.
Instead, our new proof method overcomes this limitation by choosing
$\ov h_{\mu}$ based on the softmax, as the softmax corresponding to
the label zero can still vary due to the influence of changes in other
scores, even when the scoring function on label zero is fixed.

\subsection{Realizable \texorpdfstring{$\sH$}{H}-consistency}
\label{sec:realizable}

Recently, \citet{pmlr-v206-mozannar23a} showed that even in the
straightforward single-expert setting, existing Bayes-consistent
single-stage surrogates
\citep{mozannar2020consistent,verma2022calibrated} are not
\emph{realizable $\sH$-consistent}
\citep{long2013consistency,zhang2020bayes} for learning with
deferral. This can pose significant challenges when learning with a
restricted hypothesis set $\sH$, even for simple linear
models. Instead, they proposed a new surrogate loss that is realizable
$\sH$-consistent when $\sH$ is \emph{closed under scaling}, meaning
that it satisfies the condition $h\in \sH \Rightarrow \tau h\in \sH$
for all $\tau$ in the set of real numbers. However, they stated that
they could not prove or disprove whether their proposed surrogate loss
is Bayes-consistent. Consequently, it has become crucial to identify a
surrogate loss that is both consistent and realizable-consistent,
which has remained an open problem.

\begin{definition}[\textbf{Realizable $\sH$-consistency}]
\label{def:rel-consistency}
A surrogate loss $\sfL$ is considered a \emph{realizable
$\sH$-consistent} loss function for the deferral loss $\ldef$ if, for
any distribution that is \emph{$\sH$-realizable}, that is, there
exists a zero loss solution $h^*\in \sH$ with $\sE_{\ldef}(h^*) =0$,
optimizing the surrogate loss results in obtaining the zero-error
solution:
\begin{equation*}
\sE_{\sfL}(h_n) - \sE^*_{\sfL}(\sH) \xrightarrow{n \rightarrow +\infty} 0 \implies \sE_{\ldef}(h_n) - \sE^*_{\ldef}(\sH) \xrightarrow{n \rightarrow +\infty} 0.
\end{equation*}
\end{definition}
In the following result, we show our two-stage surrogate losses
are realizable $\sH$-consistent. Combined with their Bayes-consistency
properties, which have already been established in
Section~\ref{sec:two-stage-bounds}, we effectively find surrogate
losses that are both Bayes-consistent and realizable consistent in the
multi-expert setting, including the single-expert setting as a special
case. For simplicity, here, we study the case where $\ell_1 = \ell_2 =
\ell_{\rm{\log}}$, a similar proof holds for other choices of $\ell_1$
and $\ell_2$ defined in Table~\ref{tab:sur}. The proof is included in
Appendix~\ref{app:realizable-score}.

\begin{restatable}[\textbf{Realizable $\sH$-consistency for score-based
      two-stage surrogates}]{theorem}{BoundGenralTwoStepScoreRealizable}
\label{Thm:bound-general-two-stage-score-realizable}
Assume that $\sH$ is closed under scaling and $c_j(x, y) = \beta_j,
\forall (x, y)\in \sX \times \sY$. Let $\ell_1$ and $\ell_2$ be the
logistic loss. Let $\hat{h}_p$ be the minimizer of $\sE_{\ell_1}$ and
$\hat{h}_d$ be the minimizer of $\sE_{\sfL_{\hat{h}_p}}$ such that
$\sE_{\sfL_{\hat{h}_p}}(\hat{h}_d)
=\min_{h}\sE_{\sfL_{h_p}}(h_d)$. Then, the following equality holds
for any ($\sH$, $\sR$) -realizable distribution,
\begin{align*}
 \sE_{\ldef}(\hat h) = 0, \text{ where } \hat h = (\hat{h}_p, \hat{h}_d).
\end{align*}
\end{restatable}
Theorem~\ref{Thm:bound-general-two-stage-score-realizable} suggests
that when the estimation error of the first-stage surrogate loss,
$\sE_{\ell_1}(\hp^n) - \sE_{\ell_1}^*(\sH_p) \xrightarrow{n
  \rightarrow +\infty} 0$, and that of the
second-stage surrogate loss, $\sE_{\sfL_{\hp}}(\hd^n) -
\sE_{\sfL_{\hp}}^*(\sH_d) \xrightarrow{n \rightarrow +\infty} 0$, the
estimation error of the deferral loss, $\sE_{\ldefsc}(h^n) -
\sE_{\ldefsc}^*(\sH) \xrightarrow{n \rightarrow +\infty} 0$. This
result demonstrates that our two-stage surrogate losses are not only
Bayes-consistent, but also realizable $\sH$-consistent when only the
inference cost ($\beta_j$) exists.

\section{Predictor-rejector setting}
\label{sec:two-stage-predictor-rejector}

The results of the previous sections were all given for the
score-based setting. We note that another popular setting in learning
with deferral/abstention is the \emph{predictor-rejector setting}
\citep{CortesDeSalvoMohri2016,CortesDeSalvoMohri2023}, where the
deferral corresponds to a separate function $\sR$ instead of extra
scores. For completeness, we introduce this setting as well. Here too,
we design a family of two-stage surrogate losses benefiting from both
$(\sH,\sR)$-consistency bounds and realizable consistency. For
simplicity, we overload the notation as with score-based setting based
on the context.

Let $\sH$ be a
hypothesis set of prediction functions mapping from $\sX \times \sY$
to $\Rset$. The label predicted for $x \in \sX$ using a hypothesis $h
\in \sH$ is denoted by $\hh(x)$ and defined as one with the highest
score, $\hh(x) = \argmax_{y\in \sY}h(x, y)$, with an arbitrary but
fixed deterministic strategy for breaking ties. Let $\sR$ be a family
of \emph{deferring} functions mapping from $\sX$ to $\Rset^{\num}$,
where $\num$ is the number of experts. A deferral $r =
\paren*{r_1,\ldots,r_{\num}} \in \sR$ is used to defer the prediction
on input $x$ to the $j$th expert $h_{j}$ if $r_j(x)\leq 0$ and
$r_j(x)< \min_{i=1,i\neq j}^{\num} r_i(x)$, in which case a cost
$c_j(x, y) = 1 - \bar c_j(x, y) \in [1 - \ov c_j, 1 - \uv c_j]$ is incurred with $0<\uv c_j\leq \ov
c_j\leq 1$. A natural choice of the cost is $ c_j(x, y) =
\alpha_j\1_{\hh_{j}(x)\neq y}+ \beta_j$, where $\alpha_j,\beta_j>0$
and $\hh_{j}$ is the prediction of the $j$th expert. The $\beta_j$ in
the second term corresponds to the inference cost incurred by expert
$h_j$. Let $r_0=0$ and define $\rr(x) = 0$ if $r_0(x) < \min_{j\in
  [\num]} r_j(x)$; otherwise, $\rr(x) = \argmin_{j \in [\num]}
r_j(x)$, with an arbitrary but fixed deterministic strategy for
breaking ties. The \emph{learning to defer loss} $\ldef$ with $\num$
experts is defined as follows for any $(h, r) \in \sH \times \sR$ and
$(x, y) \in \sX \times \sY$:
\begin{equation}
\label{eq:def-multiple}
\begin{aligned}
\ldef(h, r, x, y)
& = \1_{\hh(x) \neq y} \1_{\rr(x) =0} + \sum_{j = 1}^{\num} c_j(x, y) \1_{\rr(x) =j}.
\end{aligned}
\end{equation}
Given a distribution $\sD$ over $\sX \times \sY$, we will denote by
$\sE_{\ldefsc}(h,r)$ the expected deferral loss of a predictor $h \in
\sH$ and a deferral $r\in \sR$, $\sE_{\ldefsc}(h,r) = \E_{(x, y) \sim
  \sD}[\ldefsc(h, r,x, y)]$, and by $\sE^*_{\ldefsc}(\sH,\sR) =
\inf_{h \in \sH,r\in \sR} \sE_{\ldefsc}(h,r)$ its infimum or best-in
class expected loss. We will adopt similar definitions for other loss
functions. We denote by $\sM_{\sfL}(\sH,\sR) = \sE^*_{\sfL}(\sH,\sR) -
\mathbb{E}_{x} \bracket* {\inf_{h\in
    \sH,r\in\sR}\E_{y|x}\bracket*{\sfL(h,r,x, y)}}$ the minimizability
gap for hypothesis sets ($\sH$,$\sR$) and a loss function $\sfL$.

Let $\ell_1$ be a surrogate loss for standard multi-class
classification with $n$ classes. We consider the following two-stage
scenario: in the first stage, a predictor $h$ is learned using the
surrogate loss $\ell_1$; in the second stage, $r$ is learned using a
surrogate loss $\sfL_{h}$ that depends on the prediction function $h$
learned in the first stage.

To any $r\in \sR$, we associate a hypothesis $\ov r$ defined over
$(\num + 1)$ classes $\curl*{0, 1, \ldots, \num}$ by $\ov r(x, 0) =
0$, that is zero scoring function, and $\ov r(x, j) = -r_j(x)$ for $j
\in [\num]$. We can then define our suggested surrogate loss for the
second stage:
\begin{equation}
\label{eq:ell-Phi-h-multi}
\begin{aligned}
\sfL_{h} \paren*{r, x, y}
& = \1_{\hh(x) = y} \ell_2(\ov r, x,0)
+ \sum_{j = 1}^{\num}\ov c_j(x, y)\ell_2(\ov r, x, j).
\end{aligned}
\end{equation}
Here, $\ell_2(\ov r, x, j)$ is a surrogate loss for standard
multi-class classification with $(\num + 1)$ categories $\curl*{0, 1,
  \ldots, \num}$. Intuitively, the indicator term $\1_{r(x) \neq j}$
in the deferral loss penalizes $r_j(x)$ when it has a large value. However, a standard surrogate loss $\ell_2(\ov r, x, j)$ such as the logistic loss penalizes $\ov r(x,j)$ when it has
a small value. This is why we use a negative sign in the definition of
$\ov r$ to maintain consistency between the definitions of
$\sfL_h$ and $\ldef$. In
Table~\ref{tab:sur-general-two-stage}, we present a summary of
examples of such second-stage surrogate losses, where $\ell_2$ is
selected from common surrogate losses in standard multi-class
classification defined in Table~\ref{tab:sur}. A detailed derivation
is presented in Appendix~\ref{app:sur-general-example-two-stage}.
 
From the point of view of the second stage, we will denote by $\ov
\sR$ the family of hypotheses $\ov r \colon \sX \times \curl*{0, 1,
  \ldots, \num} \to \Rset$ whose first scoring function, $\ov r(\cdot,
0)$, is zero function and will not be learned in the second stage. We will provide strong guarantees for two-stage surrogate
losses, provided that the first-stage loss function $\ell_1$ admits an
$\sH$-consistency bound, and the second-stage loss function $\ell_2$
admits an $\ov\sR$-consistency bound.

\begin{table}[t]
\caption{Examples for predictor-rejector second-stage surrogate losses
 \eqref{eq:ell-Phi-h-multi}.}
 \label{tab:sur-general-two-stage}
 \centering
 \begin{tabular}{@{\hspace{0cm}}ll@{\hspace{0cm}}}
  \toprule
  $\ell_2$ & $\lsc_{\hp}$ \\
  \midrule
  $\ell_{\rm{exp}}$ & $\1_{\hh(x) = y} \sum_{i=1}^{\num}e^{-r_i(x)}
  + \sum_{j = 1}^{\num}\bar c_j(x, y)\bracket*{\sum_{i=1,i\neq j }^{\num}e^{r_j(x) -r_i(x)}+e^{r_j(x)}}$ \\
  $\ell_{\rm{log}}$
  & $- \1_{\hh(x) = y} \log\paren*{\frac{1}{1+ \sum_{i=1}^{\num}e^{-r_i(x)}}}
  - \sum_{j = 1}^{\num}\bar c_j(x, y)\log\paren*{\frac{e^{-r_j(x)}}{1 + \sum_{i=1}^{\num}e^{-r_i(x)}}}$ \\
  $\ell_{\rm{gce}}$
  & $\1_{\hh(x) = y}\frac{1}{\alpha}\bracket*{1 - \bracket*{\frac{1}{1 + \sum_{i=1}^{\num}e^{-r_i(x)}}}^{\alpha}}
  + \sum_{j = 1}^{\num}\bar c_j(x, y)\frac{1}{\alpha}\bracket*{1- \bracket*{\frac{e^{-r_j(x)}}{1 + \sum_{i=1}^{\num}e^{-r_i(x)}}}^{\alpha}}$ 
   \\
   $\ell_{\rm{mae}}$ & $\1_{\hh(x) = y}\bracket*{1- \frac{1}{1 + \sum_{i=1}^{\num}e^{-r_i(x)}}}
   + \sum_{j = 1}^{\num}\bar c_j(x, y)\bracket*{1- \frac{e^{-r_j(x)}}{1 + \sum_{i=1}^{\num}e^{-r_i(x)}}}$ \\
  \bottomrule
 \end{tabular}
\end{table}

\begin{restatable}[\textbf{$(\sH,\sR)$-consistency bounds
      for predictor-rejector two-stage surrogates}]{theorem}{BoundGeneralTwoStepMulti}
\label{Thm:bound-general-two-step-multi}
  
Assume that $\ell_1$ admits an $\sH$-consistency bound and $\ell_2$
admits an $\ov \sR$-consistency bound with respect to the multi-class
zero-one classification loss $ \ell_{0-1}$ respectively. Thus, there
are non-decreasing concave functions $\Gamma_1$ and $\Gamma_2$ such
that, for all $h\in \sH$ and $\ov r\in \ov \sR$, we have
\begin{align*}
\sE_{\ell_{0-1}}(h) - \sE_{\ell_{0-1}}^*(\sH) + \sM_{\ell_{0-1}}(\sH)
& \leq \Gamma_1\paren*{\sE_{\ell_1}(h) - \sE_{\ell_1}^*(\sH) + \sM_{\ell_1}(\sH}\\
\sE_{\ell_{0-1}}(\ov r) - \sE_{\ell_{0-1}}^*(\ov \sR) + \sM_{\ell_{0-1}}(\ov \sR)
& \leq \Gamma_2\paren*{\sE_{\ell_2}(\ov r) - \sE_{\ell_2}^*(\ov \sR) + \sM_{\ell_2}(\ov \sR)}.
\end{align*}
Then, the following holds for all $h\in \sH$ and $r\in \sR$:
\begin{align*}
&\sE_{\ldef}(h,r) - \sE_{\ldef}^*(\sH,\sR) + \sM_{\ldef}(\sH,\sR)\\
&\quad \leq\Gamma_1\paren*{\sE_{\ell_1}(h) - \sE_{\ell_1}^*(\sH) + \sM_{\ell_1}(\sH)} + \paren*{1+ \sum_{j = 1}^{\num}\ov c_j}\Gamma_2\paren*{\frac{\sE_{\sfL_{h}}(r) - \sE_{\sfL_{h}}^*(\sR) + \sM_{\sfL_{h}}(\sR)}{\sum_{j = 1}^{\num}\uv c_j}},
\end{align*}
where the
constant factors $\paren*{1+ \sum_{j = 1}^{\num}\ov c_j}$ and $\frac{1}{\sum_{j = 1}^{\num}\uv c_j}$ can be removed 
when $\Gamma_2$ is linear.
\end{restatable}
As with the score-based setting, a specific instance of
Theorem~\ref{Thm:bound-general-two-step-multi} holds for the case
where $\sE^*_{\ell_1}(\sH) = \sE^*_{\ell_1}(\sH_{\rm{all}})$ and
$\sE^*_{\sfL_{h}}(\sR) = \sE^*_{\sfL_{h}}(\sR_{\rm{all}})$, ensuring
that the Bayes-error coincides with the best-in-class error and,
consequently, $\sM_{\ell_1}(\sH) = \sM_{\lsc_{h}}(\sR) = 0$. In these
cases, when the estimation error of the first-stage surrogate loss,
$\sE_{\ell_1}(h) - \sE_{\ell_1}^*(\sH)$, is reduced to $\e_1$, and the
estimation error of the second-stage surrogate loss,
$\sE_{\sfL_{h}}(r) - \sE_{\sfL_{h}}^*(\sR)$, is reduced to $\e_2$, the
estimation error of the deferral loss, $\sE_{\ldefsc}(h,r) -
\sE_{\ldefsc}^*(\sH,\sR)$, is upper-bounded by
\begin{equation*}
  \Gamma_1\paren*{\e_1} + \paren*{1+ \sum_{j = 1}^{\num}\ov c_j}
  \Gamma_2\paren*{\frac{\e_2}{\sum_{j = 1}^{\num}\uv c_j}}.
\end{equation*}
Next, we show that our two-stage surrogate losses are realizable
$(\sH,\sR)$-consistent. We say that the distribution is
\emph{$(\sH,\sR)$-realizable}, if there exists a zero error solution
$(h^*,r^*)\in \sH\times\sR$ with $\sE_{\ldef}(h^*,r^*) =0$.
\begin{restatable}[\textbf{Realizable $(\sH,\sR)$-consistency for
      predictor-rejector two-stage surrogates}]{theorem}{BoundGenralTwoStepGeneralRealizable}
\label{Thm:bound-general-two-stage-general-realizable}
Assume that $\sH$ and $\sR$ is closed under scaling and $c_j(x, y) =
\beta_j, \forall (x, y)\in \sX \times \sY$. Let $\ell_1$ and $\ell_2$
be the logistic loss. Let $\hat{h}$ be the minimizer of $\sE_{\ell_1}$
and $\hat{r}$ be the minimizer of $\sE_{\sfL_{\hat h}}$. Then, the
following holds for any ($\sH$, $\sR$) -realizable distribution,
\begin{align*}
 \sE_{\ldef}(\hat h, \hat r) = 0.
\end{align*}
\end{restatable}
The proof is included in Appendix~\ref{app:realizable-general}.
Theorem~\ref{Thm:bound-general-two-stage-general-realizable} suggests
that the two-stage surrogate loss is realizable consistent: when the estimation
error of the first-stage surrogate loss $\sE_{\ell_1}(h_n) -
\sE_{\ell_1}^*(\sH) \xrightarrow{n \rightarrow +\infty} 0$, and the
estimation error of the second-stage surrogate loss
$\sE_{\sfL_{h}}(r_n) - \sE_{\sfL_{h}}^*(\sR) \xrightarrow{n
  \rightarrow +\infty} 0$, the estimation error of the deferral loss,
$\sE_{\ldefsc}(h_n,r_n) - \sE_{\ldefsc}^*(\sH, \sR) \xrightarrow{n
  \rightarrow +\infty} 0$.  By
Theorem~\ref{Thm:bound-general-two-step-multi} and
Theorem~\ref{Thm:bound-general-two-stage-general-realizable}, in the
predictor-rejector setting, we also effectively find both
Bayes-consistent and realizable consistent surrogate losses with
multiple experts when only the inference cost ($\beta_j$) exists.

Note that while Sections~\ref{sec:two-stage}
and~\ref{sec:two-stage-predictor-rejector} both propose new two-stage
algorithms based on $\sH$-consistent surrogate losses, they differ in
an important way. Section~\ref{sec:two-stage} learns with deferral in
a score-based framework, where deferral is associated with extra
scores. In contrast, Section~\ref{sec:two-stage-predictor-rejector}
learns with deferral in a predictor-rejector setting, where deferral
corresponds to a separate function. These represent two distinct
learning frameworks that have been studied in the literature. Deriving
consistent surrogate losses in the predictor-rejector setting has
historically been challenging for traditional single-stage scenarios,
leading many to opt for the score-based approach.

We should also highlight that our $\sH$-consistency bounds in
Theorems~\ref{Thm:bound-general-two-stage-score}
and~\ref{Thm:bound-general-two-step-multi} can be used to derive
finite sample estimation bounds for the minimizer of the surrogate
loss over a hypothesis set $\sH$. This is achieved by upper-bounding
the estimation error of the minimizer of the surrogate loss using
standard Rademacher complexity bounds (see \citep{MaoMohriZhong2023cross}).

\section{Experiments}
\label{sec:experiments}

In this section, we report the results of our experiments on CIFAR-10
\citep{Krizhevsky09learningmultiple} and SVHN \citep{Netzer2011}
datasets to test the effectiveness of our proposed algorithms for
two-stage learning with multi-expert deferral. We evaluated the
overall accuracy of the learned pairs of predictor and deferral model
across different scenarios involving varying the number of experts,
where the predictor is pre-learned in the first stage and the deferral
is subsequently learned using our proposed surrogate loss. We find
that as the number of experts increases, the overall accuracy of the
learned pairs also increases, in both scenarios with zero and non-zero
base costs. This observation highlights the significance of using a
multiple expert framework in our approach and the effectiveness of our
surrogate loss within the framework.

We used ResNet architectures \citep{he2016deep} for the prediction
model, the deferral model and expert models. More precisely, we used
ResNet-$4$ for both the predictor and the deferral. We adopted three
expert models: ResNet-$10$, ResNet-$16$, ResNet-$28$ with increasing
capacity. For training, we used the Adam optimizer
\citep{kingma2014adam} with a batch size of $128$ and weight decay
$1\times 10^{-4}$.  Training was run for $15$ epochs for SVHN and $50$
epochs for CIFAR-10 with the default learning rate. No data
augmentation was used in our experiments.
We used our two-stage surrogate loss \eqref{eq:ell-Phi-h-score} with
the logistic loss $\ell = \ell_{\rm{log}}$ to train the deferral model
ResNet-$4$, with a pre-learned predictor ResNet-$4$ trained using logistic loss. A check mark indicates the
presence of a base cost in the cost function, whereas a cross mark
signifies its absence. We first set the cost function to be
$\1_{\hh_{j}(x) \neq y}$ without a base cost. Next, for the
experimental results shown in the last two row of
Table~\ref{tab:deferral}, we chose base costs $\beta_j$ associated
with each expert model as: $0.1$, $0.12$, $0.14$ increasing with model
capacity for SVHN and $0.3$, $0.32$, $0.34$ increasing with model
capacity for CIFAR-10. A base cost value that is close to the
misclassification loss can strike a balance between improving accuracy
and maintaining the ratio of deferral.  We observed that other
neighboring values lead to similar results. Note that the accuracy
here refers to the overall accuracy of the learned pairs of predictor
and deferral model. It is related to the deferral loss. Specifically,
in the absence of the base cost, the accuracy aligns precisely with
one minus the expected deferral loss.  The results of
Table~\ref{tab:deferral} demonstrate the effectiveness of our proposed
algorithms for two-stage learning with multi-expert deferral.

To the best of our knowledge, our study pioneers the exploration of a
two-stage learning approach for deferral, a framework that is
essential in numerous practical applications. Thus, we are unaware of
any established baselines within this context.

It is important to underscore the differences between our learning
scenario and those presented in
\citep{okati2021differentiable,narasimhanpost}. While both of them
involve two phases, their methodologies are considerably different from
ours. \citet{okati2021differentiable} required conditional
probabilities paired with loss estimates from the expert—a component
not available in our framework, as emphasized by
\citet{pmlr-v206-mozannar23a}. On the other hand,
\citet{narasimhanpost} proposed a post-hoc correction for single-stage
learning to defer surrogate losses. This approach, however, is
not applicable to a pre-trained predictor from the standard multi-class
classification. In contrast, our work focuses on enhancing the
pre-trained predictor within the standard framework.

A limitation of our study is that the cost function used within the
deferral loss is not fixed, and is typically determined through
cross-validation in practice. There exists potential to introduce a
principled method for selecting the cost function, which we have
reserved for future research.

\begin{table}[t]
  \caption{Accuracy of deferral with multiple experts:
    mean $\pm$ standard deviation over three runs.}
 \label{tab:deferral}
 \centering
 \begin{tabular}{@{\hspace{0cm}}llllll@{\hspace{0cm}}}
  \toprule
 Dataset & Base cost & Base model & Single expert & Two experts & Three experts \\
  \midrule
  SVHN & \xmark & 91.12 & 91.85 $\pm$ 0.01\% & 92.77 $\pm$ 0.02\% & 93.30 $\pm$ 0.02\% \\
  CIFAR-10 & \xmark & 70.56 & 72.63 $\pm$ 0.20\% & 75.84 $\pm$ 0.35\% & 77.68 $\pm$ 0.07\%\\
  SVHN & \cmark & 91.12 & 91.66 $\pm$ 0.01\% & 92.05 $\pm$ 0.10\% & 92.19 $\pm$ 0.03\% \\
  CIFAR-10 & \cmark & 70.56 & 71.73 $\pm$ 0.06\% & 72.31 $\pm$ 0.31\% & 72.42 $\pm$ 0.12\%\\
  \bottomrule
 \end{tabular}
\end{table}

\section{Conclusion}

We introduced a novel family of surrogate loss functions and
algorithms for a crucial two-stage learning to defer approach with
multiple experts. We proved that these surrogate losses are supported
by $\sH$-consistency bounds and established their realizable
$\sH$-consistency properties for a constant cost function. This work
paves the way for comparing different surrogate losses and cost
functions within our framework.  Further exploration, both
theoretically and empirically, holds the potential to identify optimal
choices for these quantities across diverse tasks.

\chapter{Regression with Multi-Expert Deferral} \label{ch6}
In this chapter, we deal with the problem of learning with multi-expert deferral in the regression setting.  While this problem has received
significant attention in classification contexts
\citep{hemmer2022forming, keswani2021towards, kerrigan2021combining,
  straitouri2022provably,
  benz2022counterfactual,verma2023learning,MaoMohriMohriZhong2023two,
  MaoMohriZhong2024deferral}, it presents unique challenges in regression due
to the infinite and continuous nature of the label space. In
particular, the \emph{score-based formulation} commonly used in
classification is inapplicable here, since regression problems cannot
be represented using multi-class scoring functions, with auxiliary
labels corresponding to each expert.

Our approach involves defining prediction and deferral functions,
consistent with previous studies in classification
\citep{MaoMohriMohriZhong2023two, MaoMohriZhong2024deferral}. We present a
comprehensive analysis for both the single-stage scenario
(simultaneous learning of predictor and deferral functions)
(Section~\ref{sec:single-stage}), and the two-stage scenario
(pre-trained predictor with learned deferral function)
(Section~\ref{sec:two-stage-regdef}). We introduce new surrogate loss
functions for both scenarios and prove that they are supported by
$\sH$-consistency bounds. These are consistency guarantees that are
stronger than Bayes consistency, as they are non-asymptotic and
hypothesis set-specific. Our framework is versatile, applying to
multiple experts, accommodating any bounded regression losses,
addressing both instance-dependent and label-dependent costs, and
supporting both single-stage and two-stage methods. We also
instantiate our formulations in the special case of a single expert
(Section~\ref{sec:single-expert}), and demonstrate that our
single-stage formulation includes the recent \emph{regression with
abstention} framework \citep{cheng2023regression} as a special case,
where only a single expert, the squared loss and a label-independent
cost are considered. In Section~\ref{sec:experiments-regdef}, we report the
results of extensive experiments showing the effectiveness of our
proposed algorithms.

The presentation in this chapter is based on \citep{mao2024regression}.

\section{Preliminaries}

\textbf{Learning scenario of regression.} We first describe the
familiar problem of supervised regression and introduce our
notation. Let $\sX$ be the input space and $\sY \subseteq \Rset$ the
label space. We write $\sD$ to denote a distribution over $\sX \times
\sY$. Let $\sH_{\rm{all}}$ be the family of all real-valued measurable
functions $h \colon \sX \to \sY$, and let $\sH \subseteq
\sH_{\rm{all}}$ be the hypothesis set adopted. The learning challenge
in regression is to use the labeled sample to find a hypothesis $h \in
\sH$ with small expected loss or generalization error $\sE_{\sfL}(h)$,
with $\sE_{\sfL}(h) = \E_{(x, y) \sim \sD}\bracket*{\sfL (h(x), y)}$,
where $\sfL \colon \sY \times \sY \to \Rset_{+}$ is a loss function
used to measure the magnitude of error in the regression. In the most
common case, where $\sfL$ is the squared loss $\sfL_2$ defined by
$\sfL_2(y', y) = \abs*{y' - y}^2$, this represents the mean squared
error. In the case where $\sfL$ is the $\sfL_1$ loss defined by
$\sfL_1(y', y) = \abs*{y' - y}$, this represents the mean absolute
error. More generally, $\sfL$ can be an $\sfL_p$ loss, defined by
$\sfL_p(y', y) = \abs*{y' - y}^p$ for all $y', y \in \sY$, for some $p
\geq 1$. In this work, we will consider an arbitrary regression loss
function $\sfL$, subject to the boundedness assumption, that is
$\sfL(y', y) \leq \ul$ for some constant $\ul > 0$ and for all $y, y'
\in \sY$. This assumption is commonly adopted in the theoretical
analysis of regression \citep{MohriRostamizadehTalwalkar2018}.

\textbf{Regression with deferral.}
We introduce a novel framework where a learner can defer predictions to multiple experts, $\expert_1, \ldots, \expert_{n_e}$. Each expert may represent a pre-trained model or a human expert.
The learner's output is a pair $(h, r)$, where $h \colon \sX \to \sY$
is a prediction function and $r \colon \sX \times \curl*{0, 1, \ldots,
  \num} \to \Rset$ a deferral function.
For any input $x$, $\rr(x) = \argmax_{y \in [\num]} r(x, y) = j$ is
the expert deferred to when $j > 0$, no deferral if $j = 0$.  The
learner makes the prediction $h(x)$ when $\rr(x) = 0$, or defers to
$\expert_j$ when $\rr(x) = j > 0$.  Deferral incurs the cost
$\sfL(\expert_{j}(x), y) + \alpha_j$, where $\alpha_j$ is a base
cost. The base cost can be the inference cost incurred when querying an expert, factoring in scenarios where engaging experts entails certain costs. Non-deferral incurs the cost $\sfL(h(x), y)$.

Let $\sH_{\rm{all}}$ and $\sR_{\rm{all}}$ denote the family of all
measurable functions, $h\colon \sX \to \sY$ and $r\colon \sX \times
\curl*{0, 1, \ldots, \num} \to \Rset$ respectively.  Given a
hypothesis set $\sH \subset \sH_{\rm{all}}$ and a hypothesis set $\sR
\subset \sR_{\rm{all}}$, the goal of the regression with deferral
problem consists of using the labeled sample to find a pair $(h, r)
\in (\sH, \sR)$ with small expected deferral loss
$\sE_{\ldef}(h, r) = \E_{(x, y) \sim
  \sD}\bracket*{\ldef(h, r, x, y)}$, where $\ldef$ is
defined for any $(h, r)\in
\sH \times \sR$ and $(x, y)\in \sX \times \sY$ by
\begin{equation}
\label{eq:def}
 \ldef(h, r, x, y)
\!=\! \sfL(h(x), y) 1_{\rr(x) = 0}
+  \sum_{j = 1}^{\num} c_j(x,y) 1_{\rr(x) = j}
\end{equation}
and $c_j(x, y) > 0$ is a cost function, which can be typically chosen
as $\alpha_j + \sfL\paren*{\expert_{j}(x), y}$ for an expert
$\expert_{j}$ and a base cost $\alpha_j> 0$ as mentioned before.
Here, we adopt a general cost functions $c_j$ for any $j$, and only
require that the cost remains bounded: $c_j(x, y) \leq \uc_j$ for all
$(x, y) \in \sX \times \sY$, for some constant $\uc_j > 0$.

\textbf{Learning with surrogate losses.} As with most target losses
in learning problems, such as the zero-one loss in classification
\citep{Zhang2003,bartlett2006convexity,zhang2004statistical,
  tewari2007consistency} and the classification with abstention loss
\citep{bartlett2008classification,CortesDeSalvoMohri2016}, directly
minimizing the deferral loss $\ldef$ is computationally hard
for most hypothesis sets due to its non-continuity and
non-differentiability. Instead, surrogate losses are proposed and
adopted in practice. Examples include the hinge loss in binary
classification \citep{cortes1995support}, the (multinomial) logistic
loss in multi-class classification
\citep{Verhulst1838,Verhulst1845,Berkson1944,Berkson1951}, and the
predictor-rejector abstention loss in classification with abstention
\citep{CortesDeSalvoMohri2016}. We will derive surrogate losses for the
deferral loss.

Given a surrogate loss $L \colon (h, r, x, y) \mapsto \Rset_{+}$, we
denote by $\sE_{L}(h, r)$ the generalization error of a pair $(h, r)$,
defined as
\begin{equation*}
\sE_{L}(h, r) = \E_{(x, y) \sim \sD} \bracket*{L(h, r, x, y)}.
\end{equation*}
Let $\sE_{L}(\sH, \sR) = \inf_{h\in \sH, r\in \sR} \sE_{L}(h, r)$ be
the best-in-class error within the family $\sH \times \sR$. One
desired property for surrogate losses in this context is
\emph{Bayes-consistency} \citep{steinwart2007compare}. This means that
minimizing the expected surrogate loss over the family of all
measurable functions leads to minimizing the expected deferral loss
over the same family. More precisely, for a surrogate loss $L \colon
(h, r, x, y) \mapsto \Rset_{+}$, it is \emph{Bayes-consistent} with
respect to $\ldef$ if,
\begin{align*}
\sE_{L}(h_n, r_n) - \sE_{L}(\sH, \sR) \xrightarrow{n \rightarrow + \infty} 0
\implies & \sE_{\ldef}(h_n, r_n) - \sE_{ \ldef}(\sH, \sR) \xrightarrow{n \rightarrow + \infty} 0
\end{align*}
for all sequences $\curl*{(h_n, r_n)}_{n \in \Nset} \subset \sH \times
\sR$ and all distributions. Recently,
\citet{awasthi2022Hconsistency,AwasthiMaoMohriZhong2022multi} (see also
\citep{awasthi2021calibration,awasthi2021finer,AwasthiMaoMohriZhong2023theoretically,awasthi2024dc,MaoMohriZhong2023cross,MaoMohriZhong2023ranking,MaoMohriZhong2023rankingabs,zheng2023revisiting,MaoMohriZhong2023characterization,MaoMohriZhong2023structured,mao2024h,mao2024regression,mao2024universal,mao2025enhanced,mao2024multi,mao2024realizable,MohriAndorChoiCollinsMaoZhong2024learning,cortes2024cardinality,zhong2025fundamental,MaoMohriZhong2025principled,MaoMohriZhong2025mastering,cortes2025balancing,cortes2025improved,desalvo2025budgeted,mohri2025beyond}) pointed out that
Bayes-consistency does not take into account the hypothesis set $\sH$
and is non-asymptotic. Thus, they proposed a stronger guarantee called
\emph{$\sH$-consistency bounds}. In our context, a surrogate loss $L$
is said to admit an $(\sH, \sR)$-consistency bound with respect to
$\ldef$ if, for all $(h, r) \in \sH \times \sR$ and all
distributions, the following inequality holds:
\begin{equation*}
f \paren*{\sE_{\ldef}(h, r) - \sE^*_{\ldef}(\sH, \sR)} \leq \sE_{L}(h, r) - \sE^*_{L}(\sH, \sR)
\end{equation*}
for some non-decreasing function $f \colon \sR_{+} \to \sR_{+}$. In particular, when
$(\sH, \sR) = \paren*{\sH_{\rm{all}}, \sR_{\rm{all}}}$, the $(\sH,
\sR)$-consistency bound implies Bayes-consistency.

We will prove $(\sH,
\sR)$-consistency bounds for our proposed surrogate losses, which
imply their Bayes-consistency. One key term in our bound is the \emph{minimizability gap},
defined as \[\sM_{L}(\sH, \sR) = \sE^*_{L}(\sH, \sR) - \E_{x} \E_{y
  \mid x}\bracket*{L(h, r, x, y)}.\] The minimizability gap
characterizes the difference between the best-in-class error and the
expected best-in-class point-wise error, and is non-negative. As shown
by \citet{MaoMohriZhong2023cross}, the minimizability gap is upper-bounded by
the approximation error, satisfying $0 \leq \sM_{L}(\sH, \sR) \leq
\sE^*_{L}(\sH, \sR) - \sE^*_{L}(\sH_{\rm{all}}, \sR_{\rm{all}})$ and
is generally a finer quantity. The minimizability gap
vanishes when $(\sH, \sR) =
(\sH_{\rm{all}}, \sR_{\rm{all}})$, or, more generally, when
$\sE^*_{L}(\sH, \sR) = \sE^*_{L}(\sH_{\rm{all}}, \sR_{\rm{all}})$.

Given a loss function $\ell \colon (r, x, y) \mapsto
\Rset_{+}$ that only depends on the hypothesis $r$, 
the notions of generalization error, best-in-class generalization
error, and minimizability gaps, as well as Bayes-consistency and
$\sR$-consistency bounds, are similarly defined \citep{awasthi2022Hconsistency,AwasthiMaoMohriZhong2022multi}.

In the next sections, we study the problem of learning a pair $(h, r)$
in the framework of regression with deferral. We will derive a family  of surrogate losses of $\ldef$, starting from first principles. We will
show that these loss functions benefit from strong consistency guarantees, which yield directly principled algorithms for our deferral
problem. We will specifically distinguish two approaches: the
single-stage surrogate losses, where the predictor $h$ and the
deferral function $r$ are jointly learned, and the two-stage surrogate
losses wherein the predictor $h$ have been previously trained and is
fixed and subsequently
used in the learning process of the deferral function $r$.

\section{Single-Stage Scenario}
\label{sec:single-stage}

In this section, we derive single-stage surrogate losses for the
deferral loss and prove their strong $(\sH, \sR)$-consistency bounds
guarantees. To do so, we first prove that the following alternative
expression holds for $\ldef$.

\begin{restatable}{lemma}{Ldef}
\label{lemma:Ldef}
For any $(h, r) \in \sH \times \sR$ and $(x, y) \in \sX \times \sY$,
the loss function $\ldef$ can be expressed as follows:
\begin{align*}
  \ldef(h, r, x, y) 
  & = \bracket*{\sum_{j = 1}^{\num} c_j(x,y)} 1_{\rr(x) \neq 0} 
  + \sum_{j = 1}^{\num} \bracket*{\sfL(h(x), y)
    + \sum_{k = 1}^{\num} c_k(x, y) 1_{k \neq j}} 1_{\rr(x) \neq j}\\
  & \quad - \paren*{\num - 1} \bracket*{\sfL(h(x), y) + \sum_{j = 1}^{\num} c_j(x, y)}.
\end{align*}
\end{restatable}
Let $\ell_{0-1}$ be the zero-one multi-class classification loss defined
by $\ell_{0-1}(r, x, y) = 1_{\rr(x) \neq y}$ for all $r \in \Rset$ and
$(x, y) \in \sX \times \sY$ and let $\ell \colon \sR \times \sX \times
[\num] \to \Rset_{+}$ be a surrogate loss for $\ell_{0-1}$ such that
$\ell \geq \ell_{0-1}$. $\ell$ may be chosen to be the logistic loss,
for example. Since the last term $\paren*{\num - 1} \sum_{j =
  1}^{\num} c_j(x,y)$ in the expression of $\ldef$ in
Lemma~\ref{lemma:Ldef} does not depend on $h$ and $r$, the following
loss function $L_{\ell}$ defined for all $(h, r) \in \sH \times \sR$
and $(x, y) \in \sX \times \sY$ by
\begin{align}
\label{eq:sur}
L_{\ell}(h, r, x, y) 
& = \bracket*{\sum_{j = 1}^{\num} c_j(x,y)} \ell(r, x, 0)\\
& \quad + \sum_{j = 1}^{\num} \bracket*{\sfL(h(x), y)
  + \sum_{j' \neq j}^{\num} c_{j'}(x,y)} \ell(r, x, j) - \paren*{\num - 1} \sfL(h(x), y), \nonumber
\end{align}
is a natural single-stage surrogate loss for $\ldef$.  We will show
that when $\ell$ admits a strong $\sR$-consistency bound with respect
to $\ell_{0-1}$, then $L_{\ell}$ admits an $(\sH, \sR)$-consistency
bound with respect to $\ldef$.

Let us underscore the novelty of the surrogate loss formulation
presented in equation \eqref{eq:sur} in the context of learning to
defer with multiple experts. This formulation represents a substantial
departure from the existing score-based approach prevalent in
classification. As previously highlighted, the score-based formulation
becomes inapplicable in regression.  Our new predictor-rejector
formulation not only overcomes this limitation, but also provides the
foundation for the design of new deferral algorithms for
classification.

We say that a hypothesis set $\sR$ is \emph{regular} if for any $x \in
\sX$, the predictions made by the hypotheses in $\sR$ cover the
complete set of possible classification labels: $\curl*{\rr(x) \colon
  r\in \sR} = \curl*{0, 1, \ldots, \num}$. Widely used hypothesis sets
such as linear hypotheses, neural networks, and of course the family
of all measurable functions are all regular.

Recent studies by \citet{AwasthiMaoMohriZhong2022multi} and \citet{MaoMohriZhong2023cross}
demonstrate that common multi-class surrogate losses, such as
constrained losses and comp-sum losses (including the logistic loss),
admit strong $\sR$-consistency bounds with respect to the multi-class
zero-one loss $\ell_{0-1}$, when using such regular hypothesis sets.
The next result shows that, for multi-class loss functions $\ell$,
their corresponding deferral surrogate losses $L_{\ell}$
(Eq.~\eqref{eq:sur}) also exhibit $(\sH, \sR)$-consistency bounds with
respect to the deferral loss (Eq.~\eqref{eq:def}).

\begin{restatable}{theorem}{Single}
\label{thm:single}
Let $\sR$ be a regular hypothesis set and $\ell$ a surrogate loss for
the multi-class loss function $\ell_{0-1}$ upper-bounding
$\ell_{0-1}$. Assume that there exists a function $\Gamma(t) = \beta\,
t^{\alpha}$ for some $\alpha \in (0, 1]$ and $\beta > 0$, such that
  the following $\sR$-consistency bound holds for all $r \in \sR$ and
  any distribution,
\begin{equation*}
  \sE_{\ell_{0-1}}(r) - \sE^*_{\ell_{0-1}}(\sR) + \sM_{\ell_{0-1}}(\sR)
  \leq \Gamma\paren*{\sE_{\ell}(r)
    - \sE^*_{\ell}(\sR) + \sM_{\ell}(\sR)}.
\end{equation*}
Then, the following $(\sH,\sR)$-consistency bound holds for all $h\in
\sH$, $r\in \sR$ and any distribution,
\begin{equation*}
   \sE_{\ldef}(h, r) - \sE_{\ldef}^*(\sH,\sR) + \sM_{\ldef}(\sH,\sR)
   \leq 
  \ov \Gamma\paren*{\sE_{L_{\ell}}(h, r)
    -  \sE_{L_{\ell}}^*(\sH,\sR) + \sM_{L_{\ell}}(\sH,\sR)},
\end{equation*}
where $\ov \Gamma(t)
= \max\curl*{t, \paren*{\num\paren*{\ul
      + \sum_{j = 1}^{\num}\uc_j}}^{1 - \alpha} \beta\, t^{\alpha}}$.
\end{restatable}
The proof is given in Appendix~\ref{app:exp}. As already mentioned,
when the best-in-class error coincides with the Bayes error
$\sE^*_{L}(\sH, \sR) = \sE^*_{L}\paren*{\sH_{\rm{all}},
  \sR_{\rm{all}}}$ for $L = L_{\ell}$ and $L = \ldef$, the
minimizability gaps $\sM_{L_{\ell}}(\sH,\sR)$ and
$\sM_{\ldef}(\sH,\sR)$ vanish. In such cases, the
$(\sH, \sR)$-consistency bound guarantees that when the surrogate
estimation error $\sE_{L_{\ell}}(h, r) - \sE_{L_{\ell}}^*(\sH,\sR)$ is
optimized up to $\e$, the estimation error of the deferral loss $
\sE_{\ldef}(h, r) - \sE_{\ldef}^*(\sH,\sR)$ is upper-bounded by $\ov
\Gamma(\e)$.

In particular, when both $\sH$ and $\sR$ include all measurable
functions, all the minimizability gap terms in
Theorem~\ref{thm:single} vanish, which yields the following result.
\begin{corollary}
\label{cor:single}
Given a multi-class loss function $\ell \geq \ell_{0-1}$. Assume that
there exists a function $\Gamma(t) = \beta\, t^{\alpha}$ for some
$\alpha \in (0, 1]$ and $\beta > 0$, such that the following excess
  error bound holds for all $r \in \sR_{\rm{all}}$ and any
  distribution,
\begin{equation*}
  \sE_{\ell_{0-1}}(r) - \sE^*_{\ell_{0-1}}(\sR_{\rm{all}})
  \leq \Gamma\paren*{\sE_{\ell}(r) - \sE^*_{\ell}(\sR_{\rm{all}})}.
\end{equation*}
Then, the following excess error bound holds for all $h\in \sH_{\rm{all}}$, $r\in \sR_{\rm{all}}$ and any distribution,
\begin{equation*}
   \sE_{\ldef}(h, r) - \sE_{\ldef}^*(\sH_{\rm{all}},\sR_{\rm{all}})
   \leq 
    \ov \Gamma\paren*{\sE_{L_{\ell}}(h, r) -  \sE_{L_{\ell}}^*(\sH_{\rm{all}},\sR_{\rm{all}})},
\end{equation*}
where $\ov \Gamma(t) = \max\curl*{t, \paren*{\num\paren*{\ul + \sum_{j = 1}^{\num}\uc_j}}^{1 - \alpha} \beta\, t^{\alpha}}$.
\end{corollary}
In this case, as shown by \citet{MaoMohriZhong2023cross}, $\Gamma(t)$ can be
expressed as $\sqrt{2t}$ for the logistic loss $\ell_{\rm{\log}}
\colon (r, x, y) \mapsto \log_2 \paren*{\sum_{j = 0}^{\num} e^{r(x, j)
    - r(x, y)}}$. Then, by Corollary~\ref{cor:single}, we further
obtain the following corollary.
\begin{corollary}
\label{cor:single-example}
For any $h\in \sH_{\rm{all}}$, $r\in \sR_{\rm{all}}$ and distribution,
\begin{equation*}
    \sE_{\ldef}(h, r) - \sE_{\ldef}^*(\sH_{\rm{all}},\sR_{\rm{all}})
    \leq 
    \ov\Gamma\paren*{\sE_{L_{\ell_{\rm{log}}}}(h, r) -  \sE_{L_{\ell_{\rm{log}}}}^*(\sH_{\rm{all}},\sR_{\rm{all}})},
\end{equation*}
where $\ov \Gamma (t) = \max\curl*{t, \sqrt{2 \num} \paren*{\ul +
    \sum_{j = 1}^{\num}\uc_j}^{\frac12} t^{\frac 12}}$.
\end{corollary}
By taking the limit on both sides, we derive the Bayes-consistency of
these single-stage surrogate losses $L_{\ell}$ with respect to the
deferral loss $\ldef$. More generally, Corollary~\ref{cor:single}
shows that $L_{\ell}$ admits an excess error bound with respect to
$\ldef$ when $\ell$ admits an excess error bound with respect to
$\ell_{0-1}$.

\section{Two-Stage Scenario}
\label{sec:two-stage-regdef}

In the single-stage scenario, we introduced a family of surrogate
losses and resulting algorithms for effectively learning the pair $(h,
r)$. However, practical applications often encounter a \emph{two-stage
scenario}, where deferral decisions are based on a fixed, pre-trained
predictor $h$. Retraining this predictor is often prohibitively
expensive or time-consuming. Thus, this two-stage scenario
\citep{MaoMohriMohriZhong2023two} requires a different approach to optimize
deferral decisions controlled by $r$, while using the existing
predictor $h$.

In this section, we will introduce a principled
two-stage algorithm for regression with deferral, with favorable
consistency guarantees. Remarkably, we show that the single-stage approach can be adapted for the two-stage scenario if we fix the predictor $h$ and disregard constant
terms.

Let $h$ be a predictor learned by minimizing a regression loss
$\sfL$ in a first stage. A deferral function $r$ is then learned based on
that predictor $h$ and the following loss function $L_{\ell}^h$ in the
second stage: for any $r \in \sR$, $x \in \sX$ and $y \in \sY$,
\begin{align}
\label{eq:two-stage-sur}
  L_{\ell}^h (r, x, y)
  = \bracket*{\sum_{j = 1}^{\num} c_j(x,y)} \, \ell(r, x, 0) + \sum_{j = 1}^{\num} \bracket*{\sfL(h(x), y)
    + \sum_{j' \neq j}^{\num} c_{j'}(x,y)} \, \ell(r, x, j),
\end{align}
where $\ell$ is a surrogate loss in the standard multi-class
classification.  
Equation \eqref{eq:two-stage-sur} resembles \eqref{eq:sur}, except for the constant term $(\num - 1)\sfL(h(x), y)$. In \eqref{eq:two-stage-sur}, the predictor $h$ remains fixed while only the deferral function $r$ is optimized. In \eqref{eq:sur}, both $h$ and $r$ are learned jointly.
\ignore{
Note that \eqref{eq:two-stage-sur} has a same
formulation as \eqref{eq:sur} modulo the constant term $(\num -
1)\sfL(h(x), y)$ in the second-stage. However, the main difference is
that in \eqref{eq:two-stage-sur}, $h$ is fixed and only $r$ is
learned, while both $h$ and $r$ are jointly learned in
\eqref{eq:sur}. 
}

Similarly, we define $\ldef^h$ as the
deferral loss \eqref{eq:def} with a fixed predictor $h$ as follows:
\begin{equation}
\label{eq:two-stage-deferral-loss}
\ldef^h (r, x, y) = \sfL(h(x), y) 1_{\rr(x) = 0} + \sum_{j = 1}^{\num} c_j(x,y) 1_{\rr(x) = j}.
\end{equation}
Here too, $h$ is fixed in \eqref{eq:two-stage-deferral-loss}. Both $L_{\ell}^h$ and $\ldef^h$ are loss functions defined for deferral function $r$, while $\ell_{\ell}$ and $\ell_{\rm{def}}$ are loss functions defined for pairs $(h, r) \in (\sH, \sR)$.

As with the proposed single-stage approach, the two-stage surrogate losses $L_{\ell}^h$ benefit from strong consistency guarantees. We show that in the second stage where a predictor $h$ is fixed, the surrogate loss function $L_{\ell}^h$ benefits from $\sR$-consistency bounds with respect to $\ldef^h$ when $\ell$ admits a strong $\sR$-consistency bound with respect to the multi-class zero-one loss $\ell_{0-1}$. 
\begin{restatable}{theorem}{TwostageR}
\label{thm:tsr}
Given a hypothesis set $\sR$, a multi-class loss function $\ell \geq \ell_{0-1}$ and a predictor $h$. Assume that there exists a function $\Gamma(t) = \beta\, t^{\alpha}$ for some $\alpha \in (0, 1]$ and $\beta > 0$, such that the following $\sR$-consistency bound holds for all $r \in \sR$ and any distribution,
\begin{align*}
\sE_{\ell_{0-1}}(r) - \sE^*_{\ell_{0-1}}(\sR) + \sM_{\ell_{0-1}}(\sR) \leq \Gamma\paren*{\sE_{\ell}(r) - \sE^*_{\ell}(\sR) + \sM_{\ell}(\sR)}.
\end{align*}
Then, the following $\sR$-consistency bound holds for all $r\in \sR$ and any distribution,
\begin{align*}
   \sE_{L^h_{\rm{def}}}(r) - \sE_{L^h_{\rm{def}}}^*(\sR) + \sM_{L^h_{\rm{def}}}(\sR) \leq 
    \ov \Gamma\paren*{\sE_{L^h_{\ell}}(r) -  \sE_{L^h_{\ell}}^*(\sR) + \sM_{L^h_{\ell}}(\sR)},
\end{align*}
where $\ov \Gamma(t) = \paren*{\num\paren*{\ul + \sum_{j = 1}^{\num}\uc_j}}^{1 - \alpha} \beta\, t^{\alpha}$.
\end{restatable}
The proof is given in Appendix~\ref{app:tsr}. When the best-in-class error coincides with the Bayes error, $\sE^*_{L}(\sR) = \sE^*_{L}\paren*{\sR_{\rm{all}}}$ for $L = L^h_{\ell}$ and $L = L^h_{\rm{def}}$, the minimizability gaps $\sM_{L^h_{\ell}}(\sR)$ and $\sM_{L^h_{\rm{def}}}(\sR)$ vanish. In that case, the  $\sR$-consistency bound guarantees that when the surrogate estimation error $\sE_{L^h_{\ell}}(r) -  \sE_{L^h_{\ell}}^*(\sR)$ is optimized up to $\e$, the target estimation error $ \sE_{L^h_{\rm{def}}}(r) - \sE_{L^h_{\rm{def}}}^*(\sR)$ is upper-bounded by $\ov \Gamma(\e)$. In the special case where $\sH$ and $\sR$ are the family of all measurable functions, all the minimizability gap terms in Theorem~\ref{thm:tsr} vanish. Thus, we obtain the following corollary.
\begin{corollary}
\label{cor:tsr}
Given a multi-class loss function $\ell \geq \ell_{0-1}$ and a predictor $h$. Assume that there exists a function $\Gamma(t) = \beta\, t^{\alpha}$ for some $\alpha \in (0, 1]$ and $\beta > 0$, such that the following excess error bound holds for all $r \in \sR$ and any distribution,
\begin{equation*}
\sE_{\ell_{0-1}}(r) - \sE^*_{\ell_{0-1}}(\sR_{\rm{all}}) \leq \Gamma\paren*{\sE_{\ell}(r) - \sE^*_{\ell}(\sR_{\rm{all}})}.
\end{equation*}
Then, the following excess error bound holds for all $r\in \sR_{\rm{all}}$ and any distribution,
\begin{equation}
\label{eq:bound-tsr-all}
   \sE_{L^h_{\rm{def}}}(r) - \sE_{L^h_{\rm{def}}}^*(\sR_{\rm{all}})) \leq 
    \ov \Gamma\paren*{\sE_{L^h_{\ell}}(r) -  \sE_{L^h_{\ell}}^*(\sR_{\rm{all}})},
\end{equation}
where $\ov \Gamma(t) = \paren*{\num\paren*{\ul + \sum_{j = 1}^{\num}\uc_j}}^{1 - \alpha} \beta\, t^{\alpha}$.
\end{corollary}
Corollary~\ref{cor:tsr} shows that $L^h_{\ell}$ admits an excess error bound with respect to $L^h_{\rm{def}}$ when $\ell$ admits an excess error bound with respect to $\ell_{0-1}$.

We now establish $(\sH, \sR)$-consistency bounds the entire two-stage approach with respect to the deferral loss function $\ldef$. This result applies to any multi-class loss function $\ell$ that satisfies a strong $\sR$-consistency bound with respect to the multi-class zero-one loss $\ell_{0-1}$.
\begin{restatable}{theorem}{TwostageHR}
\label{thm:tshr}
Given a hypothesis set $\sH$, a regular hypothesis set $\sR$ and a multi-class loss function $\ell \geq \ell_{0-1}$. Assume that there exists a function $\Gamma(t) = \beta\, t^{\alpha}$ for some $\alpha \in (0, 1]$ and $\beta > 0$, such that the following $\sR$-consistency bound holds for all $r \in \sR$ and any distribution,
\begin{align*}
\sE_{\ell_{0-1}}(r) - \sE^*_{\ell_{0-1}}(\sR) + \sM_{\ell_{0-1}}(\sR) \leq \Gamma\paren*{\sE_{\ell}(r) - \sE^*_{\ell}(\sR) + \sM_{\ell}(\sR)}.
\end{align*}
Then, the following $(\sH,\sR)$-consistency bound holds for all $h \in \sH$, $r \in \sR$ and any distribution,
\begin{align*}
& \sE_{\ldef}(h, r) - \sE_{\ldef}^*(\sH,\sR) + \sM_{\ldef}(\sH,\sR)\\
& \qquad \leq \sE_{\sfL}(h) - \sE_{\sfL}(\sH) + \sM_{\sfL}(\sH) + \ov \Gamma\paren*{\sE_{L^h_{\ell}}(r) -  \sE_{L^h_{\ell}}^*(\sR) + \sM_{L^h_{\ell}}(\sR)},
\end{align*}
where $\ov \Gamma(t) = \paren*{\num\paren*{\ul + \sum_{j = 1}^{\num}\uc_j}}^{1 - \alpha}\beta\, t^{\alpha}$.
\end{restatable}
The proof is presented in Appendix~\ref{app:tshr}. As before, when $\sH$ and $\sR$ are the family of all measurable functions, all the minimizability gap terms in Theorem~\ref{thm:tshr} vanish. In particular, $\Gamma(t)$ can be expressed as $\sqrt{2t}$ for the logistic loss. Thus, we obtain the following on excess error bounds.
\begin{corollary}
\label{cor:tshr}
Given a multi-class loss function $\ell \geq \ell_{0-1}$. Assume that there exists a function $\Gamma(t) = \beta\, t^{\alpha}$ for some $\alpha \in (0, 1]$ and $\beta > 0$, such that the following excess error bound holds for all $r \in \sR_{\rm{all}}$ and any distribution,
\begin{equation*}
\sE_{\ell_{0-1}}(r) - \sE^*_{\ell_{0-1}}(\sR_{\rm{all}}) \leq \Gamma\paren*{\sE_{\ell}(r) - \sE^*_{\ell}(\sR_{\rm{all}})}.
\end{equation*}
Then, the following excess error bound holds for all $h\in \sH_{\rm{all}}$, $r\in \sR_{\rm{all}}$ and any distribution,
\begin{equation}
\begin{aligned}
\label{eq:bound-tshr-all}
   \sE_{\ldef}(h, r) - \sE_{\ldef}^*(\sH_{\rm{all}},\sR_{\rm{all}}) \leq \sE_{\sfL}(h) - \sE_{\sfL}(\sH_{\rm{all}}) + \ov \Gamma\paren*{\sE_{L^h_{\ell}}(r) -  \sE_{L^h_{\ell}}^*(\sR_{\rm{all}})},
\end{aligned}
\end{equation}
where $\ov \Gamma(t) = \paren*{\num\paren*{\ul + \sum_{j = 1}^{\num}\uc_j}}^{1 - \alpha} \beta\, t^{\alpha}$. In particular, $\ov \Gamma (t) = \sqrt{2 \num} \paren*{\ul + \sum_{j = 1}^{\num}\uc_j}^{\frac12} t^{\frac 12}$ for $\ell = \ell_{\rm{log}}$.
\end{corollary}
Corollary~\ref{cor:tshr} shows that our two-stage approach admits an excess error bound with respect to $\ldef$ when $\ell$ admits an excess error bound with respect to $\ell_{0-1}$. More generally, when the minimizability gaps are zero, as when the best-in-class errors coincide with the Bayes errors, the $(\sH, \sR)$-consistency bound of Theorem~\ref{thm:tshr} guarantees that the target estimation error, $\sE_{\ldef}(h, r) - \sE_{\ldef}^*(\sH,\sR)$, is upper-bounded by $\e_1 + \ov \Gamma(\e_2)$
provided that the surrogate estimation error in the first stage, $\sE_{\sfL}(h) - \sE_{\sfL}(\sH)$, is reduced to $\e_1$ and the surrogate estimation error in the second stage, $\sE_{L^h_{\ell}}(r) -  \sE_{L^h_{\ell}}^*(\sR)$, reduced to $\e_2$.

\textbf{Significance and novelty.} The challenges in dealing with
multiple experts in the theoretical analysis of learning with deferral
in regression arise first from the need to formulate new surrogate
losses that cannot be directly extended from previous
work. Furthermore, proving theoretical guarantees requires analyzing
the conditional regret for both the surrogate and target deferral
loss, which becomes more complex with multiple experts. The novelty
and significance of our work are rooted in these new surrogate losses
and algorithmic solutions, which come with strong theoretical
guarantees specifically tailored for this context. These enhancements
are non-trivial and represent a substantial extension beyond the
existing framework of regression with abstention, which is limited in
scope to a single expert, squared loss, and label-independent cost.

\section{Special Case of a Single Expert}
\label{sec:single-expert}

In the special case of a single expert, $\num = 1$, both the
single-stage surrogate loss $L_{\ell}$ and the two-stage surrogate
loss $L^h_{\ell}$ can be simplified as follows:
\begin{equation*}
c(x,y) \ell(r, x, 0) + \sfL(h(x), y) \ell(r, x, 1).
\end{equation*}
Let $\ell(r, x, 0) = \Phi(r(x))$ and $\ell(r, x, 1) = \Phi(-r(x))$, where $\Phi\colon \Rset \to \Rset_{+}$ is a non-increasing auxiliary function upper-bounding the indicator $u \mapsto 1_{u \leq 0}$. Here, $r \colon \sX \to \Rset$ is a function whose sign determines if there is
deferral, that is $r(x) \leq 0$:
\begin{equation*}
 \ell_{\rm{def}}(h, r, x, y) = \sfL(h(x), y) 1_{r(x) > 0} + c(x,y) 1_{r(x) \leq 0}.
\end{equation*}
As an example, $\Phi$ can be the auxiliary function that defines a margin-based loss in the binary classification. Thus, both the single-stage surrogate loss $\ell_{\Phi}$ and the two-stage surrogate loss $\ell^h_{\Phi}$ can be reformulated as follows: 
\begin{equation}
\label{eq:single}
c(x,y) \Phi(r(x)) + \sfL(h(x), y) \Phi(-r(x)).
\end{equation}
Some common examples of $\Phi$ are listed in Table~\ref{tab:sur-binary} in Appendix~\ref{app:sur-binary}.
Note that \eqref{eq:single} is a straightforward extension of the two-stage formulation given in \citep[Eq.~(5)]{MaoMohriZhong2024predictor}. In their formulation, the zero-one loss function replaces the regression loss and is tailored for the classification context. A special case of the straightforward extension \eqref{eq:single} is one where the cost does not depend on the label $y$ and the squared loss is considered. This coincides with the loss function \citep[Eq.~(10)]{cheng2023regression} in the context of regression with abstention. It is important to note that incorporating $y$ as argument of the cost functions is crucial in the more general deferral setting, as each cost takes into account the accuracy of the corresponding expert. 

Let the binary zero-one loss be $\ell^{\rm{bi}}_{0-1}(r, x, y) = 1_{\sign\paren*{r(x)} \neq y}$, where $\sign(\alpha) = 1_{\alpha > 0} - 1_{\alpha \leq 0}$. We say a hypothesis set $\sR$ consists of functions mapping from $\sX$ to $\Rset$ is \emph{regular}, if  $\curl*{\sign \paren*{r(x)} \colon r \in \sR} = \curl*{+ 1, - 1}$ for any $x \in \sX$.

Then, Theorems~\ref{thm:single} and ~\ref{thm:tshr} can be reduced to Theorems~\ref{thm:single-binary} and ~\ref{thm:tshr-binary} below respectively. We present these guarantees and their corresponding corollaries in the following sections.

\subsection{Single-Stage Guarantees}

\begin{table*}[t]
  \caption{System MSE of deferral with multiple experts:
    mean $\pm$ standard deviation over three runs.}
 \label{tab:deferral-1}
 \centering
 \begin{tabular}{@{\hspace{0cm}}lllllll@{\hspace{0cm}}}
  \toprule
 Dataset & Base cost & Method & Base model & Single expert & Two experts & Three experts \\
  \midrule
  \multirow{4}{*}{\texttt{Airfoil}} & \xmark & Single & ---              & $18.98 \pm 2.44$ & $13.16 \pm 0.93$ & \bm{$8.53 \pm 1.57$} \\
                           & \xmark & Two    & $23.35 \pm 1.90$ & $18.64 \pm 1.96$ & $13.33 \pm 0.92 $ & \bm{$8.81 \pm 1.56$} \\
                           & \cmark & Single & ---              & $18.83 \pm 2.14$ & $13.79 \pm 0.75 $ & \bm{$8.64 \pm 1.40$} \\
                           & \cmark & Two    & $23.35 \pm 1.90$ & $19.15 \pm 1.99$ & $15.12 \pm 0.62 $ & \bm{$10.06 \pm 1.54$}\\
  \midrule
  \multirow{4}{*}{\texttt{Housing}} & \xmark & Single & ---              & $14.85 \pm 5.40$ & $14.75 \pm 3.53$ & \bm{$12.43 \pm 2.03$} \\
                           & \xmark & Two    & $22.72 \pm 7.68$ & $16.26 \pm 5.58$ & $14.82 \pm 3.60$ & \bm{$12.02 \pm 1.97$} \\
                           & \cmark & Single & ---              & $15.17 \pm 5.18$ & $15.07 \pm 2.88$ & \bm{$14.80 \pm 3.48$} \\
                           & \cmark & Two    & $22.72 \pm 7.68$ & $16.24 \pm 4.64$ & $15.62 \pm 3.04$ & \bm{$14.87 \pm 4.04$} \\
  \midrule
  \multirow{4}{*}{\texttt{Concrete}}& \xmark & Single & ---              & $104.38 \pm 5.55$ & $41.08 \pm 2.05$ & \bm{$37.83 \pm 2.60$} \\
                           & \xmark & Two    & $120.20 \pm 8.09$& $114.73 \pm 6.50$ & $44.46 \pm 5.34$ & \bm{$36.75 \pm 1.76$} \\
                           & \cmark & Single & ---              & $105.01 \pm 5.40$ & $39.52 \pm 2.81$ & \bm{$38.46 \pm 1.79$} \\
                           & \cmark & Two    & $120.20 \pm 8.09$& $114.11 \pm 5.34$ & $39.93 \pm 2.77$ & \bm{$37.51 \pm 2.32$} \\
  \bottomrule
 \end{tabular}
\end{table*}

Here, we present guarantees for the single-stage surrogate.
\begin{restatable}{theorem}{SingleBinary}
\label{thm:single-binary}
Given a hypothesis set $\sH$, a regular hypothesis set $\sR$ and a margin-based loss function $\Phi$. Assume that there exists a function $\Gamma(t) = \beta\, t^{\alpha}$ for some $\alpha \in (0, 1]$ and $\beta > 0$, such that the following $\sR$-consistency bound holds for all $r \in \sR$ and any distribution,
\begin{align*}
\sE_{\ell^{\rm{bi}}_{0-1}}(r) - \sE^*_{\ell^{\rm{bi}}_{0-1}}(\sR) + \sM_{\ell^{\rm{bi}}_{0-1}}(\sR) \leq \Gamma\paren*{\sE_{\Phi}(r) - \sE^*_{\Phi}(\sR) + \sM_{\Phi}(\sR)}.
\end{align*}
Then, the following $(\sH,\sR)$-consistency bound holds for all $h\in \sH$, $r\in \sR$ and any distribution,
\begin{align*}
   \sE_{\ell_{\rm{def}}}(h, r) - \sE_{\ell_{\rm{def}}}^*(\sH,\sR) + \sM_{\ell_{\rm{def}}}(\sH,\sR) \leq 
    \ov \Gamma\paren*{\sE_{\ell_{\Phi}}(h, r) -  \sE_{\ell_{\Phi}}^*(\sH,\sR) + \sM_{\ell_{\Phi}}(\sH,\sR)},
\end{align*}
where $\ov \Gamma(t) = \max\curl*{t, \paren*{\ul + \uc}^{1 - \alpha}
  \beta\, t^{\alpha}}$.
\end{restatable}
In particular, when $\sH$ and $\sR$ are the family of all measurable
functions, all the minimizability gap terms in
Theorem~\ref{thm:single-binary} vanish. In this case, as shown by
\citet{awasthi2022Hconsistency}, $\Gamma(t)$ can be expressed as $\sqrt{2t}$
for exponential and logistic loss, $\sqrt{t}$ for quadratic loss and $t$
for hinge, sigmoid and $\rho$-margin losses. Thus, the following
result holds.
\begin{corollary}
\label{cor:single-binary}
Given a margin-based loss function $\Phi$. Assume that there exists a
function $\Gamma(t) = \beta\, t^{\alpha}$ for some $\alpha \in (0, 1]$
  and $\beta > 0$, such that the following excess error bound holds
  for all $r \in \sR_{\rm{all}}$ and any distribution,
\begin{equation*}
  \sE_{\ell^{\rm{bi}}_{0-1}}(r) - \sE^*_{\ell^{\rm{bi}}_{0-1}}(\sR_{\rm{all}})
  \leq \Gamma\paren*{\sE_{\Phi}(r) - \sE^*_{\Phi}(\sR_{\rm{all}})}.
\end{equation*}
Then, the following excess error bound holds for all $h\in
\sH_{\rm{all}}$, $r\in \sR_{\rm{all}}$ and any distribution,
\begin{align*}
   \sE_{\ell_{\rm{def}}}(h, r) - \sE_{\ell_{\rm{def}}}^*(\sH_{\rm{all}}, \sR_{\rm{all}}) \leq 
    \ov \Gamma\paren*{\sE_{\ell_{\Phi}}(h, r) -  \sE_{\ell_{\Phi}}^*(\sH_{\rm{all}},\sR_{\rm{all}})},
\end{align*}
where $\ov \Gamma(t) = \max\curl*{t, \paren*{\ul + \uc}^{1 - \alpha}
  \beta\, t^{\alpha}}$. In particular, $\ov \Gamma (t) = \max\curl*{t,
  \frac12 \paren{\ul + \uc}^{\frac12} t^{\frac 12}}$ for $\Phi =
\Phi_{\rm{exp}}$ and $\Phi_{\rm{log}}$, $\ov \Gamma (t) =
\max\curl*{t, \paren{\ul + \uc}^{\frac12} t^{\frac 12}}$ for $\Phi =
\Phi_{\rm{quad}}$, and $\ov \Gamma (t) = t$ for $\Phi =
\Phi_{\rm{hinge}}$, $\Phi_{\rm{sig}}$, and $\Phi_{\rho}$.
\end{corollary}
By taking the limit on both sides, we derive the Bayes-consistency and
excess error bound of these single-stage surrogate losses
$\ell_{\Phi}$ with respect to the deferral loss
$\ell_{\rm{def}}$. More generally, Corollary~\ref{cor:single-binary}
shows that $\ell_{\Phi}$ admits an excess error bound with respect to
$\ell_{\rm{def}}$ when $\Phi$ admits an excess error bound with
respect to $\ell_{0-1}$. Corollary~\ref{cor:single-binary} also
include the theoretical guarantees in \citep[Theorems~7 and
  8]{cheng2023regression} as a special case where the cost does not
depend on the label $y$ and the squared loss is considered.

\subsection{Two-Stage Guarantees}

Here, we present guarantees for the two-stage surrogate.
\begin{restatable}{theorem}{TwostageHRBinary}
\label{thm:tshr-binary}
Given a hypothesis set $\sH$, a regular hypothesis set $\sR$ and a
margin-based loss function $\Phi$. Assume that there exists a function
$\Gamma(t) = \beta\, t^{\alpha}$ for some $\alpha \in (0, 1]$ and
  $\beta > 0$, such that the following $\sR$-consistency bound holds
  for all $r \in \sR$ and any distribution,
\begin{align*}
\sE_{\ell^{\rm{bi}}_{0-1}}(r) - \sE^*_{\ell^{\rm{bi}}_{0-1}}(\sR) + \sM_{\ell^{\rm{bi}}_{0-1}}(\sR) \leq \Gamma\paren*{\sE_{\Phi}(r) - \sE^*_{\Phi}(\sR) + \sM_{\Phi}(\sR)}.
\end{align*}
Then, the following $(\sH,\sR)$-consistency bound holds for all $h\in
\sH$, $r\in \sR$ and any distribution,
\begin{align*}
   & \sE_{\ell_{\rm{def}}}(h, r) - \sE_{\ell_{\rm{def}}}^*(\sH,\sR) + \sM_{\ell_{\rm{def}}}(\sH,\sR)\\
   & \qquad \leq \sE_{\sfL}(h) - \sE_{\sfL}(\sH) + \sM_{\sfL}(\sH)  + \ov \Gamma\paren*{\sE_{\ell^h_{\Phi}}(r) -  \sE_{\ell^h_{\Phi}}^*(\sR) + \sM_{\ell^h_{\Phi}}(\sR)},
\end{align*}
where $\ov \Gamma(t) = \paren*{\ul + \uc}^{1 - \alpha} \beta\, t^{\alpha}$.
\end{restatable}
 As before, when $\sH$ and $\sR$ include all measurable functions, all the minimizability gap terms in Theorem~\ref{thm:tshr} vanish. In particular, $\Gamma(t)$ can be expressed as $\sqrt{2t}$ for exponential and logistic loss, $\sqrt{t}$ for quadratic loss and $t$ for hinge, sigmoid and $\rho$-margin losses \citep{awasthi2022Hconsistency}. Thus, we obtain the following result.
\begin{corollary}
\label{cor:tshr-binary}
Given a margin-based loss function $\Phi$. Assume that there exists a function $\Gamma(t) = \beta\, t^{\alpha}$ for some $\alpha \in (0, 1]$ and $\beta > 0$, such that the following excess error bound holds for all $r \in \sR_{\rm{all}}$ and any distribution,
\begin{equation*}
\sE_{\ell^{\rm{bi}}_{0-1}}(r) - \sE^*_{\ell^{\rm{bi}}_{0-1}}(\sR_{\rm{all}}) \leq \Gamma\paren*{\sE_{\Phi}(r) - \sE^*_{\Phi}(\sR_{\rm{all}})}.
\end{equation*}
Then, the following excess error bound holds for all $h\in \sH_{\rm{all}}$, $r\in \sR_{\rm{all}}$ and any distribution,
\begin{equation}
   \sE_{\ell_{\rm{def}}}(h, r) - \sE_{\ell_{\rm{def}}}^*(\sH_{\rm{all}},\sR_{\rm{all}}) \leq \sE_{\sfL}(h) - \sE_{\sfL}(\sH_{\rm{all}}) + \ov \Gamma\paren*{\sE_{\ell^h_{\Phi}}(r) -  \sE_{\ell^h_{\Phi}}^*(\sR_{\rm{all}})},
\end{equation}
where $\ov \Gamma(t) = \paren*{\ul + \uc}^{1 - \alpha} \beta\, t^{\alpha}$. In particular, $\ov \Gamma (t) = \frac12 \paren{\ul + \uc}^{\frac12} t^{\frac 12}$ for $\Phi = \Phi_{\rm{exp}}$ and $\Phi_{\rm{log}}$, $\ov \Gamma (t) = \paren{\ul + \uc}^{\frac12} t^{\frac 12}$ for $\Phi = \Phi_{\rm{quad}}$, and $\ov \Gamma (t) = t$ for $\Phi = \Phi_{\rm{hinge}}$, $\Phi_{\rm{sig}}$, and $\Phi_{\rho}$.
\end{corollary}
Corollary~\ref{cor:tshr-binary} shows that the proposed two-stage approach admits an excess error bound with respect to $\ell_{\rm{def}}$ when $\Phi$ admits an excess error bound with respect to $\ell_{0-1}$. More generally, in the cases where the minimizability gaps are zero (as when the best-in-class errors coincide with the Bayes errors), the $(\sH, \sR)$-consistency bound in Theorem~\ref{thm:tshr-binary} guarantees that when the surrogate estimation error in the first stage $\sE_{\sfL}(h) - \sE_{\sfL}(\sH)$ is minimized up to $\e_1$ and the surrogate estimation error in the second stage $\sE_{\ell^h_{\Phi}}(r) -  \sE_{\ell^h_{\Phi}}^*(\sR)$ is minimized up to $\e_2$, the target estimation error $ \sE_{\ell_{\rm{def}}}(h, r) - \sE_{\ell_{\rm{def}}}^*(\sH,\sR)$ is upper-bounded by $\e_1 + \ov \Gamma(\e_2)$.

It is worth noting that while our theoretical results are general, they can be effectively applied to derive bounds for specific loss functions. Applicable regression loss functions are those with a boundedness assumption, and any classification loss function that benefits from $\sH$-consistency bounds is suitable. Our theoretical analysis guides the selection of the loss function by considering several key factors: the functional form $\overline{\Gamma}$ of the bound, the approximation properties indicated by the minimizability gaps, the optimization advantages of each loss function, and how favorably the bounds depend on the number of experts.

\section{Experiments}
\label{sec:experiments-regdef}

In this section, we report the empirical results for our single-stage and two-stage algorithms for regression with deferral on three datasets from the UCI machine learning repository \citep{asuncion2007uci}, the \texttt{Airfoil}, \texttt{Housing} and \texttt{Concrete}, which have also been studied in \citep{cheng2023regression}.

\textbf{Setup and Metrics.} 
For each dataset, we randomly split it into a training set of 60\% examples, a validation set of 20\% examples and a test set of 20\% examples. We report results averaged over three such random splits. We adopted linear models for both the predictor $h$ and the deferral function $r$. We considered three experts
$g_1$, $g_2$ and $g_3$, each trained by feedforward neural networks
with ReLU activation functions \citep{nair2010rectified}
with one, two, and three hidden layers, respectively.
We used the Adam optimizer
\citep{kingma2014adam} with a batch size of $256$ and $2\mathord,000$ training epochs. The learning rate for all datasets is selected from $\curl*{0.01, 0.05, 0.1}$.
We adopted the squared loss as the regression loss ($\sfL = \sfL_2$).
For our single-stage surrogate loss \eqref{eq:sur} and two-stage surrogate loss \eqref{eq:two-stage-sur}, we choose $\ell = \ell_{\rm{log}}$ as the logistic loss. In the experiments, we considered two types of costs: $c_j(x, y) = L(g_j(x), y)$ and $c_j(x, y) = L(g_j(x), y) + \alpha_j$, for $1 \leq j \leq \num$. In the first case, the cost corresponds exactly to the expert's squared error. In the second case, the constant $\alpha_j$ is the base cost for deferring to expert $g_j$. We chose $(\alpha_1, \alpha_2, \alpha_3) = (4.0, 8.0, 12.0)$. We selected those values of $\alpha$ because they are empirically determined to encourage an optimal balance, ensuring a reasonable number of input instances are deferred to each expert.
For evaluation, we compute the system mean squared error (MSE), that is the average squared difference between the target value and the prediction made by the predictor $h$ or the expert selected by the deferral function $r$.\ignore{ $\frac{\sum_{i = 1}^n \ldef(h, r, x_i, y_i)}{n}$ and the deferral ratio to the $j$th expert:
$\frac{\sum_{i = 1}^n 1_{\rr(x) = j}}{n}$} We also report the empirical regression loss, $\frac{1}{n}\sum_{i = 1}^n \sfL(h(x_i), y_i)$, of the base model used in the two-stage algorithm.

\textbf{Results.}  In Table~\ref{tab:deferral-1}, we report the mean
and standard deviation of the empirical regression loss of the base
model, as well as the System MSE obtained by using a single expert
$g_1$, two experts $g_1$ and $g_2$ and three experts $g_1$, $g_2$ and
$g_3$, over three random splits of the dataset. Here, the base model
is the predictor. We did not report its performance in the
single-stage method because it is independently trained and exhibits
varying accuracies across settings with single expert, two experts,
and three experts, in contrast to the two-stage method where the base
model is pre-learned and fixed. Additionally, in the single-stage
method, the base model is always used in conjunction with deferral,
rather than being used separately.

Table~\ref{tab:deferral-1} shows that the performance of both our
single-stage and two-stage algorithms improves as more experts are
taken into account across the \texttt{Airfoil}, \texttt{Housing} and
\texttt{Concrete} datasets. In particular, our algorithms are able to
effectively defer difficult test instances to more suitable experts
and outperform the base model. Table~\ref{tab:deferral-1} also shows
that the two-stage algorithm usually does not perform as well as the
single-stage one in the regression setting, mainly due to the error
accumulation in the two-stage process, particularly if the first-stage
predictor has large errors. However, the two-stage algorithm is still
useful when an existing predictor cannot be retrained due to cost or
time constraints. In such cases, we can still improve its performance
by learning a multi-expert deferral with our two-stage surrogate
losses.

Note that since our work is the first to study regression with
multi-expert deferral, there are no existing baselines for direct
comparison. Nevertheless, for completeness, we include additional
experimental results with three simple baselines in
Appendix~\ref{app:additional_experiments}, further demonstrating the
effectiveness of our approach.

In real-life scenarios, ``cancellation effects'' might occur when
using experts with similar expertise
\citep{verma2023learning}. However, the experts we have used do not
exhibit such an effect. This is because, for each expert, there are
specific input instances that they can predict correctly, which others
cannot. Therefore, without base costs, the system's error after using
deferral is lower than that of any individual expert. Furthermore, in
our scenario, we operate under the assumption that the experts are
predefined.

Our $\sH$-consistency guarantees demonstrate that in both
single- and two-stage scenarios, given a sufficient amount of data,
our algorithms can approximate results close to the optimal deferral
loss for the given experts. Our analysis and experiments do not
directly address the process of selecting experts
beforehand. Optimally selecting diverse and accurate experts is an
interesting research question.

\section{Conclusion}
\label{sec:conclusion}

We introduced a novel and principled framework for regression with
deferral, enhancing the accuracy and reliability of regression
predictions through the strategic use of multiple experts. Our
comprehensive analysis of this framework includes the formulation of
novel surrogate losses for both single-stage and two-stage scenarios,
and the proof of strong $\sH$-consistency bounds. These theoretical
guarantees lead to powerful algorithms that leverage multiple experts,
providing a powerful tool for addressing the inherent challenges of
regression problems.  Empirical results validate the effectiveness of
our approach, showcasing its practical significance and opening up new
avenues for developing robust solutions across diverse regression
tasks.

\chapter*{Conclusion} \label{chp-conclusion}
\addcontentsline{toc}{chapter}{Conclusion}

In this thesis, we presented a comprehensive study of learning with multi-class abstention and multi-expert deferral, supported by strong consistency guarantees. 

Our detailed study of score-based and  predictor-rejector multi-class abstention introduced novel surrogate loss families with strong hypothesis set-specific and non-asymptotic theoretical guarantees, providing resolutions to two open questions within the literature. 
Our theoretical analysis, including proofs of $(\sH, \sR)$-consistency
bounds and realizable $(\sH, \sR)$-consistency, covers both single-stage
and two-stage surrogate losses.
These results further provide valuable tools applicable to the analysis of other loss functions in learning with abstention.
Empirical results demonstrate the practical advantage of our proposed surrogate losses
and their derived algorithms. This work establishes a powerful
framework for designing new, more reliable abstention-aware algorithms
applicable across diverse domains.

We presented a series of theoretical, algorithmic, and empirical results for the core
challenge of learning with multi-expert deferral in the classification setting for both the single-stage and two-stage scenarios. In particular, the two-stage approach proves particularly advantageous in scenarios where a large pre-trained prediction model is readily available, and the expense associated with retraining is prohibitive.
We proved that the novel surrogate losses we introduced are supported by $\sH$-consistency bounds and established their realizable
$\sH$-consistency properties for a constant cost function. This work
paves the way for comparing different surrogate losses and cost
functions within our framework.  Further exploration, both
theoretically and empirically, holds the potential to identify optimal
choices for these quantities across diverse tasks.

We finally introduced a novel and principled framework for regression with
deferral, enhancing the accuracy and reliability of regression
predictions through the strategic use of multiple experts. Our
comprehensive analysis of this framework includes the formulation of
novel surrogate losses for both single-stage and two-stage scenarios,
and the proof of strong $\sH$-consistency bounds. These theoretical
guarantees lead to powerful algorithms that leverage multiple experts,
providing a powerful tool for addressing the inherent challenges of
regression problems.  Empirical results validate the effectiveness of
our approach, showcasing its practical significance and opening up new
avenues for developing robust solutions across diverse regression
tasks.

This thesis offers valuable avenues for future work, which can be grouped into three key areas: \emph{theoretical and algorithmic}, \emph{application-focused}, and \emph{interdisciplinary}.

First, in the theoretical and algorithmic area, an intriguing direction is to further study surrogate loss functions for learning with multi-expert deferral. This includes proposing new surrogate loss functions for \emph{any general cost function} that achieve Bayes-consistency, realizable $\sH$-consistency, and $\sH$-consistency bounds \emph{simultaneously}, leading to new algorithms that perform better in realizable scenarios and remain comparable in non-realizable scenarios \citep{mao2024realizable,MaoMohriZhong2025mastering}. Another avenue of exploration is extending multi-expert deferral and multi-class abstention to the online setting, where the learner can decide to defer or abstain \emph{sequentially} \citep{cortes2018online}, as well as to the budgeted setting, where the learner has access to only \emph{partial} or \emph{bandit} information about the experts' predictions \citep{foster2020adapting,desalvo2025budgeted}.

Second, in the application-focused area, our deferral algorithms and their extensions show promising potential for large language models (LLMs), particularly in enhancing their \emph{factuality} and \emph{efficiency}. Regarding factuality, LLMs are susceptible to generating erroneous information, often referred to as hallucinations. Their response quality can be significantly improved by deferring uncertain predictions to more advanced or domain-specific pre-trained models, ensuring greater factuality.
In terms of efficiency, relying exclusively on the largest available model is often prohibitively costly. For faster inference, smaller models can be used when they suffice for a given task. A well-designed deferral algorithm can efficiently balance performance and cost, achieving high accuracy while minimizing the number of instances deferred to larger, more resource-intensive models.

Third, in the interdisciplinary area, extending our study of multi-class abstention and multi-expert deferral beyond standard settings offers opportunities to benefit a variety of other learning scenarios. One intriguing direction involves enhancing adversarial robustness through abstention or deferral \citep{chen2023stratified}, where perturbed instances can be abstained from or deferred to multiple models trained with different perturbation sizes, thereby improving both model robustness and generalization. Another promising avenue is integrating abstention or deferral into active learning \citep{zhu2022efficient}, where an active learning algorithm can abstain from making predictions or defer them to designated models or experts for challenging examples, thereby improving label complexity and enabling more efficient learning.


\appendix

\chapter{Appendix to Chapter~\ref{ch2}}

\disableatoc
\section{Discussion on Experiments}
\label{app:experimemts}

This section presents a detailed analysis of the experimental results.

For CIFAR-10, the two-stage score-based abstention surrogate loss
outperforms the cross-entropy scored-based abstention surrogate loss
($\mu = 1.0$) used in \citep{mozannar2020consistent} by 1.26\%, and
outperforms the cross-entropy scored-based abstention surrogate loss
($\mu = 1.7$) used in \citep{caogeneralizing} by 0.4\%. Our results
for the score-based surrogate losses are also consistent with those of
\citet{caogeneralizing}, who showed that the scored-based abstention
loss \eqref{eq:sur-score-mabsc} with $\ell_{\mu}$ adopted as the generalized
cross-entropy loss ($\mu=1.7$) performs better than the scored-based
abstention loss with $\ell_{\mu}$ adopted as the logistic loss
($\mu=1$). This agrees with our theoretical analysis based on
$\sH$-consistency bounds and minimizability gaps in
Theorem~\ref{Thm:bound_comp_sum} and
Theorem~\ref{Thm:gap-upper-bound-determi}, since both losses have the
same square-root functional form while the magnitude of the
minimizability gap decreases with $\mu$ in light of the fact that
$\sE_{\lsc_{\mu}}^*(\sH)$ is close for both losses.

Table~\ref{tab:comparison} also shows that on SVHN, using deeper
neural networks than \citep{caogeneralizing}, the cross-entropy
scored-based abstention loss ($\mu=1.7$) actually performs worse than
the cross-entropy scored-based abstention loss ($\mu=1$) in
\citep{mozannar2020consistent}, in contrast with the opposite results
observed in \citep{caogeneralizing} when using shallower neural
networks. This is consistent with our theoretical analysis based on
their $\sH$-consistency bounds (Theorem~\ref{Thm:bound_comp_sum}): the
minimizability gaps are basically the same while the dependency of the
multiplicative constant on the number of classes appears for
$\mu=1.7$, which makes the scored-based abstention loss
\eqref{eq:sur-score-mabsc} with $\ell_{\mu}$ adopted as the generalized
cross-entropy loss ($\mu=1.7$) less favorable. Here too, the two-stage
score-based abstention surrogate loss is superior to both, with an
abstention loss 1.23\% lower than that of \citep{caogeneralizing} and
0.68\% lower than that of \citep{mozannar2020consistent}.

To further test the algorithms, we also carried out experiments on
CIFAR-100, with deeper neural networks. Table~\ref{tab:comparison}
shows that score-based abstention loss with generalized cross-entropy
adopted in \citep{caogeneralizing} does not perform well in this
case. In contrast, the score-based abstention loss with the logistic
loss adopted in \citep{mozannar2020consistent} performs better and
surpasses it by 4.59\%. Our two-stage score-based abstention loss is
still the most favorable, here too, with 0.86\% lower abstention loss
than that of \citep{mozannar2020consistent}. As with the case of SVHN,
the inferior performance of the cross-entropy scored-based abstention
surrogate loss ($\mu = 1.7$) can be seen from the dependency of the
multiplicative constant on the number of classes in $\sH$-consistency
bounds (Theorem~\ref{Thm:bound_comp_sum}), which is worse when the
number of classes is much larger as in the case of CIFAR-100.

\section{Proofs for Score-Based Abstention Losses}
\label{app:score-mabsc}

To begin with the proof, we first introduce some notation. Recall
that we denote by $\sfp(y \!\mid\! x) = \sD(Y = y \!\mid\! X = x)$ the
conditional probability of $Y=y$ given $X = x$. For simplicity of the
notation, we let $\sfp(n + 1 \!\mid\! x) = 1 - c$ and denote by $ y_{\max}\in \sY \bigcup \curl*{n+1}$ the label associated to an input $x\in
\sX$, defined as $ y_{\max}=n+1$ if $1-c \geq \max_{y \in \sY} \sfp(y \!\mid\! x)$; otherwise, $ y_{\max}$ is defined as an element in $\sY$ with
the highest conditional probability, $ y_{\max} = \argmax_{y \in
  \sY} \sfp(y \!\mid\! x)$, with the same deterministic strategy for breaking
ties as that of $ \hh(x)$.  Thus, the generalization error for a
score-based abstention surrogate loss can be rewritten as $
\sE_{\lsc}( h) = \mathbb{E}_{X} \bracket*{\sC_{\lsc}( h, x)} $,
where $\sC_{\lsc}( h,x)$ is the conditional $\lsc$-risk, defined by
\begin{align*}
\sC_{\lsc}( h,x) = \sum_{y\in \sY\bigcup \curl*{n+1}} \sfp(y \!\mid\! x)  \ell( h,x, y).
\end{align*}
We denote by $\sC_{\lsc}^*( \sH,x) = \inf_{ h\in
  \sH}\sC_{\lsc}( h,x)$ the minimal conditional
$\lsc$-risk. Then, the minimizability gap can be rewritten as follows:
\begin{align*}
\sM_{\lsc}( \sH)
 = \sE^*_{\lsc}( \sH) - \mathbb{E}_{X} \bracket* {\sC_{\lsc}^*( \sH, x)}.
\end{align*}
We further refer to $\sC_{\lsc}( h,x)-\sC_{\lsc}^*( \sH,x)$ as
the calibration gap and denote it by $\Delta\sC_{\lsc, \sH}(
h,x)$.  We first prove a lemma on the calibration gap of the
score-based abstention loss. For any $x \in \sX$, we will denote by
$\mathsf H(x)$ the set of labels generated by hypotheses in $ \sH$:
$\mathsf H(x) = \curl*{ \hh(x) \colon  h \in  \sH}$.
\begin{restatable}{lemma}{CalibrationGapScoreMabsc}
\label{lemma:calibration_gap_score-mabsc}
For any $x \in \sX$,
the minimal conditional $\labsc$-risk and
the calibration gap for $\labsc$ can be expressed as follows:
\begin{align*}
\sC^*_{\labsc}( \sH,x) & = 1 - \max_{y\in \mathsf H(x)} \sfp(y \!\mid\! x)\\
\Delta\sC_{\labsc, \sH}( h,x) & = \max_{y\in \mathsf H(x)} \sfp(y \!\mid\! x) - \sfp(\hh(x) \!\mid\! x).
\end{align*}
\end{restatable}
\begin{proof}
The conditional
$\labsc$-risk of $ h$ can be expressed as follows:
\begin{align*}
\sC_{\labsc}( h, x)
 = \sum_{y\in \sY}\sfp(y \!\mid\! x)\1_{ \hh(x)\neq y}\1_{ \hh(x)\neq n + 1}+c \1_{ \hh(x) = n + 1}=1-\sfp(\hh(x) \!\mid\! x).
\end{align*}
Then, the minimal conditional $\labsc$-risk is given by
\[
\sC_{\labsc}^*(\sH,x) = 1 - \max_{y\in \mathsf H(x)} \sfp(y \!\mid\! x),
\]
and the calibration gap can be expressed as follows:
\begin{align*}
  \Delta \sC_{\labsc,\sH}( h, x)
  = \sC_{\labsc}( h, x)-\sC_{\labsc}^*(\sH,x)= \max_{y\in \mathsf H(x)} \sfp(y \!\mid\! x)-\sfp(\hh(x) \!\mid\! x).
\end{align*}
This completes the proof.
\end{proof}

Note that when $ \sH$ is symmetric, $\mathsf H(x)= \sY \bigcup \curl*{n+1}$. By
Lemma~\ref{lemma:calibration_gap_score-mabsc}, in those cases, we obtain the
following result,
\begin{corollary}
\label{cor:calibration_gap_score}
Assume that $ \sH$ is symmetric. Then, for any $x \in \sX$,
the minimal conditional $\labsc$-risk and
the calibration gap for $\labsc$ can be expressed as follows:
\begin{align*}
\sC^*_{\labsc}( \sH,x) & =  1- \sfp(y_{\max} \!\mid\! x)\\
\Delta\sC_{\labsc, \sH}( h,x) & =  \sfp(y_{\max} \!\mid\! x) - \sfp(\hh(x) \!\mid\! x).
\end{align*}
\end{corollary}

\subsection{Proof of \texorpdfstring{$\sH$}{H}-Consistency bounds
  for Cross-Entropy Score-Based Surrogates (Theorem~\ref{Thm:bound_comp_sum})}
\label{app:bound_comp_sum}
\BoundCompSum*
\begin{proof}
The main proof idea is similar for each case of $\mu$: we will lower
bound the calibration gap of $\lsc_{\mu}$ by that of $\labsc$ by
carefully selecting a hypothesis $ h_{\lambda}$ in the hypothesis
set $ \sH$. In particular, we analyze different cases as follows.
\paragraph{The Case Where $\mu \in [0,1)$} 

For any $ h \in \sH$ and $x\in \sX$, choose hypothesis $
h_{\lambda} \in  \sH$ such that
\begin{align*}
 h_{\lambda}(x, y) = 
\begin{cases}
   h(x, y) & \text{if $y \not \in \curl*{ y_{\max},  \hh(x)}$}\\
  \log\paren*{\exp\bracket*{ h(x, y_{\max})} + \lambda} & \text{if $y =  \hh(x)$}\\
  \log\paren*{\exp\bracket*{ h(x, \hh(x))} -\lambda} & \text{if $y = y_{\max}$},
\end{cases} 
\end{align*}
where $\lambda = \frac{\exp\bracket*{ h(x, \hh(x))}\sfp(\hh(x) \!\mid\! x)^{\frac1{2-\mu}}-\exp\bracket*{ h(x, y_{\max})}\sfp(y_{\max} \!\mid\! x)^{\frac1{2-\mu}}}{\sfp(y_{\max} \!\mid\! x)^{\frac1{2-\mu}}+\sfp(\hh(x) \!\mid\! x)^{\frac1{2-\mu}}}$. The existence of such a $ h_{\lambda}$
in the hypothesis set $ \sH$ is guaranteed by the assumption that $ \sH$
is symmetry and complete. Thus, the calibration gap can be expressed
and lower-bounded as follows:
\begin{align*}
&(1-\mu)\Delta \sC_{\lsc_{\mu}, \sH}( h, x)\\
& = (1-\mu)\paren*{\sC_{\lsc_{\mu}}( h, x) - \sC^*_{\lsc_{\mu}}( \sH, x)}\\
& \geq (1-\mu)\paren*{\sC_{\lsc_{\mu}}( h, x) - \sC_{\lsc_{\mu}}( h_{\lambda}, x)}\\
& = \sfp(y_{\max} \!\mid\! x) \paren*{\bracket*{\sum_{y'\in  \sY\bigcup \curl*{n+1}} e^{ h(x, y')- h(x, y_{\max})}}^{1 - \mu }-1}\\
&\qquad +\sfp(\hh(x) \!\mid\! x) \paren*{\bracket*{\sum_{y'\in  \sY\bigcup \curl*{n+1}} e^{ h(x, y')- h(x, \hh(x))}}^{1 - \mu }-1}\\
&  \qquad \quad -\sfp(y_{\max} \!\mid\! x) \paren*{\bracket*{\sum_{y'\in  \sY\bigcup \curl*{n+1}}e^{ h(x, y')- h(x, \hh(x))+\lambda}}^{1 - \mu }-1}\\
&\qquad \quad \quad -\sfp(\hh(x) \!\mid\! x) \paren*{\bracket*{\sum_{y'\in  \sY\bigcup \curl*{n+1}}e^{ h(x, y')- h(x, y_{\max})-\lambda}}^{1 - \mu }-1}\\
& = \sfp(y_{\max} \!\mid\! x)\bracket*{\sum_{y'\in  \sY\bigcup \curl*{n+1}} e^{ h(x, y')-e^{ h(x, y_{\max})}}}^{1 - \mu }\\
&\qquad - \sfp(y_{\max} \!\mid\! x)\bracket*{\frac{\sum_{y'\in  \sY\bigcup \curl*{n+1}}e^{ h(x, y')}\bracket*{\sfp(y_{\max} \!\mid\! x)^{\frac1{2 - \mu }}+\sfp(\hh(x) \!\mid\! x)^{\frac1{2 - \mu }}}}{\bracket*{e^{ h(x, y_{\max})} + e^{ h(x, \hh(x))}}\sfp(y_{\max} \!\mid\! x)^{\frac1{2 - \mu }}}}^{1 - \mu }\\
&+\sfp(\hh(x) \!\mid\! x)\bracket*{\sum_{y'\in  \sY\bigcup \curl*{n+1}} e^{ h(x, y')- h(x, \hh(x))}}^{1 - \mu }\\
&\qquad - \sfp(\hh(x) \!\mid\! x)\bracket*{\frac{\sum_{y'\in  \sY\bigcup \curl*{n+1}}e^{ h(x, y')}\bracket*{\sfp(y_{\max} \!\mid\! x)^{\frac1{2 - \mu }}+\sfp(\hh(x) \!\mid\! x)^{\frac1{2 - \mu }}}}{\bracket*{e^{ h(x, y_{\max})} + e^{ h(x, \hh(x))}}\sfp(\hh(x) \!\mid\! x)^{\frac1{2 - \mu }}}}^{1 - \mu }.
\end{align*}
Then, by using the fact that $\sum_{y'\in  \sY\bigcup \curl*{n+1}} e^{ h(x, y')}\geq e^{ h(x, \hh(x))}+e^{ h(x, y_{\max})}$, we further have
\begin{align*}
&(1-\mu)\Delta \sC_{\lsc_{\mu}, \sH}( h, x)\\
&\geq \sfp(y_{\max} \!\mid\! x)\bracket*{e^{ h(x, \hh(x))- h(x, y_{\max})}+1}^{1 - \mu } - \sfp(y_{\max} \!\mid\! x)\bracket*{\frac{\sfp(y_{\max} \!\mid\! x)^{\frac1{2 - \mu }}+\sfp(\hh(x) \!\mid\! x)^{\frac1{2 - \mu }}}{\sfp(y_{\max} \!\mid\! x)^{\frac1{2 - \mu }}}}^{1 - \mu }\\
&+\sfp(\hh(x) \!\mid\! x)\bracket*{e^{ h(x, y_{\max})- h(x, \hh(x))}+1}^{1 - \mu } - \sfp(\hh(x) \!\mid\! x)\bracket*{\frac{\sfp(y_{\max} \!\mid\! x)^{\frac1{2 - \mu }}+\sfp(\hh(x) \!\mid\! x)^{\frac1{2 - \mu }}}{\sfp(\hh(x) \!\mid\! x)^{\frac1{2 - \mu }}}}^{1 - \mu }
\tag{$\sum_{y'\in  \sY\bigcup \curl*{n+1}} e^{ h(x, y')}\geq e^{ h(x, \hh(x))}+e^{ h(x, y_{\max})}$}\\
&\geq \sfp(y_{\max} \!\mid\! x)2^{1 - \mu } - \sfp(y_{\max} \!\mid\! x)^{\frac1{2 - \mu }} \bracket*{\sfp(y_{\max} \!\mid\! x)^{\frac1{2 - \mu }}+\sfp(\hh(x) \!\mid\! x)^{\frac1{2 - \mu }}}^{1 - \mu }\\
&+\sfp(\hh(x) \!\mid\! x)2^{1 - \mu } - \sfp(\hh(x) \!\mid\! x)^{\frac1{2 - \mu }}\bracket*{\sfp(y_{\max} \!\mid\! x)^{\frac1{2 - \mu }}+\sfp(\hh(x) \!\mid\! x)^{\frac1{2 - \mu }}}^{1 - \mu }
\tag{minimum is attained when $e^{ h(x, \hh(x))}=e^{ h(x, y_{\max})}$}\\
& = 2^{1-\mu}\paren*{\sfp(y_{\max} \!\mid\! x)+\sfp(\hh(x) \!\mid\! x)}-\bracket*{\sfp(y_{\max} \!\mid\! x)^{\frac1{2 - \mu }}+\sfp(\hh(x) \!\mid\! x)^{\frac1{2 - \mu }}}^{2 - \mu }\\
& = 2^{2-\mu}\bracket*{\paren*{\frac{\sfp(y_{\max} \!\mid\! x)+\sfp(\hh(x) \!\mid\! x)}{2}}-\bracket*{\frac{\sfp(y_{\max} \!\mid\! x)^{\frac1{2 - \mu }}+\sfp(\hh(x) \!\mid\! x)^{\frac1{2 - \mu }}}{2}}^{2 - \mu }}\\
& \geq \frac{1-\mu}{(2-c)2^{\mu}(2-\mu)}\paren*{\sfp(y_{\max} \!\mid\! x) - \sfp(\hh(x) \!\mid\! x)}^2
\tag{$\sfp(y_{\max} \!\mid\! x)+\sfp(\hh(x) \!\mid\! x)\leq 2-c$ and by analyzing the Taylor expansion}\\
& = \frac{1-\mu}{(2-c)2^{\mu}(2-\mu)}\Delta\sC_{\labsc, \sH}( h,x)^2 \tag{Corollary~\ref{cor:calibration_gap_score}}
\end{align*}
Thus, we have
\begin{align*}
\sE_{\labsc}( h) - \sE_{\labsc}^*( \sH) + \sM_{\labsc}( \sH)
& = \E_{X}\bracket*{\Delta \sC_{\labsc,\sH}( h, x)}\\
& \leq \E_X\bracket*{\Gamma_{\mu}\paren*{\Delta \sC_{\lsc_{\mu},\sH}( h, x)}}\\
& \leq \Gamma_{\mu}\paren*{\E_X\bracket*{\Delta \sC_{\lsc_{\mu},\sH}( h, x)}}
\tag{$\Gamma_{\mu}$ is concave}\\
& = \Gamma_{\mu}\paren*{\sE_{\lsc_{\mu}}( h)-\sE_{\lsc_{\mu}}^*( \sH) +\sM_{\lsc_{\mu}}( \sH)},
\end{align*}
where $\Gamma_{\mu}(t)=\sqrt{(2-c)2^{\mu}(2-\mu) t}$.

\paragraph{The Case Where $\mu =1$} 

For any $ h \in \sH$ and $x\in \sX$, choose hypothesis $ h_{\lambda} \in  \sH$ such that
\begin{align*}
 h_{\lambda}(x, y) = 
\begin{cases}
   h(x, y) & \text{if $y \not \in \curl*{ y_{\max},  \hh(x)}$}\\
  \log\paren*{\exp\bracket*{ h(x, y_{\max})} + \lambda} & \text{if $y =  \hh(x)$}\\
  \log\paren*{\exp\bracket*{ h(x, \hh(x))} -\lambda} & \text{if $y = y_{\max}$}
\end{cases} 
\end{align*}
where $\lambda = \frac{\exp\bracket*{ h(x, \hh(x))}\sfp(\hh(x) \!\mid\! x)-\exp\bracket*{ h(x,  y_{\max})}\sfp(y_{\max} \!\mid\! x)}{\sfp(y_{\max} \!\mid\! x)+\sfp(\hh(x) \!\mid\! x)}$. The existence of
such a $ h_{\lambda}$ in hypothesis set $ \sH$ is guaranteed by
the fact that $ \sH$ is symmetry and complete. Thus, the
calibration gap can be expressed and lower-bounded as follows:
\begin{align*}
& \Delta \sC_{\lsc_{\mu}, \sH}( h, x)\\
& = \sC_{\lsc_{\mu}}( h, x) - \sC^*_{\lsc_{\mu}}( \sH, x)\\
& \geq \sC_{\lsc_{\mu}}( h, x) - \sC_{\lsc_{\mu}}( h_{\lambda}, x)\\
& =-\sfp(y_{\max} \!\mid\! x) \log\bracket*{e^{ h(x, y_{\max})}}-\sfp(\hh(x) \!\mid\! x) \log\bracket*{e^{ h(x, \hh(x))}}\\
& \qquad +\sfp(y_{\max} \!\mid\! x)\log\bracket*{ e^{ h(x, \hh(x))}-\lambda}+\sfp(\hh(x) \!\mid\! x)\log\bracket*{e^{ h(x, y_{\max})}+\lambda}\\
& = \sfp(y_{\max} \!\mid\! x)\log\bracket*{\frac{\bracket*{e^{ h(x, y_{\max})} + e^{ h(x, \hh(x))}}\sfp(y_{\max} \!\mid\! x)}{e^{ h(x, y_{\max})}\bracket*{\sfp(y_{\max} \!\mid\! x) + \sfp(\hh(x) \!\mid\! x)}}}\\ 
&\qquad + \sfp(\hh(x) \!\mid\! x)\log\bracket*{\frac{\bracket*{e^{ h(x, y_{\max})} + e^{ h(x, \hh(x))}}\sfp(\hh(x) \!\mid\! x)}{ e^{ h(x, \hh(x))}\bracket*{\sfp(y_{\max} \!\mid\! x)+\sfp(\hh(x) \!\mid\! x)}}}\\
&\geq \sfp(y_{\max} \!\mid\! x)\log\bracket*{\frac{2\sfp(y_{\max} \!\mid\! x)}{\sfp(y_{\max} \!\mid\! x)+\sfp(\hh(x) \!\mid\! x)}} + \sfp(\hh(x) \!\mid\! x)\log\bracket*{\frac{2\sfp(\hh(x) \!\mid\! x)}{\sfp(y_{\max} \!\mid\! x)+\sfp(\hh(x) \!\mid\! x)}}
\tag{minimum is attained when $e^{ h(x, \hh(x))}=e^{ h(x, y_{\max})}$}.
\end{align*}
Then, by using the Pinsker’s inequality \citep[Proposition~E.7]{MohriRostamizadehTalwalkar2018}, we further have
\begin{align*}
& \Delta \sC_{\lsc_{\mu}, \sH}( h, x)\\
& \geq \bracket*{\sfp(y_{\max} \!\mid\! x)+\sfp(\hh(x) \!\mid\! x)}\\
&\qquad \times \frac12\bracket*{ \abs*{\frac{\sfp(y_{\max} \!\mid\! x)}{\sfp(y_{\max} \!\mid\! x)+\sfp(\hh(x) \!\mid\! x)}-\frac12}+\abs*{\frac{\sfp(\hh(x) \!\mid\! x)}{\sfp(y_{\max} \!\mid\! x)+\sfp(\hh(x) \!\mid\! x)}-\frac12}}^2
\tag{Pinsker’s inequality \citep[Proposition~E.7]{MohriRostamizadehTalwalkar2018}}\\
& = \bracket*{\sfp(y_{\max} \!\mid\! x)+\sfp(\hh(x) \!\mid\! x)} \times \frac12 \bracket*{\frac{\sfp(y_{\max} \!\mid\! x)-\sfp(\hh(x) \!\mid\! x)}{\sfp(y_{\max} \!\mid\! x)+\sfp(\hh(x) \!\mid\! x)}}^2
\tag{$\sfp(y_{\max} \!\mid\! x)\geq \sfp(\hh(x) \!\mid\! x)$}\\\\
& \geq \frac1{2(2-c)} \paren*{\sfp(y_{\max} \!\mid\! x) - \sfp(\hh(x) \!\mid\! x)}^2\\
\tag{$\sfp(y_{\max} \!\mid\! x)+\sfp(\hh(x) \!\mid\! x)\leq 2-c$}\\
& = \frac1{2(2-c)}\Delta\sC_{\labsc, \sH}( h,x)^2 \tag{Corollary~\ref{cor:calibration_gap_score}}
\end{align*}
Thus, we have
\begin{align*}
\sE_{\labsc}( h) - \sE_{\labsc}^*( \sH) + \sM_{\labsc}( \sH)
& = \E_{X}\bracket*{\Delta \sC_{\labsc,\sH}( h, x)}\\
& \leq \E_X\bracket*{\Gamma_{\mu}\paren*{\Delta \sC_{\lsc_{\mu},\sH}( h, x)}}\\
& \leq \Gamma_{\mu}\paren*{\E_X\bracket*{\Delta \sC_{\lsc_{\mu},\sH}( h, x)}}
\tag{$\Gamma_{\mu}$ is concave}\\
& = \Gamma_{\mu}\paren*{\sE_{\lsc_{\mu}}( h)-\sE_{\lsc_{\mu}}^*( \sH) +\sM_{\lsc_{\mu}}( \sH)},
\end{align*}
where $\Gamma_{\mu}(t)=\sqrt{2(2-c)t }$.

\paragraph{The Case Where $\mu \in [2,+ \infty)$} 

For any $ h \in \sH$ and $x\in \sX$, choose hypothesis $ h_{\lambda} \in  \sH$ such that
\begin{align*}
 h_{\lambda}(x, y) = 
\begin{cases}
   h(x, y) & \text{if $y \not \in \curl*{ y_{\max},  \hh(x)}$}\\
  \log\paren*{\exp\bracket*{ h(x, y_{\max})} + \lambda} & \text{if $y =  \hh(x)$}\\
  \log\paren*{\exp\bracket*{ h(x, \hh(x))} -\lambda} & \text{if $y = y_{\max}$}
\end{cases} 
\end{align*}
where $\lambda = -\exp\bracket*{ h(x, y_{\max})}$. The existence of
such a $ h_{\lambda}$ in hypothesis set $ \sH$ is guaranteed by
the fact that $ \sH$ is symmetry and complete. Thus, the
calibration gap can be expressed and lower-bounded as follows:
\begin{align*}
& (\mu-1)\Delta \sC_{\lsc_{\mu}, \sH}( h, x)\\
& = (\mu-1) \paren*{ \sC_{\lsc_{\mu}}( h, x) - \sC^*_{\lsc_{\mu}}( \sH, x)}\\
& \geq (\mu-1)\paren*{\sC_{\lsc_{\mu}}( h, x) - \sC_{\lsc_{\mu}}( h_{\lambda}, x)}\\
& =\sfp(y_{\max} \!\mid\! x) \paren*{1-\bracket*{\frac{e^{ h(x, y_{\max})}}{\sum_{y'\in  \sY\bigcup \curl*{n+1}}e^{ h(x, y')}}}^{\mu-1}} +\sfp(\hh(x) \!\mid\! x) \paren*{1-\bracket*{\frac{e^{ h(x, \hh(x))}}{\sum_{y'\in  \sY\bigcup \curl*{n+1}}e^{ h(x, y')}}}^{\mu-1}}\\
&  -\sfp(y_{\max} \!\mid\! x) \paren*{1-\bracket*{\frac{e^{ h(x, \hh(x))}-\mu}{\sum_{y'\in  \sY\bigcup \curl*{n+1}}e^{ h(x, y')}}}^{\mu-1}} -\sfp(\hh(x) \!\mid\! x) \paren*{1-\bracket*{\frac{e^{ h(x, y_{\max})}+\mu}{\sum_{y'\in  \sY\bigcup \curl*{n+1}}e^{ h(x, y')}}}^{\mu-1}}\\
& = \sfp(y_{\max} \!\mid\! x)\bracket*{\frac{e^{ h(x, \hh(x))}+e^{ h(x, y_{\max})}}{\sum_{y'\in  \sY\bigcup \curl*{n+1}}e^{h(x, y')}}}^{\mu-1}-\sfp(y_{\max} \!\mid\! x)\bracket*{\frac{e^{ h(x, y_{\max})}}{\sum_{y'\in  \sY\bigcup \curl*{n+1}}e^{ h(x, y')}}}^{\mu-1}\\
& \qquad -\sfp(\hh(x) \!\mid\! x)\bracket*{\frac{e^{ h(x, \hh(x))}}{\sum_{y'\in  \sY\bigcup \curl*{n+1}}e^{ h(x, y')}}}^{\mu-1}\\
&\geq \sfp(y_{\max} \!\mid\! x)\bracket*{\frac{e^{ h(x, \hh(x))}}{\sum_{y'\in  \sY\bigcup \curl*{n+1}}e^{ h(x, y')}}}^{\mu-1}-\sfp(\hh(x) \!\mid\! x)\bracket*{\frac{e^{ h(x, \hh(x))}}{\sum_{y'\in  \sY\bigcup \curl*{n+1}}e^{ h(x, y')}}}^{\mu-1}
\tag{$(x+y)^{\mu-1}\geq x^{\mu -1 } + y^{\mu-1}$, $\forall\, x, y\geq 0$, $\mu\geq 2$}\\
&\geq \frac{1}{(n+1)^{\mu-1}}
\paren*{\sfp(y_{\max} \!\mid\! x) - \sfp(\hh(x) \!\mid\! x)} \tag{$\frac{e^{ h(x, \hh(x))}}{\sum_{y'\in  \sY\bigcup \curl*{n+1}}e^{h(x, y')}}\geq \frac1{n+1}$}\\
& = \frac{1}{(n+1)^{\mu-1}}\Delta\sC_{\labsc, \sH}( h,x) \tag{Corollary~\ref{cor:calibration_gap_score}}
\end{align*}
Thus, we have
\begin{align*}
\sE_{\labsc}( h) - \sE_{\labsc}^*( \sH) + \sM_{\labsc}( \sH)
& = \E_{X}\bracket*{\Delta \sC_{\labsc,\sH}( h, x)}\\
& \leq \E_X\bracket*{\Gamma_{\mu}\paren*{\Delta \sC_{\lsc_{\mu},\sH}( h, x)}}\\
& \leq \Gamma_{\mu}\paren*{\E_X\bracket*{\Delta \sC_{\lsc_{\mu},\sH}( h, x)}}
\tag{$\Gamma_{\mu}$ is concave}\\
& = \Gamma_{\mu}\paren*{\sE_{\lsc_{\mu}}( h)-\sE_{\lsc_{\mu}}^*( \sH) +\sM_{\lsc_{\mu}}( \sH)},
\end{align*}
where $\Gamma_{\mu}(t)=(\mu - 1)(n+1)^{\mu - 1} t$.

\paragraph{The Case Where $\mu \in (1,2)$}

For any $ h \in \sH$ and $x\in \sX$, choose hypothesis $ h_{\lambda} \in  \sH$ such that
\begin{align*}
 h_{\lambda}(x, y) = 
\begin{cases}
   h(x, y) & \text{if $y \not \in \curl*{ y_{\max},  \hh(x)}$}\\
  \log\paren*{\exp\bracket*{ h(x, y_{\max})} + \lambda} & \text{if $y =  \hh(x)$}\\
  \log\paren*{\exp\bracket*{ h(x, \hh(x))} -\lambda} & \text{if $y = y_{\max}$}
\end{cases} 
\end{align*}
where $\lambda = \frac{\exp\bracket*{ h(x, \hh(x))}\sfp(y_{\max} \!\mid\! x)^{\frac1{\mu-2}}-\exp\bracket*{ h(x,
    y_{\max})}\sfp(\hh(x) \!\mid\! x)^{\frac1{\mu-2}}}{\sfp(y_{\max} \!\mid\! x)^{\frac1{\mu-2}}+\sfp(\hh(x) \!\mid\! x)^{\frac1{\mu-2}}}$. The
existence of such a $ h_{\lambda}$ in hypothesis set $ \sH$ is
guaranteed by the fact that $ \sH$ is symmetry and complete. Thus,
the calibration gap can be lower-bounded as follows:
\begin{align*}
& (\mu-1)\Delta \sC_{\lsc_{\mu}, \sH}( h, x)\\
& = (\mu-1)\paren*{\sC_{\lsc_{\mu}}( h, x) - \sC^*_{\lsc_{\mu}}( \sH, x)}\\
& \geq (\mu-1)\paren*{\sC_{\lsc_{\mu}}( h, x) - \sC_{\lsc_{\mu}}( h_{\lambda}, x)}\\
& = \sfp(y_{\max} \!\mid\! x) \paren*{1-\bracket*{\sum_{y'\in  \sY\bigcup \curl*{n+1}} e^{ h(x, y')- h(x, y_{\max})}}^{1 - \mu }}\\ &\qquad +\sfp(\hh(x) \!\mid\! x) \paren*{1-\bracket*{\sum_{y'\in  \sY\bigcup \curl*{n+1}} e^{ h(x, y')- h(x, \hh(x))}}^{1 - \mu }}\\
& \qquad \quad -\sfp(y_{\max} \!\mid\! x) \paren*{1-\bracket*{\sum_{y'\in  \sY\bigcup \curl*{n+1}}e^{ h(x, y')- h(x, \hh(x))+\lambda}}^{1 - \mu }}\\
&\qquad \quad \quad -\sfp(\hh(x) \!\mid\! x) \paren*{1-\bracket*{\sum_{y'\in  \sY\bigcup \curl*{n+1}}e^{ h(x, y')- h(x, y_{\max})-\lambda}}^{1 - \mu }}\\
& = -\sfp(y_{\max} \!\mid\! x)\bracket*{\sum_{y'\in  \sY\bigcup \curl*{n+1}} e^{ h(x, y')-e^{ h(x, y_{\max})}}}^{1 - \mu }\\
&\qquad + \sfp(y_{\max} \!\mid\! x)\bracket*{\frac{\sum_{y'\in  \sY\bigcup \curl*{n+1}}e^{ h(x, y')}\bracket*{\sfp(y_{\max} \!\mid\! x)^{\frac1{\mu-2}}+\sfp(\hh(x) \!\mid\! x)^{\frac1{\mu-2 }}}}{\bracket*{e^{ h(x, y_{\max})} + e^{ h(x, \hh(x))}}\sfp(\hh(x) \!\mid\! x)^{\frac1{\mu-2}}}}^{1 - \mu }\\
&-\sfp(\hh(x) \!\mid\! x)\bracket*{\sum_{y'\in  \sY\bigcup \curl*{n+1}} e^{ h(x, y')- h(x, \hh(x))}}^{1 - \mu }\\
&\qquad + \sfp(\hh(x) \!\mid\! x)\bracket*{\frac{\sum_{y'\in  \sY\bigcup \curl*{n+1}}e^{ h(x, y')}\bracket*{\sfp(y_{\max} \!\mid\! x)^{\frac1{\mu-2 }}+\sfp(\hh(x) \!\mid\! x)^{\frac1{\mu-2}}}}{\bracket*{e^{ h(x, y_{\max})} + e^{ h(x, \hh(x))}}\sfp(y_{\max} \!\mid\! x)^{\frac1{\mu-2}}}}^{1 - \mu }.
\end{align*}
Then, by using the fact that $\frac{e^{ h(x, \hh(x))}}{\sum_{y'\in  \sY\bigcup \curl*{n+1}} e^{ h(x, y')}}\geq \frac{1}{(n+1)^{\mu-1}}$, we further have
\begin{align*}
& (\mu-1)\Delta \sC_{\lsc_{\mu}, \sH}( h, x)\\
&\geq \frac{1}{(n+1)^{\mu-1}}\paren[\Bigg]{\sfp(y_{\max} \!\mid\! x)\bracket*{\frac{\bracket*{e^{h(x, y_{\max})} + e^{h(x, \hh(x))}}\sfp(\hh(x) \!\mid\! x)^{\frac1{\mu-2}}}{e^{h(x, \hh(x))}\bracket*{\sfp(y_{\max} \!\mid\! x)^{\frac1{\mu-2}}+\sfp(\hh(x) \!\mid\! x)^{\frac1{\mu-2}}}}}^{\mu-1}\\
&\qquad -\sfp(y_{\max} \!\mid\! x)\bracket*{e^{h(x, y_{\max})-h(x, \hh(x))}}^{\mu-1}}\\
&\qquad \quad +\frac{1}{(n+1)^{\mu-1}}\paren*{\sfp(\hh(x) \!\mid\! x)\bracket*{\frac{\bracket*{e^{h(x, y_{\max})} + e^{h(x, \hh(x))}}\sfp(y_{\max} \!\mid\! x)^{\frac1{\mu-2}}}{e^{h(x, \hh(x))}\bracket*{\sfp(y_{\max} \!\mid\! x)^{\frac1{\mu-2}}+\sfp(\hh(x) \!\mid\! x)^{\frac1{\mu-2}}}}}^{\mu-1}-\sfp(\hh(x) \!\mid\! x)}
\tag{$\frac{e^{ h(x, \hh(x))}}{\sum_{y'\in  \sY\bigcup \curl*{n+1}} e^{ h(x, y')}}\geq \frac{1}{(n+1)^{\mu-1}}$}\\
&\geq \frac{1}{(n+1)^{\mu-1}}\paren*{\sfp(y_{\max} \!\mid\! x)\bracket*{\frac{2\sfp(\hh(x) \!\mid\! x)^{\frac1{\mu-2}}}{\sfp(y_{\max} \!\mid\! x)^{\frac1{\mu-2}}+\sfp(\hh(x) \!\mid\! x)^{\frac1{\mu-2}}}}^{\mu-1}-\sfp(y_{\max} \!\mid\! x)}\\
& \qquad + \frac{1}{(n+1)^{\mu-1}}\paren*{\sfp(\hh(x) \!\mid\! x)\bracket*{\frac{2\sfp(y_{\max} \!\mid\! x)^{\frac1{\mu-2}}}{\sfp(y_{\max} \!\mid\! x)^{\frac1{\mu-2}}+\sfp(\hh(x) \!\mid\! x)^{\frac1{\mu-2}}}}^{\mu-1}-\sfp(\hh(x) \!\mid\! x)}
\tag{minimum is attained when $e^{ h(x, \hh(x))}=e^{ h(x, y_{\max})}$}\\
& = \frac{1}{(n+1)^{\mu-1}}\paren*{2^{\mu-1}\bracket*{\sfp(y_{\max} \!\mid\! x)^{\frac1{2-\mu}}+\sfp(\hh(x) \!\mid\! x)^{\frac1{2-\mu}}}^{2-\mu}-\sfp(y_{\max} \!\mid\! x)-\sfp(\hh(x) \!\mid\! x)}\\
& = \frac{2}{(n+1)^{\mu-1}}\paren*{\bracket*{\frac{\sfp(y_{\max} \!\mid\! x)^{\frac1{2-\mu}}+\sfp(\hh(x) \!\mid\! x)^{\frac1{2-\mu}}}{2}}^{2-\mu}-\frac{\sfp(y_{\max} \!\mid\! x)+\sfp(\hh(x) \!\mid\! x)}{2}}\\
& \geq \frac{\mu-1}{2(2-c)(n+1)^{\mu-1}}\paren*{\sfp(y_{\max} \!\mid\! x) - \sfp(\hh(x) \!\mid\! x)}^2
\tag{$\sfp(y_{\max} \!\mid\! x)+\sfp(\hh(x) \!\mid\! x)\leq 2-c$ and by analyzing the Taylor expansion}\\
& = \frac{\mu-1}{2(2-c)(n+1)^{\mu-1}}\Delta\sC_{\labsc, \sH}( h,x)^2 \tag{Corollary~\ref{cor:calibration_gap_score}}
\end{align*}
Thus, we have
\begin{align*}
\sE_{\labsc}( h) - \sE_{\labsc}^*( \sH) + \sM_{\labsc}( \sH)
& = \E_{X}\bracket*{\Delta \sC_{\labsc,\sH}( h, x)}\\
& \leq \E_X\bracket*{\Gamma_{\mu}\paren*{\Delta \sC_{\lsc_{\mu},\sH}( h, x)}}\\
& \leq \Gamma_{\mu}\paren*{\E_X\bracket*{\Delta \sC_{\lsc_{\mu},\sH}( h, x)}}
\tag{$\Gamma_{\mu}$ is concave}\\
& = \Gamma_{\mu}\paren*{\sE_{\lsc_{\mu}}( h)-\sE_{\lsc_{\mu}}^*( \sH) +\sM_{\lsc_{\mu}}( \sH)},
\end{align*}
where $\Gamma_{\mu}(t)=\sqrt{2(2-c)(n+1)^{\mu-1}t }$.
\end{proof}

\subsection{Characterization of Minimizability
  Gaps (Theorem~\ref{Thm:gap-upper-bound-determi})}
\label{app:gap-upper-bound-determi}
\GapUpperBoundDetermi*
\begin{proof}
Let $s_{ h}(x, y)=\frac{e^{ h(x, y)}}{\sum_{y'\in  \sY\bigcup \curl*{n+1}} h(x, y')}\in [0,1]$, $\forall y\in  \sY$.
By the definition, for any deterministic distribution, $\sM_{\lsc_{\mu}}( \sH)
 = \sE^*_{\lsc_{\mu}}( \sH) - \mathbb{E}_{X} \bracket* {\inf_{ h \in  \sH}\sC_{\lsc_{\mu}}( \sH, x)}$, where 
\begin{align*}
&\sC_{\lsc_{\mu}}( h,x)\\
& = \sum_{y\in \sY\bigcup \curl*{n+1}} \sfp(y \!\mid\! x)  \ell_{\mu}( h,x, y)\\
& =  \ell_{\mu}( h,x, y_{\max}) + (1-c) \ell_{\mu}( h,x,n+1)\\
& =\begin{cases}
\frac{1}{1 - \mu} \paren[\Big]{\bracket*{\sum_{y'\in \sY\bigcup \curl*{n+1}} e^{{ h(x, y') -  h(x, y_{\max})}}}^{1 - \mu} - 1}\\ \qquad + (1-c) \frac{1}{1 - \mu} \paren*{\bracket*{\sum_{y'\in \sY\bigcup \curl*{n+1}} e^{{ h(x, y') -  h(x, n+1)}}}^{1 - \mu} - 1} & \mu\neq 1  \\
\log\paren*{\sum_{y'\in  \sY\bigcup \curl*{n+1}} e^{ h(x, y') -  h(x, y_{\max})}}\\ \qquad + (1-c) \log\paren*{\sum_{y'\in  \sY\bigcup \curl*{n+1}} e^{ h(x, y') -  h(x, n+1)}} &  \mu = 1.
\end{cases}\\
& = \begin{cases}
\frac{1}{1 - \mu} \paren*{s_{ h}\paren*{x, y_{\max}}^{\mu-1} - 1} + (1-c) \frac{1}{1 - \mu} \paren*{\bracket*{s_{ h}\paren*{x,n+1}}^{\mu-1} - 1} & \mu\neq 1  \\
-\log\paren*{s_{ h}\paren*{x, y_{\max}}} - (1-c) \log\paren*{s_{ h}\paren*{x,n+1}} & \mu = 1.
\end{cases}
\end{align*}
Since $0\leq s_{ h}(x, y_{\max})+s_{ h}(x,n+1)\leq 1$, by taking the partial derivative, we obtain that the minimum can be attained by
\begin{align}
\label{eq:min}
\begin{cases}
  s^*_{ h}(x, y_{\max})=\frac{1}{1+(1-c)^{\frac{1}{2-\mu}}} \text{ and } s^*_{ h}(x,n+1)
  =\frac{(1-c)^{\frac{1}{2-\mu}}}{1+(1-c)^{\frac{1}{2-\mu}}} & \mu \neq 2\\
s^*_{ h}(x, y_{\max}) = 1 \text{ and } s^*_{ h}(x,n+1)=0 & \mu =2.
\end{cases}
\end{align}
Since $ \sH$ is symmetric and complete, there exists $ h \in 
\sH$ such that \eqref{eq:min} is achieved. Therefore,
\begin{align*}
\inf_{ h \in  \sH}\sC_{\lsc_{\mu}}( \sH, x) 
& =
\begin{cases}
\frac{1}{1 - \mu} \paren*{s^*_{ h}\paren*{x, y_{\max}}^{\mu-1} - 1} + (1-c) \frac{1}{1 - \mu} \paren*{\bracket*{s^*_{ h}\paren*{x,n+1}}^{\mu-1} - 1} & \mu\neq 1  \\
-\log\paren*{s^*_{ h}\paren*{x, y_{\max}}} - (1-c) \log\paren*{s^*_{ h}\paren*{x,n+1}} & \mu = 1
\end{cases}\\
& =
 \begin{cases}
   \frac{1}{1 - \mu} \bracket*{\bracket*{1+\paren*{1-c}^{\frac{1}{2-\mu}}}^{2 - \mu}
     - (2-c)} & \mu \notin \curl*{1,2}\\
-\log \paren*{\frac{1}{2-c}}-(1-c)\log \paren*{\frac{1-c}{2-c}}
 & \mu=1\\
1-c & \mu =2.
\end{cases}
\end{align*}
Since $\inf_{ h \in  \sH}\sC_{\lsc_{\mu}}( \sH, x) $ is
independent of $x$, we obtain $\mathbb{E}_{X} \bracket*
{\inf_{ h \in  \sH}\sC_{\lsc_{\mu}}( \sH, x)}=\inf_{ h \in
   \sH}\sC_{\lsc_{\mu}}( \sH, x)$,
which completes the proof.
\end{proof}

\subsection{Proof of General Transformation of \texorpdfstring{$\sH$}{H}-Consistency Bounds (Theorem~\ref{Thm:bound-score-mabsc})}
\label{app:bound-score}
\BoundScoreMabsc*
\begin{proof}
By Lemma~\ref{lemma:calibration_gap_score-mabsc}, the calibration gap of
$\labsc$ can be expressed and upper-bounded as follows:
\begin{align*}
& \Delta \sC_{\labsc,\sH}( h, x)\\
& = \sC_{\labsc}( h, x)-\sC_{\labsc}^*(\sH,x)\\
& = \max_{y\in \mathsf H(x)} \sfp(y \!\mid\! x) - \sfp(\hh(x) \!\mid\! x)\\
& = (2 - c) \paren*{\max_{y\in \mathsf H(x)} \ov \sfp(y \!\mid\! x) -  \ov \sfp(\hh(x) \!\mid\! x)}\tag{Let $ \ov \sfp(y \!\mid\! x) = \frac{\sfp(y \!\mid\! x)}{2 - c}\1_{y\in \sY}+\frac{1 - c}{2 - c}\1_{y = n + 1}$}\\
& =(2 - c) \Delta \sC_{\ell_{0-1},\sH}( h, x)\tag{By \citep[Lemma~3]{AwasthiMaoMohriZhong2022multi}}\\
& \leq (2 - c)\Gamma\paren*{\Delta \sC_{ \ell, \sH}( h, x)} \tag{By $ \sH$-consistency
bound of $ \ell$}\\
& = (2 - c)\Gamma\paren*{\sum_{y\in \sY \bigcup \curl*{n+1}} \ov \sfp(y \!\mid\! x) \ell( h, x, y)-\inf_{ h \in \sH}\sum_{y\in \sY \bigcup \curl*{n+1}} \ov \sfp(y \!\mid\! x) \ell( h, x, y)} \\
& = (2 - c)\Gamma\paren[\Bigg]{\sum_{y\in \sY}\frac{\sfp(y \!\mid\! x)}{2 - c} \ell( h, x, y)+\frac{1 - c}{2 - c} \ell( h, x, n + 1)\\
&\qquad -\inf_{ h \in \sH}\paren*{\sum_{y\in \sY}\frac{\sfp(y \!\mid\! x)}{2 - c} \ell( h, x, y)+\frac{1 - c}{2 - c} \ell( h, x, n + 1)}}\tag{Plug in $ \ov \sfp(y \!\mid\! x) = \frac{\sfp(y \!\mid\! x)}{2 - c}\1_{y\in \sY}+\frac{1 - c}{2 - c}\1_{y = n + 1}$} \\
& = (2 - c)\Gamma\paren*{\frac{1}{2 - c}\bracket*{\sum_{y\in \sY}\sfp(y \!\mid\! x)\lsc( h, x, y)-\inf_{ h \in \sH}\sum_{y\in \sY}\sfp(y \!\mid\! x)\lsc( h, x, y)}}\\
& = (2 - c)\Gamma\paren*{\frac{1}{2 - c}\Delta \sC_{\lsc,\sH}( h, x)}.
\end{align*}
Thus, we have
\begin{align*}
\sE_{\labsc}( h) - \sE_{\labsc}^*( \sH) + \sM_{\labsc}( \sH)
& = \E_{X}\bracket*{\Delta \sC_{\labsc,\sH}( h, x)}\\
& \leq \E_X\bracket*{(2 - c)\Gamma\paren*{\frac{1}{2 - c}\Delta \sC_{\lsc,\sH}( h, x)}}\\
& \leq (2 - c) \Gamma\paren*{\frac{1}{2 - c} \E_X\bracket*{\Delta \sC_{\lsc,\sH}( h, x)}}
\tag{$\Gamma$ is concave}\\
& = (2 - c) \Gamma\paren*{\frac{\sE_{\lsc}( h)-\sE_{\lsc}^*( \sH) +\sM_{\lsc}( \sH)}{2 - c}},
\end{align*}
which completes the proof.
\end{proof}

\subsection{Proof of \texorpdfstring{$\sH$}{H}-Consistency Bounds for Two-Stage Surrogates (Theorem~\ref{Thm:bound-general-two-step})}
\label{app:bound-general-two-step}
\BoundGenralTwoStep*
\begin{proof}
For any $h=(h_{\sY},h_{n+1})$,
we can rewrite $\sE_{\labs}(h)-\sE^*_{\labs}\paren*{\sH}+\sM_{\labs}(\sH)$ as 
\begin{equation}
\label{eq:expression-two-step}
\begin{aligned}
& \sE_{\labs}(h)-\sE^*_{\labs}\paren*{\sH}+\sM_{\labs}(\sH)\\
& =  \E_{X}\bracket*{\sC_{\labs}(h,x)-\sC^*_{\labs}(\sH,x)} \\
& =  \E_{X}\bracket*{\sC_{\labs}(h,x)-\inf_{h_{n+1}\in \sH_{n+1}}\sC_{\labs}(h,x)+\inf_{h_{n+1}\in \sH_{n+1}}\sC_{\labs}(h,x)-\sC^*_{\labs}(\sH, x)}\\
& = \E_{X}\bracket*{\sC_{\labs}(h,x)-\inf_{h_{n+1}\in \sH_{n+1}}\sC_{\labs}(h,x)}+\E_{X}\bracket*{\inf_{h_{n+1}\in \sH_{n+1}}\sC_{\labs}(h,x)-\sC^*_{\labs}(\sH,x)}
\end{aligned}
\end{equation}
By the assumptions, we have
\begin{align*}
& \sC_{\labs}(h,x)-\inf_{h_{n+1}\in \sH_{n+1}}\sC_{\labs}(h, x)\\
& = \sum_{y\in \sY}\sfp(y \!\mid\! x)\1_{ \hh_{\sY}(x)\neq y}\1_{ \hh(x)\neq n + 1} + c \1_{ \hh(x) = n + 1}-\inf_{h_{n+1}\in \sH_{n+1}}\paren*{\sum_{y\in \sY}\sfp(y \!\mid\! x)\1_{ \hh_{\sY}(x)\neq y}\1_{ \hh(x)\neq n + 1} + c \1_{ \hh(x) = n + 1}}\\
& = \paren*{\sum_{y\in \sY}\sfp(y \!\mid\! x)\1_{\hh_{\sY}(x)\neq y} + c}\times \bigg[\eta(x)\ell_{0-1}^{\rm{binary}}\paren*{h_{n+1}-\max_{y\in \sY}h_{\sY}(x, y),x,+1}\\
&\qquad +(1-\eta(x))\ell_{0-1}^{\rm{binary}}\paren*{h_{n+1}-\max_{y\in \sY}h_{\sY}(x, y),x,-1}\\
& \quad \quad -\inf_{h_{n+1}\in \sH_{n+1}}\paren[\bigg]{\eta(x)\ell_{0-1}^{\rm{binary}}\paren*{h_{n+1}-\max_{y\in \sY}h_{\sY}(x, y),x,+1}\\
&\qquad \quad \quad +(1-\eta(x))\ell_{0-1}^{\rm{binary}}\paren*{h_{n+1}-\max_{y\in \sY}h_{\sY}(x, y),x,-1}}\bigg]\tag{Let $\eta(x) = \frac{\sum_{y\in \sY}\sfp(y \!\mid\! x)\1_{\hh_{\sY}(x)\neq y}}{\sum_{y\in \sY}\sfp(y \!\mid\! x)\1_{\hh_{\sY}(x)\neq y} + c}$}.
\end{align*}
Then, by $\sH_{n+1}^{\tau}$-consistency bounds of $\Phi$ under assumption, $\tau=\max_{y\in \sY}h_{\sY}(x, y)$, we further have
\begin{align*}
& \sC_{\labs}(h,x)-\inf_{h_{n+1}\in \sH_{n+1}}\sC_{\labs}(h, x)\\
& \leq \paren*{\sum_{y\in \sY}\sfp(y \!\mid\! x)\1_{\hh_{\sY}(x)\neq y} + c}\Gamma_2\bigg[\eta(x)\Phi\paren*{h_{n+1}(x)-\max_{y\in \sY}h_{\sY}(x, y)}\\
&\qquad +(1-\eta(x))\Phi\paren*{\max_{y\in \sY}h_{\sY}(x, y)-h_{n+1}(x)}\\
&\qquad \quad-\inf_{h_{n+1}\in \sH_{n+1}}\paren*{\eta(x)\Phi\paren*{h_{n+1}(x)-\max_{y\in \sY}h_{\sY}(x, y)} +(1-\eta(x))\Phi\paren*{\max_{y\in \sY}h_{\sY}(x, y)-h_{n+1}(x)}}\bigg]\tag{By $\sH_{n+1}^{\tau}$-consistency bounds of $\Phi$ under assumption, $\tau=\max_{y\in \sY}h_{\sY}(x, y)$}\\
& =  \paren*{\sum_{y\in \sY}\sfp(y \!\mid\! x)\1_{\hh_{\sY}(x)\neq y} + c}
\\ & \qquad \Gamma_2\paren*{\frac{\sum_{y\in \sY}\sfp(y \!\mid\! x)\ell_{h_{\sY}}(h_{n+1},x, y)-\inf_{h_{n+1}\in \sH_{n+1}}\sum_{y\in \sY}\sfp(y \!\mid\! x)\ell_{h_{\sY}}(h_{n+1},x, y)}{ \sum_{y\in \sY}\sfp(y \!\mid\! x)\1_{\hh_{\sY}(x)\neq y} + c}}\tag{ $\eta(x) = \frac{\sum_{y\in \sY}\sfp(y \!\mid\! x)\1_{\hh_{\sY}(x)\neq y}}{\sum_{y\in \sY}\sfp(y \!\mid\! x)\1_{\hh_{\sY}(x)\neq y} + c}$ and formulation \eqref{eq:ell-Phi-h}}\\
& = \paren*{\sum_{y\in \sY}\sfp(y \!\mid\! x)\1_{\hh_{\sY}(x)\neq y} + c} \Gamma_2\paren*{\frac{\sC_{\ell_{h_{\sY}}}(h_{n+1},x)-\sC^*_{\ell_{h_{\sY}}}(\sH_{n+1},x)}{ \sum_{y\in \sY}\sfp(y \!\mid\! x)\1_{\hh_{\sY}(x)\neq y} + c}}\\
& \leq
\begin{cases}
\Gamma_2\paren*{\sC_{\ell_{h_{\sY}}}(h_{n+1},x)-\sC^*_{\ell_{h_{\sY}}}(\sH_{n+1},x)} & \text{when $\Gamma_2$ is linear}\\
(1+c)\Gamma_2\paren*{\frac {\sC_{\ell_{h_{\sY}}}(h_{n+1},x)-\sC^*_{\ell_{h_{\sY}}}(\sH_{n+1},x)}{c}} & \text{otherwise}
\end{cases}\\
\tag{$c\leq \sum_{y\in \sY}\sfp(y \!\mid\! x)\1_{\hh_{\sY}(x)\neq y} + c\leq 1+c$ and $\Gamma_2$ is non-decreasing}\\
& = \begin{cases}
\Gamma_2\paren*{\Delta\sC_{\ell_{h_{\sY}},\sH_{n+1}}(h_{n+1},x)}  & \text{when $\Gamma_2$ is linear}\\
(1+c)\Gamma_2\paren*{\frac {\Delta\sC_{\ell_{h_{\sY}},\sH_{n+1}}(h_{n+1},x)}{c}} & \text{otherwise}
\end{cases}
\end{align*}
and 
\begin{align*}
& \inf_{h_{n+1}\in \sH_{n+1}}\sC_{\labs}(h,x)-\sC^*_{\labs}(\sH,x)\\
& = \inf_{h_{n+1}\in \sH_{n+1}}\sC_{\labs}(h,x)-\inf_{h_{\sY}\in \sH_{\sY},h_{n+1}\in \sH_{n+1}}\sC_{\labs}(h,x)\\
& = \inf_{h_{n+1}\in \sH_{n+1}} \paren*{\sum_{y\in \sY}\sfp(y \!\mid\! x)\1_{ \hh_{\sY}(x)\neq y}\1_{ \hh(x)\neq n + 1} + c \1_{ \hh(x) = n + 1}}\\
&\qquad -\inf_{h_{\sY}\in \sH_{\sY},h_{n+1}\in \sH_{n+1}} \paren*{\sum_{y\in \sY}\sfp(y \!\mid\! x)\1_{ \hh_{\sY}(x)\neq y}\1_{ \hh(x)\neq n + 1} + c \1_{ \hh(x) = n + 1}}\\
& = \inf_{h_{n+1}\in \sH_{n+1}} \paren*{\sum_{y\in \sY}\sfp(y \!\mid\! x)\1_{ \hh_{\sY}(x)\neq y}\1_{ \hh(x)\neq n + 1} + c \1_{ \hh(x) = n + 1}}\\
&\qquad -\inf_{h_{n+1}\in \sH_{n+1}} \paren*{\inf_{h_{\sY}\in \sH_{\sY}}\sum_{y\in \sY}\sfp(y \!\mid\! x)\1_{ \hh_{\sY}(x)\neq y}\1_{ \hh(x)\neq n + 1} + c \1_{ \hh(x) = n + 1}}\\
& = \min\curl*{\sum_{y\in \sY}\sfp(y \!\mid\! x)\1_{\hh_{\sY}(x)\neq y},c}-\min\curl*{\inf_{h_{\sY}\in \sH_{\sY}}\sum_{y\in \sY}\sfp(y \!\mid\! x)\1_{\hh_{\sY}(x)\neq y},c}\\
& \leq \sum_{y\in \sY}\sfp(y \!\mid\! x)\1_{\hh_{\sY}(x)\neq y} -\inf_{h_{\sY}\in \sH_{\sY}}\sum_{y\in \sY}\sfp(y \!\mid\! x)\1_{\hh_{\sY}(x)\neq y}\\
& = \sC_{\ell_{0-1}}(h_{\sY}, x)-\sC^*_{\ell_{0-1}}(\sH_{\sY},x)\\
& = \Delta\sC_{\ell_{0-1},\sH_{\sY}}(h_{\sY}, x)\\
& \leq \Gamma_1\paren*{\Delta\sC_{\ell,\sH_{\sY}}(h_{\sY}, x)}.
\tag{By $\sH_{\sY}$-consistency bounds of $\ell$ under assumption}
\end{align*}
Therefore, by \eqref{eq:expression-two-step}, we obtain
\begin{align*}
& \sE_{\labs}(h)-\sE^*_{\labs}\paren*{\sH_{\sY}}+\sM_{\labs}(\sH_{\sY})\\
& \leq 
\begin{cases}
\E_X\bracket*{\Gamma_2\paren*{\Delta\sC_{\ell_{h_{\sY}},\sH_{n+1}}(h_{n+1},x)}} + \E_X\bracket*{\Gamma_1\paren*{\Delta\sC_{\ell,\sH_{\sY}}(h_{\sY}, x)}} & \text{when $\Gamma_2$ is linear}\\
(1+c)\E_X\bracket*{\Gamma_2\paren*{\frac {\Delta\sC_{\ell_{h_{\sY}},\sH_{n+1}}(h_{n+1},x)}{c}}} + \E_X\bracket*{\Gamma_1\paren*{\Delta\sC_{\ell,\sH_{\sY}}(h_{\sY}, x)}} & \text{otherwise}
\end{cases}\\
& \leq 
\begin{cases}
\Gamma_2\paren*{\E_X\bracket*{\Delta\sC_{\ell_{h_{\sY}},\sH_{n+1}}(h_{n+1},x)}} + \Gamma_1\paren*{\E_X\bracket*{\Delta\sC_{\ell,\sH_{\sY}}(h_{\sY}, x)}} & \text{when $\Gamma_2$ is linear}\\
(1+c)\Gamma_2\paren*{\frac1c\E_X\bracket*{\Delta\sC_{\ell_{h_{\sY}},\sH_{n+1}}(h_{n+1},x)}} + \Gamma_1\paren*{\E_X\bracket*{\Delta\sC_{\ell,\sH_{\sY}}(h_{\sY}, x)}} & \text{otherwise}
\end{cases}
\tag{$\Gamma_1$ and $\Gamma_2$ are concave}\\
& =
\begin{cases}
\Gamma_1\paren*{\sE_{\ell}(h)-\sE_{\ell}^*(\sH_{\sY}) +\sM_{\ell}(\sH_{\sY})}
\\\qquad + \Gamma_2\paren*{\sE_{\ell_{h_{\sY}}}(h_{n+1})-\sE_{\ell_{h_{\sY}}}^*(\sH_{n+1}) +\sM_{\ell_{h_{\sY}}}(\sH_{n+1})} & \text{when $\Gamma_2$ is linear}\\
(\Gamma_1\paren*{\sE_{\ell}(h)-\sE_{\ell}^*(\sH_{\sY}) +\sM_{\ell}(\sH_{\sY})}
\\\qquad + (1+c)\Gamma_2\paren*{\frac{\sE_{\ell_{h_{\sY}}}(h_{n+1})-\sE_{\ell_{h_{\sY}}}^*(\sH_{n+1}) +\sM_{\ell_{h_{\sY}}}(\sH_{n+1})}{c}} & \text{otherwise},
\end{cases}
\end{align*}
which completes the proof.

\end{proof}

\subsection{Proof of Realizable \texorpdfstring{$\sH$}{H}-Consistency
  for Two-Stage Surrogates
  (Theorem~\ref{Thm:bound-general-two-step-realizable})}
\label{app:bound-general-two-step-realizable}

\begin{definition}[\textbf{Realizable $\sH$-consistency}]
\label{def:rel-consistency-mabsc} Let $\hat h$ denote a hypothesis attaining the infimum of the expected surrogate loss, $\sE_{\sfL}(\hat h) = \sE^*_{\sfL}(\sH)$. A score-based abstention surrogate loss $\sfL$ is said to be
\emph{realizable $\sH$-consistent} with respect to the abstention loss
$\labs$ if, for any distribution in which an optimal hypothesis $h^*$
exists in $\sH$ with an abstention loss of zero (i.e.,
$\sE_{\labs}(h^*)=0$), we have $\sE_{\labs}(\hat h) = 0$.
\end{definition}

Next, we demonstrate that our proposed two-stage score-based surrogate
losses are not only Bayes-consistent, as previously established in
Section~\ref{sec:two-stage-mabsc}, but also realizable $\sH$-consistent,
which will be shown in
Theorem~\ref{Thm:bound-general-two-step-realizable}. This effectively
addresses the open question posed by \citet{pmlr-v206-mozannar23a} in
the context of score-based multi-class abstention and highlights the
benefits of the two-stage formulation.

\begin{restatable}[\textbf{Realizable $\sH$-consistency for
      two-stage surrogates}]{theorem}{BoundGenralTwoStepRealizable}
\label{Thm:bound-general-two-step-realizable}

Given a hypothesis set $\sH=\sH_{\sY}\times \sH_{n+1}$ that is closed
under scaling.  Let $\Phi$ be a function that satisfies the condition
$\lim_{t\to + \infty}\Phi(t)=0$ and $\Phi(t) \geq 1_{t \leq 0}$
for any $ t \in \Rset$. Assume that $\hat h = (\hat h_{\sY}, \hat
h_{n+1})\in \sH$ attains the infimum of the expected surrogate loss,
$\sE_{\ell}(\hat h_{\sY}) = \inf_{h_{\sY} \in
  \sH_{\sY}}\sE_{\ell}(h_{\sY})$ and $\sE_{\ell_{\hat h_{\sY}}}(\hat
h_{n+1}) = \inf_{h \in \sH}\sE_{\ell_{h_{\sY}}}(h_{n+1})$. Then, for
any distribution in which an optimal hypothesis
$h^*=(h^*_{\sY},h^*_{n+1})$ exists in $\sH$ with $\sE_{\labs}(h^*)=0$,
we have $\sE_{\labs}(\hat h) = 0$.
\end{restatable}

\begin{proof}
By the assumptions, $\ell_{h_{\sY}}$ serves as an upper bound for
$\labs$ and thus $\sE_{\labs}(\hat h) \leq \sE_{\ell_{\hat
    h_{\sY}}}(\hat h_{n + 1})$. If abstention happens, that is
$h^*_{n+1}(x) > \max_{y\in \sY}h^*_{\sY}(x, y)$ for some point $x$,
then we must have $c=0$ by the realizability assumption. Therefore,
there exists an optimal $h^{**}$ such that $h^{**}_{n+1}(x) >
\max_{y\in \sY}h^{**}_{\sY}(x, y)$ for all $x \in \sX$ without
incurring any cost.  Then, by the Lebesgue dominated convergence
theorem and the assumption that $\sH$ is closed under scaling,
\begin{align*}
\sE_{\labs}(\hat h) & \leq \sE_{\ell_{\hat h_{\sY}}}(\hat h_{n + 1}) \\
& \leq \lim_{\alpha \to +\infty}\sE_{\ell_{\alpha h^{**}_{\sY}}}(\alpha h^{**}_{n+1})\\
& =\lim_{\alpha \to +\infty}\mathbb{E}\bracket*{ \ell_{\alpha h^{**}_{\sY}}\paren*{\alpha h^{**}_{n+1}, x, y}  }\\
& =\lim_{\alpha\to +\infty}\mathbb{E}\bracket*{ \1_{ \hh^{**}_{\sY}(x) \neq y} \Phi\paren*{\alpha\paren*{h^{**}_{n+1}(x)-\max_{y\in \sY}h^{**}_{\sY}(x, y)}} + c \Phi\paren*{\alpha\paren*{\max_{y\in \sY} h^{**}_{\sY}(x, y)-h^{**}_{n+1}(x)}}}\\
& =\lim_{\alpha\to +\infty}\mathbb{E}\bracket*{\1_{ \hh^{**}_{\sY}(x) \neq y} \Phi\paren*{\alpha\paren*{ h^{**}_{n+1}(x)-\max_{y\in \sY} h^{**}_{\sY}(x, y)}}} \tag{$c=0$}\\
& =0.  \tag{using $\lim_{t\to + \infty}\Phi(t)=0$ and the Lebesgue dominated convergence theorem}
\end{align*}
If abstention does not happen, that is $h^*_{n+1}(x)-\max_{y\in \sY}h^*_{\sY}(x, y)<0$ for all $x \in \sX$, then we must have $h^*_{\sY}(x, y)-\max_{y'\neq y}h^*_{\sY}(x, y')>0$ for all $x \in \sX$ and $y \in \sY$ by the realizability assumption. Then, by the Lebesgue dominated convergence theorem and the assumption that $\sH$ is closed under scaling,
\begin{align*}
\sE_{\labs}(\hat h) & \leq \sE_{\ell_{\hat h_{\sY}}}(\hat h_{n + 1})\\
& \leq \lim_{\alpha \to +\infty}\sE_{\ell_{\alpha h^*_{\sY}}}(\alpha h^{*}_{n+1})\\
& = \lim_{\alpha \to +\infty}\mathbb{E}\bracket*{ \ell_{ \alpha h^{*}_{\sY}}\paren*{\alpha h^{*}_{n+1}, x, y}  }\\
& = \lim_{\alpha\to +\infty}\mathbb{E}\bracket*{ \1_{ \hh^{*}_{\sY}(x) \neq y} \Phi\paren*{\alpha\paren*{h^*_{n+1}(x)-\max_{y\in \sY}h^*_{\sY}(x, y)}} + c \Phi\paren*{\alpha\paren*{\max_{y\in \sY} h^*_{\sY}(x, y)-h^*_{n+1}(x)}}}\\
& = \lim_{\alpha\to +\infty}\mathbb{E}\bracket*{ c \Phi\paren*{\alpha\paren*{\max_{y\in \sY} h^*_{\sY}(x, y)-h^*_{n+1}(x)}}} \tag{$h^{*}_{\sY}(x, y)-\max_{y'\neq y} h^{*}_{\sY}(x, y')>0$}\\
& = 0.  \tag{using $\lim_{t\to + \infty}\Phi(t)=0$ and the Lebesgue dominated convergence theorem}
\end{align*}
By combining the above two analysis, we conclude the proof.
\end{proof}
\ignore{
\begin{proof}
For any distribution in which an optimal hypothesis
$h^*=(h^*_{\sY},h^*_{n+1})$ exists in $\sH$ with $\sE_{\labs}(h^*)=0$,
we have for any $x\in \sX$, either $c=0$ and $h^*_{n+1}(x)> \max_{y\in
  \sY}h^*_{\sY}(x, y)$, or there exists $y_{\max}$ such that
$\sfp(y_{\max} \!\mid\! x)=1$, $h^*_{\sY}(x, y_{\max})> \max_{y'\neq
  y_{\max}}h^*_{\sY}(x, y')$ and $\max_{y\in
  \sY}h^*_{\sY}(x, y)>h^*_{n+1}(x)$.  Since $\sH$ is closed under
scaling, $\alpha h^*\in \sH$ for any $\alpha>0$. Using the fact that
$\lim_{t\to + \infty}\Phi(t)=0$ and $\lim_{\alpha \to +
  \infty}\sE_{\ell}(\alpha h^*_{\sY})=0$ for $\ell$ being the logistic
loss, we obtain
$\sE_{\ell}^*(\sH_{\sY})=\sE_{\ell_{h_{\sY}}}^*(\sH_{n+1})=0$. By
establishing that $\ell_{h_{\sY}}$ serves as an upper bound for
$\labs$, we conclude the proof.
\end{proof}
}

\section{Significance of \texorpdfstring{$\sH$}{H}-Consistency Bounds with Minimizability Gaps}
\label{app:better-bounds}

As previously highlighted, the minimizabiliy gap can be upper-bounded
by the approximation error $\sA_{\sfL}(\sH)= \sE^*_{\lsc}(\sH) -
\E_x\bracket[\big]{\inf_{ h \in \sH_{\rm{all}}} \E_y \bracket{\lsc( h,
    X, y) \mid X =
    x}}=\sE^*_{\sfL}(\sH)-\sE^*_{\sfL}(\sH_{\rm{all}})$. However, it
is a finer quantity than the approximation error, and as such, it can
potentially provide more significant guarantees.  To elaborate, as
shown by \citep{awasthi2022Hconsistency,AwasthiMaoMohriZhong2022multi}, for a
target loss function $\sfL_2$ and a surrogate loss function $\sfL_1$,
the excess error bound $\sE_{\sfL_2} (h) -
\sE^*_{\sfL_2}(\sH_{\rm{all}})\leq \Gamma\paren*{ \sE_{\sfL_1} (h) -
  \sE^*_{\sfL_1}(\sH_{\rm{all}})}$ can be reformulated as
\begin{align*}
\sE_{\sfL_2} (h) - \sE^*_{\sfL_2}(\sH) +\sA_{\sfL_2}(\sH)\leq \Gamma\paren*{ \sE_{\sfL_1} (h) - \sE^*_{\sfL_1}(\sH)+\sA_{\sfL_1}(\sH)},
\end{align*}
where $\Gamma$ is typically linear or the square-root function modulo
constants.  On the other hand, an $\sH$-consistency bound can be
expressed as follows:
\begin{equation*}
\sE_{\sfL_2} (h) - \sE^*_{\sfL_2}(\sH) +  \sM_{\sfL_2}(\sH)  \leq \Gamma\paren*{ \sE_{\sfL_1} (h) - \sE^*_{\sfL_1}(\sH) + \sM_{\sfL_1}(\sH}.
\end{equation*}
For a target loss function $\sfL_2$ with discrete outputs, such as the
zero-one loss or the deferral loss, we have
$\E_{x}\bracket[\big]{\inf_{h \in\sH}\E_{y}\bracket*{\sfL_2(h, x,
    y)\mid X = x}}=\E_x\bracket[\big]{\inf_{h \in \sH_{\rm{all}}} \E_{y}
  \bracket*{\sfL_2(h,x, y)\mid X = x}}$ when the hypothesis set generates labels
that cover all possible outcomes for each input (See
\citep[Lemma~3]{AwasthiMaoMohriZhong2022multi},
Lemma~\ref{lemma:calibration_gap_score-mabsc} in
Appendix~\ref{app:score-mabsc}). Consequently, we have
$\sM_{\sfL_2}(\sH) = \sA_{\sfL_2}(\sH)$. However, for a surrogate loss function
$\sfL_1$, the minimizability gap is upper-bounded by the approximation
error, $\sM_{\sfL_1}(\sH)\leq \sA_{\sfL_1}(\sH)$, and is generally
finer.

Let us consider a straightforward binary classification example where the conditional distribution is denoted as $\eta(x)=D(Y=1 | X = x)$. We will define $\sH$ as a set of functions $h$, such that $|h(x)| \leq \Lambda$ for all $x \in \sX$, for some $\Lambda > 0$, and it is also possible to achieve any value in the range $[-\Lambda, +\Lambda]$. For the exponential-based margin loss, which we define as $\sfL(h, x, y) = e^{-yh(x)}$, we obtain the following equation:
\begin{equation*}
\E_{y}[\sfL(h, x, y)\mid X = x] = \eta(x)
e^{-h(x)} + (1 - \eta(x)) e^{h(x)}.
\end{equation*}
Upon observing this, it becomes apparent that the infimum over all measurable functions can be expressed in the following way, for all $x$:
\begin{equation*}
\inf_{h
  \in \sH_{\mathrm{all}}}\E_{y}[\sfL(h, x, y)\mid X = x] =
2\sqrt{\eta(x)(1-\eta(x))},
\end{equation*}
while the infimum over $\sH$, $\inf_{h \in \sH}\E_{y}[\sfL(h, x, y)\mid X = x]$, depends on $\Lambda$ and can be expressed as
  \begin{align*}
   &\inf_{h \in \sH}\E_{y}[\sfL(h, x, y)\mid X = x]\\
   &=\begin{cases}
   \max\curl*{\eta(x),1 - \eta(x)}
e^{-\Lambda} + \min\curl*{\eta(x),1 - \eta(x)} e^{\Lambda} & \Lambda<\frac{1}{2} \abs*{\log \frac{\eta(x)}{1
    -\eta(x)}}\\
    2\sqrt{\eta(x)(1-\eta(x))} & \text{otherwise}.
   \end{cases} 
  \end{align*}
Thus, in the deterministic scenario, the discrepancy between the
approximation error $\sA_{\sfL}(\sH)$ and the minimizability gap
$\sM_{\sfL}(\sH)$ is:
\begin{equation*}
  \sA_{\sfL}(\sH)-\sM_{\sfL}(\sH)
  = \E_{x}\bracket*{\inf_{h \in\sH}\E_{y}\bracket*{\sfL(h, x, y)\mid X = x}
    - \inf_{h \in
      \sH_{\rm{all}}} \E_{y} \bracket*{\sfL(h,x, y)\mid X = x}}
  = e^{-\Lambda}.
\end{equation*}
Therefore, for a surrogate loss, the minimizability gap can be
strictly less than the approximation error.  In summary, an
$\sH$-consistency bound can be more significant than the excess error
bound as $\sM_{\sfL_2}(\sH) = \sA_{\sfL_2}(\sH)$ when $\sfL_2$
represents the zero-one loss or deferral loss, and $\sM_{\sfL_1}(\sH)
\leq \sA_{\sfL_1}(\sH)$. They can also be directly used to derive
finite sample estimation bounds for a surrogate loss minimizer, which
are more favorable and relevant than a similar finite sample guarantee
that could be derived from an excess error bound (see
Section~\ref{sec:finite-sample}).
\restoreatoc

\chapter{Appendix to Chapter~\ref{ch3}}

\disableatoc
\section{Remarks on some key results}
\label{app:remark}

\begin{remark}
\label{remark:spcific-loss-bound}
When the best-in-class error coincides with Bayes error $\sE^*_{\sfL}(\sH, \sR) = \sE^*_{\sfL}\paren*{\sH_{\rm{all}}, \sR_{\rm{all}}}$, the minimizability gaps $\sM_{\sfL}(\sH,\sR)$ vanish. In those cases, the  $(\sH,\sR)$-consistency bound in Theorem~\ref{Thm:spcific-loss-bound} guarantees that when the surrogate estimation error $\sE_{\sfL}(h, r) -  \sE_{\sfL}^*(\sH,\sR)$ is optimized up to $\e$, the estimation error of the abstention loss $ \sE_{\labs}(h, r) - \sE_{\labs}^*(\sH,\sR)$ is upper-bounded by $\Gamma(\e)$. For all the three loss functions, when $\e$ is sufficiently small, the
dependence of $\Gamma$ on $\e$ exhibits a square root
relationship. However, if this is not the case, the dependence becomes
linear. Note that the dependence is subject to the number of classes $n$ for $\ell = \ell_{\rm{mae}}$ and
$\ell =
\ell_{\rho-\mathrm{hinge}}$.
\end{remark}

\begin{remark}
\label{remark:bound-general-second-step}
When the best-in-class error coincides with the Bayes error $\sE^*_{\ell}(\sR) = \sE^*_{\ell}\paren*{\sR_{\rm{all}}}$ for $\ell = \ell_{\Phi, h}$ and $\ell = \ell_{\mathrm{abs}, h}$, the minimizability gaps $\sM_{\ell_{\Phi, h}}(\sR)$ and $\sM_{\ell_{\mathrm{abs}, h}}(\sR)$ vanish. In those cases, the $\sR$-consistency bound in Theorem~\ref{Thm:bound-general-second-step} guarantees that when the surrogate estimation error $\sE_{\ell_{\Phi, h}}(r) -  \sE_{\ell_{\Phi, h}}^*(\sR)$ is optimized up to $\e$, the target estimation error $ \sE_{\ell_{\mathrm{abs}, h}}(r) - \sE_{\ell_{\mathrm{abs}, h}}^*(\sR)$ is upper-bounded by $\Gamma(\frac{\e}{c})$.
\end{remark}

\begin{remark}
\label{remark:tsr}
Corollary~\ref{cor:tsr-mabs} shows that $\ell_{\Phi, h}$ admits an excess error bound with respect to $\ell_{\mathrm{abs}, h}$ with functional form $\Gamma(\frac{\cdot}{c})$ when $\Phi$ admits an excess error bound with respect to $\ell_{0-1}^{\rm{binary}}$ with functional form $\Gamma(\cdot)$.
\end{remark}


\begin{remark}
\label{remark:bound-general-two-step}
Note that the minimizability gaps vanish when $\sH$ and $\sR$ are
families of all measurable functions or when they include the Bayes
predictor and rejector. In their absence,
Theorem~\ref{Thm:bound-general-two-step-mabs} shows that if the estimation
prediction loss $(\sE_{\ell}(h)-\sE_{\ell}^*(\sH))$ is reduced to
$\e_1$ and the estimation rejection loss
$(\sE_{\ell_{\Phi,h}}(r)-\sE_{\ell_{\Phi,h}}^*(\sR))$ to $\e_2$, then
the abstention estimation loss $(\sE_{\labs}(h, r) -
\sE_{\labs}^*(\sH, \sR))$ is, up to constant factors, bounded by
$\Gamma_1(\e_1) + \Gamma_2(\e_2)$.
\end{remark}

\begin{remark}
\label{remark:tshr}
Corollary~\ref{cor:tshr-mabs} shows that the proposed two-stage approach
admits an excess error bound with respect to $\labs$ with functional
form $\Gamma_1(\cdot) + (1 + c)\Gamma_2(\frac{\cdot}{c})$ when $\ell$
admits an excess error bound with respect to $\ell_{0-1}$ with
functional form $\Gamma_1(\cdot)$ and $\Phi$ admits an excess error
bound with respect to $\ell_{0-1}^{\rm{binary}}$ with functional form
$\Gamma_2(\cdot)$.
\end{remark}

\section{Significance of two-stage formulation compared
  with single-stage losses}
\label{app:two-stage}

Here, we wish to further highlight the significance of our findings
regarding the two-stage formulation. The $(\sH, \sR)$-consistency
bounds we established for this scenario directly motivate an algorithm
for a crucial scenario. As already indicated, in applications, often a
prediction function is already available and has been trained using a
standard loss function such as cross-entropy. Training may take days
or months for some large models. The cost of a one-stage approach,
which involves "retraining" to find a pair $(h, r)$ with a new $h$,
can thus be prohibitive. Instead, we demonstrate that a rejector $r$
can be learned using a suitable surrogate loss function based on the
existing predictor $h$ and that the solution formed by the existing
$h$ and this rejector $r$ benefits from $(\sH, \sR)$-consistency
bounds.

Both our one-stage and two-stage solutions using our surrogate losses
benefit from strong $(\sH, \sR)$-consistency bounds: in the limit of
large samples, both methods, one-stage and two-stage converge to the
same joint minimizer of the target abstention loss. However, as
already emphasized, the two-stage approach is advantageous in some
scenarios where a predictor $h$ is already available.  Moreover, the
two-stage solution is more beneficial from the optimization point of
view: the first-stage optimization can be standard and based on say
cross-entropy, and the second stage is based on a loss function
\eqref{eq:ell-Phi-h-mabs} that is straightforward to minimize. In contrast,
the one-stage minimization with the MAE loss is known to be more
difficult, see \citep{zhang2018generalized}. In
Section~\ref{sec:experiments-mabs}, our empirical results show a more
favorable performance for the two-stage solution, which we believe
reflects this difference in optimization.

\section{Difference between predictor-rejector and score-based formulations}
\label{app:difference}
Here, we hope to further emphasize our contributions by pointing out that the score-based formulation does not offer a direct loss function applicable to the predictor-rejector formulation. It is important to emphasize that the hypothesis set used in the score-based setting constitutes a subset of real-valued functions defined over $ \sX \times \tilde \sY$, where $\tilde \sY$ is the original label set $\sY$ augmented with an additional label corresponding to rejection. In contrast, the predictor-rejector function involves selecting a predictor $h$ out of a collection of real-valued functions defined over $\sX \times \sY$ and a rejector $r$ from a sub-family of real-valued functions defined over $\sX$. Thus, the hypothesis sets in these two frameworks differ entirely. Consequently, a score-based loss function cannot be directly applied to the hypothesis set of the predictor-rejector formulation.

One can instead, given the hypothesis sets $\sH$ and $\sR$ for the predictor and rejector functions in the predictor-rejector formulation, define a distinct hypothesis set $\tilde \sH$ of real-valued functions defined over $\sX \times \tilde \sY$. Functions $\tilde h\in\tilde \sH$ are defined from a pair $(h, r)\in \sH \times \sR$. The score-based loss function for $\tilde h\in \tilde \sH$ then coincides with the predictor-rejector loss of $(h, r)$. However, the family $\tilde \sH$ is complex. As pointed out in Section~\ref{sec:score-example}, for instance, when $\sH$ and $\sR$ are families of linear functions, $\tilde \sH$ is not linear and is more complex.  

Moreover, there is a non-trivial coupling relating the scoring functions defined for the rejection label $(n+1)$ and other scoring functions, while such a coupling is not present in the standard score-based formulation. This makes it more difficult to minimize
the loss function $\wt \ell(\wt h, x, y)$.
Indeed, the minimization problem requires that the constraint $\wt
h(x, n + 1) = \max_{y \in \sY} \wt h(x, y) - r(x)$ be satisfied. This
constraint relates the first $n$ scoring functions $\wt h(\cdot, y)$, $y
\in \sY$, to the last scoring function $\wt h(\cdot, n + 1)$, via a
maximum operator (and the function $r$). The constraint is
non-differentiable and non-convex, which makes the minimization
problem more challenging. 

This augmented complexity and coupling fundamentally differentiate the two formulations. Our counterexample in Section~\ref{sec:score-example} underscores this intrinsic distinction between these two formulations: While the predictor-rejector formulation can easily handle certain instances, score-based framework falls short to tackle them unless a more complex hypothesis set is adopted. This key difference between the two formulations is also the underlying reason for the historical difficulty in devising a consistent surrogate loss function for the predictor-rejector formulation within the standard multi-class setting, while the task has been comparatively more straightforward within the score-based formulation.

It is important to highlight that our novel families of predictor-rejector surrogate losses, alongside similar variants, establish the first Bayes-consistent and realizable consistent surrogate losses within the predictor-rejector formulation and they address two previously open questions in the literature \citep{NiCHS19} and \citep{pmlr-v206-mozannar23a} (see Section~\ref{sec:general}). Moreover, they outperform the state-of-the-art surrogate losses found in the score-based formulation (see Section~\ref{sec:experiments-mabs}). This underscores both the innovative nature and the significant contribution of our work.

In the following section,
we will further showcase the advantages of the predictor-rejector
formulation through empirical evidence.

\section{Experimental details}
\label{app:setup}
\paragraph{Setup.}
We adopt ResNet-$34$ \citep{he2016deep}, a residual network with $34$ convolutional layers, for SVHN and CIFAR-10,
and WRN-$28$-$10$ \citep{zagoruyko2016wide}, a residual network with $28$ convolutional layers and a widening factor of $10$, for CIFAR-100 both with ReLU activations
\citep{zagoruyko2016wide}. We train for 200 epochs using Stochastic Gradient Descent (SGD) with Nesterov momentum
\citep{nesterov1983method} following the cosine decay learning rate schedule
\citep{loshchilov2016sgdr} of an initial learning rate $0.1$.
During the training, the batch size is set to $1\mathord,024$ and the weight decay is $1\times 10^{-4}$. Except for SVHN, we adopt the standard data augmentation: a four pixel padding with $32 \times 32$ random crops and random horizontal flips.

We compare with a score-based surrogate
loss proposed in \citep{mozannar2020consistent} based on cross-entropy and a score-based surrogate loss used in \citep{caogeneralizing} based on generalized cross-entropy \citep{zhang2018generalized}. For our single-stage predictor-rejector surrogate
loss, we set $\ell$ to be the mean absolute error loss
$\ell_{\rm{mae}}$ since the constrained hinge loss imposes a
restriction incompatible with the standard use of the softmax function
with neural network hypotheses, and the $\rho$-margin loss is
non-convex. For our two-stage predictor-rejector
surrogate loss, we first use standard training with the logistic loss
to learn a predictor $h^*$, and then in the second stage, optimize the
loss function $\ell_{\Phi, h^*}$ with $\Phi(t) = \exp(-t)$ to learn a
rejector. 
We set the cost $c$ to $0.03$ for SVHN, $0.05$ for CIFAR-10 and $0.15$ for CIFAR-100. We observe that the performance remains close for other neighboring values of $c$. We highlight this particular choice of cost because a cost value that is not too far from the best-in-class zero-one classification loss encourages in practice a reasonable amount of input instances to be abstained.

\paragraph{Metrics.}  We use as evaluation metrics the average
abstention loss, $\labs$ for predictor-rejector surrogate losses and
$\labsc$ for score-based abstention surrogate losses, which share the
same semantic meaning. It's important to emphasize that the two abstention losses, $\labs$ and $\labsc$ are indeed the same metric, albeit tailored for two distinct formulations. Consequently, their average numerical values can be directly compared. Note that both $\labs$ and $\labsc$  account for the zero-one misclassification error when the sample is accepted, and the cost when the sample is rejected. The reason they are adapted to the two formulations is due to the difference in the rejection method: $r(x) \leq 0$ in the predictor-rejector formulation and $\tilde h(x) = n + 1$ in the score-based abstention formulation. It should also be noted that the abstention loss serves as a comprehensive metric that integrates the rejection ratio and zero-one misclassification error on the accepted data, thereby providing a singular, fair ground for comparison in Table~\ref{tab:comparison-mabs}. Nevertheless, we include all three metrics in Table~\ref{tab:comparison-cifar10} as a detailed comparison.

\begin{table}[t]
\caption{Abstention loss, zero-one misclassification error on the accepted data and rejection ratio of our predictor-rejector surrogate losses against baselines: 
the state-of-the-art score-based abstention surrogate losses in
\citep{mozannar2020consistent,caogeneralizing} on CIFAR-10.}
    \label{tab:comparison-cifar10}
\begin{center}
\resizebox{\columnwidth}{!}{
    \begin{tabular}{@{\hspace{0pt}}llll@{\hspace{0pt}}}
    \toprule
      Method & Abstention loss & Misclassification error & Rejection ratio  \\
    \toprule
     \citep{mozannar2020consistent} &  4.48\% $\pm$ 0.10\% & 4.30\% $\pm$ 0.14\% & 25.99\% ± 0.41\%       \\
     \citep{caogeneralizing}  & 3.62\% $\pm$ 0.07\% & 3.08\% $\pm$ 0.10\%  & 28.27\% $\pm$ 0.18\%  \\
     single-stage predictor-rejector ($\ell_{\rm{mae}}$) & 3.64\% $\pm$ 0.05\% & 3.54\% $\pm$ 0.05\% &                  \textbf{17.21\% $\pm$ 0.22\%} \\
    two-stage predictor-rejector  & \textbf{3.31\% \!$\pm$ 0.02\%} &  \textbf{2.69\% $\pm$ 0.05\%}  &                 22.83\% $\pm$ 0.21\% \\
    \bottomrule
    \end{tabular}
    }
\end{center}
\end{table}

\section{Useful lemmas}
\label{app:general}

We first
introduce some notation before presenting a lemma that will be used in our proofs. Recall that we denote by
$\sfp(y \!\mid\! x) = \sD(Y = y \!\mid\! X = x)$ the conditional probability of
$Y=y$ given $X = x$. Thus, the generalization error for a general
abstention surrogate loss can be rewritten as $ \sE_{\sfL}(h, r) =
\mathbb{E}_{X} \bracket*{\sC_{\sfL}(h, r, x)} $, where $\sC_{\sfL}(h,
r,x)$ is the conditional risk of $\sfL$, defined by
\begin{align*}
\sC_{\sfL}(h, r, x) = \sum_{y\in \sY} \sfp(y \!\mid\! x) \sfL(h, r, x, y).
\end{align*}
We denote by $\sC_{\sfL}^*(\sH, \sR, x) = \inf_{h\in \sH, r\in
  \sR}\sC_{\sfL}(h, r, x)$ the best-in-class conditional risk of $\sfL$. Then,
the minimizability gap can be rewritten as follows:
\begin{align*}
\sM_{\sfL}(\sH, \sR)
= \sE^*_{\sfL}(\sH, \sR)
- \mathbb{E}_{X} \bracket* {\sC_{\sfL}^*(\sH, \sR, x)}.
\end{align*}
We further refer to $\sC_{\sfL}(h, r, x)-\sC_{\sfL}^*(\sH, \sR, x)$ as
the calibration gap and denote it by $\Delta\sC_{\sfL,\sH, \sR}(h,
r,x)$.  We first prove a lemma on the calibration gap of the general
abstention loss. For any $x \in \sX$, we define the set of labels
generated by hypotheses in $\sH$ as $\mathsf H(x) := \curl*{\hh(x)
  \colon h \in \sH}$.  We will consider hypothesis sets $\sR$ which
are \emph{regular for abstention}.
\begin{definition}[Regularity for Abstention]
We say that a hypothesis set $\sR$ is \emph{regular for abstention} 
if for any $x\in \sX$, there exist $f, g \in \sR$ 
such that $f(x)>0$ and 
$g(x)\leq 0$.
\end{definition}
In other words, if $\sR$ is regular for abstention, then, for any
instance $x$, there is an option to accept and an option to
reject. 

\subsection{Lemma~\ref{lemma:calibration_gap_general} and proof}
The following lemma characterizes the calibration gap of the
predictor-rejector abstention.
\begin{restatable}{lemma}{ConditionalRegret}
\label{lemma:calibration_gap_general}
Assume that $\sR$ is regular for abstention. For any $x \in \sX$,
the minimal conditional $\labs$-risk and
the calibration gap for $\labs$ can be expressed as follows:
\begin{align*}
\sC^*_{\labs}(\sH,\sR,x)  & = 1 - \max\curl*{\max_{y\in
    \mathsf H(x)}\sfp(y \!\mid\! x),1 - c},\\
 \Delta\sC_{\labs,\sH, \sR}(h, r, x) & =
\begin{cases}
\max\curl*{\max_{y\in \mathsf H(x)} \sfp(y \!\mid\! x),1 - c} - \sfp(\hh(x) \!\mid\! x) &  r(x)>0\\
\max\curl*{\max_{y\in \mathsf H(x)} \sfp(y \!\mid\! x)-1+c,0} & r(x)\leq 0.
\end{cases}
\end{align*}
\end{restatable}
\begin{proof}
By the definition, the conditional $\labs$-risk can be expressed as
follows:
\begin{align}
\label{eq:cond}
\sC_{\labs}(h, r, x)
=  \sum_{y\in \sY} \sfp(y \!\mid\! x) \1_{\hh(x)\neq y}\1_{r(x)> 0} + c \1_{r(x)\leq 0}
=
\begin{cases}
1-\sfp(\hh(x) \!\mid\! x) & \text{if } r(x)>0\\
c & \text{if } r(x)\leq 0.
\end{cases}
\end{align}
Since $\sR$ is regular for
abstention, the minimal conditional $\labs$-risk can be expressed as
follows:
\begin{align*}
\sC^*_{\labs}(\sH,\sR,x) = 1 - \max\curl*{\max_{y\in
    \mathsf H(x)}\sfp(y \!\mid\! x),1 - c},
\end{align*}
which proves the first part of the lemma. By the definition of the calibration gap, we have
\begin{align*}
\Delta\sC_{\labs,\sH,\sR}(h, r, x)
& = \sC_{\labs}(h, r, x)- \sC^*_{\labs}(\sH,\sR,x)\\
& =
\begin{cases}
\max\curl*{\max_{y\in \mathsf H(x)} \sfp(y \!\mid\! x),1 - c} - \sfp(\hh(x) \!\mid\! x) & \text{if } r(x)>0\\
\max\curl*{\max_{y\in \mathsf H(x)} \sfp(y \!\mid\! x)-1+c,0} & \text{if } r(x)\leq 0,
\end{cases}
\end{align*}
which completes the proof.
\end{proof}

\subsection{Lemma~\ref{lemma:aux-mabs} and proof}
The following lemma would be useful in the proofs for two-stage surrogate losses.
\begin{lemma}
\label{lemma:aux-mabs}
Assume that the following $\sR$-consistency bound holds for all $r \in \sR$ and any distribution,
\begin{equation*}
\sE_{\ell_{0-1}}(r) - \sE^*_{\ell_{0-1}}(\sR) + \sM_{\ell_{0-1}}(\sR) \leq \Gamma\paren*{\sE_{\Phi}(r) - \sE^*_{\Phi}(\sR) + \sM_{\Phi}(\sR)}.
\end{equation*}
Then, for any $p_1, p_2\in [0,1]$ such that $p_1 + p_2 =1$ and $x \in \sX$, we have
\begin{align*}
& p_1 1_{r(x) > 0} + p_2 1_{r(x) \leq 0} - \inf_{r \in \sR} \paren*{p_1 1_{r(x) > 0} + p_2 1_{r(x) \leq 0}}\\
&\quad \leq \Gamma \paren*{p_1 \Phi(-r(x)) + p_2 \Phi(r(x)) - \inf_{r \in \sR} \paren*{p_1 \Phi(-r(x)) + p_2 \Phi(r(x))}}
\end{align*}
\end{lemma}
\begin{proof}
For any $x \in \sX$, consider a distribution $\delta_{x}$ that concentrates on that point. Let $p_1 = \mathbb{P}(y = - 1 \mid x)$ and $p_2 = \mathbb{P}(y = + 1 \mid x)$. Then, by definition, $\sE_{\ell_{0-1}}(r) - \sE^*_{\ell_{0-1}}(\sR) + \sM_{\ell_{0-1}}(\sR)$ can be expressed as 
\begin{equation*}
\sE_{\ell_{0-1}}(r) - \sE^*_{\ell_{0-1}}(\sR) + \sM_{\ell_{0-1}}(\sR) = p_1 1_{r(x) > 0} + p_2 1_{r(x) \leq 0} - \inf_{r \in \sR} \paren*{p_1 1_{r(x) > 0} + 1_{r(x) \leq 0}}.
\end{equation*}
Similarly, $\sE_{\Phi}(r) - \sE^*_{\Phi}(\sR) + \sM_{\Phi}(\sR)$ can be expressed as
\begin{equation*}
\sE_{\Phi}(r) - \sE^*_{\Phi}(\sR) + \sM_{\Phi}(\sR) = p_1 \Phi(-r(x)) + p_2 \Phi(r(x)) - \inf_{r \in \sR} \paren*{p_1 \Phi(-r(x)) + p_2 \Phi(r(x))}.
\end{equation*}
Since the $\sR$-consistency bound holds by the assumption, we complete the proof.
\end{proof}

\section{Proofs of main theorems}

\subsection{Proof of negative result for single-stage surrogates (Theorem~\ref{Thm:negative-bound})}
\label{app:general-negativ}
\NegativeBound*
\begin{proof}
We prove by contradiction. Assume that the bound holds with some non-decreasing function $\Gamma$ with
$\lim_{t\to 0^{+}}\Gamma(t) = 0$, then, for all
$h\in \sH$, $r\in \sR$, and any distribution, $\sE_{\sfL}(h,
r)-\sE_{\sfL}^*(\sH, \sR) +\sM_{\sfL}(\sH, \sR) \to 0 \implies
\sE_{\labs}(h, r) - \sE_{\labs}^*(\sH, \sR) + \sM_{\labs}(\sH, \sR)
\to 0$. This further implies that for any $x\in \sX$, the minimizer
$h^*$ and $r^*$ of $\sC_{\sfL}(h, r, x)$ within $\sH$ and $\sR$ also
achieves the minimum of $\sC_{\labs}(h, r, x)$ within $\sH$ and
$\sR$. When $\sH$ is symmetric and complete, we have
$\mathsf{H}(x)=\sY$. Since $\sR$ is complete, $\sR$ is regular for
abstention. By Lemma~\ref{lemma:calibration_gap_general}, $h^*$ and
$r^*$ need to satisfy the following conditions:
\begin{equation}
\label{eq:star-general-cond}
\sfp(\hh^*(x)\!\mid\! x)
= \max_{y\in \sY} \sfp(y \!\mid\! x),\quad \sign(r^*(x))
= \sign\paren*{\max_{y\in \sY}\sfp(y \!\mid\! x) - (1 - c)}.
\end{equation}
Next, we will show that \eqref{eq:star-general-cond} contradicts the
assumption that there exists $x\in \sX$ such that \[\inf_{h \in \sH}
\E_y\bracket*{\ell(h,X, y) \mid X =x}\neq \frac{\beta\Psi\paren*{1 -
    \max_{y\in \sY}\sfp(y \!\mid\! x)}}{\alpha}.\]  By definition, the conditional
$\sfL$-risk can be expressed as follows:
\begin{equation*}
\sC_{\sfL}(h, r, x)  =  \exp(\alpha r(x))\mathbb{E}_y
  \bracket*{\ell(h,X, y) \mid X = x} + \Psi(c)\exp(-\beta r(x) ).
\end{equation*}
Then, for any fixed $r\in \sR$, $\alpha>0$ and $\beta>0$, we have
\begin{equation*}
  \inf_{h\in \sH}\sC_{\sfL}(h, r, x)
  = \exp(\alpha r(x))\inf_{h \in \sH} \mathbb{E}_y
\bracket*{\ell(h,X, y) \mid X = x}
+ \Psi(c)\exp(-\beta r(x)):=\sF(r(x)).
\end{equation*}
By taking the derivative, we obtain
\begin{equation*}
\sF'(r(x)) = \alpha \exp(\alpha r(x))\inf_{h \in \sH} \mathbb{E}_y
  \bracket*{\ell(h,X, y) \mid X = x}  - \beta \Psi(c) \exp(-\beta r(x) ).
\end{equation*}
Since $\sF(r(x))$ is convex with respect to $r(x)$, we know that
$\sF'(r(x))$ is non-decreasing with respect to $r(x)$. The equation \eqref{eq:star-general-cond} implies that
 $\sF(r(x))$ is attained at $r^*(x)$ such that
$\sign(r^*(x)) = \sign\paren*{\max_{y\in \sY}\sfp(y \!\mid\! x) - (1 - c)}$. Thus, we have
\begin{align*}
& \max_{y\in \sY}\sfp(y \!\mid\! x)\geq (1 - c) \implies r^*(x)\geq 0 \implies \sF'(0)\leq \sF'(r^*(x)) =0\\
& \max_{y\in \sY}\sfp(y \!\mid\! x)< (1 - c) \implies r^*(x)< 0 \implies \sF'(0)\geq \sF'(r^*(x))=0.
\end{align*}
This implies that
\begin{align*}
\alpha\inf_{h \in \sH} \mathbb{E}_y
  \bracket*{\ell(h,X, y) \mid X = x}  - \beta\Psi(c)\leq 0 \text{ whenever } \max_{y\in \sY}\sfp(y \!\mid\! x)\geq (1 - c)\\
\alpha\inf_{h \in \sH} \mathbb{E}_y
  \bracket*{\ell(h,X, y) \mid X = x}  - \beta\Psi(c) \geq 0 \text{ whenever } \max_{y\in \sY}\sfp(y \!\mid\! x)\leq (1 - c),
\end{align*} 
which leads to
\begin{align}
\label{eq:general-con}
\alpha\inf_{h \in \sH} \mathbb{E}_y
  \bracket*{\ell(h,X, y) \mid X = x}  - \beta\Psi(c) = 0 \text{ whenever } \max_{y\in \sY}\sfp(y \!\mid\! x) = (1 - c).
\end{align}
It is clear that \eqref{eq:general-con} contradicts the assumption that $\exists \,x\in \sX$ such that \[\inf_{h \in \sH} \E_y\bracket*{\ell(h,X, y) \mid X =x}\neq \frac{\beta\Psi\paren*{1 - \max_{y\in \sY}\sfp(y \!\mid\! x)}}{\alpha}.\] 
\end{proof}

\subsection{Proof of \texorpdfstring{$(\sH, \sR)$}{HR}-consistency
  bounds for single-stage surrogates
  (Theorem~\ref{Thm:spcific-loss-bound})}
\label{app:general-positive-single-stage}

\SpecificLossBound*
\begin{proof}
When $\sH$ is symmetric and complete, $\mathsf H(x)=\sY$. Since $\sR$
is complete, $\sR$ is regular for abstention. By
Lemma~\ref{lemma:calibration_gap_general},
\begin{equation}
\label{eq:calibration_gap_general}
\begin{aligned}
\sC^*_{\labs}(\sH,\sR,x)  & = 1 - \max\curl*{\max_{y\in
    \sY}\sfp(y \!\mid\! x),1 - c},\\
 \Delta\sC_{\labs,\sH, \sR}(h, r, x) & =
\begin{cases}
\max\curl*{\max_{y\in \sY} \sfp(y \!\mid\! x),1 - c} - \sfp(\hh(x) \!\mid\! x) &  r(x)>0\\
\max\curl*{\max_{y\in \sY} \sfp(y \!\mid\! x)-1+c,0} & r(x)\leq 0.
\end{cases}
\end{aligned}
\end{equation}
Then, the idea of proof for each loss $\ell$ is similar, which
consists of the analysis in four cases depending on the sign of
$\max_{y\in \sY}\sfp(y \!\mid\! x)-(1-c)$ and the sign of $r(x)$, as shown below.

\textbf{Mean absolute error loss: $\ell= \ell_{\rm{mae}}$.}
\ignore{By Lemma~\ref{lemma:calibration_gap_general}, $\sC_{\labs}(h, r, x) =
\begin{cases}
1-\sfp(\hh(x) \!\mid\! x) & r(x)\leq 0\\
c & r(x)> 0
\end{cases}
$ and $\sC^*_{\labs}(\sH, \sR, x) =1 - \max\curl*{\max_{y\in
    \sY}\sfp(y \!\mid\! x),1 - c}$.} When $\ell= \ell_{\rm{mae}}$ with $\Psi(t)=t$, let $s_h(x, y) = \frac{e^{h(x, y)}}{\sum_{y'\in
    \sY}e^{h(x, y')}}$, by
the assumption that $\sH$ is symmetric and complete and $\sR$ is complete, we obtain $\inf_{h \in \sH} \E_y\bracket*{\ell_{\rm{mae}}(h,X, y) \mid X =x} =1 - \max_{y\in \sY}\sfp(y \!\mid\! x)$ and
\begin{align*}
\sC_{\sfL}(h, r, x) &= \sum_{y\in \sY}\sfp(y \!\mid\! x)(1-s_h(x, y))e^{\alpha r(x)} + c e^{-\alpha r(x)}\\
\sC^*_{\sfL}(\sH, \sR, x) & =2\sqrt{c\paren*{1 - \max_{y\in \sY}\sfp(y \!\mid\! x)}}.
\end{align*}
Note that for any $h\in \sH$,
\begin{align*}
& \sum_{y\in \sY}\sfp(y \!\mid\! x)(1-s_h(x, y)) - \paren*{1 - \max_{y\in \sY}\sfp(y \!\mid\! x)}\\
& = \max_{y\in \sY}\sfp(y \!\mid\! x) - \sum_{y\in \sY}\sfp(y \!\mid\! x)s_h(x, y)\\
& \geq \max_{y\in \sY}\sfp(y \!\mid\! x) - \sfp(\hh(x) \!\mid\! x)s_h(x, \hh(x))- \max_{y\in \sY}\sfp(y \!\mid\! x)\paren*{1-s_h(x, \hh(x))}
\tag{$\sfp(y \!\mid\! x)\leq \max_{y\in \sY}\sfp(y \!\mid\! x)$, $\forall y\neq \hh(x)$}\\
& = s_h(x, \hh(x))\paren*{\max_{y\in \sY}\sfp(y \!\mid\! x)-\sfp(\hh(x) \!\mid\! x)}\\
& \geq \frac{1}{n}\paren*{\max_{y\in \sY}\sfp(y \!\mid\! x)-\sfp(\hh(x) \!\mid\! x)}
\tag{$s_h(x, \hh(x)) = \max_{y\in \sY}s_h(x, y)\geq \frac{1}{n}$}.
\end{align*}
We will then analyze the following four cases.
     \paragraph{(i)} $\max_{y\in \sY}\sfp(y \!\mid\! x) > (1 - c)$ and $r(x)>0$. In this case, by \eqref{eq:calibration_gap_general},
      we have $\sC^*_{\labs}(\sH, \sR, x) =1 - \max_{y\in
        \sY}\sfp(y \!\mid\! x)$ and $\Delta\sC_{\labs,\sH,
        \sR}(h, r, x) = \max_{y\in \sY}\sfp(y \!\mid\! x)-p\paren*{x, \hh(x)}$. For
      the surrogate loss, we have
    \begin{align*}
    & \Delta\sC_{\sfL,\sH, \sR}(h, r, x)\\
    & = \sum_{y\in \sY}\sfp(y \!\mid\! x)(1-s_h(x, y))e^{\alpha r(x)} + c e^{-\alpha r(x)} - 2\sqrt{c\paren*{1 - \max_{y\in \sY}\sfp(y \!\mid\! x)}}\\
    & \geq \sum_{y\in \sY}\sfp(y \!\mid\! x)(1-s_h(x, y))e^{\alpha r(x)} + c e^{-\alpha r(x)} - \paren*{1 - \max_{y\in \sY}\sfp(y \!\mid\! x)}e^{\alpha r(x)} - c e^{-\alpha r(x)}
    \tag{AM–GM inequality}\\
    & \geq \sum_{y\in \sY}\sfp(y \!\mid\! x)(1-s_h(x, y)) - \paren*{1 - \max_{y\in \sY}\sfp(y \!\mid\! x)}
    \tag{$r(x)>0$}\\
    & \geq \frac{1}{n}\paren*{\max_{y\in \sY}\sfp(y \!\mid\! x)-\sfp(\hh(x) \!\mid\! x)}\\
    \tag{$\sum_{y\in \sY}\sfp(y \!\mid\! x)(1-s_h(x, y)) - \paren*{1 - \max_{y\in \sY}\sfp(y \!\mid\! x)}\geq \frac{1}{n}\paren*{\max_{y\in \sY}\sfp(y \!\mid\! x)-\sfp(\hh(x) \!\mid\! x)}$}\\
    & = \frac{1}{n}\Delta\sC_{\labs,\sH, \sR}(h, r, x).
    \end{align*}
     Therefore, 
    \begin{align*}
\sE_{\labs}(h, r) - \sE_{\labs}^*(\sH, \sR) + \sM_{\labs}(\sH, \sR) & = \mathbb{E}_{X}\bracket*{\Delta\sC_{\labs,\sH, \sR}(h, r, x)}\\
& \leq \mathbb{E}_{X}\bracket*{\Gamma_1\paren*{\Delta\sC_{\sfL,\sH, \sR}(h, r, x)}}\\
& \leq \Gamma_1\paren*{\mathbb{E}_{X}\bracket*{\Delta\sC_{\sfL,\sH, \sR}(h, r, x)}}
\tag{$\Gamma_1$ is concave}\\
& = \Gamma_1\paren*{\sE_{\sfL}(h, r)-\sE_{\sfL}^*(\sH, \sR) +\sM_{\sfL}(\sH, \sR)}
\end{align*}
where $\Gamma_1(t)= n\,t$.
  \paragraph{(ii)} $\max_{y\in \sY}\sfp(y \!\mid\! x) \leq (1 - c)$ and $r(x)>0$. In this
      case, by \eqref{eq:calibration_gap_general}, we have $\sC^*_{\labs}(\sH, \sR, x) =c$ and
      $\Delta\sC_{\labs,\sH, \sR}(h, r, x) =1-c-p\paren*{x,
        \hh(x)}$. For the surrogate loss, we have
    \begin{align*}
    & \Delta\sC_{\sfL,\sH, \sR}(h, r, x) \\
    & = \sum_{y\in \sY}\sfp(y \!\mid\! x)(1-s_h(x, y))e^{\alpha r(x)} + c e^{-\alpha r(x)} - 2\sqrt{c\paren*{1 - \max_{y\in \sY}\sfp(y \!\mid\! x)}}\\
    & \geq \sum_{y\in \sY}\sfp(y \!\mid\! x)(1-s_h(x, y))e^{\alpha r(x)} + c e^{-\alpha r(x)} - 2\sqrt{c\paren*{\sum_{y\in \sY}\sfp(y \!\mid\! x)(1-s_h(x, y))}}
   \tag{$\sum_{y\in \sY}\sfp(y \!\mid\! x)(1-s_h(x, y))\geq 1 - \max_{y\in \sY}\sfp(y \!\mid\! x)$}\\
    & \geq \sum_{y\in \sY}\sfp(y \!\mid\! x)(1-s_h(x, y)) + c - 2\sqrt{c\paren*{\sum_{y\in \sY}\sfp(y \!\mid\! x)(1-s_h(x, y))}}
    \tag{increasing for $r(x)\geq 0$}\\
    & = \paren*{\sqrt{\sum_{y\in \sY}\sfp(y \!\mid\! x)(1-s_h(x, y))}-\sqrt{c}}^2\\
    & = \paren*{\frac{\sum_{y\in \sY}\sfp(y \!\mid\! x)(1-s_h(x, y))-c}{\sqrt{\sum_{y\in \sY}\sfp(y \!\mid\! x)(1-s_h(x, y))}+\sqrt{c}}}^2\\
    & \geq  \paren*{\frac{\sum_{y\in \sY}\sfp(y \!\mid\! x)(1-s_h(x, y))-\paren*{1 - \max_{y\in \sY}\sfp(y \!\mid\! x)}+\paren*{1 - \max_{y\in \sY}\sfp(y \!\mid\! x)-c}}{2}}^2
    \tag{$\sqrt{\sum_{y\in \sY}\sfp(y \!\mid\! x)(1-s_h(x, y))}+\sqrt{c}\leq 2$}\\
    & \geq  \paren*{\frac{\frac{1}{n}\paren*{\max_{y\in \sY}\sfp(y \!\mid\! x)-\sfp(\hh(x) \!\mid\! x)}+\frac{1}{n}\paren*{1 - \max_{y\in \sY}\sfp(y \!\mid\! x)-c}}{2}}^2
    \tag{$\sum_{y\in \sY}\sfp(y \!\mid\! x)(1-s_h(x, y)) - \paren*{1 - \max_{y\in \sY}\sfp(y \!\mid\! x)}
    \geq \frac{1}{n}\paren*{\max_{y\in \sY}\sfp(y \!\mid\! x)-\sfp(\hh(x) \!\mid\! x)}$}\\
    & = \frac{1}{4n^2}\paren*{1-c-p\paren*{x,\hh(x)}}^2\\
    & = \frac{\Delta\sC_{\labs,\sH, \sR}(h, r, x)^2}{4n^2}
    \end{align*}
    Therefore, 
    \begin{align*}
\sE_{\labs}(h, r) - \sE_{\labs}^*(\sH, \sR) + \sM_{\labs}(\sH, \sR) & = \mathbb{E}_{X}\bracket*{\Delta\sC_{\labs,\sH, \sR}(h, r, x)}\\
& \leq \mathbb{E}_{X}\bracket*{\Gamma_2\paren*{\Delta\sC_{\sfL,\sH, \sR}(h, r, x)}}\\
& \leq \Gamma_2\paren*{\mathbb{E}_{X}\bracket*{\Delta\sC_{\sfL,\sH, \sR}(h, r, x)}}
\tag{$\Gamma_2$ is concave}\\
& = \Gamma_2\paren*{\sE_{\sfL}(h, r)-\sE_{\sfL}^*(\sH, \sR) +\sM_{\sfL}(\sH, \sR)}
\end{align*}
where $\Gamma_2(t)=2n\sqrt{t}$.
  \paragraph{(iii)} $\max_{y\in \sY}\sfp(y \!\mid\! x) \leq  (1 - c)$ and $r(x)\leq 0$. In this case, by \eqref{eq:calibration_gap_general}, we have $\sC^*_{\labs}(\sH, \sR, x) =c$ and $\Delta\sC_{\labs,\sH, \sR}(h, r, x)=0$, which implies that 
  $
  \sE_{\labs}(h, r) - \sE_{\labs}^*(\sH, \sR) + \sM_{\labs}(\sH, \sR)\\
  = \mathbb{E}_{X}\bracket*{\Delta\sC_{\labs,\sH, \sR}(h, r, x)}
   = 0
  \leq \Gamma\paren*{\sE_{\sfL}(h, r)-\sE_{\sfL}^*(\sH, \sR) +\sM_{\sfL}(\sH, \sR)}    
  $
  for any $\Gamma\geq 0$.
  \paragraph{(iv)} $\max_{y\in \sY}\sfp(y \!\mid\! x) > (1 - c)$ and $r(x)\leq 0$. In this
      case, by \eqref{eq:calibration_gap_general}, we have $\sC^*_{\labs}(\sH, \sR, x) =1 - \max_{y\in
        \sY}\sfp(y \!\mid\! x)$ and $\Delta\sC_{\labs,\sH,
        \sR}(h, r, x) =\max_{y\in \sY} \sfp(y \!\mid\! x)-1+c$. For the
      surrogate loss, we have
    \begin{align*}
    &\Delta\sC_{\sfL,\sH, \sR}(h, r, x)\\
    & = \sum_{y\in \sY}\sfp(y \!\mid\! x)(1-s_h(x, y))e^{\alpha r(x)} + c e^{-\alpha r(x)} - 2\sqrt{c\paren*{1 - \max_{y\in \sY}\sfp(y \!\mid\! x)}}\\
    & \geq \paren*{1 - \max_{y\in \sY}\sfp(y \!\mid\! x)}e^{\alpha r(x)} + c e^{-\alpha r(x)} - 2\sqrt{c\paren*{1 - \max_{y\in \sY}\sfp(y \!\mid\! x)}}
    \tag{$\sum_{y\in \sY}\sfp(y \!\mid\! x)(1-s_h(x, y))\geq 1 - \max_{y\in \sY}\sfp(y \!\mid\! x)$}\\
    & \geq 1 - \max_{y\in \sY}\sfp(y \!\mid\! x)+c - 2\sqrt{c\paren*{1 - \max_{y\in \sY}\sfp(y \!\mid\! x)}}
    \tag{decreasing for $r(x)\leq 0$}\\
    & = \paren*{\sqrt{1 - \max_{y\in \sY}\sfp(y \!\mid\! x)}-\sqrt{c}}^2\\
    & = \paren*{\frac{1 - \max_{y\in \sY}\sfp(y \!\mid\! x)-c}{\sqrt{1 - \max_{y\in \sY}\sfp(y \!\mid\! x)}+\sqrt{c}}}^2\\
    & \geq \paren*{\frac{\max_{y\in \sY} \sfp(y \!\mid\! x)-1+c}{2}}^2
    \tag{$\sqrt{1 - \max_{y\in \sY}\sfp(y \!\mid\! x)}+\sqrt{c}\leq 2$}\\
    & = \frac{\Delta\sC_{\labs,\sH, \sR}(h, r, x)^2}{4}.
    \end{align*}
     Therefore, 
    \begin{align*}
      \sE_{\labs}(h, r) - \sE_{\labs}^*(\sH, \sR) + \sM_{\labs}(\sH, \sR)
      & = \mathbb{E}_{X}\bracket*{\Delta\sC_{\labs,\sH, \sR}(h, r, x)}\\
      & \leq \mathbb{E}_{X}\bracket*{\Gamma_3\paren*{\Delta\sC_{\sfL,\sH, \sR}(h, r, x)}}\\
      & \leq \Gamma_3\paren*{\mathbb{E}_{X}\bracket*{\Delta\sC_{\sfL,\sH, \sR}(h, r, x)}}
      \tag{$\Gamma_3$ is concave}\\
& = \Gamma_3\paren*{\sE_{\sfL}(h, r)-\sE_{\sfL}^*(\sH, \sR) +\sM_{\sfL}(\sH, \sR)}
\end{align*}
where $\Gamma_3(t)=2\sqrt{t}$.

Overall, we obtain
\begin{align*}
\sE_{\labs}(h, r) - \sE_{\labs}^*(\sH, \sR) + \sM_{\labs}(\sH, \sR) \leq \Gamma\paren*{\sE_{\sfL}(h, r)-\sE_{\sfL}^*(\sH, \sR) +\sM_{\sfL}(\sH, \sR)}
\end{align*}
where $\Gamma(t)=\max\curl*{\Gamma_1(t),\Gamma_2(t),\Gamma_3(t)}= \max\curl*{2n\sqrt{t},n\,t}$, which completes the proof.

\textbf{$\rho$-Margin loss: $\ell= \ell_{\rho}$.}
When $\ell= \ell_{\rho}$ with $\Psi(t)=t$, by
the assumption that $\sH$ is symmetric and complete and $\sR$ is complete, we obtain\ignore{$\sC_{\ell_{\rho}}(h, x) = \sum_{y\in \sY}\sfp(y \!\mid\! x)\min\curl*{\max\curl*{0,1 - \frac{\rho_h(x,y)}{\rho}},1}
=1- \min\curl*{1,\frac{\rho_h(x, \hh(x))}{\rho}}\,\sfp(\hh(x) \!\mid\! x)$}
$
\inf_{h \in \sH} \E_y\bracket*{\ell_{\rho}(h,X, y) \mid X =x}  =1 - \max_{y\in \sY}\sfp(y \!\mid\! x)$ and
\begin{align*}
\sC_{\sfL}(h, r, x) &= \sum_{y\in \sY}\sfp(y \!\mid\! x)\min\curl*{\max\curl*{0,1 - \frac{\rho_h(x,y)}{\rho}},1}e^{\alpha r(x)} + c e^{-\alpha r(x)}\\
&=\paren*{1- \min\curl*{1,\frac{\rho_h(x, \hh(x))}{\rho}}\,\sfp(\hh(x) \!\mid\! x)}e^{\alpha r(x)} + c e^{-\alpha r(x)}\\
\sC^*_{\sfL}(\sH, \sR, x) & =2\sqrt{c\paren*{1 - \max_{y\in \sY}\sfp(y \!\mid\! x)}}.
\end{align*}
where $\rho_h(x,y) = h(x, y) - \max_{y' \neq y} h(x, y')$ is the margin. Note that for any $h\in \sH$,
\begin{align*}
& 1- \min\curl*{1,\frac{\rho_h(x, \hh(x))}{\rho}}\,\sfp(\hh(x) \!\mid\! x) - \paren*{1 - \max_{y\in \sY}\sfp(y \!\mid\! x)}\\
& = \max_{y\in \sY}\sfp(y \!\mid\! x) - \min\curl*{1,\frac{\rho_h(x, \hh(x))}{\rho}}\,\sfp(\hh(x) \!\mid\! x)\\
& \geq \max_{y\in \sY}\sfp(y \!\mid\! x)-\sfp(\hh(x) \!\mid\! x)
\tag{$\min\curl*{1,\frac{\rho_h(x, \hh(x))}{\rho}}\leq 1$}.
\end{align*}
We will then analyze the following four cases.
\paragraph{(i)} $\max_{y\in \sY}\sfp(y \!\mid\! x) \leq (1 - c)$ and $r(x)>0$. In this
      case, by \eqref{eq:calibration_gap_general}, we have $\sC^*_{\labs}(\sH, \sR, x) =c$ and
      $\Delta\sC_{\labs,\sH, \sR}(h, r, x) =1-c-p\paren*{x,
        \hh(x)}$. For the surrogate loss, we have
    \begin{align*}
    &\Delta\sC_{\sfL,\sH, \sR}(h, r, x)\\ 
    & = \paren*{1- \min\curl*{1,\frac{\rho_h(x, \hh(x))}{\rho}}\,\sfp(\hh(x) \!\mid\! x)}e^{\alpha r(x)} + c e^{-\alpha r(x)} - 2\sqrt{c\paren*{1 - \max_{y\in \sY}\sfp(y \!\mid\! x)}}\\
    & \geq \paren*{1- \min\curl*{1,\frac{\rho_h(x, \hh(x))}{\rho}}\,\sfp(\hh(x) \!\mid\! x)}e^{\alpha r(x)} + c e^{-\alpha r(x)}\\
    & \qquad - 2\sqrt{c\paren*{1- \min\curl*{1,\frac{\rho_h(x, \hh(x))}{\rho}}\,\sfp(\hh(x) \!\mid\! x)}}
   \tag{$1- \min\curl*{1,\frac{\rho_h(x, \hh(x))}{\rho}}\,\sfp(\hh(x) \!\mid\! x)\geq 1 - \max_{y\in \sY}\sfp(y \!\mid\! x)$}\\
    & \geq 1- \min\curl*{1,\frac{\rho_h(x, \hh(x))}{\rho}}\,\sfp(\hh(x) \!\mid\! x) + c -  2\sqrt{c\paren*{1- \min\curl*{1,\frac{\rho_h(x, \hh(x))}{\rho}}\,\sfp(\hh(x) \!\mid\! x)}}
    \tag{increasing for $r(x)\geq 0$}\\
    & = \paren*{\frac{1- \min\curl*{1,\frac{\rho_h(x, \hh(x))}{\rho}}\,\sfp(\hh(x) \!\mid\! x)-c}{\sqrt{1- \min\curl*{1,\frac{\rho_h(x, \hh(x))}{\rho}}\,\sfp(\hh(x) \!\mid\! x)}+\sqrt{c}}}^2\\
    & \geq  \paren*{\frac{\1- \min\curl*{1,\frac{\rho_h(x, \hh(x))}{\rho}}\,\sfp(\hh(x) \!\mid\! x)-\paren*{1 - \max_{y\in \sY}\sfp(y \!\mid\! x)}+\paren*{1 - \max_{y\in \sY}\sfp(y \!\mid\! x)-c}}{2}}^2
    \tag{$\sqrt{1- \min\curl*{1,\frac{\rho_h(x, \hh(x))}{\rho}}\,\sfp(\hh(x) \!\mid\! x)}+\sqrt{c}\leq 2$}\\
    & \geq  \paren*{\frac{\max_{y\in \sY}\sfp(y \!\mid\! x)-\sfp(\hh(x) \!\mid\! x)+\paren*{1 - \max_{y\in \sY}\sfp(y \!\mid\! x)-c}}{2}}^2
    \tag{$1- \min\curl*{1,\frac{\rho_h(x, \hh(x))}{\rho}}\,\sfp(\hh(x) \!\mid\! x) - \paren*{1 - \max_{y\in \sY}\sfp(y \!\mid\! x)}\geq \max_{y\in \sY}\sfp(y \!\mid\! x)-\sfp(\hh(x) \!\mid\! x)$}\\
    & = \frac{1}{4}\paren*{1-c-p\paren*{x,\hh(x)}}^2\\
    & = \frac{\Delta\sC_{\labs,\sH, \sR}(h, r, x)^2}{4}
    \end{align*}
    Therefore, 
    \begin{align*}
\sE_{\labs}(h, r) - \sE_{\labs}^*(\sH, \sR) + \sM_{\labs}(\sH, \sR) & = \mathbb{E}_{X}\bracket*{\Delta\sC_{\labs,\sH, \sR}(h, r, x)}\\
& \leq \mathbb{E}_{X}\bracket*{\Gamma_2\paren*{\Delta\sC_{\sfL,\sH, \sR}(h, r, x)}}\\
& \leq \Gamma_2\paren*{\mathbb{E}_{X}\bracket*{\Delta\sC_{\sfL,\sH, \sR}(h, r, x)}}
\tag{$\Gamma_2$ is concave}\\
& = \Gamma_2\paren*{\sE_{\sfL}(h, r)-\sE_{\sfL}^*(\sH, \sR) +\sM_{\sfL}(\sH, \sR)}
\end{align*}
where $\Gamma_2(t)=2\sqrt{t}$.

\paragraph{(ii)} $\max_{y\in \sY}\sfp(y \!\mid\! x) > (1 - c)$ and $r(x)>0$. In this case, by \eqref{eq:calibration_gap_general},
      we have $\sC^*_{\labs}(\sH, \sR, x) =1 - \max_{y\in
        \sY}\sfp(y \!\mid\! x)$ and $\Delta\sC_{\labs,\sH,
        \sR}(h, r, x) = \max_{y\in \sY}\sfp(y \!\mid\! x)-p\paren*{x, \hh(x)}$. For
      the surrogate loss, we have
    \begin{align*}
    &\Delta\sC_{\sfL,\sH, \sR}(h, r, x)\\ 
    & = \paren*{1- \min\curl*{1,\frac{\rho_h(x, \hh(x))}{\rho}}\,\sfp(\hh(x) \!\mid\! x)}e^{\alpha r(x)} + c e^{-\alpha r(x)} - 2\sqrt{c\paren*{1 - \max_{y\in \sY}\sfp(y \!\mid\! x)}}\\
    & \geq \paren*{1- \min\curl*{1,\frac{\rho_h(x, \hh(x))}{\rho}}\,\sfp(\hh(x) \!\mid\! x)}e^{\alpha r(x)} + c e^{-\alpha r(x)} - \paren*{1 - \max_{y\in \sY}\sfp(y \!\mid\! x)}e^{\alpha r(x)} - c e^{-\alpha r(x)}
    \tag{AM–GM inequality}\\
    & \geq 1- \min\curl*{1,\frac{\rho_h(x, \hh(x))}{\rho}}\,\sfp(\hh(x) \!\mid\! x) - \paren*{1 - \max_{y\in \sY}\sfp(y \!\mid\! x)}
    \tag{$r(x)>0$}\\
    & \geq \max_{y\in \sY}\sfp(y \!\mid\! x)-\sfp(\hh(x) \!\mid\! x)\\
    \tag{$1- \min\curl*{1,\frac{\rho_h(x, \hh(x))}{\rho}}\,\sfp(\hh(x) \!\mid\! x)- \paren*{1 - \max_{y\in \sY}\sfp(y \!\mid\! x)}\geq \max_{y\in \sY}\sfp(y \!\mid\! x)-\sfp(\hh(x) \!\mid\! x)$}\\
    & = \Delta\sC_{\labs,\sH, \sR}(h, r, x).
    \end{align*}
     Thus, $\sE_{\labs}(h, r) - \sE_{\labs}^*(\sH, \sR) + \sM_{\labs}(\sH, \sR)  = \mathbb{E}_{X}\bracket*{\Delta\sC_{\labs,\sH, \sR}(h, r, x)}
\leq \mathbb{E}_{X}\bracket*{\Gamma_1\paren*{\Delta\sC_{\sfL,\sH, \sR}(h, r, x)}}\\
 \leq \Gamma_1\paren*{\mathbb{E}_{X}\bracket*{\Delta\sC_{\sfL,\sH, \sR}(h, r, x)}}
= \Gamma_1\paren[big]{\sE_{\sfL}(h, r)-\sE_{\sfL}^*(\sH, \sR) +\sM_{\sfL}(\sH, \sR)}$, where $\Gamma_1(t)= t$ is concave.
\ignore{\begin{align*}
\sE_{\labs}(h, r) - \sE_{\labs}^*(\sH, \sR) + \sM_{\labs}(\sH, \sR) & = \mathbb{E}_{X}\bracket*{\Delta\sC_{\labs,\sH, \sR}(h, r, x)}\\
& \leq \mathbb{E}_{X}\bracket*{\Gamma_1\paren*{\Delta\sC_{\sfL,\sH, \sR}(h, r, x)}}\\
& \leq \Gamma_1\paren*{\mathbb{E}_{X}\bracket*{\Delta\sC_{\sfL,\sH, \sR}(h, r, x)}}
\tag{$\Gamma_1$ is concave}\\
& = \Gamma_1\paren*{\sE_{\sfL}(h, r)-\sE_{\sfL}^*(\sH, \sR) +\sM_{\sfL}(\sH, \sR)}
\end{align*}}

\paragraph{(iii)} $\max_{y\in \sY}\sfp(y \!\mid\! x) \leq  (1 - c)$ and $r(x)\leq 0$. In this case, by \eqref{eq:calibration_gap_general}, we have $\sC^*_{\labs}(\sH, \sR, x) =c$ and $\Delta\sC_{\labs,\sH, \sR}(h, r, x)=0$, which implies that $\sE_{\labs}(h, r) - \sE_{\labs}^*(\sH, \sR) + \sM_{\labs}(\sH, \sR)\\
= \mathbb{E}_{X}\bracket*{\Delta\sC_{\labs,\sH, \sR}(h, r, x)}=0\leq \Gamma\paren*{\sE_{\sfL}(h, r)-\sE_{\sfL}^*(\sH, \sR) +\sM_{\sfL}(\sH, \sR)}$ for any $\Gamma\geq 0$.
\paragraph{(iv)} $\max_{y\in \sY}\sfp(y \!\mid\! x) > (1 - c)$ and $r(x)\leq 0$. In this
      case, by \eqref{eq:calibration_gap_general}, we have $\sC^*_{\labs}(\sH, \sR, x) =1 - \max_{y\in
        \sY}\sfp(y \!\mid\! x)$ and $\Delta\sC_{\labs,\sH,
        \sR}(h, r, x) =\max_{y\in \sY} \sfp(y \!\mid\! x)-1+c$. For the
      surrogate loss, we have
    \begin{align*}
    &\Delta\sC_{\sfL,\sH, \sR}(h, r, x)\\ 
    & = \paren*{1- \min\curl*{1,\frac{\rho_h(x, \hh(x))}{\rho}}\,\sfp(\hh(x) \!\mid\! x)}e^{\alpha r(x)} + c e^{-\alpha r(x)} - 2\sqrt{c\paren*{1 - \max_{y\in \sY}\sfp(y \!\mid\! x)}}\\
    & \geq \paren*{1 - \max_{y\in \sY}\sfp(y \!\mid\! x)}e^{\alpha r(x)} + c e^{-\alpha r(x)} - 2\sqrt{c\paren*{1 - \max_{y\in \sY}\sfp(y \!\mid\! x)}}
    \tag{$1- \min\curl*{1,\frac{\rho_h(x, \hh(x))}{\rho}}\,\sfp(\hh(x) \!\mid\! x)\geq 1 - \max_{y\in \sY}\sfp(y \!\mid\! x)$}\\
    & \geq 1 - \max_{y\in \sY}\sfp(y \!\mid\! x)+c - 2\sqrt{c\paren*{1 - \max_{y\in \sY}\sfp(y \!\mid\! x)}}
    \tag{decreasing for $r(x)\leq 0$}\\
    & = \paren*{\sqrt{1 - \max_{y\in \sY}\sfp(y \!\mid\! x)}-\sqrt{c}}^2\\
    & = \paren*{\frac{1 - \max_{y\in \sY}\sfp(y \!\mid\! x)-c}{\sqrt{1 - \max_{y\in \sY}\sfp(y \!\mid\! x)}+\sqrt{c}}}^2\\
    & \geq \paren*{\frac{\max_{y\in \sY} \sfp(y \!\mid\! x)-1+c}{2}}^2
    \tag{$\sqrt{1 - \max_{y\in \sY}\sfp(y \!\mid\! x)}+\sqrt{c}\leq 2$}\\
    & = \frac{\Delta\sC_{\labs,\sH, \sR}(h, r, x)^2}{4}.
    \end{align*}
     Therefore, 
    \begin{align*}
      \sE_{\labs}(h, r) - \sE_{\labs}^*(\sH, \sR) + \sM_{\labs}(\sH, \sR)
      & = \mathbb{E}_{X}\bracket*{\Delta\sC_{\labs,\sH, \sR}(h, r, x)}\\
      & \leq \mathbb{E}_{X}\bracket*{\Gamma_3\paren*{\Delta\sC_{\sfL,\sH, \sR}(h, r, x)}}\\
      & \leq \Gamma_3\paren*{\mathbb{E}_{X}\bracket*{\Delta\sC_{\sfL,\sH, \sR}(h, r, x)}}
      \tag{$\Gamma_3$ is concave}\\
& = \Gamma_3\paren*{\sE_{\sfL}(h, r)-\sE_{\sfL}^*(\sH, \sR) +\sM_{\sfL}(\sH, \sR)}
\end{align*}
where $\Gamma_3(t)=2\sqrt{t}$.

Overall, we obtain
\begin{align*}
\sE_{\labs}(h, r) - \sE_{\labs}^*(\sH, \sR) + \sM_{\labs}(\sH, \sR) \leq \Gamma\paren*{\sE_{\sfL}(h, r)-\sE_{\sfL}^*(\sH, \sR) +\sM_{\sfL}(\sH, \sR)}
\end{align*}
where $\Gamma(t)=\max\curl*{\Gamma_1(t),\Gamma_2(t),\Gamma_3(t)}= \max\curl*{2\sqrt{t},t}$, which completes the proof.

\textbf{Constrained $\rho$-hinge loss: $\ell=\ell_{\rho-\mathrm{hinge}}$.}
When $\ell= \ell_{\rm{\rho-\mathrm{hinge}}}$ with $\Psi(t)=nt$, by
the assumption that $\sH$ is symmetric and complete and $\sR$ is complete, we have $\inf_{h \in \sH} \E_y\bracket*{\ell_{\rm{\rho-\mathrm{hinge}}}(h,X, y) \mid X =x}\\ =n\paren*{1 - \max_{y\in \sY}\sfp(y \!\mid\! x)}$ and with the constraint $\sum_{y\in \sY}h(x, y) = 0$,
\begin{align*}
\sC_{\sfL}(h, r, x) & = \sum_{y\in \sY}\sfp(y \!\mid\! x)\sum_{y'\neq y}\max\curl[\big]{0,1 + \frac{h(x, y')}{\rho}}e^{\alpha r(x)} + n c e^{-\alpha r(x)}\\
& = \sum_{y\in \sY} \paren*{1-\sfp(y \!\mid\! x)}\max\curl*{0,1 + \frac{h(x, y)}{\rho}}e^{\alpha r(x)} + n c e^{-\alpha r(x)}\\
\sC^*_{\sfL}(\sH, \sR, x) &=2\sqrt{n^2c\paren*{1 - \max_{y\in \sY}\sfp(y \!\mid\! x)}}.
\end{align*}
Take $h_{\rho}\in \sH$ such that $ h_{\rho}(x,y) = 
\begin{cases}
  h(x, y) & \text{if $y \not \in \curl*{y_{\max}, \hh(x)}$}\\
  -\rho & \text{if $y = \hh(x)$}\\
  h(x, y_{\max})+h(x,\hh(x))+\rho& \text{if $y = y_{\max}$}.
\end{cases}   $
with the constraint $\sum_{y\in \sY}h_{\rho}(x, y) = 0$, where $y_{\max}= \argmax_{y\in \sY}\sfp(y \!\mid\! x)$.
Note that for any $h\in \sH$,
\begin{align*}
& \sum_{y\in \sY} \paren*{1-\sfp(y \!\mid\! x)}\max\curl*{0,1 + \frac{h(x, y)}{\rho}} - n\paren*{1 - \max_{y\in \sY}\sfp(y \!\mid\! x)}\\
& \geq  \sum_{y\in \sY} \paren*{1-\sfp(y \!\mid\! x)}\min\curl*{n,\max\curl*{0,1 + \frac{h(x, y)}{\rho}}} - n\paren*{1 - \max_{y\in \sY}\sfp(y \!\mid\! x)}\\
& \geq \sum_{y\in \sY} \paren*{1-\sfp(y \!\mid\! x)}\min\curl*{n,\max\curl*{0,1 + \frac{h(x, y)}{\rho}}}\\
&\qquad - \sum_{y\in \sY} \paren*{1-\sfp(y \!\mid\! x)}\min\curl*{n,\max\curl*{0,1 + \frac{ h_{\rho}(x, y)}{\rho}}}\\ 
& \geq \min\curl*{n,1+\frac{h(x,\hh(x))}{\rho}}\paren*{\max_{y\in \sY}\sfp(y \!\mid\! x)-\sfp(\hh(x) \!\mid\! x)}
\tag{plug in $h_{\rho}(x,y)$}\\
& \geq \max_{y\in \sY}\sfp(y \!\mid\! x)-\sfp(\hh(x) \!\mid\! x). \tag{$h(x,\hh(x))\geq 0$}
\end{align*}
We will then analyze the following four cases.
    \paragraph{(i)} $\max_{y\in \sY}\sfp(y \!\mid\! x) > (1 - c)$ and $r(x)>0$. In this case, by \eqref{eq:calibration_gap_general},
      we have $\sC^*_{\labs}(\sH, \sR, x) =1 - \max_{y\in
        \sY}\sfp(y \!\mid\! x)$ and $\Delta\sC_{\labs,\sH,
        \sR}(h, r, x) = \max_{y\in \sY}\sfp(y \!\mid\! x)-p\paren*{x, \hh(x)}$. For
      the surrogate loss, we have
    \begin{align*}
    &\Delta\sC_{\sfL,\sH, \sR}(h, r, x)\\ 
    & = \sum_{y\in \sY} \paren*{1-\sfp(y \!\mid\! x)}\max\curl*{0,1 + \frac{h(x, y)}{\rho}}e^{\alpha r(x)} + n c e^{-\alpha r(x)} - 2\sqrt{n^2c\paren*{1 - \max_{y\in \sY}\sfp(y \!\mid\! x)}}\\
    & \geq \sum_{y\in \sY} \paren*{1-\sfp(y \!\mid\! x)}\max\curl*{0,1 + \frac{h(x, y)}{\rho}}e^{\alpha r(x)} + n c e^{-\alpha r(x)} - n\paren*{1 - \max_{y\in \sY}\sfp(y \!\mid\! x)}e^{\alpha r(x)} - n c e^{-\alpha r(x)}
    \tag{AM–GM inequality}\\
    & =\sum_{y\in \sY} \paren*{1-\sfp(y \!\mid\! x)}\max\curl*{0,1 + \frac{h(x, y)}{\rho}} - n\paren*{1 - \max_{y\in \sY}\sfp(y \!\mid\! x)}
    \tag{$r(x)>0$}\\
    & \geq \max_{y\in \sY}\sfp(y \!\mid\! x)-\sfp(\hh(x) \!\mid\! x)\\
    \tag{$\sum_{y\in \sY} \paren*{1-\sfp(y \!\mid\! x)}\max\curl*{0,1 + \frac{h(x, y)}{\rho}} - n\paren*{1 - \max_{y\in \sY}\sfp(y \!\mid\! x)}\geq \max_{y\in \sY}\sfp(y \!\mid\! x)-\sfp(\hh(x) \!\mid\! x)$}\\
    & = \Delta\sC_{\labs,\sH, \sR}(h, r, x).
    \end{align*}
     Therefore, 
    \begin{align*}
\sE_{\labs}(h, r) - \sE_{\labs}^*(\sH, \sR) + \sM_{\labs}(\sH, \sR) & = \mathbb{E}_{X}\bracket*{\Delta\sC_{\labs,\sH, \sR}(h, r, x)}\\
& \leq \mathbb{E}_{X}\bracket*{\Gamma_1\paren*{\Delta\sC_{\sfL,\sH, \sR}(h, r, x)}}\\
& \leq \Gamma_1\paren*{\mathbb{E}_{X}\bracket*{\Delta\sC_{\sfL,\sH, \sR}(h, r, x)}}
\tag{$\Gamma_1$ is concave}\\
& = \Gamma_1\paren*{\sE_{\sfL}(h, r)-\sE_{\sfL}^*(\sH, \sR) +\sM_{\sfL}(\sH, \sR)}
\end{align*}
where $\Gamma_1(t)= t$.
\paragraph{(ii)} $\max_{y\in \sY}\sfp(y \!\mid\! x) \leq (1 - c)$ and $r(x)>0$. In this
      case, by \eqref{eq:calibration_gap_general}, we have $\sC^*_{\labs}(\sH, \sR, x) =c$ and
      $\Delta\sC_{\labs,\sH, \sR}(h, r, x) =1-c-p\paren*{x,
        \hh(x)}$. For the surrogate loss, we have
    \begin{align*}
    & \Delta\sC_{\sfL,\sH, \sR}(h, r, x)\\
    & =  \sum_{y\in \sY} \paren*{1-\sfp(y \!\mid\! x)}\max\curl*{0,1 + \frac{h(x, y)}{\rho}}e^{\alpha r(x)} + n c e^{-\alpha r(x)} - 2\sqrt{n^2c\paren*{1 - \max_{y\in \sY}\sfp(y \!\mid\! x)}}\\
    & \geq \sum_{y\in \sY} \paren*{1-\sfp(y \!\mid\! x)}\max\curl*{0,1 + \frac{h(x, y)}{\rho}}e^{\alpha r(x)} + n c e^{-\alpha r(x)}\\
    &\qquad - 2\sqrt{nc\paren*{\sum_{y\in \sY} \paren*{1-\sfp(y \!\mid\! x)}\max\curl*{0,1 + \frac{h(x, y)}{\rho}}}}
   \tag{$\sum_{y\in \sY} \paren*{1-\sfp(y \!\mid\! x)}\max\curl*{0,1 + \frac{h(x, y)}{\rho}}\geq n\paren*{1 - \max_{y\in \sY}\sfp(y \!\mid\! x)}$}\\
    & \geq \sum_{y\in \sY} \paren*{1-\sfp(y \!\mid\! x)}\max\curl*{0,1 + \frac{h(x, y)}{\rho}} + n c -  2\sqrt{nc\paren*{\sum_{y\in \sY} \paren*{1-\sfp(y \!\mid\! x)}\max\curl*{0,1 + \frac{h(x, y)}{\rho}}}}
    \tag{increasing for $r(x)\geq 0$}\\
    & = \paren*{\sqrt{\sum_{y\in \sY} \paren*{1-\sfp(y \!\mid\! x)}\max\curl*{0,1 + \frac{h(x, y)}{\rho}}}-\sqrt{nc}}^2\\
    &\geq \paren*{\sqrt{\sum_{y\in \sY} \paren*{1-\sfp(y \!\mid\! x)}\min\curl*{n,\max\curl*{0,1 + \frac{h(x, y)}{\rho}}}}-\sqrt{nc}}^2\\
    & = \paren*{\frac{\sum_{y\in \sY} \paren*{1-\sfp(y \!\mid\! x)}\min\curl*{n,\max\curl*{0,1 + \frac{h(x, y)}{\rho}}}-n c}{\sqrt{\sum_{y\in \sY} \paren*{1-\sfp(y \!\mid\! x)}\min\curl*{n,\max\curl*{0,1 + \frac{h(x, y)}{\rho}}}}+\sqrt{nc}}}^2.
    \end{align*}
    Thus, by using the fact that $\sqrt{\sum_{y\in \sY} \paren*{1-\sfp(y \!\mid\! x)}\min\curl*{n,\max\curl*{0,1 + \frac{h(x, y)}{\rho}}}}+\sqrt{nc}\leq 2\sqrt{n}$, we further have
    \begin{align*}
    & \Delta\sC_{\sfL,\sH, \sR}(h, r, x)\\
    & \geq \frac{1}{4n} \paren[\Bigg]{\sum_{y\in \sY} \paren*{1-\sfp(y \!\mid\! x)}\min\curl*{n,\max\curl*{0,1 + \frac{h(x, y)}{\rho}}}\\
    &\qquad -n\paren*{1 - \max_{y\in \sY}\sfp(y \!\mid\! x)}+n\paren*{1 - \max_{y\in \sY}\sfp(y \!\mid\! x)-c}}^2 
    \tag{$\sqrt{\sum_{y\in \sY} \paren*{1-\sfp(y \!\mid\! x)}\min\curl*{n,\max\curl*{0,1 + \frac{h(x, y)}{\rho}}}}+\sqrt{nc}\leq 2\sqrt{n}$}\\
    & \geq  \paren*{\frac{\max_{y\in \sY}\sfp(y \!\mid\! x)-\sfp(\hh(x) \!\mid\! x)+1 - \max_{y\in \sY}\sfp(y \!\mid\! x)-c}{2\sqrt{n}}}^2\\
    & = \frac{1}{4n}\paren*{1-c-p\paren*{x,\hh(x)}}^2\\
    & = \frac{\Delta\sC_{\labs,\sH, \sR}(h, r, x)^2}{4n}
    \end{align*}
    Therefore, 
    \begin{align*}
\sE_{\labs}(h, r) - \sE_{\labs}^*(\sH, \sR) + \sM_{\labs}(\sH, \sR) & = \mathbb{E}_{X}\bracket*{\Delta\sC_{\labs,\sH, \sR}(h, r, x)}\\
& \leq \mathbb{E}_{X}\bracket*{\Gamma_2\paren*{\Delta\sC_{\sfL,\sH, \sR}(h, r, x)}}\\
& \leq \Gamma_2\paren*{\mathbb{E}_{X}\bracket*{\Delta\sC_{\sfL,\sH, \sR}(h, r, x)}}
\tag{$\Gamma_2$ is concave}\\
& = \Gamma_2\paren*{\sE_{\sfL}(h, r)-\sE_{\sfL}^*(\sH, \sR) +\sM_{\sfL}(\sH, \sR)}
\end{align*}
where $\Gamma_2(t)=2\sqrt{n t}$.
\paragraph{(iii)} $\max_{y\in \sY}\sfp(y \!\mid\! x) \leq  (1 - c)$ and $r(x)\leq 0$. In this case, by \eqref{eq:calibration_gap_general}, we have $\sC^*_{\labs}(\sH, \sR, x) =c$ and $\Delta\sC_{\labs,\sH, \sR}(h, r, x)=0$, which implies that $\sE_{\labs}(h, r) - \sE_{\labs}^*(\sH, \sR) + \sM_{\labs}(\sH, \sR)\\
= \mathbb{E}_{X}\bracket*{\Delta\sC_{\labs,\sH, \sR}(h, r, x)}=0 \leq \Gamma\paren*{\sE_{\sfL}(h, r)-\sE_{\sfL}^*(\sH, \sR) +\sM_{\sfL}(\sH, \sR)}$ for any $\Gamma\geq 0$.
\paragraph{(iv)} $\max_{y\in \sY}\sfp(y \!\mid\! x) > (1 - c)$ and $r(x)\leq 0$. In this
      case, by \eqref{eq:calibration_gap_general}, we have $\sC^*_{\labs}(\sH, \sR, x) =1 - \max_{y\in
        \sY}\sfp(y \!\mid\! x)$ and $\Delta\sC_{\labs,\sH,
        \sR}(h, r, x) =\max_{y\in \sY} \sfp(y \!\mid\! x)-1+c$. For the
      surrogate loss, we have
    \begin{align*}
    &\Delta\sC_{\sfL,\sH, \sR}(h, r, x)\\ 
    & = \sum_{y\in \sY} \paren*{1-\sfp(y \!\mid\! x)}\max\curl*{0,1 + \frac{h(x, y)}{\rho}}e^{\alpha r(x)} + n c e^{-\alpha r(x)} - 2\sqrt{n^2c\paren*{1 - \max_{y\in \sY}\sfp(y \!\mid\! x)}}\\
    & \geq n\paren*{1 - \max_{y\in \sY}\sfp(y \!\mid\! x)}e^{\alpha r(x)} + nc e^{-\alpha r(x)} - 2\sqrt{n^2c\paren*{1 - \max_{y\in \sY}\sfp(y \!\mid\! x)}}
    \tag{$\sum_{y\in \sY} \paren*{1-\sfp(y \!\mid\! x)}\max\curl*{0,1 + \frac{h(x, y)}{\rho}}\geq n\paren*{1 - \max_{y\in \sY}\sfp(y \!\mid\! x)}$}\\
    & \geq n\paren*{1 - \max_{y\in \sY}\sfp(y \!\mid\! x)} + nc - 2\sqrt{n^2c\paren*{1 - \max_{y\in \sY}\sfp(y \!\mid\! x)}}
    \tag{decreasing for $r(x)\leq 0$}\\
    & = n\paren*{\sqrt{1 - \max_{y\in \sY}\sfp(y \!\mid\! x)}-\sqrt{c}}^2\\
    & = n\paren*{\frac{1 - \max_{y\in \sY}\sfp(y \!\mid\! x)-c}{\sqrt{1 - \max_{y\in \sY}\sfp(y \!\mid\! x)}+\sqrt{c}}}^2\\
    & \geq n\paren*{\frac{\max_{y\in \sY} \sfp(y \!\mid\! x)-1+c}{2}}^2
    \tag{$\sqrt{1 - \max_{y\in \sY}\sfp(y \!\mid\! x)}+\sqrt{c}\leq 2$}\\
    & = \frac{n\Delta\sC_{\labs,\sH, \sR}(h, r, x)^2}{4}.
    \end{align*}
     Therefore, 
    \begin{align*}
      \sE_{\labs}(h, r) - \sE_{\labs}^*(\sH, \sR) + \sM_{\labs}(\sH, \sR)
      & = \mathbb{E}_{X}\bracket*{\Delta\sC_{\labs,\sH, \sR}(h, r, x)}\\
      & \leq \mathbb{E}_{X}\bracket*{\Gamma_3\paren*{\Delta\sC_{\sfL,\sH, \sR}(h, r, x)}}\\
      & \leq \Gamma_3\paren*{\mathbb{E}_{X}\bracket*{\Delta\sC_{\sfL,\sH, \sR}(h, r, x)}}
      \tag{$\Gamma_3$ is concave}\\
& = \Gamma_3\paren*{\sE_{\sfL}(h, r)-\sE_{\sfL}^*(\sH, \sR) +\sM_{\sfL}(\sH, \sR)}
\end{align*}
where $\Gamma_3(t)=2\sqrt{t/n}$.

Overall, we obtain
\begin{align*}
\sE_{\labs}(h, r) - \sE_{\labs}^*(\sH, \sR) + \sM_{\labs}(\sH, \sR) \leq \Gamma\paren*{\sE_{\sfL}(h, r)-\sE_{\sfL}^*(\sH, \sR) +\sM_{\sfL}(\sH, \sR)}
\end{align*}
where $\Gamma(t)=\max\curl*{\Gamma_1(t),\Gamma_2(t),\Gamma_3(t)}= \max\curl*{2\sqrt{nt},t}$, which completes the proof.
\end{proof}

\subsection{Proof of \texorpdfstring{$\sR$}{R}-consistency bounds for second-stage surrogates (Theorem~\ref{Thm:bound-general-second-step})}
\label{app:general-positive-second-stage}
\BoundGenralSecondStep*
\begin{proof}
Given any fixed predictor $h$.  For any $r \in \sR$, $x \in \sX$ and $y \in \sY$, the conditional risk of $\ell_{\mathrm{abs}, h}$ and $\ell_{\Phi, h}$ can be written as
\begin{equation}
\label{eq:tsr-cond-error-mabs}
\begin{aligned}
\sC_{\ell_{\mathrm{abs}, h}}(r, x) &=  \sum_{y \in \sY} \sfp(y \!\mid\! x) \1_{\hh(x) \neq y} \1_{r(x) > 0} + c \1_{r(x) \leq 0}\\
\sC_{\ell_{\Phi, h}}(r, x) &= \sum_{y \in \sY} \sfp(y \!\mid\! x) \1_{\hh(x) \neq y} \Phi \paren*{-r(x)} + c \Phi \paren*{r(x)}.
\end{aligned}
\end{equation}
Thus, the best-in class conditional risk of $\ell_{\mathrm{abs}, h}$ and $\ell_{\Phi, h}$ can be expressed as
\begin{equation}
\label{eq:tsr-best-cond-error-mabs}
\begin{aligned}
\sC^*_{\ell_{\mathrm{abs}, h}}(\sR, x) &= \inf_{r \in \sR}\paren*{ \sum_{y \in \sY} \sfp(y \!\mid\! x) \1_{\hh(x) \neq y} \1_{r(x) > 0} + c \1_{r(x) \leq 0}}\\
\sC^*_{\ell_{\Phi, h}}(\sR, x) &= \inf_{r \in \sR}\paren*{\sum_{y \in \sY} \sfp(y \!\mid\! x) \1_{\hh(x) \neq y} \Phi \paren*{-r(x)} + c \Phi \paren*{r(x)}}.
\end{aligned}
\end{equation}
Let $p_1 = \frac{\sum_{y \in \sY} \sfp(y \!\mid\! x) \1_{\hh(x) \neq y}}{\sum_{y \in \sY} \sfp(y \!\mid\! x) \1_{\hh(x) \neq y} + c}$ and $p_2 = \frac{c}{\sum_{y \in \sY} \sfp(y \!\mid\! x) \1_{\hh(x) \neq y} + c}$. Then, the calibration gap of $\ell_{\mathrm{abs}, h}$ can be written as 
\begin{align*}
& \sC_{\ell_{\Phi, h}}(r, x) - \sC^*_{\ell_{\Phi, h}}(\sR, x)\\
& = \paren*{\sum_{y \in \sY} \sfp(y \!\mid\! x) \1_{\hh(x) \neq y} + c} \bracket*{p_1 \Phi(-r(x)) + p_2 \Phi(r(x)) - \inf_{r \in \sR} \paren*{p_1 \Phi(-r(x)) + p_2 \Phi(r(x))}}.
\end{align*}
By Lemma~\ref{lemma:aux-mabs}, we have
\begin{align*}
& \sC_{\ell_{\mathrm{abs}, h}}(r, x) - \sC^*_{\ell_{\mathrm{abs}, h}}(\sR, x)\\
& = p_1 1_{r(x) > 0} + p_2 1_{r(x) \leq 0} - \inf_{r \in \sR} \paren*{p_1 1_{r(x) > 0} +  p_2 1_{r(x) \leq 0}}\\
& \leq \Gamma \paren*{p_1 \Phi(-r(x)) + p_2 \Phi(r(x)) - \inf_{r \in \sR} \paren*{p_1 \Phi(-r(x)) + p_2 \Phi(r(x))}}\\
& = \Gamma\paren*{\frac{\sC_{\ell_{\Phi, h}}(r, x) - \inf_{r\in \sR}\sC_{\ell_{\Phi, h}}(r, x)}{\sum_{y \in \sY} \sfp(y \!\mid\! x) \1_{\hh(x) \neq y} + c}}\\
& \leq \Gamma \paren*{\frac{\sC_{\ell_{\Phi, h}}(r, x) - \inf_{r\in \sR}\sC_{\ell_{\Phi, h}}(r, x)}{c}},
\end{align*}
where we use the fact that $\Gamma$ is non-decreasing and $\sum_{y \in \sY} \sfp(y \!\mid\! x) \1_{\hh(x) \neq y} + c \geq c$ in the last inequality.
Since $\Gamma$ is concave, taking the expectation on both sides and using Jensen's inequality, we obtain 
\begin{equation*}
\E_{X}\bracket*{\sC_{\ell_{\mathrm{abs}, h}}(r, x) - \sC^*_{\ell_{\mathrm{abs}, h}}(\sR, x)} \leq \Gamma \paren*{\frac{\E_{X}\bracket*{\sC_{\ell_{\Phi, h}}(r, x) - \inf_{r\in \sR}\sC_{\ell_{\Phi, h}}(r, x)}}{c}}.
\end{equation*}
Since the term $\E_{X}\bracket*{\sC_{\ell_{\mathrm{abs}, h}}(r, x) - \sC^*_{\ell_{\mathrm{abs}, h}}(\sR, x)}$ and $\E_{X}\bracket*{\sC_{\ell_{\Phi, h}}(r, x) - \inf_{r\in \sR}\sC_{\ell_{\Phi, h}}(r, x)}$ can be expressed as
\begin{align*}
 \E_{X}\bracket*{\sC_{\ell_{\mathrm{abs}, h}}(r, x) - \sC^*_{\ell_{\mathrm{abs}, h}}(\sR, x)} &= \sE_{\ell_{\mathrm{abs}, h} }(r)-\sE_{\ell_{\mathrm{abs}, h} }^*(\sR) +\sM_{\ell_{\mathrm{abs}, h} }(\sR)\\
 \E_{X}\bracket*{\sC_{\ell_{\Phi, h}}(r, x) - \inf_{r\in \sR}\sC_{\ell_{\Phi, h}}(r, x)} &= \sE_{\ell_{\Phi, h} }(r)-\sE_{\ell_{\Phi, h} }^*(\sR) +\sM_{\ell_{\Phi, h}}(\sR),
\end{align*}
we have
\begin{equation*}
\sE_{\ell_{\mathrm{abs}, h} }(r)-\sE_{\ell_{\mathrm{abs}, h} }^*(\sR) +\sM_{\ell_{\mathrm{abs}, h} }(\sR) \leq \Gamma\paren*{\frac{\sE_{\ell_{\Phi, h} }(r)-\sE_{\ell_{\Phi, h} }^*(\sR) +\sM_{\ell_{\Phi, h}}(\sR)}{c}},
\end{equation*}
which completes the proof.
\end{proof}

\subsection{Proof of \texorpdfstring{$(\sH, \sR)$}{HR}-consistency bounds for two-stage surrogates (Theorem~\ref{Thm:bound-general-two-step-mabs})}
\label{app:general-positive-two-stage}
\BoundGenralTwoStepMabs*
\begin{proof}
Since $\sR$ is regular, the conditional risk and the best-in-class conditional risk of the abstention loss $\labs$ can be expressed as
\begin{equation}
\label{eq:tshr-cond-error-def}
\begin{aligned}
\sC_{\labs}(h, r, x) &=  \sum_{y \in \sY} \sfp(y \!\mid\! x) \1_{\hh(x) \neq y} \1_{r(x) > 0} + c \1_{r(x) \leq 0}\\
\sC^*_{\labs}(\sH, \sR, x) &= \min\curl*{\inf_{h \in \sH} \sum_{y \in \sY} \sfp(y \!\mid\! x) \1_{\hh(x) \neq y}, c}.
\end{aligned}
\end{equation}
Thus, by introducing the term $\min\curl*{\sum_{y \in \sY} \sfp(y \!\mid\! x) \1_{\hh(x) \neq y}, c}$ and subsequently subtracting it after rearranging, the calibration gap of the abstention loss $\labs$ can be written as follows
\begin{equation}
\label{eq:tshr-cond-reg-def-mabs}
\begin{aligned}
& \sC_{\labs}(h, r, x) - \sC^*_{\labs}(\sH, \sR, x)\\
& = \sum_{y \in \sY} \sfp(y \!\mid\! x) \1_{\hh(x) \neq y} \1_{r(x) > 0} + c \1_{r(x) \leq 0} - \min\curl*{\inf_{h \in \sH} \sum_{y \in \sY} \sfp(y \!\mid\! x) \1_{\hh(x) \neq y}, c}\\ 
& =  \sum_{y \in \sY} \sfp(y \!\mid\! x) \1_{\hh(x) \neq y} \1_{r(x) > 0} + c \1_{r(x) \leq 0}  - \min\curl*{\sum_{y \in \sY} \sfp(y \!\mid\! x) \1_{\hh(x) \neq y}, c}\\
& \quad + \min\curl*{\sum_{y \in \sY} \sfp(y \!\mid\! x) \1_{\hh(x) \neq y}, c} - \min\curl*{\inf_{h \in \sH} \sum_{y \in \sY} \sfp(y \!\mid\! x) \1_{\hh(x) \neq y}, c}.
\end{aligned}
\end{equation}
Note that by the property of the minimum, the second term can be upper-bounded as 
\begin{align*}
& \min\curl*{\sum_{y \in \sY} \sfp(y \!\mid\! x) \1_{\hh(x) \neq y}, c} - \min\curl*{\inf_{h \in \sH} \sum_{y \in \sY} \sfp(y \!\mid\! x) \1_{\hh(x) \neq y}, c}\\
& \leq \sum_{y \in \sY} \sfp(y \!\mid\! x) \1_{\hh(x) \neq y} - \inf_{h \in \sH} \sum_{y \in \sY} \sfp(y \!\mid\! x) \1_{\hh(x) \neq y}\\
& = \sC_{\ell_{0-1}}(h, x)-\sC^*_{\ell_{0-1}}(\sH,x)\\
& \leq \Gamma_1\paren*{\sC_{\ell}(h, x)-\sC^*_{\ell}(\sH,x)},
\end{align*}
where we use the $\sH$-consistency bound of $\ell$ on the pointwise distribution $\delta_{x}$ that concentrates on a point $x$ in the last inequality.
Next, we will upper-bound the first term. Note that the conditional risk and the best-in class conditional risk of $\ell_{\Phi, h}$ can be expressed as
\begin{equation}
\label{eq:tshr-cond-error-sur-mabs}
\begin{aligned}
\sC_{\ell_{\Phi, h}}(r, x) &= \sum_{y \in \sY} \sfp(y \!\mid\! x) \1_{\hh(x) \neq y} \Phi \paren*{-r(x)} + c \Phi \paren*{r(x)}\\
\sC^*_{\ell_{\Phi, h}}(\sR, x) &= \inf_{r \in \sR}\paren*{\sum_{y \in \sY} \sfp(y \!\mid\! x) \1_{\hh(x) \neq y} \Phi \paren*{-r(x)} + c \Phi \paren*{r(x)}}.
\end{aligned}
\end{equation}
Let $p_1 = \frac{\sum_{y \in \sY} \sfp(y \!\mid\! x) \1_{\hh(x) \neq y}}{\sum_{y \in \sY} \sfp(y \!\mid\! x) \1_{\hh(x) \neq y} + c}$ and $p_2 = \frac{c}{\sum_{y \in \sY} \sfp(y \!\mid\! x) \1_{\hh(x) \neq y} + c}$. Then, the first term can be rewritten as 
\begin{align*}
& \sum_{y \in \sY} \sfp(y \!\mid\! x) \1_{\hh(x) \neq y} \1_{r(x) > 0} + c \1_{r(x) \leq 0}  - \min\curl*{\sum_{y \in \sY} \sfp(y \!\mid\! x) \1_{\hh(x) \neq y}, c}\\
& = \paren*{\sum_{y \in \sY} \sfp(y \!\mid\! x) \1_{\hh(x) \neq y} + c} \bracket*{p_1 1_{r(x) > 0} + p_2 1_{r(x) \leq 0} - \inf_{r \in \sR} \paren*{p_1 1_{r(x) > 0} +  p_2 1_{r(x) \leq 0}}}
\end{align*}
By Lemma~\ref{lemma:aux-mabs}, we have
\begin{align*}
& p_1 1_{r(x) > 0} + p_2 1_{r(x) \leq 0} - \inf_{r \in \sR} \paren*{p_1 1_{r(x) > 0} + p_2 1_{r(x) \leq 0}}\\
& \leq \Gamma_2 \paren*{p_1 \Phi(-r(x)) + p_2 \Phi(r(x)) - \inf_{r \in \sR} \paren*{p_1 \Phi(-r(x)) + p_2 \Phi(r(x))}}\\
& = \Gamma_2 \paren*{\frac{\sC_{\ell_{\Phi, h}}(r, x) - \sC^*_{\ell_{\Phi, h}}(\sR, x)}{\sum_{y \in \sY} \sfp(y \!\mid\! x) \1_{\hh(x) \neq y} + c}}.
\end{align*}
Therefore, the first term can be upper-bounded as
\begin{align*}
& \sum_{y \in \sY} \sfp(y \!\mid\! x) \1_{\hh(x) \neq y} \1_{r(x) > 0} + c \1_{r(x) \leq 0}  - \min\curl*{\sum_{y \in \sY} \sfp(y \!\mid\! x) \1_{\hh(x) \neq y}, c}\\
& = \paren*{\sum_{y \in \sY} \sfp(y \!\mid\! x) \1_{\hh(x) \neq y} + c} \bracket*{p_1 1_{r(x) > 0} + p_2 1_{r(x) \leq 0} - \inf_{r \in \sR} \paren*{p_1 1_{r(x) > 0} +  p_2 1_{r(x) \leq 0}}}\\
& \leq \paren*{\sum_{y \in \sY} \sfp(y \!\mid\! x) \1_{\hh(x) \neq y} + c} \Gamma_2 \paren*{\frac{\sC_{\ell_{\Phi, h}}(r, x) - \sC^*_{\ell_{\Phi, h}}(\sR, x)}{\sum_{y \in \sY} \sfp(y \!\mid\! x) \1_{\hh(x) \neq y} + c}} \\
& \leq
\begin{cases}
\Gamma_2\paren*{\sC_{\ell_{\Phi,h}}(r,x)-\sC^*_{\ell_{\Phi,h}}(\sR,x)} & \text{when $\Gamma_2$ is linear}\\
(1+c)\Gamma_2\paren*{\frac {\sC_{\ell_{\Phi,h}}(r,x)-\sC^*_{\ell_{\Phi,h}}(\sR,x)}{c}} & \text{otherwise}
\end{cases}
\end{align*}
where we use the fact that $c\leq \sum_{y\in \sY}\sfp(y \!\mid\! x)\1_{\hh(x)\neq y} + c\leq 1+c$ and $\Gamma_2$ is non-decreasing in the last inequality. After upper-bounding the first term and the second term in \eqref{eq:tshr-cond-reg-def-mabs} as above, taking the expectation on both sides, using the fact that $\Gamma_1$ and $\Gamma_2$ are concave, we obtain
\begin{align*}
  &\E_{X}\bracket*{\sC_{\labs}(h, r, x) - \sC^*_{\labs}(\sH, \sR, x)}\\
  &\leq 
\begin{cases}
\Gamma_2\paren*{\E_X\bracket*{\sC_{\ell_{\Phi,h}}(r,x)-\sC^*_{\ell_{\Phi,h}}(\sR,x)}} + \Gamma_1\paren*{\E_X\bracket*{\sC_{\ell}(h, x)-\sC^*_{\ell}(\sH,x)}} & \text{$\Gamma_2$ is linear}\\
(1+c)\Gamma_2\paren*{\frac1c\E_X\bracket*{\sC_{\ell_{\Phi,h}}(r,x)-\sC^*_{\ell_{\Phi,h}}(\sR,x))}} + \Gamma_1\paren*{\E_X\bracket*{\sC_{\ell}(h, x)-\sC^*_{\ell}(\sH,x)}} & \text{otherwise}
\end{cases}
\end{align*}
Since the three expected terms can be expressed as
\begin{align*}
\E_{X}\bracket*{\sC_{\labs}(h, r, x) - \sC^*_{\labs}(\sH, \sR, x)} &= \sE_{\labs}(h, r)-\sE^*_{\labs}\paren*{\sH,\sR}+\sM_{\labs}(\sH, \sR)\\
\E_X\bracket*{\sC_{\ell_{\Phi,h}}(r,x)-\sC^*_{\ell_{\Phi,h}}(\sR,x)} &= \sE_{\ell_{\Phi,h}}(r)-\sE_{\ell_{\Phi,h}}^*(\sR) +\sM_{\ell_{\Phi,h}}(\sR)\\
\E_X\bracket*{\sC_{\ell}(h, x)-\sC^*_{\ell}(\sH,x)} &= \sE_{\ell}(h)-\sE_{\ell}^*(\sH) +\sM_{\ell}(\sH),
\end{align*}
we have
\begin{align*}
\sE_{\labs}(h, r) - \sE_{\labs}^*(\sH, \sR) + \sM_{\labs}(\sH, \sR)
& \leq \Gamma_1\paren*{\sE_{\ell}(h)-\sE_{\ell}^*(\sH) +\sM_{\ell}(\sH)}\\
& \quad + (1+c)\Gamma_2\paren*{\frac{\sE_{\ell_{\Phi,h}}(r)-\sE_{\ell_{\Phi,h}}^*(\sR) +\sM_{\ell_{\Phi,h}}(\sR)}{c}},
\end{align*}
where the
constant factors $(1 + c)$ and $\frac{1}{c}$ can be removed 
when $\Gamma_2$ is linear.
\end{proof}

\subsection{Proof of realizable \texorpdfstring{$(\sH, \sR)$}{HR}-consistency bounds for single-stage surrogates (Theorem~\ref{Thm:spcific-loss-bound-realizable})}
\label{app:general-positive-single-stage-realizable}
\SpecificLossBoundRealizable*
\begin{proof}
It is straightforward to see that $\sfL$ serves as an upper bound for $\labs$ when $\ell$ serves as an upper bound for $\ell_{0-1}$ under Assumption~\ref{assumption:phi}.
By definition, for any $(\sH,\sR)$-realizable distribution, there exists $h^*\in \sH$ and $r^*\in \sR$ such that $\sE_{\labs}(h^*,r^*) = \sE_{\labs}^*(\sH, \sR) = 0$. Then, by the assumption that $\sH$ and $\sR$ are closed under scaling, for any $\nu>0$,
\begin{align*}
\sE^*_{\sfL}(\sH,\sR)
&\leq\sE_{\sfL}(\nu h^*,\nu r^*)\\
&=\mathbb{E}\bracket*{\sfL(\nu h^*,\nu r^*,x,y)\mid r^*< 0}\mathbb{P}(r^*< 0) + \mathbb{E}\bracket*{\sfL(\nu h^*,\nu r^*,x,y)\mid r^*>0}\mathbb{P}(r^*> 0)
\end{align*}
Next, we investigate the two terms.
The first term is when $r^*< 0$, then we must have $c=0$ since the data is realizable. By taking the limit, we obtain:
\begin{align*}
&\lim_{\nu\to +\infty}\mathbb{E}\bracket*{\sfL(\nu h^*,\nu r^*,x,y)\mid r^*< 0}\mathbb{P}(r^*< 0)\\
&=\lim_{\nu\to +\infty}\mathbb{E}\bracket*{\ell(\nu h^*, x, y)\Phi\paren*{-\alpha \nu r^*(x)} + \Psi(c) \Phi\paren*{\beta \nu r^*(x)}\mid r^*< 0}\mathbb{P}(r^*< 0)\\
&=\lim_{\nu\to +\infty}\mathbb{E}\bracket*{\ell(\nu h^*, x, y)\Phi\paren*{-\alpha \nu r^*(x)}\mid r^*< 0}\mathbb{P}(r^*< 0) \tag{$c=0$ and $\Psi(0)=0$}\\
&=0. \tag{by the Lebesgue dominated convergence theorem and $\lim_{t\to + \infty}\Phi(t)=0$}
\end{align*}
The second term is when $r^*> 0$, then we must have $h^*(x,y) - \max_{y'\neq y}h^*(x,y') > 0$ since the data is realizable. Thus, using the fact that $\lim_{\nu\to +\infty}\ell(\nu h^*,x,y)=0$ and taking the limit, we obtain
\begin{align*}
&\lim_{\nu\to +\infty}\mathbb{E}\bracket*{\sfL(\nu h^*,\nu r^*,x,y)\mid r^*< 0}\mathbb{P}(r^*< 0)\\
&=\lim_{\nu\to +\infty}\mathbb{E}\bracket*{\ell(\nu h^*, x, y)\Phi\paren*{-\alpha \nu r^*(x)} + \Psi(c) \Phi\paren*{\beta \nu r^*(x)}\mid r^*< 0}\mathbb{P}(r^*< 0)\\
&=0. \tag{by the Lebesgue dominated convergence theorem, $\lim_{t\to + \infty}\Phi(t)=0$, $\lim_{\nu\to +\infty}\ell(\nu h^*,x,y)=0$}
\end{align*}
Therefore, by combining the above two analysis, we obtain
\begin{align*}
\sE^*_{\sfL}(\sH,\sR)\leq \lim_{\nu\to +\infty}\sE_{\sfL}(\nu h^*,\nu r^*)=0.
\end{align*}
By using the fact that $\sfL$ serves as an upper bound for $\labs$ and $\sE_{\labs}^*(\sH, \sR)=0$, we conclude that
\begin{equation*}
\sE_{\labs}(h, r) - \sE_{\labs}^*(\sH, \sR)
\leq \sE_{\sfL}(h, r)-\sE_{\sfL}^*(\sH, \sR).
\end{equation*}
\end{proof}
\subsection{Proof of realizable \texorpdfstring{$(\sH, \sR)$}{HR}-consistency for two-stage surrogates (Theorem~\ref{Thm:bound-general-two-step-realizable-mabs})}
\label{app:general-positive-two-stage-realizable}
\BoundGenralTwoStepRealizableMabs*
\begin{proof}
It is straightforward to see that $\ell_{\Phi, h}$ upper-bounds the abstention loss $\labs$ under Assumption~\ref{assumption:phi}. By definition, for any $(\sH,\sR)$-realizable distribution, there exists $h^*\in \sH$ and $r^*\in \sR$ such that $\sE_{\labs}(h^*,r^*)=0$. Let $\hat h$ be the minimizer of $\sE_{\ell}$ and $\hat r$ be the minimizer of $\sE_{\ell_{\Phi, \hat h}}$.
Then, using the fact that $\ell_{\Phi, h}$ upper-bounds the abstention loss $\sE_{\labs}$, we have $\sE_{\sE_{\labs}}(\hat h, \hat r)\leq \sE_{\ell_{\Phi, \hat h}}(\hat r)$.

Next, we analyze two cases. If for a point $x$, abstention happens, that is $r^*(x) < 0$, then we must have $c = 0$ since the data is realizable. Therefore, there exists an optimal $r^{**}$ abstaining all the points with zero cost: $r^{**}(x) < 0$ for all $x \in \sX$. Then, by the assumption that $\sR$ is closed under scaling and the Lebesgue dominated convergence theorem, using the fact that $\lim_{t \to +
  \infty} \Phi(t) = 0$, we obtain
\begin{align*}
\sE_{\labs}(\hat h, \hat r) 
&\leq \sE_{\ell_{\Phi, \hat h}}(\hat r)\\
&\leq \lim_{\nu \to +\infty} \sE_{\ell_{\Phi, \hat h}}(\nu r^{**}) \tag{$\hat r$ is the minimizer of $\sE_{\ell_{\Phi, \hat h}}$}\\
&= \lim_{\nu\to +\infty}\mathbb{E}\bracket*{\1_{\hat \hh(x) \neq y} \Phi\paren*{-\nu r^{**}(x)} + c \Phi\paren*{\nu r^{**}(x)}} \tag{By \eqref{eq:ell-Phi-h-mabs}}\\
&= \lim_{\nu\to +\infty}\mathbb{E}\bracket*{\1_{\hat \hh(x) \neq y} \Phi\paren*{-\nu r^{**}(x)}} \tag{$c=0$}\\
&= 0.  \tag{by the Lebesgue dominated convergence theorem and $\lim_{t\to + \infty}\Phi(t)=0$}
\end{align*}
On the other hand, if no abstention occurs for any point, that is $r^*(x) > 0$ for any $x \in \sX$, then we must have $\1_{\mathsf h^*(x) \neq y} = 0$ for all $(x,y ) \in \sX \times \sY$ since the data is realizable. Using the fact that $\ell$ is realizable $\sH$-consistent with respect to $\ell_{0-1}$ when $\sH$ is closed under scaling, we obtain $\1_{\hat{\mathsf h}(x) \neq y} = 0$ for all $(x,y) \in \sX \times \sY$. Then, by the assumption that $\sR$ is closed under scaling and the Lebesgue dominated convergence theorem, using the fact that $\lim_{t \to + \infty} \Phi(t) = 0$, we obtain
\begin{align*}
\sE_{\labs}(\hat h, \hat r) 
&\leq \sE_{\ell_{\Phi, \hat h}}(\hat r)\\
&\leq \lim_{\nu \to +\infty} \sE_{\ell_{\Phi, \hat h}}(\nu r^{*}) \tag{$\hat r$ is the minimizer of $\sE_{\ell_{\Phi, \hat h}}$}\\
&= \lim_{\nu\to +\infty}\mathbb{E}\bracket*{\1_{\hat \hh(x) \neq y} \Phi\paren*{-\nu r^{*}(x)} + c \Phi\paren*{\nu r^{*}(x)}} \tag{By \eqref{eq:ell-Phi-h-mabs}}\\
&= \lim_{\nu\to +\infty}\mathbb{E}\bracket*{c \Phi\paren*{\nu r^{*}(x)}} \tag{$\1_{\hat{\mathsf h}(x) \neq y} = 0$}\\
&= 0.  \tag{by the Lebesgue dominated convergence theorem and $\lim_{t\to + \infty}\Phi(t)=0$}
\end{align*}
This completes the proof.
\end{proof}
\restoreatoc

\chapter{Appendix to Chapter~\ref{ch4}}

\disableatoc
\section{Experimental details} 
\label{app:experiments}
\paragraph{Experimental setup.} For experiments, we used two popular
datasets: CIFAR-10 \citep{Krizhevsky09learningmultiple} and SVHN
(Street View House Numbers) \citep{Netzer2011}. CIFAR-10 consists of
$60\mathord,000$ color images in 10 different classes, with
$6\mathord,000$ images per class. The dataset is split into
$50\mathord,000$ training images and $10\mathord,000$ test
images. SVHN contains images of house numbers captured from Google
Street View. It consists of $73\mathord,257$ images for training and
$26\mathord,032$ images for testing. We trained for 50 epochs on
CIFAR-10 and 15 epochs on SVHN without any data augmentation.

In our experiments, we adopted the ResNet \citep{he2016deep}
architecture for the base model and selected various sizes of ResNet
models as experts in each scenario. Throughout all three scenarios, we
used ResNet-$4$ for both the predictor and the deferral models. In the
first scenario, we chose ResNet-$10$ as the expert model. In the
second scenario, we included ResNet-$10$ and ResNet-$16$ as expert
models. The third scenario involves ResNet-$10$, ResNet-$16$, and
ResNet-$28$ as expert models with increasing complexity. The expert
models are pre-trained on the training data of SVHN and CIFAR-10
respectively.

During the training process, we simultaneously trained the predictor
ResNet-$4$ and the deferral model ResNet-$4$. We adopted the Adam
optimizer \citep{kingma2014adam} with a batch size of $128$ and a
weight decay of $1\times 10^{-4}$. We used our proposed deferral
surrogate loss \eqref{eq:sur-score} with the generalized cross-entropy
loss being adopted for $\ell$. As suggested by
\citet{zhang2018generalized}, we set the parameter $\alpha$ to $0.7$.

For the second type of cost functions, we set the base costs as
follows: $\beta_1=0.1$, $\beta_2=0.12$ and $\beta_3=0.14$ for the SVHN
dataset and $\beta_1=0.3$, $\beta_2=0.32$, $\beta_3=0.34$ for the
CIFAR-10 dataset, where $\beta_1$ corresponds to the cost associated
with the smallest expert model, ResNet-$10$, $\beta_2$ to that of the
medium model, ResNet-$16$, and $\beta_3$ to that of the largest expert
model, ResNet-$28$. A base cost value that is not too far from the
misclassification error of expert models encourages in practice a
reasonable amount of input instances to be deferred.  We observed that
the performance remains close for other neighboring values of base
costs.

\paragraph{Additional experiments.} Here, we share additional experimental results in an intriguing setting where multiple experts are available and each of them has a clear domain of expertise. We report below the empirical results of our proposed deferral surrogate loss and the one-vs-all (OvA) surrogate loss proposed in recent work \citep{verma2023learning}, which is the state-of-the-art surrogate loss for learning with multi-expert deferral, on CIFAR-10. In this setting, the two experts have a clear domain of expertise. The expert 1 is always correct on the first three classes, 0 to 2, and predicts uniformly at random for other classes; the expert 2 is always correct on the next three classes, 3 to 5, and generates random predictions otherwise. We train a ResNet-16 for the predictor/deferral model. 

As shown in Table~\ref{tab:additional}, our method achieves comparable system accuracy with OvA. Among the images in classes 0 to 2, only $3.57\%$ is deferred to expert 2 which predicts uniformly at random. Similarly, among the images in classes 3 to 5, only $3.33\%$ is deferred to expert 1. For the rest of the images in classes 6 to 9, the predictor decides to learn to classify them by itself and actually makes $92.88\%$ of the final predictions. This illustrates that our proposed surrogate loss is effective and comparable to the baseline.

\begin{table*}[t]
  \centering
  \resizebox{\textwidth}{!}{
  \begin{tabular}{@{\hspace{0cm}}llllllllllllll@{\hspace{0cm}}}
    \toprule
    \multirow{3}{*}{Method} & \multirow{3}{*}{System accuracy (\%)} & \multicolumn{12}{c}{Ratio of deferral (\%)} \\
    \cmidrule{3-14}
    & & \multicolumn{3}{|c|}{all the classes} & \multicolumn{3}{|c|}{classes 0 to 2} & \multicolumn{3}{|c|}{classes 3 to 5} & \multicolumn{3}{|c|}{classes 6 to 9} \\
    \cmidrule{3-14}
    & & predictor & expert 1 & expert 2 & predictor & expert 1 & expert 2 & predictor & expert 1 & expert 2 & predictor & expert 1 & expert 2\\
    \midrule
    Ours & 92.19 & 61.43 & 17.38 & 21.19 & 46.77 & 49.67 & 3.57 & 33.60 & 3.33 & 63.07 & 92.88 & 3.43 & 3.70 \\
    OvA &  91.39 & 59.72 & 16.78 & 23.50 & 48.63 & 47.67 & 3.70 & 27.87 & 2.47 & 69.67 & 92.73 & 3.50 & 3.78 \\
    \bottomrule
  \end{tabular}
  }
 \caption{Comparison of our proposed deferral surrogate loss with the one-vs-all (OvA) surrogate loss in an intriguing setting where multiple experts are available and each of them has a clear domain of expertise.}
 \label{tab:additional}
\end{table*}

\section{Proof of \texorpdfstring{$\sH$}{H}-consistency bounds for deferral surrogate losses}

To prove $\sH$-consistency bounds for our deferral surrogate loss
functions, we will show how the \emph{conditional regret} of the
deferral loss can be upper-bounded in terms of the \emph{conditional
regret} of the surrogate loss. The general theorems proven by
\citet[Theorem~4, Theorem~5]{AwasthiMaoMohriZhong2022multi} then
guarantee our $\sH$-consistency bounds.

For any $x \in \sX$ and $y \in \sY$, let $\sfp(y \!\mid\! x)$ denote the
conditional probability of $Y = y$ given $X = x$ for any $y\in \sY$.
Then, for any $x \in \sX$, the \emph{conditional $\ldefsc$-loss}
$\sC_{\ldefsc}(h, x)$ and \emph{conditional regret} (or \emph{calibration
gap}) $\Delta \sC_{\ldefsc}(h, x)$ of a hypothesis $h \in \sH$ are
defined by
\begin{align*}
  \sC_{\ldefsc}(h, x)
& = \E_{y | x}[\ldefsc(h, x, y)]
= \sum_{y \in \sY} \sfp(y \!\mid\! x) \ldefsc(h, x, y)\\
\Delta \sC_{\ldefsc}(h, x)
& = \sC_{\ldefsc}(h, x) - \sC^*_{\ldefsc}(\sH, x),
\end{align*}
where $\sC^*_{\ldefsc}(\sH, x) = \inf_{h \in \sH} \sC_{\ldefsc}(h, x)$.
Similar definitions hold for the surrogate loss $\lsc$. To bound
$\Delta \sC_{\ldefsc}(h, x)$ in terms of
$\Delta \sC_{\lsc}(h, x)$, we first give more explicit expressions
for these conditional regrets.

To do so, it will be convenient to use the following definition
for any $x \in \sX$ and $y \in [n + \num]$:
\begin{align*}
q(x, y) =
\begin{cases}
\sfp(y \!\mid\! x) & y \in \sY \\
\ignore{\E_{y | x} \bracket*{1 - c_j(x, y)} = } 1 - \sum_{y \in \sY} \sfp(y \!\mid\! x) c_j(x, y) & n + 1 \le y \leq n + \num.
\end{cases}
\end{align*}
Note that $q(x, y)$ is non-negative but, in general, these quantities
do not sum to one. We denote by $\ov q(x, y) = \frac{q(x, y)}{Q}$ their
normalized counterparts which represent probabilities, where
$Q = \sum_{y\in [n + \num]} q(x, y)$. 

For any $x \in \sX$, we will denote by $\mathsf H(x)$ the set of
labels generated by hypotheses in $\sH$: $\mathsf H(x) = \curl*{
  \hh(x) \colon h \in \sH}$. We denote by $y_{\max} \in [n + \num]$
the label associated by $q$ to an input $x\in \sX$, defined as $
y_{\max} = \argmax_{y \in[n + \num]} q(x, y)$, with the same
deterministic strategy for breaking ties as that of $ \hh(x)$.

\subsection{Conditional regret of the deferral loss}
\label{sec:app_def_cond}

With these definitions, we can now express the conditional loss
and regret of the deferral loss.

\begin{restatable}{lemma}{CalibrationGapScore}
\label{lemma:calibration_gap_score}
For any $x \in \sX$,
the minimal conditional $\ldefsc$-loss and
the calibration gap for $\ldefsc$ can be expressed as follows:
\begin{align*}
\sC^*_{\ldefsc}( \sH,x) & = 1 - \max_{y\in \mathsf H(x)} q(x, y)\\
\Delta\sC_{\ldefsc, \sH}(h,x) & = \max_{y\in \mathsf H(x)} q(x, y) - q(x, \hh(x)).
\end{align*}
\end{restatable}
\begin{proof}
The conditional
$\ldefsc$-risk of $h$ can be expressed as follows:
\begin{align*}
& \sC_{\ldefsc}(h, x)\\
& = \E_{y | x} \bracket*{\ldefsc(h, x, y)}\\
& = \E_{y | x}
  \bracket*{\1_{\hh(x)\neq y}}\1_{\hh(x)\in [n]} + \sum_{j = 1}^{\num} \E_{y | x}
  \bracket*{c_j(x, y)}\1_{\hh(x) = n + j}\\
& = \sum_{y\in \sY} q(x, y) \1_{\hh(x)\neq y} \1_{\hh(x)\in [n]} + \sum_{j = 1}^{\num} (1 - q(x, n + j)) \1_{\hh(x) = n + j}\\
& = (1 - q(x, \hh(x))) \1_{\hh(x)\in [n]} + \sum_{j = 1}^{\num} (1 - q(x, \hh(x))) \1_{\hh(x) = n + j}\\
& = 1 - q(x, \hh(x)).
\end{align*}
Then, the minimal conditional $\ldefsc$-risk is given by
\[
\sC_{\ldefsc}^*(\sH,x) = 1 - \max_{y\in \mathsf H(x)} q(x, y),
\]
and the calibration gap can be expressed as follows:
\begin{align*}
  \Delta \sC_{\ldefsc,\sH}(h, x)
  = \sC_{\ldefsc}(h, x)-\sC_{\ldefsc}^*(\sH,x)
  = \max_{y\in \mathsf H(x)} q(x, y)-q(x,\hh(x)),
\end{align*}
which completes the proof.
\end{proof}

\subsection{Conditional regret of a surrogate deferral loss}

\begin{restatable}{lemma}{SurrogateCalibrationGapScore}
\label{lemma:surrogate_calibration_gap_score}
For any $x \in \sX$,
the conditional surrogate $\lsc$-loss and regret 
can be expressed as follows:
\begin{align*}
  \sC_{\lsc}(h, x)
  & = \sum_{y \in [n + \num]} q(x, y) \ell(h, x, y)\\
  \Delta \sC_{\lsc}(h, x)
  & = \sum_{y \in [n + \num]} q(x, y) \ell(h, x, y) - \inf_{h \in \sH} \sum_{y \in [n + \num]} q(x, y) \ell(h, x, y).
\end{align*}
\end{restatable}

\begin{proof}
By definition, $\sC_{\lsc}(h, x)$ is the conditional-$\lsc$ loss can
be expressed as follows:
\begin{equation}
\label{eq:cond-surrogate}
\begin{aligned}
  & \sC_{\lsc}(h, x)\\
  & = \E_y
  \bracket*{\lsc(h, x, y) }\\
  & = \E_y
  \bracket*{\ell \paren*{h, x, y} } +\sum_{j = 1}^{\num}\E_{y | x}
  \bracket*{\paren*{1 - c_j(x, y)} }\ell\paren*{h, x, n + j}\\  
  & = \sum_{y \in \sY} q(x, y) \ell \paren*{h, x, y} + \sum_{j = 1}^{\num}  q(x,n+j)  \ell\paren*{h, x, n + j}\\  
  & = \sum_{y\in [n + \num]} q(x, y)  \ell(h, x, y),
\end{aligned}
\end{equation}
which ends the proof.
\end{proof}

\subsection{Conditional regret of zero-one loss}

We will also make use of the following result for the zero-one loss
$\ell_{0-1}(h, x, y) = \1_{\hh(x) \neq y}$ with label space $[n+\num]$
and the conditional probability vector $\ov q(x,\cdot)$, which
characterizes the minimal conditional $\ell_{0-1}$-loss and the
corresponding calibration gap
\citep[Lemma~3]{AwasthiMaoMohriZhong2022multi}.

\begin{restatable}{lemma}{ExplicitAssumptionQ}
\label{lemma:explicit_assumption_01_q}
For any $x \in \sX$,
the minimal conditional $\ell_{0-1}$-loss and
the calibration gap for $\ell_{0-1}$ can be expressed as follows:
\begin{align*}
\sC^*_{\ell_{0-1}}(x) & = 1 - \max_{y\in \mathsf H(x)} \ov q(x, y)\\
\Delta\sC_{\ell_{0-1}}(h,x) & = \max_{y\in \mathsf H(x)} \ov q(x, y) - \ov q(x,\hh(x)).
\end{align*}
\end{restatable}

\subsection{Proof of \texorpdfstring{$\sH$}{H}-consistency bounds
  for  deferral surrogate losses (Theorem~\ref{Thm:bound-score})}
\label{app:score}

\BoundScore*
\begin{proof}
We denote the normalization factor as $Q=\sum_{y\in [n+\num]}q(x,
y)=\num+1-\E_y \bracket*{c_j(x, y)}$, which is a constant that ensures
the sum of $\ov q(x, y)=\frac{q(x, y)}{Q}$ is equal to 1.  By
Lemma~\ref{lemma:calibration_gap_score}, the calibration gap of
$\ldefsc$ can be expressed and upper-bounded as follows:
\begin{align*}
& \Delta \sC_{\ldefsc}(h, x)\\
& = \max_{y\in \mathsf H(x)} q(x, y) - q(x, \hh(x)) \tag{Lemma~\ref{lemma:calibration_gap_score}}\\
& = Q \paren*{\max_{y\in \mathsf H(x)} \ov q(x, y) -  \ov q(x,  \hh(x))}\\
& = Q \Delta \sC_{\ell_{0-1}}(h, x)\tag{Lemma~\ref{lemma:explicit_assumption_01_q}}\\
& \leq Q\Gamma\paren*{\Delta \sC_{ \ell, \sH}(h, x)} \tag{$\sH$-consistency
bound of $ \ell$}\\
& = Q\Gamma\bigg(\sum_{y\in [n + \num]} \ov q(x, y)  \ell(h,x, y) -\inf_{h \in \sH}\sum_{y\in [n + \num]} \ov q(x, y)  \ell(h,x, y)\bigg) \\
& = Q\Gamma\bigg(\sum_{y\in [n + \num]} \frac{q(x, y)}{Q}  \ell(h,x, y) -\inf_{h \in \sH}\sum_{y\in [n + \num]} \frac{q(x, y)}{Q}  \ell(h,x, y)\bigg) \\
& = Q\Gamma\paren*{\frac{1}{Q}\Delta \sC_{\lsc}(h, x)} \tag{Lemma~\ref{lemma:surrogate_calibration_gap_score}}.
\end{align*}
Thus, taking expectations gives:
\begin{align*}
&\sE_{\ldefsc}(h) - \sE_{\ldefsc}^*( \sH) + \sM_{\ldefsc}( \sH)\\
& = \E_{X}\bracket*{\Delta \sC_{\ldefsc}(h, x)}\\
& \leq \E_X\bracket*{Q\Gamma\paren*{\frac{1}{Q}\Delta \sC_{\lsc}(h, x)}}\\
& \leq Q \Gamma\paren*{\frac{1}{Q} \E_X\bracket*{\Delta \sC_{\lsc}(h, x)}}
\tag{concavity of $\Gamma$ and Jensen's ineq.}\\
& = Q \Gamma\paren*{\frac{\sE_{\lsc}(h)-\sE_{\lsc}^*( \sH) + \sM_{\lsc}( \sH)}{Q}}\\
& = \paren*{\num+1-\E_y
  \bracket*{c_j(x, y)}} \Gamma\paren*{\frac{\sE_{\lsc}(h) - \sE_{\lsc}^*( \sH) + \sM_{\lsc}( \sH)}{\num+1-\E_y
  \bracket*{c_j(x, y)}}}\\
&\leq \paren*{\num + 1-\sum_{j = 1}^{\num}\uv c_j} \Gamma\paren*{\frac{\sE_{\lsc}(h) - \sE_{\lsc}^*( \sH) + \sM_{\lsc}( \sH)}{\num + 1 - \sum_{j = 1}^{\num}\ov c_j}}
\tag{$\uv c_j \leq c_j(x, y)\leq \ov c_j, \forall j\in [\num]$}
\end{align*}
and $\sE_{\ldefsc}(h) - \sE_{\ldefsc}^*( \sH) + \sM_{\ldefsc}( \sH)\leq \Gamma\paren*{\sE_{\lsc}(h) - \sE_{\lsc}^*( \sH) + \sM_{\lsc}( \sH)}$ when $\Gamma$ is linear, which completes the proof.
\end{proof}

\section{Examples of deferral surrogate losses and
  their \texorpdfstring{$\sH$}{H}-consistency bounds}
\label{app:sur-score-example}
\subsection{\texorpdfstring{$\ell$}{ell} being adopted as comp-sum losses}
\label{app:sur-score-example-comp}
\paragraph{Example: $\ell=\ell_{\rm{exp}}$.} Plug in $\ell=\ell_{\rm{exp}}=\sum_{y'\neq y} e^{h(x, y') - h(x, y)}$ in \eqref{eq:sur-score}, we obtain
\begin{align*}
\sfL = \sum_{y'\neq y} e^{h(x, y') - h(x, y)} +\sum_{j=1}^{\num}(1-c_j(x,y))\sum_{y'\neq n+j} e^{h(x, y') - h(x, n+j)}.
\end{align*}
By \citet[Theorem~1]{MaoMohriZhong2023cross}, $\ell_{\rm{exp}}$  admits an $\sH$-consistency bound with respect to $\ell_{0-1}$ with $\Gamma(t)=\sqrt{2t}$, using Corollary~\ref{cor:bound-score}, we obtain
\begin{equation*}
\sE_{\ldefsc}(h) - \sE_{\ldefsc}^*( \sH)
\leq \sqrt{2}\paren[\bigg]{\num + 1 - \sum_{j = 1}^{\num}\uv c_j} \paren*{\frac{\sE_{\lsc}(h) - \sE_{\lsc}^*( \sH)}{\num + 1-\sum_{j = 1}^{\num}\ov c_j}}^{\frac12}.
\end{equation*}
Since $1 \leq \num + 1 - \sum_{j = 1}^{\num}\ov
c_j\leq \num + 1 - \sum_{j = 1}^{\num}\uv c_j\leq \num + 1$, the bound can be simplified as
\begin{equation*}
\sE_{\ldefsc}(h) - \sE_{\ldefsc}^*( \sH)
\leq \sqrt{2}(\num+1)\paren*{\sE_{\lsc}(h) - \sE_{\lsc}^*( \sH)}^{\frac12}.
\end{equation*}

\paragraph{Example: $\ell=\ell_{\rm{log}}$.} Plug in $\ell=\ell_{\rm{log}}=- \log \bracket*{\frac{e^{h(x,y)}}{\sum_{y' \in \ov\sY} e^{h(x,y')}}}$ in \eqref{eq:sur-score}, we obtain
\begin{align*}
\sfL = -\log\paren*{\frac{e^{h(x,y)}}{\sum_{y'\in \ov \sY}e^{h(x,y')}}} -\sum_{j=1}^{\num}(1-c_j(x,y))\log\paren*{\frac{e^{h(x,n+j)}}{\sum_{y'\in \ov \sY}e^{h(x,y')}}}.
\end{align*}
By \citet[Theorem~1]{MaoMohriZhong2023cross}, $\ell_{\rm{log}}$  admits an $\sH$-consistency bound with respect to $\ell_{0-1}$ with $\Gamma(t)=\sqrt{2t}$, using Corollary~\ref{cor:bound-score}, we obtain
\begin{equation*}
\sE_{\ldefsc}(h) - \sE_{\ldefsc}^*( \sH)
\leq \sqrt{2}\paren[\bigg]{\num + 1 - \sum_{j = 1}^{\num}\uv c_j} \paren*{\frac{\sE_{\lsc}(h) - \sE_{\lsc}^*( \sH)}{\num + 1-\sum_{j = 1}^{\num}\ov c_j}}^{\frac12}.
\end{equation*}
Since $1 \leq \num + 1 - \sum_{j = 1}^{\num}\ov
c_j\leq \num + 1 - \sum_{j = 1}^{\num}\uv c_j\leq \num + 1$, the bound can be simplified as
\begin{equation*}
\sE_{\ldefsc}(h) - \sE_{\ldefsc}^*( \sH)
\leq \sqrt{2}(\num+1)\paren*{\sE_{\lsc}(h) - \sE_{\lsc}^*( \sH)}^{\frac12}.
\end{equation*}

\paragraph{Example: $\ell=\ell_{\rm{gce}}$.} Plug in $\ell=\ell_{\rm{gce}}==\frac{1}{\alpha}\bracket*{1 - \bracket*{\frac{e^{h(x,y)}}
    {\sum_{y'\in \ov\sY} e^{h(x,y')}}}^{\alpha}}$ in \eqref{eq:sur-score}, we obtain
\begin{align*}
\sfL = \frac{1}{\alpha}\bracket*{1 - \bracket*{\frac{e^{h(x,y)}}
    {\sum_{y'\in \ov \sY} e^{h(x,y')}}}^{\alpha}} +\frac{1}{\alpha}\sum_{j=1}^{\num}(1-c_j(x,y))\bracket*{1 - \bracket*{\frac{e^{h(x,n+j)}}
    {\sum_{y'\in \ov \sY} e^{h(x,y')}}}^{\alpha}}.
\end{align*}
By \citet[Theorem~1]{MaoMohriZhong2023cross}, $\ell_{\rm{gce}}$  admits an $\sH$-consistency bound with respect to $\ell_{0-1}$ with $\Gamma(t)=\sqrt{2n^{\alpha}t}$, using Corollary~\ref{cor:bound-score}, we obtain
\begin{equation*}
\sE_{\ldefsc}(h) - \sE_{\ldefsc}^*( \sH)
\leq \sqrt{2n^{
\alpha}}\paren[\bigg]{\num + 1 - \sum_{j = 1}^{\num}\uv c_j} \paren*{\frac{\sE_{\lsc}(h) - \sE_{\lsc}^*( \sH)}{\num + 1-\sum_{j = 1}^{\num}\ov c_j}}^{\frac12}.
\end{equation*}
Since $1 \leq \num + 1 - \sum_{j = 1}^{\num}\ov
c_j\leq \num + 1 - \sum_{j = 1}^{\num}\uv c_j\leq \num + 1$, the bound can be simplified as
\begin{equation*}
\sE_{\ldefsc}(h) - \sE_{\ldefsc}^*( \sH)
\leq \sqrt{2n^{
\alpha}}(\num+1)\paren*{\sE_{\lsc}(h) - \sE_{\lsc}^*( \sH)}^{\frac12}.
\end{equation*}

\paragraph{Example: $\ell=\ell_{\rm{mae}}$.} Plug in $\ell=\ell_{\rm{mae}}=1 - \frac{e^{h(x,y)}}{\sum_{y'\in \sY} e^{h(x, y')}}$ in \eqref{eq:sur-score}, we obtain
\begin{align*}
\sfL = 1 - \frac{e^{h(x,y)}}{\sum_{y'\in \ov \sY} e^{h(x, y')}} +\sum_{j=1}^{\num}(1-c_j(x,y))\paren*{1 - \frac{e^{h(x,n+j)}}{\sum_{y'\in \ov \sY} e^{h(x, y')}}}.
\end{align*}
By \citet[Theorem~1]{MaoMohriZhong2023cross}, $\ell_{\rm{mae}}$  admits an $\sH$-consistency bound with respect to $\ell_{0-1}$ with $\Gamma(t)=nt$, using Corollary~\ref{cor:bound-score}, we obtain
\begin{equation*}
\sE_{\ldefsc}(h) - \sE_{\ldefsc}^*( \sH)
\leq  n \paren*{\sE_{\lsc}(h) - \sE_{\lsc}^*( \sH)}.
\end{equation*}

\subsection{\texorpdfstring{$\ell$}{ell} being adopted as sum losses}
\label{app:sur-score-example-sum}
\paragraph{Example: $\ell=\Phi_{\mathrm{sq}}^{\mathrm{sum}}$.} Plug in $\ell=\Phi_{\mathrm{sq}}^{\mathrm{sum}}=\sum_{y'\neq y}\Phi_{\rm{sq}}\paren*{h(x,y)-h(x,y')}$ in \eqref{eq:sur-score}, we obtain
\begin{align*}
\sfL = \sum_{y'\neq y} \Phi_{\rm{sq}}\paren*{\Delta_h(x,y,y')} +\sum_{j=1}^{\num}(1-c_j(x,y))\sum_{y'\neq n+j} \Phi_{\rm{sq}}\paren*{\Delta_h(x,n+j,y')},
\end{align*}
where $\Delta_h(x,y,y')=h(x, y) - h(x, y')$ and $\Phi_{\mathrm{sq}}(t)=\max\curl*{0, 1 - t}^2$.
By \citet[Table~2]{AwasthiMaoMohriZhong2022multi}, $\Phi_{\mathrm{sq}}^{\mathrm{sum}}$ admits an $\sH$-consistency bound with respect to $\ell_{0-1}$ with $\Gamma(t)=\sqrt{t}$, using Corollary~\ref{cor:bound-score}, we obtain
\begin{equation*}
\sE_{\ldefsc}(h) - \sE_{\ldefsc}^*( \sH)
\leq \paren[\bigg]{\num + 1 - \sum_{j = 1}^{\num}\uv c_j} \paren*{\frac{\sE_{\lsc}(h) - \sE_{\lsc}^*( \sH)}{\num + 1-\sum_{j = 1}^{\num}\ov c_j}}^{\frac12}.
\end{equation*}
Since $1 \leq \num + 1 - \sum_{j = 1}^{\num}\ov
c_j\leq \num + 1 - \sum_{j = 1}^{\num}\uv c_j\leq \num + 1$, the bound can be simplified as
\begin{equation*}
\sE_{\ldefsc}(h) - \sE_{\ldefsc}^*( \sH)
\leq (\num+1)\paren*{\sE_{\lsc}(h) - \sE_{\lsc}^*( \sH)}^{\frac12}.
\end{equation*}

\paragraph{Example: $\ell=\Phi_{\mathrm{exp}}^{\mathrm{sum}}$.} Plug in $\ell=\Phi_{\mathrm{exp}}^{\mathrm{sum}}=\sum_{y'\neq y}\Phi_{\rm{exp}}\paren*{h(x,y)-h(x,y')}$ in \eqref{eq:sur-score}, we obtain
\begin{align*}
\sfL = \sum_{y'\neq y} \Phi_{\rm{exp}}\paren*{\Delta_h(x,y,y')} +\sum_{j=1}^{\num}(1-c_j(x,y))\sum_{y'\neq n+j} \Phi_{\rm{exp}}\paren*{\Delta_h(x,n+j,y')},
\end{align*}
where $\Delta_h(x,y,y')=h(x, y) - h(x, y')$ and $\Phi_{\mathrm{exp}}(t)=e^{-t}$.
By \citet[Table~2]{AwasthiMaoMohriZhong2022multi}, $\Phi_{\mathrm{exp}}^{\mathrm{sum}}$ admits an $\sH$-consistency bound with respect to $\ell_{0-1}$ with $\Gamma(t)=\sqrt{2t}$, using Corollary~\ref{cor:bound-score}, we obtain
\begin{equation*}
\sE_{\ldefsc}(h) - \sE_{\ldefsc}^*( \sH)
\leq \sqrt{2}\paren[\bigg]{\num + 1 - \sum_{j = 1}^{\num}\uv c_j} \paren*{\frac{\sE_{\lsc}(h) - \sE_{\lsc}^*( \sH)}{\num + 1-\sum_{j = 1}^{\num}\ov c_j}}^{\frac12}.
\end{equation*}
Since $1 \leq \num + 1 - \sum_{j = 1}^{\num}\ov
c_j\leq \num + 1 - \sum_{j = 1}^{\num}\uv c_j\leq \num + 1$, the bound can be simplified as
\begin{equation*}
\sE_{\ldefsc}(h) - \sE_{\ldefsc}^*( \sH)
\leq \sqrt{2}(\num+1)\paren*{\sE_{\lsc}(h) - \sE_{\lsc}^*( \sH)}^{\frac12}.
\end{equation*}

\paragraph{Example: $\ell=\Phi_{\rho}^{\mathrm{sum}}$.} Plug in $\ell=\Phi_{\rho}^{\mathrm{sum}}=\sum_{y'\neq y}\Phi_{\rho}\paren*{h(x,y)-h(x,y')}$ in \eqref{eq:sur-score}, we obtain
\begin{align*}
\sfL &= \sum_{y'\neq y} \Phi_{\rho}\paren*{\Delta_h(x,y,y')} + \sum_{j=1}^{\num}(1-c_j(x,y))\sum_{y'\neq n+j} \Phi_{\rho}\paren*{\Delta_h(x,n+j,y')},
\end{align*}
where $\Delta_h(x,y,y')=h(x, y) - h(x, y')$ and $\Phi_{\rho}(t)=\min\curl*{\max\curl*{0,1 - t/\rho},1}$.
By \citet[Table~2]{AwasthiMaoMohriZhong2022multi}, $\Phi_{\rho}^{\mathrm{sum}}$ admits an $\sH$-consistency bound with respect to $\ell_{0-1}$ with $\Gamma(t)=t$, using Corollary~\ref{cor:bound-score}, we obtain
\begin{equation*}
\sE_{\ldefsc}(h) - \sE_{\ldefsc}^*( \sH)
\leq \sE_{\lsc}(h) - \sE_{\lsc}^*( \sH).
\end{equation*}

\subsection{\texorpdfstring{$\ell$}{ell} being adopted as constrained losses}
\label{app:sur-score-example-cstnd}
\paragraph{Example: $\ell=\Phi_{\mathrm{hinge}}^{\mathrm{cstnd}}$.} Plug in $\ell=\Phi_{\mathrm{hinge}}^{\mathrm{cstnd}}=\sum_{y'\neq y}\Phi_{\mathrm{hinge}}\paren*{-h(x, y')}$ in \eqref{eq:sur-score}, we obtain
\begin{align*}
\sfL = \sum_{y'\neq y}\Phi_{\mathrm{hinge}}\paren*{-h(x, y')} + \sum_{j=1}^{\num}(1-c_j(x,y))\sum_{y'\neq n+j}\Phi_{\mathrm{hinge}}\paren*{-h(x, y')},
\end{align*}
where $\Phi_{\mathrm{hinge}}(t) = \max\curl*{0,1 - t}$ with the constraint that $\sum_{y\in \sY}h(x,y)=0$.
By \citet[Table~3]{AwasthiMaoMohriZhong2022multi}, $\Phi_{\mathrm{hinge}}^{\mathrm{cstnd}}$ admits an $\sH$-consistency bound with respect to $\ell_{0-1}$ with $\Gamma(t)=t$, using Corollary~\ref{cor:bound-score}, we obtain
\begin{equation*}
\sE_{\ldefsc}(h) - \sE_{\ldefsc}^*( \sH)
\leq \sE_{\lsc}(h) - \sE_{\lsc}^*( \sH).
\end{equation*}

\paragraph{Example: $\ell=\Phi_{\mathrm{sq}}^{\mathrm{cstnd}}$.} Plug in $\ell=\Phi_{\mathrm{sq}}^{\mathrm{cstnd}}=\sum_{y'\neq y}\Phi_{\mathrm{sq}}\paren*{-h(x, y')}$ in \eqref{eq:sur-score}, we obtain
\begin{align*}
\sfL = \sum_{y'\neq y}\Phi_{\mathrm{sq}}\paren*{-h(x, y')} +\sum_{j=1}^{\num}(1-c_j(x,y))\sum_{y'\neq n+j}\Phi_{\mathrm{sq}}\paren*{-h(x, y')},
\end{align*}
where $\Phi_{\mathrm{sq}}(t) = \max\curl*{0, 1 - t}^2$ with the constraint that $\sum_{y\in \sY}h(x,y)=0$.
By \citet[Table~3]{AwasthiMaoMohriZhong2022multi}, $\Phi_{\mathrm{sq}}^{\mathrm{cstnd}}$ admits an $\sH$-consistency bound with respect to $\ell_{0-1}$ with $\Gamma(t)=\sqrt{t}$, using Corollary~\ref{cor:bound-score}, we obtain
\begin{equation*}
\sE_{\ldefsc}(h) - \sE_{\ldefsc}^*( \sH)
\leq \paren[\bigg]{\num + 1 - \sum_{j = 1}^{\num}\uv c_j} \paren*{\frac{\sE_{\lsc}(h) - \sE_{\lsc}^*( \sH)}{\num + 1-\sum_{j = 1}^{\num}\ov c_j}}^{\frac12}.
\end{equation*}
Since $1 \leq \num + 1 - \sum_{j = 1}^{\num}\ov
c_j\leq \num + 1 - \sum_{j = 1}^{\num}\uv c_j\leq \num + 1$, the bound can be simplified as
\begin{equation*}
\sE_{\ldefsc}(h) - \sE_{\ldefsc}^*( \sH)
\leq (\num+1)\paren*{\sE_{\lsc}(h) - \sE_{\lsc}^*( \sH)}^{\frac12}.
\end{equation*}

\paragraph{Example: $\ell=\Phi_{\mathrm{exp}}^{\mathrm{cstnd}}$.} Plug in $\ell=\Phi_{\mathrm{exp}}^{\mathrm{cstnd}}=\sum_{y'\neq y}\Phi_{\mathrm{exp}}\paren*{-h(x, y')}$ in \eqref{eq:sur-score}, we obtain
\begin{align*}
\sfL = \sum_{y'\neq y}\Phi_{\mathrm{exp}}\paren*{-h(x, y')} + \sum_{j=1}^{\num}(1-c_j(x,y))\sum_{y'\neq n+j}\Phi_{\mathrm{exp}}\paren*{-h(x, y')},
\end{align*}
where $\Phi_{\mathrm{exp}}(t)=e^{-t}$ with the constraint that $\sum_{y\in \sY}h(x,y)=0$.
By \citet[Table~3]{AwasthiMaoMohriZhong2022multi}, $\Phi_{\mathrm{exp}}^{\mathrm{cstnd}}$ admits an $\sH$-consistency bound with respect to $\ell_{0-1}$ with $\Gamma(t)=\sqrt{2t}$, using Corollary~\ref{cor:bound-score}, we obtain
\begin{equation*}
\sE_{\ldefsc}(h) - \sE_{\ldefsc}^*( \sH)
\leq \sqrt{2}\paren[\bigg]{\num + 1 - \sum_{j = 1}^{\num}\uv c_j} \paren*{\frac{\sE_{\lsc}(h) - \sE_{\lsc}^*( \sH)}{\num + 1-\sum_{j = 1}^{\num}\ov c_j}}^{\frac12}.
\end{equation*}
Since $1 \leq \num + 1 - \sum_{j = 1}^{\num}\ov
c_j\leq \num + 1 - \sum_{j = 1}^{\num}\uv c_j\leq \num + 1$, the bound can be simplified as
\begin{equation*}
\sE_{\ldefsc}(h) - \sE_{\ldefsc}^*( \sH)
\leq \sqrt{2}(\num+1)\paren*{\sE_{\lsc}(h) - \sE_{\lsc}^*( \sH)}^{\frac12}.
\end{equation*}

\paragraph{Example: $\ell=\Phi_{\rho}^{\mathrm{cstnd}}$.} Plug in $\ell=\Phi_{\rho}^{\mathrm{cstnd}}=\sum_{y'\neq y}\Phi_{\rho}\paren*{-h(x, y')}$ in \eqref{eq:sur-score}, we obtain
\begin{align*}
\sfL =\sum_{y'\neq y}\Phi_{\rho}\paren*{-h(x, y')} + \sum_{j=1}^{\num}(1-c_j(x,y))\sum_{y'\neq n+j}\Phi_{\rho}\paren*{-h(x, y')},
\end{align*}
where $\Phi_{\rho}(t)=\min\curl*{\max\curl*{0,1 - t/\rho},1}$ with the constraint that $\sum_{y\in \sY}h(x,y)=0$.
By \citet[Table~3]{AwasthiMaoMohriZhong2022multi}, $\Phi_{\rho}^{\mathrm{cstnd}}$ admits an $\sH$-consistency bound with respect to $\ell_{0-1}$ with $\Gamma(t)=t$, using Corollary~\ref{cor:bound-score}, we obtain
\begin{equation*}
\sE_{\ldefsc}(h) - \sE_{\ldefsc}^*( \sH)
\leq \sE_{\lsc}(h) - \sE_{\lsc}^*( \sH).
\end{equation*}

\section{Proof of learning bounds
  for deferral surrogate losses (Theorem~\ref{Thm:Gbound-score})}
\label{app:Gbound-score}
\GBoundScore*
\begin{proof}
  By using the standard Rademacher complexity bounds \citep{MohriRostamizadehTalwalkar2018}, for any $\delta>0$,
  with probability at least $1 - \delta$, the following holds for all $h \in \sH$:
\[
\abs*{\sE_{\sfL}(h) - \h\sE_{\sfL,S}(h)}
\leq 2 \Rad_m^{\sfL}(\sH) +
B_{\sfL} \sqrt{\tfrac{\log (2/\delta)}{2m}}.
\]
Fix $\e > 0$. By the definition of the infimum, there exists $h^* \in
\sH$ such that $\sE_{\sfL}(h^*) \leq
\sE_{\sfL}^*(\sH) + \e$. By definition of
$\h h_S$, we have
\begin{align*}
  & \sE_{\sfL}(\h h_S) - \sE_{\sfL}^*(\sH)\\
  & = \sE_{\sfL}(\h h_S) - \h\sE_{\sfL,S}(\h h_S) + \h\sE_{\sfL,S}(\h h_S) - \sE_{\sfL}^*(\sH)\\
  & \leq \sE_{\sfL}(\h h_S) - \h\sE_{\sfL,S}(\h h_S) + \h\sE_{\sfL,S}(h^*) - \sE_{\sfL}^*(\sH)\\
  & \leq \sE_{\sfL}(\h h_S) - \h\sE_{\sfL,S}(\h h_S) + \h\sE_{\sfL,S}(h^*) - \sE_{\sfL}^*(h^*) + \e\\
  & \leq
  2 \bracket*{2 \Rad_m^{\sfL}(\sH) +
B_{\sfL} \sqrt{\tfrac{\log (2/\delta)}{2m}}} + \e.
\end{align*}
Since the inequality holds for all $\e > 0$, it implies:
\[
\sE_{\sfL}(\h h_S) - \sE_{\sfL}^*(\sH)
\leq 
4 \Rad_m^{\sfL}(\sH) +
2 B_{\sfL} \sqrt{\tfrac{\log (2/\delta)}{2m}}.
\]
Plugging in this inequality in the bound
\eqref{eq:H-consistency-bounds} completes the proof.
\end{proof}
\restoreatoc

\chapter{Appendix to Chapter~\ref{ch5}}

\disableatoc
\section{Examples of two-stage score-based surrogate losses}
\label{app:sur-score-example-two-stage}

\textbf{Example: $\ell_2 = \ell_{\rm{exp}}$.} For $\ell_2(\ov h_d,x,
y) = \ell_{\rm{exp}}(\ov h_d,x, y) = \sum_{y'\neq y}e^{\ov h_d(x, y')
  - \ov h_d(x, y)}$, by \eqref{eq:ell-Phi-h-score}, we have
\begin{align*}
&\lsc_{\hp} \paren*{\hd, x, y}\\
& = \1_{\hh_p(x) = y} \, \ell_2(\ov h_d, x, 0)
+ \sum_{j = 1}^{\num}\bar c_j(x, y) \ell_2(\ov h_d, x, j)\\
& = \1_{\hh_p(x) = y} \, \sum_{y'\neq 0}e^{\ov h_d(x, y') - \ov h_d(x, 0)}
+ \sum_{j = 1}^{\num}\bar c_j(x, y) \sum_{y'\neq j}e^{\ov h_d(x, y') - \ov h_d(x, j)}\\
& = \1_{\hhp(x) = y} \sum_{i = 1}^{\num}e^{h(x, n+i) - \max_{y\in \sY}h(x, y)} + \sum_{j = 1}^{\num}\bar c_j(x, y)\bracket*{\sum_{i = 1,i\neq j }^{\num}e^{h(x, n+i) -h(x, n+j)}+e^{\max_{y\in \sY}h(x, y) - h(x, n+j)}}.
\end{align*}

\textbf{Example: $\ell_2 = \ell_{\rm{log}}$.} For $\ell_2(\ov h_d,x,
y) = \ell_{\rm{log}}(\ov h_d,x, y) = \log\paren*{\sum_{y'\in \sY \cup
    \curl{0}}e^{\ov h_d(x, y') - \ov h_d(x, y)}}$, by
\eqref{eq:ell-Phi-h-score}, we have
\begin{align*}
&\lsc_{\hp} \paren*{\hd, x, y}\\
& = \1_{\hh_p(x) = y} \, \ell_2(\ov h_d, x, 0)
+ \sum_{j = 1}^{\num}\bar c_j(x, y) \ell_2(\ov h_d, x, j)\\
& = \1_{\hh_p(x) = y} \, \log\paren*{\sum_{y'\in \sY \cup \curl{0}}e^{\ov h_d(x, y') - \ov h_d(x, 0)}}
+ \sum_{j = 1}^{\num}\bar c_j(x, y) \log\paren*{\sum_{y'\in \sY \cup \curl{0}}e^{\ov h_d(x, y') - \ov h_d(x, j)}}\\
& = - \1_{\hhp(x) = y}\log\paren*{\frac{e^{\max_{y\in \sY} h(x, y)}}{e^{\max_{y\in \sY} h(x, y)}+ \sum_{i = 1}^{\num}e^{h(x, n+i)}}} - \sum_{j = 1}^{\num}\bar c_j(x, y)\log\paren*{\frac{e^{h(x, n+j)}}{e^{\max_{y\in \sY} h(x, y)} + \sum_{i = 1}^{\num}e^{h(x, n+i)}}}.
\end{align*}

\textbf{Example: $\ell_2 = \ell_{\rm{gce}}$.} For $\ell_2(\ov h_d,x,
y) = \ell_{\rm{gce}}(\ov h_d,x, y) = \frac{1}{\alpha}\bracket*{1 -
  \bracket*{\frac{e^{\ov h_d(x, y)}} {\sum_{y'\in \sY \cup \curl{0}}
      e^{\ov h_d(x, y')}}}^{\alpha}},\alpha\in (0,1)$, by
\eqref{eq:ell-Phi-h-score}, we have
\begin{align*}
&\lsc_{\hp} \paren*{\hd, x, y}\\
& = \1_{\hh_p(x) = y} \, \ell_2(\ov h_d, x, 0)
+ \sum_{j = 1}^{\num}\bar c_j(x, y) \ell_2(\ov h_d, x, j)\\
& = \1_{\hh_p(x) = y} \, \frac{1}{\alpha}\bracket*{1 - \bracket*{\frac{e^{\ov h_d(x,0)}}
  {\sum_{y'\in \sY \cup \curl{0}} e^{\ov h_d(x, y')}}}^{\alpha}}
+ \sum_{j = 1}^{\num}\bar c_j(x, y) \frac{1}{\alpha}\bracket*{1 - \bracket*{\frac{e^{\ov h_d(x,j)}}
  {\sum_{y'\in \sY \cup \curl{0}} e^{\ov h_d(x, y')}}}^{\alpha}}\\
& = \1_{\hhp(x) = y}\frac{1}{\alpha}\bracket*{1- \bracket*{\frac{e^{\max_{y\in \sY} h(x, y)}}{e^{\max_{y\in \sY} h(x, y)}+ \sum_{i = 1}^{\num}e^{h(x, n+i)}}}^{\alpha}}\\
&\qquad + \sum_{j = 1}^{\num}\bar c_j(x, y)\frac{1}{\alpha}\bracket*{1- \bracket*{\frac{e^{h(x, n+j)}}{e^{\max_{y\in \sY} h(x, y)} + \sum_{i = 1}^{\num}e^{h(x, n+i)}}}^{\alpha}}.
\end{align*}

\textbf{Example: $\ell_2 = \ell_{\rm{mae}}$.} For $\ell_2(\ov h_d,x,
y) = \ell_{\rm{mae}}(\ov h_d,x, y) = 1 - \frac{e^{\ov h_d(x,
    y)}}{\sum_{y'\in \sY \cup \curl{0}} e^{\ov h_d(x, y')}}$, by
\eqref{eq:ell-Phi-h-score}, we have
\begin{align*}
&\lsc_{\hp} \paren*{\hd, x, y}\\
& = \1_{\hh_p(x) = y} \, \ell_2(\ov h_d, x, 0)
+ \sum_{j = 1}^{\num}\bar c_j(x, y) \ell_2(\ov h_d, x, j)\\
& = \1_{\hh_p(x) = y} \, \paren*{1 - \frac{e^{\ov h_d(x,0)}}{\sum_{y'\in \sY \cup \curl{0}} e^{\ov h_d(x, y')}}}
+ \sum_{j = 1}^{\num}\bar c_j(x, y) \paren*{1 - \frac{e^{\ov h_d(x,j)}}{\sum_{y'\in \sY \cup \curl{0}} e^{\ov h_d(x, y')}}}\\
& = \1_{\hhp(x) = y}\bracket*{1- \frac{e^{\max_{y\in \sY} h(x, y)}}{e^{\max_{y\in \sY} h(x, y)}+ \sum_{i = 1}^{\num}e^{ h(x, n+i)}}} + \sum_{j = 1}^{\num}\bar c_j(x, y)\bracket*{1- \frac{e^{h(x, n+j)}}{e^{\max_{y\in \sY} h(x, y)} + \sum_{i = 1}^{\num}e^{h(x, n+i)}}}.
\end{align*}

\section{Examples of two-stage predictor-rejector surrogate losses}
\label{app:sur-general-example-two-stage}

\textbf{Example: $\ell_2 = \ell_{\rm{exp}}$.} For $\ell_2(\ov r,x, y)
= \ell_{\rm{exp}}(\ov r,x, y) = \sum_{y'\neq y}e^{\ov r(x, y') - \ov
  r(x, y)}$, by \eqref{eq:ell-Phi-h-multi}, we have
\begin{align*}
&\sfL_{h} \paren*{r, x, y}\\
& = \1_{\hh(x) = y} \, \ell_2(\ov r, x, 0)
+ \sum_{j = 1}^{\num}\bar c_j(x, y) \ell_2(\ov r, x, j)\\
& = \1_{\hh(x) = y} \, \sum_{y'\neq 0}e^{\ov r(x, y') - \ov r(x, 0)}
+ \sum_{j = 1}^{\num}\bar c_j(x, y) \sum_{y'\neq j}e^{\ov r(x, y') - \ov r(x, j)}\\
& = \1_{\hh(x) = y} \sum_{i = 1}^{\num}e^{-r_i(x)} + \sum_{j = 1}^{\num}\bar c_j(x, y)\bracket*{\sum_{i = 1,i\neq j }^{\num}e^{r_j(x) -r_i(x)}+e^{r_j(x)}}.
\end{align*}

\textbf{Example: $\ell_2 = \ell_{\rm{log}}$.} For $\ell_2(\ov r,x, y)
= \ell_{\rm{log}}(\ov r,x, y) = \log\paren*{\sum_{y'\in \sY \cup
    \curl{0}}e^{\ov r(x, y') - \ov r(x, y)}}$, by
\eqref{eq:ell-Phi-h-multi}, we have
\begin{align*}
&\sfL_{h} \paren*{r, x, y}\\
& = \1_{\hh(x) = y} \, \ell_2(\ov r, x, 0)
+ \sum_{j = 1}^{\num} \bar c_j(x, y) \ell_2(\ov r, x, j)\\
& = \1_{\hh(x) = y} \, \log\paren*{\sum_{y'\in \sY \cup \curl{0}}e^{\ov r(x, y') - \ov r(x, 0)}}
+ \sum_{j = 1}^{\num}\bar c_j(x, y) \log\paren*{\sum_{y'\in \sY \cup \curl{0}}e^{\ov r(x, y') - \ov r(x, j)}}\\
& = - \1_{\hh(x) = y} \log\paren*{\frac{1}{1+ \sum_{i = 1}^{\num}e^{-r_i(x)}}}- \sum_{j = 1}^{\num}\bar c_j(x, y)\log\paren*{\frac{e^{-r_j(x)}}{1+ \sum_{i = 1}^{\num}e^{-r_i(x)}}}.
\end{align*}

\textbf{Example: $\ell_2 = \ell_{\rm{gce}}$.} For $\ell_2(\ov r,x, y)
= \ell_{\rm{gce}}(\ov r,x, y) = \frac{1}{\alpha}\bracket*{1 -
  \bracket*{\frac{e^{\ov r(x, y)}} {\sum_{y'\in \sY \cup \curl{0}}
      e^{\ov r(x, y')}}}^{\alpha}},\alpha\in (0,1)$, by
\eqref{eq:ell-Phi-h-multi}, we have
\begin{align*}
&\sfL_{h} \paren*{r, x, y}\\
& = \1_{\hh(x) = y} \, \ell_2(\ov r, x, 0)
+ \sum_{j = 1}^{\num}\bar c_j(x, y) \ell_2(\ov r, x, j)\\
& = \1_{\hh(x) = y} \, \frac{1}{\alpha}\bracket*{1 - \bracket*{\frac{e^{\ov r(x,0)}}
  {\sum_{y'\in \sY \cup \curl{0}} e^{\ov r(x, y')}}}^{\alpha}}
+ \sum_{j = 1}^{\num}\bar c_j(x, y) \frac{1}{\alpha}\bracket*{1 - \bracket*{\frac{e^{\ov r(x,j)}}
  {\sum_{y'\in \sY \cup \curl{0}} e^{\ov r(x, y')}}}^{\alpha}}\\
& = \1_{\hh(x) = y}\frac{1}{\alpha}\bracket*{1- \bracket*{\frac{1}{1+ \sum_{i = 1}^{\num}e^{-r_i(x)}}}^{\alpha}} + \sum_{j = 1}^{\num}\bar c_j(x, y)\frac{1}{\alpha}\bracket*{1- \bracket*{\frac{e^{-r_j(x)}}{1+ \sum_{i = 1}^{\num}e^{-r_i(x)}}}^{\alpha}}.
\end{align*}

\textbf{Example: $\ell_2 = \ell_{\rm{mae}}$.} For $\ell_2(\ov r,x, y)
= \ell_{\rm{mae}}(\ov r,x, y) = 1 - \frac{e^{\ov r(x, y)}}{\sum_{y'\in
    \sY \cup \curl{0}} e^{\ov r(x, y')}}$, by
\eqref{eq:ell-Phi-h-multi}, we have
\begin{align*}
&\sfL_{h} \paren*{r, x, y}\\
& = \1_{\hh(x) = y} \, \ell_2(\ov r, x, 0)
+ \sum_{j = 1}^{\num}\bar c_j(x, y) \ell_2(\ov r, x, j)\\
& = \1_{\hh(x) = y} \, \paren*{1 - \frac{e^{\ov r(x,0)}}{\sum_{y'\in \sY \cup \curl{0}} e^{\ov r(x, y')}}}
+ \sum_{j = 1}^{\num}\bar c_j(x, y) \paren*{1 - \frac{e^{\ov r(x,j)}}{\sum_{y'\in \sY \cup \curl{0}} e^{\ov r(x, y')}}}\\
& = \1_{\hh(x) = y}\bracket*{1- \frac{1}{1+ \sum_{i = 1}^{\num}e^{-r_i(x)}}} + \sum_{j = 1}^{\num}\bar c_j(x, y)\bracket*{1- \frac{e^{-r_j(x)}}{1+ \sum_{i = 1}^{\num}e^{-r_i(x)}}}.
\end{align*}

\section{Proof of \texorpdfstring{$\sH$}{H}-consistency bounds for score-based
  two-stage surrogate losses (Theorem~\ref{Thm:bound-general-two-stage-score})}
\label{app:bound-general-two-stage-score}
\BoundGenralTwoStepScore*
\begin{proof}
If $\hh(x)\in [n]$, then $\hh(x) = \hhp(x)$.
Thus, the learning to defer loss can be expressed as follows:
\begin{align*}
\ldefsc(h, x, y)
& = \1_{\hh(x)\neq y}\1_{\hh(x)\in [n]} + \sum_{j = 1}^{\num} c_j(x, y) \1_{\hh(x) = n + j}\\
& = \1_{\hhp(x)\neq y}\1_{\hh(x)\in [n]} + \sum_{j = 1}^{\num} c_j(x, y) \1_{\hh(x) = n + j}.
\end{align*}
Let $\bar c_0\paren*{x, y} = \1_{\hhp(x)= y}$. Since $h = (\hp,h_d)$, we
can rewrite $\sE_{\ldefsc}(h) - \sE_{\ldefsc}^*( \sH) + \sM_{\ldefsc}(
\sH)$ as
\begin{equation}
\label{eq:expression-two-step-score}
\begin{aligned}
& \sE_{\ldefsc}(h) - \sE_{\ldefsc}^*( \sH) + \sM_{\ldefsc}( \sH)\\
& =  \mathbb{E}_{X}\bracket*{\sC_{\ldefsc}(h, x) - \sC^*_{\ldefsc}(\sH,x)} \\
  & =  \mathbb{E}_{X}\bracket*{\sC_{\ldefsc}(h, x) - \inf_{h_d\in \sH_d}\sC_{\ldefsc}(h, x)
    + \inf_{h_d\in \sH_d}\sC_{\ldefsc}(h, x) - \sC^*_{\ldefsc}(\sH,x)}\\
  & = \mathbb{E}_{X}\bracket*{\sC_{\ldefsc}(h, x) - \inf_{h_d\in \sH_d}\sC_{\ldefsc}(h, x)}
  + \mathbb{E}_{X}\bracket*{\inf_{h_d\in \sH_d}\sC_{\ldefsc}(h, x) - \sC^*_{\ldefsc}(\sH,x)}.
\end{aligned}
\end{equation}
Let $\ov \sfp(j \!\mid\! x) = \frac{\E_y\bracket*{\bar c_j(x,
    y)}}{\E_y\bracket*{\sum_{j = 0}^{\num}\bar c_j(x, y)}}$ for any $j\in
\curl*{0,\ldots,\num}$. Note that $\ov \sfp(\cdot \!\mid\! x)$ is the probability
vector on the label space $\curl*{0,\ldots,\num}$. For any $h\in \sH$,
we define $\ov h$ as its augmented hypothesis: $\ov h(x,0) =
\max_{y\in \sY}h(x, y),\ov h(x,1) = h(x,1),\ldots, \ov h(x, \num) =
h(x, \num)$.  By the assumptions, we have
\begin{align*}
& \sC_{\ldefsc}(h, x) - \inf_{h_d\in \sH_d}\sC_{\ldefsc}(h, x)\\
  & = \E_y\bracket*{\1_{\hhp(x)\neq y}\1_{\hh(x)\in [n]} + \sum_{j = 1}^{\num} c_j(x, y) \1_{\hh(x) = n + j}}
  - \inf_{h_d\in \sH_d}\E_y\bracket*{\1_{\hhp(x)\neq y}\1_{\hh(x)\in [n]} + \sum_{j = 1}^{\num} c_j(x, y) \1_{\hh(x) = n + j}}\\
  & = \E_y\bracket*{\sum_{j = 0}^{\num}\bar c_j(x, y)}\times \bracket*{\sum_{j = 0}^{\num}\ov \sfp(j \!\mid\! x)\ell_{0-1}(\ov h, x, j)
    - \inf_{h_d\in \sH_d}\sum_{j = 0}^{\num}\ov \sfp(j \!\mid\! x)\ell_{0-1}(\ov h, x, j)}\\
  & \leq \E_y\bracket*{\sum_{j = 0}^{\num}\bar c_j(x, y)}\times \Gamma_2\bracket*{\sum_{j = 0}^{\num}\ov \sfp(j \!\mid\! x)\ell_{2}(\ov h, x, j)
    - \inf_{h_d\in \sH_d}\sum_{j = 0}^{\num}\ov \sfp(j \!\mid\! x)\ell_{2}(\ov h, x, j)}
  \tag{By $\ov \sH_d$-consistency bounds of $\ell_2$ under assumption, $\lambda = \max_{y\in \sY} h(x, y)$}\\
  & =  \E_y\bracket*{\sum_{j = 0}^{\num}\bar c_j(x, y)} \Gamma_2\paren*{\frac{\E_y\bracket*{\lsc_{\hp}(h_d,x, y)}
      - \inf_{h_d\in \sH_d}\E_y\bracket*{\lsc_{\hp}(h_d,x, y)}}{ \E_y\bracket*{\sum_{j = 0}^{\num}\bar c_j(x, y)}}}\tag{ $\ov \sfp(j \!\mid\! x) = \frac{\E_y\bracket*{\bar c_j(x, y)}}{\E_y\bracket*{\sum_{j = 0}^{\num}\bar c_j(x, y)}}$, $\ov h(x,0) = \max_{y\in \sY}h(x, y)$ and formulation \eqref{eq:ell-Phi-h-score}}\\
  & =  \E_y\bracket*{\sum_{j = 0}^{\num}\bar c_j(x, y)} \Gamma_2\paren*{\frac{\sC_{\lsc_{\hp}}(h_d,x)
      - \sC^*_{\lsc_{\hp}}( \sH_d,x)}{\E_y\bracket*{ \sum_{j = 0}^{\num}\bar c_j(x, y)}}}\\
& \leq
\begin{cases}
\Gamma_2\paren*{\sC_{\lsc_{\hp}}(h_d,x) - \sC^*_{\lsc_{\hp}}( \sH_d,x)} & \text{when $\Gamma_2$ is linear}\\
\paren*{1+ \sum_{j = 1}^{\num}\ov c_j}\Gamma_2\paren*{\frac {\sC_{\lsc_{\hp}}(h_d,x) - \sC^*_{\lsc_{\hp}}( \sH_d,x)}{\sum_{j = 1}^{\num}\uv c_j}} & \text{otherwise}
\end{cases}\\
\tag{$\sum_{j = 1}^{\num}\uv c_j\leq \E_y\bracket*{ \sum_{j = 0}^{\num}\bar c_j(x, y)}
  \leq 1+ \sum_{j = 1}^{\num}\ov c_j$ and $\Gamma_2$ is non-decreasing}\\
& = \begin{cases}
\Gamma_2\paren*{\Delta\sC_{\lsc_{\hp}, \sH_d}(h_d,x)} & \text{when $\Gamma_2$ is linear}\\
\paren*{1+ \sum_{j = 1}^{\num}\ov c_j}\Gamma_2\paren*{\frac {\Delta\sC_{\lsc_{\hp}, \sH_d}(h_d,x)}{\sum_{j = 1}^{\num}\uv c_j}}
& \text{otherwise}
\end{cases}
\end{align*}
and 
\begin{align*}
& \inf_{h_d\in \sH_d}\sC_{\ldefsc}(h, x) - \sC^*_{\ldefsc}(\sH,x)\\
& = \inf_{h_d\in \sH_d}\sC_{\ldefsc}(h, x) - \inf_{\hp\in \sH_p, h_d\in \sH_d}\sC_{\ldefsc}(h, x)\\
& = \inf_{h_d\in \sH_d} \E_y\bracket*{\1_{\hhp(x)\neq y}\1_{\hh(x)\in [n]} + \sum_{j = 1}^{\num} c_j(x, y) \1_{\hh(x) = n + j}}\\
& \quad - \inf_{\hp\in \sH_p,h_d\in \sH_d} \E_y\bracket*{\1_{\hhp(x)\neq y}\1_{\hh(x)\in [n]} + \sum_{j = 1}^{\num} c_j(x, y) \1_{\hh(x) = n + j}}\\
& = \inf_{h_d\in \sH_d} \E_y\bracket*{\1_{\hhp(x)\neq y}\1_{\hh(x)\in [n]} + \sum_{j = 1}^{\num} c_j(x, y) \1_{\hh(x) = n + j}}\\
& \quad- \inf_{h_d\in \sH_d} \E_y\bracket*{\inf_{\hp\in \sH_p}\1_{\hhp(x)\neq y}\1_{\hh(x)\in [n]} + \sum_{j = 1}^{\num} c_j(x, y) \1_{\hh(x) = n + j}}\\
& = \min\curl*{\E_y\bracket*{\1_{\hhp(x)\neq y}},\min_{j\in[p]}\E_y\bracket*{c_j(x, y)}}- \min\curl*{\inf_{\hp\in \sH_p}\E_y\bracket*{\1_{\hhp(x)\neq y}},\min_{j\in[p]}\E_y\bracket*{c_j(x, y)}}\\
& \leq \E_y\bracket*{\1_{\hhp(x)\neq y}} - \inf_{\hp\in \sH_p}\E_y\bracket*{\1_{\hhp(x)\neq y}}\\
& = \sC_{\ell_{0-1}}(\hp, x) - \sC^*_{\ell_{0-1}}( \sH_p,x)\\
& = \Delta \sC_{\ell_{0-1}, \sH_p}(\hp, x)\\
& \leq \Gamma_1\paren*{\Delta\sC_{\ell_1, \sH_p}(\hp, x)}.
\tag{By $\sH_p$-consistency bounds of $\ell$ under assumption}
\end{align*}
Therefore, by \eqref{eq:expression-two-step-score}, we obtain
\begin{align*}
& \sE_{\ldefsc}(h) - \sE^*_{\ldefsc}\paren*{ \sH}+ \sM_{\ldefsc}( \sH)\\
& \leq 
\begin{cases}
\E_X\bracket*{\Gamma_2\paren*{\Delta\sC_{\lsc_{\hp}, \sH_d}(h_d,x)}} + \E_X\bracket*{\Gamma_1\paren*{\Delta\sC_{\ell_1, \sH_p}(\hp, x)}} & \text{when $\Gamma_2$ is linear}\\
\paren*{1+ \sum_{j = 1}^{\num}\ov c_j}\E_X\bracket*{\Gamma_2\paren*{\frac {\Delta\sC_{\lsc_{\hp}, \sH_d}(h_d,x)}{\sum_{j = 1}^{\num}\uv c_j}}} + \E_X\bracket*{\Gamma_1\paren*{\Delta\sC_{\ell_1, \sH_p}(\hp, x)}} & \text{otherwise}
\end{cases}\\
& \leq 
\begin{cases}
\Gamma_2\paren*{\E_X\bracket*{\Delta\sC_{\lsc_{\hp}, \sH_d}(h_d,x)}} + \Gamma_1\paren*{\E_X\bracket*{\Delta\sC_{\ell_1, \sH_p}(\hp, x)}} & \text{when $\Gamma_2$ is linear}\\
\paren*{1+ \sum_{j = 1}^{\num}\ov c_j}\Gamma_2\paren*{\frac1{\sum_{j = 1}^{\num}\uv c_j}\E_X\bracket*{\Delta\sC_{\lsc_{\hp}, \sH_d}(h_d,x)}} + \Gamma_1\paren*{\E_X\bracket*{\Delta\sC_{\ell_1, \sH_p}(\hp, x)}} & \text{otherwise}
\end{cases}
\tag{$\Gamma_1$ and $\Gamma_2$ are concave}\\
& = 
\begin{cases}
\Gamma_1\paren*{\sE_{\ell_1}(\hp) - \sE_{\ell_1}^*( \sH_p) + \sM_{\ell_1}( \sH_p)}
\\\qquad + \Gamma_2\paren*{\sE_{\lsc_{\hp}}(h_d) - \sE_{\lsc_{\hp}}^*( \sH_d) + \sM_{\lsc_{\hp}}( \sH_d)} & \text{when $\Gamma_2$ is linear}\\
\Gamma_1\paren*{\sE_{\ell_1}(\hp) - \sE_{\ell_1}^*( \sH_p) + \sM_{\ell_1}( \sH_p)}
\\\qquad + \paren*{1+ \sum_{j = 1}^{\num}\ov c_j}\Gamma_2\paren*{\frac{\sE_{\lsc_{\hp}}(h_d) - \sE_{\lsc_{\hp}}^*(\sH_d) + \sM_{\lsc_{\hp}}(\sH_d)}{\sum_{j = 1}^{\num}\uv c_j}} & \text{otherwise},
\end{cases}
\end{align*}
which completes the proof.
\end{proof}

\section{Proof of \texorpdfstring{$\ov \sH$}{H}-consistency bounds
 for standard surrogate loss functions (Theorem~\ref{Thm:bound_combine})}
\label{app:bound_comp-sum}

Recall that for a hypothesis $h\colon \sX \times \sY \to \Rset$, we
define $\ov h$ as its augmented hypothesis: $\ov h(\cdot,0) =
\lambda,\ov h(\cdot,1) = h(x,1),\ldots, \ov h(\cdot,n) = h(x, n)$ with
some constant $\lambda\in \mathbb{R}$. We define $\ov \sH$ as the
hypothesis set that consists of all such augmented hypotheses of
$\sH$: $\ov \sH = \curl*{\ov h: h\in\sH}$. The prediction associated
by $\ov h \in \ov \sH$ to an input $x \in \sX$ is denoted by $\ov
\hh(x)$ and defined as the element in $\sY \cup \curl{0}$ with the
highest score, $\ov \hh(x) = \argmax_{y \in \sY \cup \curl{0}} h(x,
y)$, with an arbitrary but fixed deterministic strategy for breaking
ties. For any $x \in \sX$ and label space $\sY \cup \curl{0}$, we will
denote, by $\ov{ \mathsf H}(x)$ the set of labels generated by
hypotheses in $\ov \sH$: $\ov{ \mathsf H}(x) = \curl*{\ov \hh(x)
  \colon h \in \ov \sH}$. By
\citep[Lemma~3]{AwasthiMaoMohriZhong2022multi} with label space $\sY
\cup \curl{0}$ and a conditional probability vector $\sfp(\cdot \!\mid\! x)$ on
$\sY \cup \curl{0}$, the minimal conditional $\ell_{0-1}$-loss and the
corresponding calibration gap can be characterized as follows.
\begin{restatable}{lemma}{ExplicitAssumption}
\label{lemma:explicit_assumption_01}
For any $x \in \sX$,
the minimal conditional $\ell_{0-1}$-risk and
the calibration gap for $\ell_{0-1}$ can be expressed as follows:
\begin{align*}
\sC^*_{\ell_{0-1}}(x) & = 1- \max_{y\in \ov {\mathsf H}(x)}\sfp(y \!\mid\! x)\\
\Delta\sC_{\ell_{0-1}}(h, x) & = \max_{y\in \ov{ \mathsf H}(x)}\sfp(y \!\mid\! x) -\sfp(\hh(x) \!\mid\! x).
\end{align*}
\end{restatable}

\subsection{Multinomial logistic loss}
\label{app:bound_log}
\begin{restatable}[\textbf{$\ov \sH$-consistency bound for
   multinomial logistic loss}]
 {theorem}{BoundLog}
\label{Thm:bound_log}
Assume that $\sH$ is symmetric and complete. Then, for any $\lambda\in
\Rset$, hypothesis $\ov h\in\ov \sH$ and any distribution,
\begin{equation*}
\sE_{\ell_{0-1}}\paren*{\ov h}- \sE_{\ell_{0-1}}^*\paren*{\ov \sH}
\leq \sqrt{2}
\paren*{\sE_{\ell_{\rm{log}}}\paren*{\ov h}
  - \sE_{\ell_{\rm{log}}}^*\paren*{\ov \sH}
  + \sM_{\ell_{\rm{log}}}\paren*{\ov \sH}}^{\frac12}
  - \sM_{\ell_{0-1}}\paren*{\ov \sH}.
\end{equation*}
\end{restatable}
\begin{proof}
For the multinomial logistic loss $\ell_{\rm{log}}$, the conditional
$\ell_{\rm{log}}$-risk can be expressed as follows:
\begin{equation*}
\begin{aligned}
 \sC_{\ell_{\rm{log}}}\paren*{\ov h, x)}
 = \sum_{y\in \sY \cup \curl{0}} \sfp(y \!\mid\! x) \log\paren*{\sum_{y'\in \sY \cup \curl{0}} e^{\ov h(x, y') - \ov h(x, y)}}
 = - \sum_{y\in \sY \cup \curl{0}}\sfp(y \!\mid\! x) \log\paren*{\sS(x, y)}
\end{aligned}
\end{equation*}
where we let $\sS(x, y) = \frac{e^{\ov h(x, y)}}{\sum_{y'\in \sY \cup
    \curl{0}}e^{\ov h(x, y')}} \in [0,1]$ for any $y\in \sY \cup
\curl{0}$ with the constraint that $\sum_{y\in \sY \cup
  \curl{0}}\sS(x, y) = 1$. Let $y_{\max} = \argmax_{y\in \sY \cup
  \curl{0}}\sfp(y \!\mid\! x)$, where we choose the label with the same
deterministic strategy for breaking ties as that of $\ov h(x)$. For
any $\ov h \in \sH$ such that $\ov \hh(x) \neq y_{\max}$ and $x\in
\sX$, by the symmetry and completeness of $\sH$, we can always find a
family of hypotheses $\curl*{\ov h_{\mu}:\mu \in [- \sS(x,
    y_{\max}),\sS(x, \ov \hh(x))]}\subset \ov \sH$ such that
$\sS_{\mu}(x,\cdot) = \frac{e^{\ov h_{\mu}(x,\cdot)}}{\sum_{y'\in \sY
    \cup \curl{0}}e^{\ov h_{\mu}(x, y')}}$ take the following values:
\begin{align*}
\sS_{\mu}(x, y) = 
\begin{cases}
 \sS(x, y) & \text{if $y \not \in \curl*{y_{\max}, \ov \hh(x)}$}\\
 \sS(x, y_{\max}) + \mu & \text{if $y = \ov \hh(x)$}\\
 \sS(x, \ov \hh(x)) - \mu & \text{if $y = y_{\max}$}.
\end{cases} 
\end{align*}
Note that $\sS_{\mu}$ satisfies the constraint:
\begin{align*}
 \sum_{y\in \sY \cup \curl{0}}\sS_{\mu}(x, y) = \sum_{y\in \sY \cup \curl{0}}
 \sS(x, y) = 1,\, \forall \mu \in [- \sS(x, y_{\max}),\sS(x, \ov \hh(x))].
\end{align*}
Let $\ov h \in \ov\sH$ be a hypothesis such that $\ov \hh(x) \neq
y_{\max}$. By the definition and using the fact that $\ov {\mathsf
  H}(x) = \sY \cup \curl{0}$ when $\sH$ is symmetric, we obtain
\begin{align*}
&\Delta\sC_{\ell_{\rm{log}},\ov \sH}\paren*{\ov h, x}\\
& = \sC_{\ell_{\rm{log}}}\paren*{\ov h, x} - \sC^*_{\ell_{\rm{log}}}\paren*{\ov \sH,x} \\
& \geq 
  \sC_{\ell_{\rm{log}}}\paren*{\ov h, x}
  - \inf_{\mu \in [- \sS(x, y_{\max}),\sS(x, \ov \hh(x))]}\sC_{\ell_{\rm{log}}}\paren*{\ov h_{\mu},x}\\
  & = \sup_{\mu \in [- \sS(x, y_{\max}),\sS(x, \ov \hh(x))]}
  \bigg\{\sfp(y_{\max} \!\mid\! x)\bracket*{- \log\paren*{\sS(x, y_{\max})} + \log\paren*{\sS(x, \ov \hh(x)) - \mu}}\\
  & \qquad + \sfp(\ov \hh(x) \!\mid\! x)\bracket*{- \log\paren*{\sS(x, \ov \hh(x))}
    + \log\paren*{\sS(x, y_{\max})+ \mu}}\bigg\}
\end{align*}
Differentiating with respect to $\mu$ yields the optimum value $\mu^* = 
\frac{\sfp(\ov \hh(x) \!\mid\! x)\sS(x, \ov \hh(x)) -\sfp(y_{\max} \!\mid\! x)\sS(x,
 y_{\max})}{\sfp(y_{\max} \!\mid\! x)+\sfp(\ov \hh(x) \!\mid\! x)}$. Plugging that value in the
inequality gives:
\begin{align*}
\Delta\sC_{\ell_{\rm{log}},\ov \sH}\paren*{\ov h, x}
& \geq \sfp(y_{\max} \!\mid\! x)\log\frac{\bracket*{\sS(x, \ov \hh(x))
    + \sS(x, y_{\max})}\sfp(y_{\max} \!\mid\! x)}{\sS(x, y_{\max})\bracket*{\sfp(y_{\max} \!\mid\! x)+\sfp(\ov \hh(x) \!\mid\! x)}}\\
& \qquad + \sfp(\ov \hh(x) \!\mid\! x)\log\frac{\bracket*{\sS(x, \ov \hh(x))
    + \sS(x, y_{\max})}\sfp(\ov \hh(x) \!\mid\! x)}{\sS(x, \ov \hh(x))\bracket*{\sfp(y_{\max} \!\mid\! x)+\sfp(\ov \hh(x) \!\mid\! x)}}.
\end{align*}
Differentiating with respect to $\sS$ to show that the minimum
is attained for $\sS(x, \ov \hh(x)) = \sS(x, y_{\max})$, which implies
\begin{align*}
\Delta\sC_{\ell_{\rm{log}},\ov \sH}\paren*{\ov h, x}
& \geq \sfp(y_{\max} \!\mid\! x)\log\frac{2\sfp(y_{\max} \!\mid\! x)}{\sfp(y_{\max} \!\mid\! x)
  + \sfp(\ov \hh(x) \!\mid\! x)}+\sfp(\ov \hh(x) \!\mid\! x)\log\frac{2\sfp(\ov \hh(x) \!\mid\! x)}{\sfp(y_{\max} \!\mid\! x)+\sfp(\ov \hh(x) \!\mid\! x)}.
\end{align*}
By Pinsker's inequality, we have, for $a, b \in [0, 1]$,
$a\log \frac{2a}{a + b} + b\log \frac{2b}{a + b}\geq \frac{(a - b)^2}{2(a + b)}$.
Using this inequality, we obtain:
\begin{align*}
\Delta\sC_{\ell_{\rm{log}},\ov \sH}\paren*{\ov h, x}
& \geq \frac{\paren*{\sfp(\ov \hh(x) \!\mid\! x) -\sfp(y_{\max} \!\mid\! x)}^2}{2\paren*{\sfp(\ov \hh(x) \!\mid\! x)+\sfp(y_{\max} \!\mid\! x)}}\\
& \geq \frac{\paren*{\sfp(\ov \hh(x) \!\mid\! x) -\sfp(y_{\max} \!\mid\! x)}^2}{2}
\tag{$0\leq \sfp(\ov \hh(x) \!\mid\! x)+\sfp(y_{\max} \!\mid\! x)\leq 1$}\\
& = \frac12 \paren*{ \Delta\sC_{\ell_{0-1},\ov \sH}\paren*{\ov h, x}}^2
\tag{by Lemma~\ref{lemma:explicit_assumption_01} and
  $\ov{\mathsf H}(x) = \sY \cup \curl{0}$}.
\end{align*}
Since the function $\frac{t^2}{2}$ is convex, by Jensen's inequality,
we obtain for any hypothesis $\ov h\in\ov \sH$ and any distribution,
\begin{align*}
\frac{\paren*{\E_{X}\bracket*{\Delta\sC_{\ell_{0-1},\ov \sH}\paren*{\ov h, x}}}^2}{2}
\leq \E_{X}\bracket*{\frac{\Delta\sC_{\ell_{0-1},\ov \sH}\paren*{\ov h, x}^2}{2}}
\leq \E_{X}\bracket*{\Delta\sC_{\ell_{\rm{log}},\ov \sH}\paren*{\ov h, x}},
\end{align*}
which leads to 
\begin{align*}
\sE_{\ell_{0-1}}\paren*{\ov h}- \sE_{\ell_{0-1}}^*\paren*{\ov \sH}
\leq \sqrt{2}
\paren*{\sE_{\ell_{\rm{log}}}\paren*{\ov h}
  - \sE_{\ell_{\rm{log}}}^*\paren*{\ov \sH}
  + \sM_{\ell_{\rm{log}}}\paren*{\ov \sH}}^{\frac12}
- \sM_{\ell_{0-1}}\paren*{\ov \sH}.
\end{align*}
\end{proof}

\subsection{Sum exponential loss}
\label{app:bound_sum_exp}

\begin{restatable}[\textbf{$\ov \sH$-consistency bound for sum exponential loss}]
 {theorem}{BoundSumExp}
\label{Thm:bound_sum_exp}
Assume that $\sH$ is symmetric and complete. Then, for any $\lambda\in
\Rset$, hypothesis $\ov h\in\ov \sH$ and any distribution,
\begin{equation*}
\sE_{\ell_{0-1}}\paren*{\ov h}- \sE_{\ell_{0-1}}^*\paren*{\ov \sH}
\leq \sqrt{2}
\paren*{\sE_{\ell_{\rm{exp}}}\paren*{\ov h}- \sE_{\ell_{\rm{exp}}}^*\paren*{\ov \sH}+ \sM_{\ell_{\rm{exp}}}\paren*{\ov \sH}}^{\frac12}
- \sM_{\ell_{0-1}}\paren*{\ov \sH}.
\end{equation*}
\end{restatable}
\begin{proof}
For the sum exponential loss $\ell_{\rm{exp}}$, the conditional
$\ell_{\rm{exp}}$-risk can be expressed as follows:
\begin{equation*}
\begin{aligned}
 \sC_{\ell_{\rm{exp}}}\paren*{\ov h, x)}
 = \sum_{y\in \sY \cup \curl{0}} \sfp(y \!\mid\! x) \paren*{\sum_{y'\in \sY \cup \curl{0}} e^{\ov h(x, y') - \ov h(x, y)}} - 1
 = \sum_{y\in \sY \cup \curl{0}} \frac{\sfp(y \!\mid\! x)}{\sS(x, y)} - 1
\end{aligned}
\end{equation*}
where we let $\sS(x, y) = \frac{e^{\ov h(x, y)}}{\sum_{y'\in \sY \cup
    \curl{0}}e^{\ov h(x, y')}} \in [0,1]$ for any $y\in \sY \cup
\curl{0}$ with the constraint that $\sum_{y\in \sY \cup
  \curl{0}}\sS(x, y) = 1$. Let $y_{\max} = \argmax_{y\in \sY \cup
  \curl{0}}\sfp(y \!\mid\! x)$, where we choose the label with the highest index
under the natural ordering of labels as the tie-breaking strategy.
For any $\ov h \in \sH$ such that $\ov \hh(x) \neq y_{\max}$ and $x\in
\sX$, by the symmetry and completeness of $\sH$, we can always find a
family of hypotheses $\curl*{\ov h_{\mu} \colon \mu \in [- \sS(x,
    y_{\max}),\sS(x, \ov \hh(x))]}\subset \ov \sH$ such that
$\sS_{\mu}(x,\cdot) = \frac{e^{\ov h_{\mu}(x,\cdot)}}{\sum_{y'\in \sY
    \cup \curl{0}}e^{\ov h_{\mu}(x, y')}}$ take the following values:
\begin{align*}
\sS_{\mu}(x, y) = 
\begin{cases}
 \sS(x, y) & \text{if $y \not \in \curl*{y_{\max}, \ov \hh(x)}$}\\
 \sS(x, y_{\max}) + \mu & \text{if $y = \ov \hh(x)$}\\
 \sS(x, \ov \hh(x)) - \mu & \text{if $y = y_{\max}$}.
\end{cases} 
\end{align*}
Note that $\sS_{\mu}$ satisfies the constraint:
\begin{align*}
 \sum_{y\in \sY}\sS_{\mu}(x, y) = \sum_{y\in \sY}
 \sS(x, y) = 1,\, \forall \mu \in [- \sS(x, y_{\max}),\sS(x, \ov \hh(x))].
\end{align*}
Let $\ov h \in \ov \sH$ be a hypothesis such that $\ov \hh(x) \neq
y_{\max}$. By the definition and using the fact that $\ov{\mathsf
  H}(x) = \sY \cup \curl{0}$ when $\sH$ is symmetric, we obtain
\begin{align*}
&\Delta\sC_{\ell_{\rm{exp}},\ov \sH}\paren*{\ov h, x}\\
& = \sC_{\ell_{\rm{exp}}}\paren*{\ov h, x} - \sC^*_{\ell_{\rm{exp}}}\paren*{\ov \sH,x} \\
& \geq 
  \sC_{\ell_{\rm{exp}}}\paren*{\ov h, x}
  - \inf_{\mu \in [- \sS(x, y_{\max}),\sS(x, \ov \hh(x))]}\sC_{\ell_{\rm{exp}}}\paren*{\ov h_{\mu},x}\\
& = \sup_{\mu \in [- \sS(x, y_{\max}),\sS(x, \ov \hh(x))]} \bigg\{\sfp(y_{\max} \!\mid\! x)\bracket*{\frac{1}{\sS(x, y_{\max})}
  - \frac{1}{\sS(x, \ov \hh(x)) - \mu}}\\
&\qquad +\sfp(\ov \hh(x) \!\mid\! x)\bracket*{\frac{1}{\sS(x, \ov \hh(x))}
  - \frac{1}{\sS(x, y_{\max})+ \mu}}\bigg\}.
\end{align*}
Differentiating with respect to $\mu$ yields the optimal value
\[
\mu^* = \frac{\sqrt{\sfp(\ov \hh(x) \!\mid\! x)})\sS(x, \ov \hh(x))
  - \sqrt{\sfp(y_{\max} \!\mid\! x)}\sS(x, y_{\max})}{\sqrt{\sfp(y_{\max} \!\mid\! x)}+ \sqrt{\sfp(\ov \hh(x) \!\mid\! x)}}.
\]
Plugging that value in the
inequality gives:
\begin{align*}
\Delta\sC_{\ell_{\rm{exp}},\ov \sH}\paren*{\ov h, x}
& \geq \frac{\sfp(y_{\max} \!\mid\! x)}{\sS(x, y_{\max})}+ \frac{\sfp(\ov \hh(x) \!\mid\! x)}{\sS(x, \ov \hh(x))}
- \frac{\paren*{\sqrt{\sfp(y_{\max} \!\mid\! x)}+ \sqrt{\sfp(\ov \hh(x) \!\mid\! x)}}^2}{\sS(x, y_{\max})+ \sS(x, \ov \hh(x))}.
\end{align*}
Differentiating with respect to $\sS$ to show that the minimum
is attained for \[\sS(x, \ov \hh(x)) = \sS(x, y_{\max}) = \frac12,\] which implies
\begin{align*}
\Delta\sC_{\ell_{\rm{exp}},\ov \sH}\paren*{\ov h, x}
& \geq \paren*{\sqrt{\sfp(y_{\max} \!\mid\! x)}- \sqrt{\sfp(\ov \hh(x) \!\mid\! x)}}^2\\
& = \frac{\paren*{\sfp(\ov \hh(x) \!\mid\! x) -\sfp(y_{\max} \!\mid\! x)}^2}{\paren*{\sqrt{\sfp(\ov \hh(x) \!\mid\! x)}+ \sqrt{\sfp(y_{\max} \!\mid\! x)}}^2}.
\end{align*}
By the concavity of the square-root function, for all $a, b \in [0, 1]$, the inequality
$\frac{1}{2} \paren*{\sqrt{a}+ \sqrt{b}} \leq \sqrt{\frac{1}{2}(a + b)}$ holds,
thus we can write
\begin{align*}
  \Delta\sC_{\ell_{\rm{exp}},\ov \sH}\paren*{\ov h, x}
& \geq  \frac{\paren*{\sfp(\ov \hh(x) \!\mid\! x) - \sfp(y_{\max} \!\mid\! x)}^2}{2\paren*{\sfp(\ov \hh(x) \!\mid\! x) + \sfp(y_{\max} \!\mid\! x)}}\\
  & \geq \frac{\paren*{\sfp(\ov \hh(x) \!\mid\! x) - \sfp(y_{\max} \!\mid\! x)}^2}{2}
  \tag{$\sfp(\ov \hh(x) \!\mid\! x) + \sfp(y_{\max} \!\mid\! x) \leq 1$}\\
  & = \frac12 \paren*{ \Delta\sC_{\ell_{0-1},\ov \sH}\paren*{\ov h, x}}^2
  \tag{by Lemma~\ref{lemma:explicit_assumption_01}
    and $\ov{\mathsf H}(x) = \sY \cup \curl{0}$}.
\end{align*}
Since the function $\frac{t^2}{2}$ is convex, by Jensen's inequality,
we obtain for any hypothesis $\ov h\in\ov \sH$ and any distribution,
\begin{align*}
\frac{\paren*{\E_{X}\bracket*{\Delta\sC_{\ell_{0-1},\ov \sH}\paren*{\ov h, x}}}^2}{2}
\leq \E_{X}\bracket*{\frac{\Delta\sC_{\ell_{0-1},\ov \sH}\paren*{\ov h, x}^2}{2}}
\leq \E_{X}\bracket*{\Delta\sC_{\ell_{\rm{exp}},\ov \sH}\paren*{\ov h, x}},
\end{align*}
which leads to 
\begin{align*}
\sE_{\ell_{0-1}}\paren*{\ov h}- \sE_{\ell_{0-1}}^*\paren*{\ov \sH}
\leq \sqrt{2}
\paren*{\sE_{\ell_{\rm{exp}}}\paren*{\ov h}
  - \sE_{\ell_{\rm{exp}}}^*\paren*{\ov \sH}
  + \sM_{\ell_{\rm{exp}}}\paren*{\ov \sH}}^{\frac12}
- \sM_{\ell_{0-1}}\paren*{\ov \sH}.
\end{align*}
\end{proof}

\subsection{Generalized cross-entropy loss}
\label{app:bound_gce}

\begin{restatable}[\textbf{$\ov \sH$-consistency bound for generalized
      cross-entropy loss}]
 {theorem}{BoundGCE}
\label{Thm:bound_gce}
Assume that $\sH$ is symmetric and complete. Then, for any $\lambda\in
\Rset$, hypothesis $\ov h\in\ov \sH$ and any distribution,
\begin{equation*}
\sE_{\ell_{0-1}}\paren*{\ov h}- \sE_{\ell_{0-1}}^*\paren*{\ov \sH}
\leq \sqrt{2(n+1)^{\alpha}}
\paren*{\sE_{\ell_{\rm{gce}}}\paren*{\ov h}
  - \sE_{\ell_{\rm{gce}}}^*\paren*{\ov \sH}
  + \sM_{\ell_{\rm{gce}}}\paren*{\ov \sH}}^{\frac12}
- \sM_{\ell_{0-1}}\paren*{\ov \sH}.
\end{equation*}
\end{restatable}
\begin{proof}
For the generalized cross-entropy loss $\ell_{\rm{gce}}$, the
conditional $\ell_{\rm{gce}}$-risk can be expressed as follows:
\begin{equation*}
\begin{aligned}
 \sC_{\ell_{\rm{gce}}}\paren*{\ov h, x)}
 = \sum_{y\in \sY \cup \curl{0}} \sfp(y \!\mid\! x) \frac{1}{\alpha}\bracket*{1 - \bracket*{\frac{e^{\ov h(x, y)}}
{\sum_{y'\in \sY \cup
    {0}} e^{\ov h(x, y')}}}^{\alpha}}
 = \frac{1}{\alpha} \sum_{y\in \sY \cup \curl{0}}\sfp(y \!\mid\! x)\paren*{1- \sS(x, y)^{\alpha}}
\end{aligned}
\end{equation*}
where we let
$\sS(x, y) = \frac{e^{\ov h(x, y)}}{\sum_{y'\in \sY \cup
\curl{0}}e^{\ov h(x, y')}} \in [0,1]$ for any $y\in \sY \cup
\curl{0}$ with the constraint that $\sum_{y\in \sY \cup
 \curl{0}}\sS(x, y) = 1$. Let $y_{\max} = \argmax_{y\in \sY \cup
 \curl{0}}\sfp(y \!\mid\! x)$, where we choose the label with the same
deterministic strategy for breaking ties as that of $\ov h(x)$. For any
$\ov h \in \sH$ such that $\ov \hh(x) \neq y_{\max}$ and
$x\in \sX$, by the symmetry and completeness of $\sH$, we can always
find a family of hypotheses $\curl*{\ov h_{\mu}:\mu \in
 [- \sS(x, y_{\max}),\sS(x, \ov \hh(x))]}\subset
\ov \sH$ such that
$\sS_{\mu}(x,\cdot) = \frac{e^{\ov h_{\mu}(x,\cdot)}}{\sum_{y'\in
\sY \cup \curl{0}}e^{\ov h_{\mu}(x, y')}}$ take the following
values:
\begin{align*}
\sS_{\mu}(x, y) = 
\begin{cases}
 \sS(x, y) & \text{if $y \not \in \curl*{y_{\max}, \ov \hh(x)}$}\\
 \sS(x, y_{\max}) + \mu & \text{if $y = \ov \hh(x)$}\\
 \sS(x, \ov \hh(x)) - \mu & \text{if $y = y_{\max}$}.
\end{cases} 
\end{align*}
Note that $\sS_{\mu}$ satisfies the constraint:
\begin{align*}
 \sum_{y\in \sY \cup \curl{0}}\sS_{\mu}(x, y) = \sum_{y\in \sY \cup \curl{0}}
 \sS(x, y) = 1,\, \forall \mu \in [- \sS(x, y_{\max}),\sS(x, \ov \hh(x))].
\end{align*}
Let $\ov h \in \ov\sH$ be a hypothesis such that $\ov \hh(x) \neq
y_{\max}$. By the definition and using the fact that $\ov {\mathsf
  H}(x) = \sY \cup \curl{0}$ when $\sH$ is symmetric, we obtain
\begin{align*}
&\Delta\sC_{\ell_{\rm{gce}},\ov \sH}\paren*{\ov h, x}\\
& = \sC_{\ell_{\rm{gce}}}\paren*{\ov h, x} - \sC^*_{\ell_{\rm{gce}}}\paren*{\ov \sH,x} \\
& \geq 
\sC_{\ell_{\rm{gce}}}\paren*{\ov h, x} - \inf_{\mu \in [- \sS(x, y_{\max}),\sS(x, \ov \hh(x))]}\sC_{\ell_{\rm{gce}}}\paren*{\ov h_{\mu},x}\\
& = \frac{1}{\alpha} \sup_{\mu \in [- \sS(x, y_{\max}),\sS(x, \ov \hh(x))]} \bigg\{\sfp(y_{\max} \!\mid\! x)\bracket*{- \sS(x, y_{\max})^{\alpha}+\paren*{\sS(x, \ov \hh(x)) - \mu}^{\alpha}}\\
&\qquad +\sfp(\ov \hh(x) \!\mid\! x)\bracket*{- \sS(x, \ov \hh(x))^{\alpha}+ \paren*{\sS(x, y_{\max})+ \mu}^{\alpha}}\bigg\}.
\end{align*}
Differentiating with respect to $\mu$ yields the optimal value
\[
\mu^* = \frac{\sfp(\ov \hh(x) \!\mid\! x)^{\frac{1}{1- \alpha}}\sS(x, \ov \hh(x)) -\sfp(y_{\max} \!\mid\! x)^{\frac{1}{1- \alpha}}\sS(x, y_{\max})}{\sfp(y_{\max} \!\mid\! x)^{\frac{1}{1- \alpha}}+\sfp(\ov \hh(x) \!\mid\! x)^{\frac{1}{1- \alpha}}}.
\]
Plugging that value in the
inequality gives:
\begin{align*}
\Delta\sC_{\ell_{\rm{gce}},\ov \sH}\paren*{\ov h, x} & \geq \frac{1}{\alpha}\paren*{\sS(x, \ov \hh(x))+ \sS(x, y_{\max})}^{\alpha}\paren*{\sfp(y_{\max} \!\mid\! x)^{\frac{1}{1- \alpha}}+\sfp(\ov \hh(x) \!\mid\! x)^{\frac{1}{1- \alpha}}}^{1- \alpha}\\
&\qquad- \frac{1}{\alpha}\sfp(y_{\max} \!\mid\! x)\sS(x, y_{\max})^{\alpha}- \frac{1}{\alpha}\sfp(\ov \hh(x) \!\mid\! x)\sS(x, \ov \hh(x))^{\alpha}.
\end{align*}
Differentiating with respect to $\sS$ to show that the minimum
is attained for \[\sS(x, \ov \hh(x)) = \sS(x, y_{\max}) = \frac{1}{n+1},\] which implies
\begin{align*}
\Delta\sC_{\ell_{\rm{gce}},\ov \sH}\paren*{\ov h, x} \geq \frac{1}{\alpha(n+1)^{\alpha}}\bracket*{2^{\alpha}\paren*{\sfp(y_{\max} \!\mid\! x)^{\frac{1}{1- \alpha}}+\sfp(\ov \hh(x) \!\mid\! x)^{\frac{1}{1- \alpha}}}^{1- \alpha}-\sfp(y_{\max} \!\mid\! x) -\sfp(\ov \hh(x) \!\mid\! x)}.
\end{align*}
By using the fact that for all $a, b \in [0, 1]$, $0 \leq a + b \leq 1$, we have  $\paren*{\frac{a^{\frac{1}{1- \alpha}}+b^{\frac{1}{1- \alpha}}}{2}}^{1- \alpha}- \frac{a+b}{2}\geq \frac{\alpha}{4}(a-b)^2$, thus we can write
\begin{align*}
\Delta\sC_{\ell_{\rm{gce}},\ov \sH}\paren*{\ov h, x} & \geq \frac{\paren*{\sfp(\ov \hh(x) \!\mid\! x) -\sfp(y_{\max} \!\mid\! x)}^2}{2(n+1)^{\alpha}} \\
& = \frac1{2(n+1)^{\alpha}} \paren*{ \Delta\sC_{\ell_{0-1},\ov \sH}\paren*{\ov h, x}}^2 \tag{by Lemma~\ref{lemma:explicit_assumption_01} and $\ov{\mathsf H}(x) = \sY \cup \curl{0}$}.
\end{align*}
Since the function $\frac{t^2}{2(n+1)^{\alpha}}$ is convex, by Jensen's inequality, we obtain for any hypothesis $\ov h\in\ov \sH$ and any distribution,
\begin{align*}
\frac{\paren*{\E_{X}\bracket*{\Delta\sC_{\ell_{0-1},\ov \sH}\paren*{\ov h, x}}}^2}{2(n+1)^{\alpha}}
\leq \E_{X}\bracket*{\frac{\Delta\sC_{\ell_{0-1},\ov \sH}\paren*{\ov h, x}^2}{2(n+1)^{\alpha}}}
\leq \E_{X}\bracket*{\Delta\sC_{\ell_{\rm{gce}},\ov \sH}\paren*{\ov h, x}}
\end{align*}
which leads to 
\begin{align*}
\sE_{\ell_{0-1}}\paren*{\ov h}- \sE_{\ell_{0-1}}^*\paren*{\ov \sH}
\leq \sqrt{2(n+1)^{\alpha}}
\paren*{\sE_{\ell_{\rm{gce}}}\paren*{\ov h}- \sE_{\ell_{\rm{gce}}}^*\paren*{\ov \sH}+ \sM_{\ell_{\rm{gce}}}\paren*{\ov \sH}}^{\frac12}
- \sM_{\ell_{0-1}}\paren*{\ov \sH}.
\end{align*}
\end{proof}

\subsection{Mean absolute error loss}
\label{app:bound_mae}

\begin{restatable}[\textbf{$\ov \sH$-consistency bound for mean absolute error loss}]
 {theorem}{BoundMAE}
\label{Thm:bound_mae}
Assume that $\sH$ is symmetric and complete. Then, for any $\lambda\in
\Rset$, hypothesis $\ov h\in\ov \sH$ and any distribution,
\begin{equation*}
\sE_{\ell_{0-1}}\paren*{\ov h}- \sE_{\ell_{0-1}}^*\paren*{\ov \sH}
\leq (n+1)
\paren*{\sE_{\ell_{\rm{mae}}}\paren*{\ov h}- \sE_{\ell_{\rm{mae}}}^*\paren*{\ov \sH}+ \sM_{\ell_{\rm{mae}}}\paren*{\ov \sH}}
- \sM_{\ell_{0-1}}\paren*{\ov \sH}.
\end{equation*}
\end{restatable}
\begin{proof}
For the mean absolute error loss $\ell_{\rm{mae}}$, the conditional $\ell_{\rm{mae}}$-risk can be expressed as follows:
\begin{equation*}
\begin{aligned}
 \sC_{\ell_{\rm{mae}}}\paren*{\ov h, x)}
 = \sum_{y\in \sY \cup \curl{0}} \sfp(y \!\mid\! x)\paren*{1 - \frac{e^{\ov h(x, y)}}
{\sum_{y'\in \sY \cup
{0}} e^{\ov h(x, y')}}} =  \sum_{y\in \sY \cup \curl{0}}\sfp(y \!\mid\! x)\paren*{1- \sS(x, y)}
\end{aligned}
\end{equation*}
where we let
$\sS(x, y) = \frac{e^{\ov h(x, y)}}{\sum_{y'\in \sY \cup
\curl{0}}e^{\ov h(x, y')}} \in [0,1]$ for any $y\in \sY \cup
\curl{0}$ with the constraint that $\sum_{y\in \sY \cup
 \curl{0}}\sS(x, y) = 1$. Let $y_{\max} = \argmax_{y\in \sY \cup
 \curl{0}}\sfp(y \!\mid\! x)$, where we choose the label with the same
deterministic strategy for breaking ties as that of $\ov h(x)$. For any
$\ov h \in \sH$ such that $\ov \hh(x) \neq y_{\max}$ and
$x\in \sX$, by the symmetry and completeness of $\sH$, we can always
find a family of hypotheses $\curl*{\ov h_{\mu}:\mu \in
 [- \sS(x, y_{\max}),\sS(x, \ov \hh(x))]}\subset
\ov \sH$ such that
$\sS_{\mu}(x,\cdot) = \frac{e^{\ov h_{\mu}(x,\cdot)}}{\sum_{y'\in
\sY \cup \curl{0}}e^{\ov h_{\mu}(x, y')}}$ take the following
values:
\begin{align*}
\sS_{\mu}(x, y) = 
\begin{cases}
 \sS(x, y) & \text{if $y \not \in \curl*{y_{\max}, \ov \hh(x)}$}\\
 \sS(x, y_{\max}) + \mu & \text{if $y = \ov \hh(x)$}\\
 \sS(x, \ov \hh(x)) - \mu & \text{if $y = y_{\max}$}.
\end{cases} 
\end{align*}
Note that $\sS_{\mu}$ satisfies the constraint:
\begin{align*}
 \sum_{y\in \sY \cup \curl{0}}\sS_{\mu}(x, y) = \sum_{y\in \sY \cup \curl{0}}
 \sS(x, y) = 1,\, \forall \mu \in [- \sS(x, y_{\max}),\sS(x, \ov \hh(x))].
\end{align*}
Let $\ov h \in \ov\sH$ be a hypothesis such that $\ov \hh(x) \neq
y_{\max}$. By the definition and using the fact that $\ov {\mathsf
  H}(x) = \sY \cup \curl{0}$ when $\sH$ is symmetric, we obtain
\begin{align*}
&\Delta\sC_{\ell_{\rm{mae}},\ov \sH}\paren*{\ov h, x}\\
& = \sC_{\ell_{\rm{mae}}}\paren*{\ov h, x} - \sC^*_{\ell_{\rm{mae}}}\paren*{\ov \sH,x} \\
& \geq 
\sC_{\ell_{\rm{mae}}}\paren*{\ov h, x} - \inf_{\mu \in [- \sS(x, y_{\max}),\sS(x, \ov \hh(x))]}\sC_{\ell_{\rm{mae}}}\paren*{\ov h_{\mu},x}\\
& = \sup_{\mu \in [- \sS(x, y_{\max}),\sS(x, \ov \hh(x))]} \bigg\{\sfp(y_{\max} \!\mid\! x)\bracket*{- \sS(x, y_{\max})+\sS(x, \ov \hh(x)) - \mu}\\
&\qquad +\sfp(\ov \hh(x) \!\mid\! x)\bracket*{- \sS(x, \ov \hh(x))+ \sS(x, y_{\max})+ \mu}\bigg\}.
\end{align*}
Differentiating with respect to $\mu$ yields the optimum value $\mu^* = - \sS(x, y_{\max})$. Plugging that value in the
inequality gives:
\begin{align*}
\Delta\sC_{\ell_{\rm{mae}},\ov \sH}\paren*{\ov h, x} & \geq \sfp(y_{\max} \!\mid\! x)\sS(x, \ov \hh(x)) -\sfp(\ov \hh(x) \!\mid\! x)\sS(x, \ov \hh(x)).
\end{align*}
Differentiating with respect to $\sS$ to show that the minimum
is attained for $\sS(x, \ov \hh(x)) = \frac{1}{n+1}$, which implies
\begin{align*}
\Delta\sC_{\ell_{\rm{mae}},\ov \sH}\paren*{\ov h, x} & \geq \frac{1}{n+1}\paren*{\sfp(y_{\max} \!\mid\! x) -\sfp(\ov \hh(x) \!\mid\! x)}\\
& = \frac{1}{n+1} \paren*{ \Delta\sC_{\ell_{0-1},\ov \sH}\paren*{\ov h, x}}
\tag{by Lemma~\ref{lemma:explicit_assumption_01} and $\ov{\mathsf H}(x) = \sY \cup \curl{0}$}.
\end{align*}
Therefore, we obtain for any hypothesis $\ov h\in\ov \sH$ and any distribution,
\begin{align*}
\frac{\E_{X}\bracket*{\Delta\sC_{\ell_{0-1},\ov \sH}\paren*{\ov h, x}}}{n+1}
\leq \E_{X}\bracket*{\Delta\sC_{\ell_{\rm{mae}},\ov \sH}\paren*{\ov h, x}},
\end{align*}
which leads to 
\begin{align*}
\sE_{\ell_{0-1}}\paren*{\ov h}- \sE_{\ell_{0-1}}^*\paren*{\ov \sH}
\leq (n+1)
\paren*{\sE_{\ell_{\rm{mae}}}\paren*{\ov h}- \sE_{\ell_{\rm{mae}}}^*\paren*{\ov \sH}+ \sM_{\ell_{\rm{mae}}}\paren*{\ov \sH}}
- \sM_{\ell_{0-1}}\paren*{\ov \sH}.
\end{align*}
\end{proof}

\section{Proof of realizable consistency for score-based two-stage
  surrogate losses (Theorem~\ref{Thm:bound-general-two-stage-score-realizable})}
\label{app:realizable-score}

\BoundGenralTwoStepScoreRealizable*
\begin{proof}
First, by definition, it is straightforward to see that for any $h, x,
y$, $\sfL_{\hp}(\hd,x, y)$ upper-bounds the deferral loss
$\ldefsc$. Consider a data distribution and costs under which there
exists $h^*\in \sH$ such that $\sE_{\ldefsc}(h^*) = 0$.

Let $\hat h_p$ be the minimizer of $\sE_{\ell_1}$ and $\hat h_d$ the
minimizer of $\sE_{\sfL_{\hat h_p}}$ Then, using the fact that
$\sfL_{h}$ upper-bounds the deferral loss $\ldefsc$, we have
$\sE_{\ldefsc}(\hat h)\leq \sE_{\sfL_{\hat h_p}}(\hat h_d)$.

Next we analyze two cases. If for a point $x$, deferral occurs, that
is there exists $j^*\in [\num]$, such that $\hh^*(x) = n + j^*$, then
we must have $c_{j^*} = 0$ for all $x$ since the data is realizable
and $c_{j^*}$ is constant. Therefore, there exists an optimal $h^{**}$
deferring all the points to the $j^*$th expert. Then, by the
assumption that $\sH$ is closed under scaling and the Lebesgue
dominated convergence theorem, for $\ell_2$ being the logistic loss,
$\sE_{\ldefsc}(\hat h)\leq\sE_{\sfL_{\hat h_p}}(\hat h_d) \leq
\lim_{\tau \to +\infty} \sE_{\sfL_{h^{**}_p}}(\tau h^{**}_d) = 0$,
where we used the fact that in the limit of $\tau \to +\infty$ the
logistic loss term $\ell_2(\overline h_d^{**},x,j)$ corresponding to
$j \neq j^*$ is zero.

On the other hand, if no deferral occurs for any point, that is
$\hh^*(x)\in[n]$ for any $x$, then we must have $\1_{\hh^*_p(x) \neq y}
= 0$ for all $(x, y)$ since the data is realizable. Using the fact
that $\sH$ is closed under scaling and that the logistic loss is
realizable $\sH$-consistent in the standard classification, we obtain
$\1_{\hat{\hh}_p(x) \neq y} = 0$ for all $(x, y)$. Then, by the
assumption that $\sH$ is closed under scaling and the Lebesgue
dominated convergence theorem, for $\ell_2$ being the logistic loss,
$\sE_{\ldefsc}(\hat h)\leq\sE_{\sfL_{\hat h_p}}(\hat h_d) \leq
\lim_{\tau \to +\infty} \sE_{\sfL_{h_p^{*}}}(\tau h_d^{*}) = 0$,
where we used the fact that in the limit of $\tau \to +\infty$ the
logistic loss term $\ell_2(\overline h_d^{*},x,j)$ corresponding to $j
\neq 0$ is zero.

Therefore, the optimal solution from minimizing score-based
two-stage surrogates leads to a zero error solution of the deferral
loss, which proves that the score-based two-stage surrogate
loss is realizable consistent.
\end{proof}

\section{Proof of \texorpdfstring{$(\sH,\sR)$}{HR}-consistency bounds for
  predictor-rejector two-stage surrogate losses (Theorem~\ref{Thm:bound-general-two-step-multi})}
\label{app:two-stage-predictor-rejector}

\BoundGeneralTwoStepMulti*
\begin{proof}
By definition,
\begin{align*}
\ldef(h, r, x, y)
 = \1_{\hh(x) \neq y} \1_{\rr(x) = 0} + \sum_{j = 1}^{\num} c_j(x, y) \1_{\rr(x) = j}.
\end{align*}
Let $\bar c_0\paren*{x, y} = \1_{\hh(x) = y}$. We can rewrite $\sE_{\ldef}(h,r)
- \sE_{\ldef}^*(\sH,\sR) + \sM_{\ldef}(\sH,\sR)$ as 
\begin{equation}
\label{eq:expression-two-step-multi}
\begin{aligned}
& \sE_{\ldef}(h,r) - \sE_{\ldef}^*(\sH,\sR) + \sM_{\ldef}(\sH,\sR)\\
& =  \mathbb{E}_{X}\bracket*{\sC_{\ldef}(h,r, x) - \sC^*_{\ldef}(\sH,\sR,x)} \\
  & =  \mathbb{E}_{X}\bracket*{\sC_{\ldef}(h, r, x) - \inf_{r\in \sR}\sC_{\ldef}(h, r, x)
    + \inf_{r\in \sR}\sC_{\ldef}(h, r, x) - \sC^*_{\ldef}(\sH,\sR,x)}\\
  & = \mathbb{E}_{X}\bracket*{\sC_{\ldef}(h, r, x) - \inf_{r\in \sR}\sC_{\ldef}(h, r, x)}
  + \mathbb{E}_{X}\bracket*{\inf_{r\in \sR}\sC_{\ldef}(h, r, x) - \sC^*_{\ldef}(\sH,\sR,x)}
\end{aligned}
\end{equation}
Let $\ov \sfp(j \!\mid\! x) = \frac{\E_y\bracket*{\bar c_j(x, y)}}{\sum_{j =
    0}^{\num}\E_y\bracket*{\bar c_j(x, y)}}$ for any $j\in
\curl*{0,\ldots,\num}$. Note that $\ov \sfp(\cdot \!\mid\! x)$ is the probability
vector on the label space $\curl*{0,\ldots,\num}$. For any $r\in \sR$,
we define $\ov r$ as its augmented hypothesis: $\ov r(x,0) = 0,\ov
r(x,1) = -r_1(x),\ldots, \ov r(x, \num) = -r_{\num}(x)$.  By the
assumptions, we have
\begin{align*}
& \sC_{\ldef}(h, r, x) - \inf_{r\in \sR}\sC_{\ldef}(h, r, x)\\
  & = \E_y\bracket*{\1_{\hh(x) \neq y} \1_{\rr(x) = 0}
    + \sum_{j = 1}^{\num} c_j(x, y) \1_{\rr(x) = j}}
  - \inf_{r\in \sR}\E_y\bracket*{\1_{\hh(x) \neq y} \1_{\rr(x) = 0} + \sum_{j = 1}^{\num} c_j(x, y) \1_{\rr(x) = j}}\\
  & = \E_y\bracket*{\sum_{j = 0}^{\num}\bar c_j(x, y)}\times \bracket*{\sum_{j = 0}^{\num}\ov \sfp(j \!\mid\! x)\ell_{0-1}(\ov r, x, j)
    - \inf_{\ov r\in \ov \sR}\sum_{j = 0}^{\num}\ov \sfp(j \!\mid\! x)\ell_{0-1}(\ov r, x, j)}\\
  & \leq \E_y\bracket*{\sum_{j = 0}^{\num}\bar c_j(x, y)}\times \Gamma_2\bracket*{\sum_{j = 0}^{\num}\ov \sfp(j \!\mid\! x)\ell_{2}(\ov r, x, j)
    - \inf_{\ov r\in \ov \sR}\sum_{j = 0}^{\num}\ov \sfp(j \!\mid\! x)\ell_{2}(\ov r, x, j)}
  \tag{By $\ov \sR$-consistency bounds of $\ell_2$ under assumption}\\
  & =  \E_y\bracket*{\sum_{j = 0}^{\num}\bar c_j(x, y)} \Gamma_2\paren*{\frac{\E_y\bracket*{\sfL_{h}(r, x, y)}
      - \inf_{r\in \sR}\E_y\bracket*{\sfL_{h}(r, x, y)}}{ \E_y\bracket*{\sum_{j = 0}^{\num}\bar c_j(x, y)}}}
  \tag{ $\ov \sfp(j \!\mid\! x) = \frac{\E_y\bracket*{\bar c_j(x, y)}}{\sum_{j = 0}^{\num}\E_y\bracket*{\bar c_j(x, y)}}$
    and formulation \eqref{eq:ell-Phi-h-multi}}\\
& \leq
\begin{cases}
  \Gamma_2\paren*{\sC_{\sfL_{h}}(r, x)
    - \sC^*_{\sfL_{h}}(\sR,x)} & \text{when $\Gamma_2$ is linear}\\
  \paren*{1+ \sum_{j = 1}^{\num}\ov c_j}\Gamma_2\paren*{\frac {\sC_{\sfL_{h}}(r, x)
      - \sC^*_{\sfL_{h}}(\sR,x)}{\sum_{j = 1}^{\num}\uv c_j}} & \text{otherwise}
\end{cases}\\
\tag{$\sum_{j = 1}^{\num}\uv c_{j}\leq \E_y\bracket*{\sum_{j = 0}^{\num}\bar c_j(x, y)}
  \leq 1+ \sum_{j = 1}^{\num}\ov c_j$ and $\Gamma_2$ is non-decreasing}\\
& = \begin{cases}
\Gamma_2\paren*{\Delta\sC_{\sfL_{h},\sR}(r, x)} & \text{when $\Gamma_2$ is linear}\\
\paren*{1+ \sum_{j = 1}^{\num}\ov c_j}\Gamma_2\paren*{\frac {\Delta\sC_{\sfL_{h},\sR}(r, x)}{\sum_{j = 1}^{\num}\uv c_j}} & \text{otherwise}
\end{cases}
\end{align*}
and 
\begin{align*}
& \inf_{r\in \sR}\sC_{\ldef}(h, r, x) - \sC^*_{\ldef}(\sH,\sR,x)\\
& = \inf_{r\in \sR}\sC_{\ldef}(h, r, x) - \inf_{h\in\sH, r\in \sR}\sC_{\ldef}(h, r, x)\\
& = \inf_{r\in \sR} \E_y\bracket*{\1_{\hh(x)\neq y}\1_{\rr(x) = 0} + \sum_{j = 1}^{\num} c_j(x, y) \1_{\rr(x) = j}}- \inf_{h\in\sH, r\in \sR}\E_y\bracket*{\1_{\hh(x)\neq y}\1_{\rr(x) = 0} + \sum_{j = 1}^{\num} c_j(x, y) \1_{\rr(x) = j}}\\
& = \min\curl*{\E_y\bracket*{\1_{\hh(x)\neq y}},\E_y\bracket*{c_j(x, y)}}- \min\curl*{\inf_{h\in \sH}\E_y\bracket*{\1_{\hh(x)\neq y}},\E_y\bracket*{c_j(x, y)}}\\
& \leq \E_y\bracket*{\1_{\hh(x)\neq y}} - \inf_{h\in \sH}\E_y\bracket*{\1_{\hh(x)\neq y}}\\
& = \sC_{\ell_{0-1}}(h, x) - \sC^*_{\ell_{0-1}}(\sH,x)\\
& = \Delta\sC_{\ell_{0-1}}(h, x)\\
& \leq \Gamma_1\paren*{\Delta\sC_{\ell}(h, x)}.
\tag{By $\sH$-consistency bounds of $\ell$ under assumption}
\end{align*}
Therefore, by \eqref{eq:expression-two-step-multi}, we obtain
\begin{align*}
& \sE_{\ldefsc}(h,r) - \sE_{\ldefsc}^*(\sH,\sR) + \sM_{\ldefsc}(\sH,\sR)\\
& \leq 
\begin{cases}
\E_X\bracket*{\Gamma_2\paren*{\Delta\sC_{\sfL_{h},\sR}(r, x)} } + \E_X\bracket*{\Gamma_1\paren*{\Delta\sC_{\ell}(h, x)}} & \text{when $\Gamma_2$ is linear}\\
\paren*{1+ \sum_{j = 1}^{\num}\ov c_j}\E_X\bracket*{\Gamma_2\paren*{\frac {\Delta\sC_{\sfL_{h},\sR}(r, x)}{\sum_{j = 1}^{\num}\uv c_{j}}}} + \E_X\bracket*{\Gamma_1\paren*{\Delta\sC_{\ell}(h, x)}} & \text{otherwise}
\end{cases}\\
& \leq 
\begin{cases}
\Gamma_2\paren*{\E_X\bracket*{\Gamma_2\paren*{\Delta\sC_{\sfL_{h},\sR}(r, x)} }} + \Gamma_1\paren*{\E_X\bracket*{\Delta\sC_{\ell}(h, x)}} & \text{when $\Gamma_2$ is linear}\\
\paren*{1+ \sum_{j = 1}^{\num}\ov c_j}\Gamma_2\paren*{\E_X\bracket*{\frac {\Delta\sC_{\sfL_{h},\sR}(r, x)}{\sum_{j = 1}^{\num}\uv c_{j}}}} + \Gamma_1\paren*{\E_X\bracket*{\Delta\sC_{\ell}(h, x)}} & \text{otherwise}
\end{cases}
\tag{$\Gamma_1$ and $\Gamma_2$ are concave}\\
& = 
\begin{cases}
\Gamma_1\paren*{\sE_{\ell}(h) - \sE_{\ell}^*(\sH) + \sM_{\ell}(\sH)} + \Gamma_2\paren*{\sE_{\sfL_{h}}(r) - \sE_{\sfL_{h}}^*(\sR) + \sM_{\sfL_{h}}(\sR)} & \text{when $\Gamma_2$ is linear}\\
\Gamma_1\paren*{\sE_{\ell}(h) - \sE_{\ell}^*(\sH) + \sM_{\ell}(\sH)} + \paren*{1+ \sum_{j = 1}^{\num}\ov c_j}\Gamma_2\paren*{\frac{\sE_{\sfL_{h}}(r) - \sE_{\sfL_{h}}^*(\sR) + \sM_{\sfL_{h}}(\sR)}{\sum_{j = 1}^{\num}\uv c_{j}}} & \text{otherwise},
\end{cases}
\end{align*}
which completes the proof.
\end{proof}

\section{Proof of realizable consistency for predictor-rejector
  two-stage surrogate losses (Theorem~\ref{Thm:bound-general-two-stage-general-realizable})}
\label{app:realizable-general}

\BoundGenralTwoStepGeneralRealizable*
\begin{proof}
First, by definition, it is straightforward to see that for any $h, r,
x, y$, $\sfL_{h}(r,x, y)$ upper-bounds the deferral loss
$\ldefsc$. Consider a data distribution and costs under which there
exists $h^*\in \sH$ and $r^*\in \sR$ such that $\sE_{\ldefsc}(h^*,r^*)
= 0$.

Let $\hat h$ be the minimizer of $\sE_{\ell_1}$ and $\hat r$ the
minimizer of $\sE_{\sfL_{\hat h}}$ Then, using the fact that
$\sfL_{h}$ upper-bounds the deferral loss $\ldefsc$, we have
$\sE_{\ldefsc}(\hat h, \hat r)\leq \sE_{\sfL_{\hat h}}(\hat r)$.

Next we analyze two cases. If for a point $x$, deferral occurs, that
is there exists $j^*\in [\num]$, such that $\rr^*(x) = j^*$, then we
must have $c_{j^*} = 0$ for all $x$ since the data is realizable and
$c_{j^*}$ is constant. Therefore, there exists an optimal $r^{**}$
deferring all the points to the $j^*$th expert. Then, by the
assumption that $\sR$ is closed under scaling and the Lebesgue
dominated convergence theorem, for $\ell_2$ being the logistic loss,
$\sE_{\ldefsc}(\hat h, \hat r)\leq\sE_{\sfL_{\hat h}}(\hat r) \leq
\lim_{\tau \to +\infty} \sE_{\sfL_{\hat h}}(\tau r^{**}) = 0$, where
we used the fact that in the limit of $\tau \to +\infty$ the logistic
loss term $\ell_2(\overline r^{**},x,j)$ corresponding to $j \neq j^*$
is zero.

On the other hand, if no deferral occurs for any point, that is
$\rr^*(x) = 0$ for any $x$, then we must have $\1_{\hh^*(x) \neq y} =
0$ for all $(x, y)$ since the data is realizable. Using the fact that
$\sH$ is closed under scaling and that the logistic loss is realizable
$\sH$-consistent in the standard classification, we obtain
$\1_{\hat{\hh}(x) \neq y} = 0$ for all $(x, y)$. Then, by the
assumption that $\sR$ is closed under scaling and the Lebesgue
dominated convergence theorem, for $\ell_2$ being the logistic loss,
$\sE_{\ldefsc}(\hat h, \hat r)\leq\sE_{\sfL_{\hat h}}(\hat r) \leq
\lim_{\tau \to + \infty} \sE_{\sfL_{\hat h}}(\tau r^{*}) = 0$,
where we used the fact that in the limit of $\tau \to +\infty$ the
logistic loss term $\ell_2(\overline r^{*},x,j)$ corresponding to $j
\neq 0$ is zero.

Therefore, the optimal solution from minimizing the predictor-rejector
two-stage surrogate loss leads to a zero error solution of the deferral
loss, which proves that the predictor-rejector two-stage surrogate
loss is realizable consistent.

\end{proof}
\restoreatoc

\chapter{Appendix to Chapter~\ref{ch6}}

\disableatoc
\section{Useful lemmas}

\Ldef*

\begin{proof}
  Observe that, for any $x \in \sX$, since $\rr(x) = 0$ if and only if
  $\rr(x) \neq j$ for all $j \geq 1$, the following equality holds:
\[
1_{\rr(x) = 0}
= 1_{\bigwedge_{j = 1}^{n_e} \curl*{\rr(x) \neq j}}
= \sum_{j = 1}^{n_e} 1_{\rr(x) \neq j} - (n_e - 1).
\]
Similarly, since $\rr(x) = j$ if and only if
  $\rr(x) \neq k$ for $k \neq j$ and $\rr(x) \neq 0$, the following equality holds:
\[
1_{\rr(x) = j} = 1_{\rr(x) \neq 0} + \sum_{k = 1}^{n_e} 1_{\rr(x) \neq k} 1_{k \neq j} - (n_e - 1).
\]
In view of these identities, starting from the definition of $\ldef$, we
can write:
\begin{align*}
  & \ldef(h, r, x, y)\\
  & = \sfL(h(x), y) 1_{\rr(x) = 0} + \sum_{j = 1}^{\num} c_j(x, y) 1_{\rr(x) = j}\\
  & = \sfL(h(x), y) \bracket*{\sum_{j = 1}^{n_e} 1_{\rr(x) \neq j} - (n_e - 1)}
  + \sum_{j = 1}^{\num} c_j(x, y) \bracket*{1_{\rr(x) \neq 0} + \sum_{k = 1}^{n_e} 1_{\rr(x) \neq k} 1_{k \neq j} - (n_e - 1)}\\
  & = \bracket*{\sum_{j = 1}^{\num} c_j(x,y)} 1_{\rr(x) \neq 0}
  + \sum_{j = 1}^{\num} \sfL(h(x), y) 1_{\rr(x) \neq j}
  + \sum_{j = 1}^{\num} \sum_{k = 1}^{\num} c_j(x, y) 1_{k \neq j} 1_{\rr(x) \neq k}\\
  & \quad - \paren*{\num - 1} \bracket*{\sfL(h(x), y) + \sum_{j = 1}^{\num} c_j(x, y)}\\
  & = \bracket*{\sum_{j = 1}^{\num} c_j(x,y)} 1_{\rr(x) \neq 0}
  + \sum_{j = 1}^{\num} \sfL(h(x), y) 1_{\rr(x) \neq j}
  + \sum_{k = 1}^{\num} \sum_{j = 1}^{\num} c_k(x, y) 1_{k \neq j} 1_{\rr(x) \neq j}\\
  & \quad - \paren*{\num - 1} \bracket*{\sfL(h(x), y) + \sum_{j = 1}^{\num} c_j(x, y)}
    \tag{change of variables $k$ and $j$}\\
  & = \bracket*{\sum_{j = 1}^{\num} c_j(x,y)} 1_{\rr(x) \neq 0} 
  + \sum_{j = 1}^{\num} \bracket*{\sfL(h(x), y)
    + \sum_{k = 1}^{\num} c_k(x, y) 1_{k \neq j}} 1_{\rr(x) \neq j}\\
  &\qquad - \paren*{\num - 1} \bracket*{\sfL(h(x), y) + \sum_{j = 1}^{\num} c_j(x, y)}.
\end{align*}
This completes the proof.
\end{proof}

\begin{lemma}
\label{lemma:aux}
Assume that the following $\sR$-consistency bound holds for all $r \in \sR$ and any distribution,
\begin{equation*}
\sE_{\ell_{0-1}}(r) - \sE^*_{\ell_{0-1}}(\sR) + \sM_{\ell_{0-1}}(\sR) \leq \Gamma\paren*{\sE_{\ell}(r) - \sE^*_{\ell}(\sR) + \sM_{\ell}(\sR)}.
\end{equation*}
Then, for any $p = (p_0, \ldots, p_{\num})\in \Delta^{\num}$ and $x \in \sX$, we have
\begin{align*}
\sum_{j = 0}^{\num} p_j 1_{\rr(x) \neq j} - \inf_{r \in \sR} \paren*{\sum_{j = 0}^{\num} p_j 1_{\rr(x) \neq j}} \leq \Gamma\paren*{\sum_{j = 0}^{\num} p_j \ell(r, x, j) - \inf_{r \in \sR} \paren*{\sum_{j = 0}^{\num} p_j \ell(r, x, j) }}.
\end{align*}
\end{lemma}
\begin{proof}
For any $x \in \sX$, consider a distribution $\delta_{x}$ that concentrates on that point. Let $p_j = \mathbb{P}(y = j \mid x)$, $j \in [\num]$. Then, by definition, $\sE_{\ell_{0-1}}(r) - \sE^*_{\ell_{0-1}}(\sR) + \sM_{\ell_{0-1}}(\sR)$ can be expressed as 
\begin{equation*}
\sE_{\ell_{0-1}}(r) - \sE^*_{\ell_{0-1}}(\sR) + \sM_{\ell_{0-1}}(\sR) = \sum_{j = 0}^{\num} p_j 1_{\rr(x) \neq j} - \inf_{r \in \sR} \paren*{\sum_{j = 0}^{\num} p_j 1_{\rr(x) \neq j}}.
\end{equation*}
Similarly, $\sE_{\ell}(r) - \sE^*_{\ell}(\sR) + \sM_{\ell}(\sR)$ can be expressed as
\begin{equation*}
\sE_{\ell}(r) - \sE^*_{\ell}(\sR) + \sM_{\ell}(\sR) = \sum_{j = 0}^{\num} p_j \ell(r, x, j) - \inf_{r \in \sR} \paren*{\sum_{j = 0}^{\num} p_j \ell(r, x, j) }.
\end{equation*}
Since the $\sR$-consistency bound holds by the assumption, we complete the proof.
\end{proof}

\section{Proof of Theorem~\ref{thm:single}}
\label{app:exp}

\Single*
\begin{proof}
The conditional error of the deferral loss can be expressed as
\begin{equation}
\label{eq:cond-error}
\begin{aligned}
\E_{y | x}\bracket*{\ldef(h, r, x, y)} =
\E_{y | x}\bracket*{\sfL(h(x), y)} 1_{\rr(x) = 0} + \sum_{j = 1}^{\num} \E_{y | x}\bracket*{c_j(x,y)} 1_{\rr(x) = j}.
\end{aligned}
\end{equation}
Let $\ov c_0(x) = \inf_{h\in \sH}\E_{y | x}\bracket*{\sfL(h(x), y)}$ and $\ov c_j(x) = \E_{y | x}\bracket*{c(x, y)}$.
Thus, the best-in class conditional error of the deferral loss can be expressed as
\begin{equation}
\label{eq:best-cond-error}
\inf_{h\in \sH, r\in \sR}\E_{y | x}\bracket*{\ldef(h, r, x, y)} =  \min_{j \in [\num]} \ov c_j(x).
\end{equation}
The conditional error of the surrogate loss can be expressed as
\begin{equation}
\label{eq:cond-error-sur}
\begin{aligned}
\E_{y | x}\bracket*{\ell_{\ell}(h, r, x, y)} & =
\paren*{\sum_{j = 1}^{\num} \E_{y | x}\bracket*{c_j(x,y)}} \ell(r, x, 0)
+ \sum_{j = 1}^{\num} \paren*{\E_{y | x}\bracket*{\sfL(h(x), y)} + \sum_{j' \neq j}^{\num} \E_{y | x}\bracket*{c_{j'}(x,y)}} \ell(r, x, j)\\
& \quad - \paren*{\num - 1} \E_{y | x}\bracket*{\sfL(h(x), y)}.
\end{aligned}
\end{equation}
Note that the coefficient of term $\E_{y | x}\bracket*{\sfL(h(x), y)}$ satisfies $\sum_{j = 1}^{\num}  \ell(r, x, j) - \paren*{\num - 1} \geq 0$ since $\ell \geq \ell_{0-1}$.
Thus, the best-in class conditional error of the surrogate loss can be expressed as
\begin{equation}
\label{eq:best-cond-error-sur}
\begin{aligned}
& \inf_{h\in \sH,r\in \sR}\E_{y | x}\bracket*{L_{\ell}(h, r, x, y)}\\
& = \inf_{r \in \sR}\bracket*{\paren*{\sum_{j = 1}^{\num} \ov c_j(x)} \ell(r, x, 0)
+ \sum_{j = 1}^{\num} \paren*{\ov c_0(x) + \sum_{j' \neq j}^{\num} \ov c_{j'}(x)} \ell(r, x, j)} - \paren*{\num - 1} \ov c_0(x).
\end{aligned}
\end{equation}
Next, we analyze four cases separately to show that the calibration gap of the surrogate loss can be lower bounded by that of the deferral loss.
\paragraph{Case I: \texorpdfstring{$\rr(x) = 0$}{I} and \texorpdfstring{$\ov c_0(x) \leq \min_{j = 1}^{\num} \ov c_j(x)$}{I}.} In this case, by \eqref{eq:cond-error} and \eqref{eq:best-cond-error}, the calibration gap of the deferral loss can be expressed as
\begin{align*}
\E_{y | x}\bracket*{\ldef(h, r, x, y)}-\inf_{h\in \sH, r\in \sR}\E_{y | x}\bracket*{\ldef(h, r, x, y)} = \E_{y | x}\bracket*{\sfL(h(x), y)} - \inf_{h\in \sH}\E_{y | x}\bracket*{\sfL(h(x), y)}.
\end{align*}
By \eqref{eq:cond-error-sur} and \eqref{eq:best-cond-error-sur}, the calibration gap of the surrogate loss can be expressed as
\begin{align*}
&\E_{y | x}\bracket*{L_{\ell}(h, r, x, y)} - \inf_{h\in \sH, r\in \sR}\E_{y | x}\bracket*{L_{\ell}(h, r, x, y)}\\
& =  \paren*{\sum_{j = 1}^{\num} \ov c_j(x)} \ell(r, x, 0)
+ \sum_{j = 1}^{\num} \paren*{\E_{y | x}\bracket*{\sfL(h(x), y)} + \sum_{j' \neq j}^{\num} \ov c_{j'}(x)} \ell(r, x, j) - \paren*{\num - 1} \E_{y | x} \sfL(h(x), y)\\
&\quad - \inf_{r \in \sR}\bracket*{\paren*{\sum_{j = 1}^{\num} \ov c_j(x)} \ell(r, x, 0)
+ \sum_{j = 1}^{\num} \paren*{\ov c_0(x) + \sum_{j' \neq j}^{\num} \ov c_{j'}(x)} \ell(r, x, j)} + \paren*{\num - 1} \ov c_0(x).
\end{align*}
Since $\ell \geq \ell_{0-1}$, we have $\ell(r, x, j) \geq 1$ for $j \neq 0$. By eliminating the infimum over $\sR$ from the final line, and consequently canceling the terms related to $\ov c_j(x)$ for $j \neq 0$, the calibration gap of the surrogate loss can be lower bounded as
\begin{align*}
&\E_{y | x}\bracket*{L_{\ell}(h, r, x, y)} - \inf_{h\in \sH,r\in \sR}\E_{y | x}\bracket*{L_{\ell}(h, r, x, y)}\\
&\geq \paren*{\E_{y | x}\bracket*{\sfL(h(x), y)} - \inf_{h\in \sH}\E_{y | x}\bracket*{\sfL(h(x), y)}}\paren*{\sum_{j = 1}^{\num} \ell(r, x, j) - \num + 1}\\
&\geq \E_{y | x}\bracket*{\sfL(h(x), y)}- \inf_{h\in \sH}\E_{y | x}\bracket*{\sfL(h(x), y)} \tag{$\sum_{j = 1}^{\num} \ell(r, x, j) - \num + 1 \geq 1$}\\
& = \E_{y | x}\bracket*{\ldef(h, r, x, y)}-\inf_{h\in \sH,r\in \sR}\E_{y | x}\bracket*{\ldef(h, r, x, y)}.
\end{align*}

\paragraph{Case II: \texorpdfstring{$\ov c_0(x) > \min_{j = 1}^{\num} \ov c_j(x)$}{II}.} In this case, by \eqref{eq:cond-error} and \eqref{eq:best-cond-error},  the calibration gap of the deferral loss can be expressed as
\begin{align*}
\E_{y | x}\bracket*{\ldef(h, r, x, y)}-\inf_{h\in \sH,r\in \sR}\E_{y | x}\bracket*{\ldef(h, r, x, y)} = \ov c_{\rr(x)}(x) - \min_{j = 1}^{\num} \ov c_j(x).
\end{align*}
By \eqref{eq:cond-error-sur} and \eqref{eq:best-cond-error-sur}, the calibration gap of the surrogate loss can be expressed as
\begin{align*}
&\E_{y | x}\bracket*{L_{\ell}(h, r, x, y)} - \inf_{h\in \sH, r\in \sR}\E_{y | x}\bracket*{L_{\ell}(h, r, x, y)}\\
& =  \paren*{\sum_{j = 1}^{\num} \ov c_j(x)} \ell(r, x, 0)
+ \sum_{j = 1}^{\num} \paren*{\E_{y | x}\bracket*{\sfL(h(x), y)} + \sum_{j' \neq j}^{\num} \ov c_{j'}(x)} \ell(r, x, j) - \paren*{\num - 1} \E_{y | x}\bracket*{\sfL(h(x), y)}\\
&\quad - \inf_{r \in \sR}\bracket*{\paren*{\sum_{j = 1}^{\num} \ov c_j(x)} \ell(r, x, 0)
+ \sum_{j = 1}^{\num} \paren*{\ov c_0(x) + \sum_{j' \neq j}^{\num} \ov c_{j'}(x)} \ell(r, x, j)} + \paren*{\num - 1} \ov c_0(x).
\end{align*}
Using the fact that $\ov c_0(x) = \inf_{h\in \sH}\E_{y | x}\bracket*{\sfL(h(x), y)} \leq \E_{y | x}\bracket*{\sfL(h(x), y)}$, the calibration gap of the surrogate loss can be lower bounded as
\begin{align*}
& \E_{y | x}\bracket*{L_{\ell}(h, r, x, y)} - \inf_{h\in \sH,r\in \sR}\E_{y | x}\bracket*{L_{\ell}(h, r, x, y)}\\
& \geq \paren*{\sum_{j = 1}^{\num} \ov c_j(x)} \ell(r, x, 0)
+ \sum_{j = 1}^{\num} \paren*{\E_{y | x}\bracket*{\sfL(h(x), y)} + \sum_{j' \neq j}^{\num} \ov c_{j'}(x)} \ell(r, x, j)\\
&\quad - \inf_{r \in \sR}\bracket*{\paren*{\sum_{j = 1}^{\num} \ov c_j(x)} \ell(r, x, 0)
+ \sum_{j = 1}^{\num} \paren*{\E_{y | x}\bracket*{\sfL(h(x), y)} + \sum_{j' \neq j}^{\num} \ov c_{j'}(x)} \ell(r, x, j)}\\
& = \num \paren*{\E_{y | x}\bracket*{\sfL(h(x), y)} + \sum_{j = 1 }^{\num} \ov c_{j}(x)} \bracket*{\sum_{j = 0}^{\num} p_j \ell(r, x, j) - \inf_{r \in \sR} \paren*{\sum_{j = 0}^{\num} p_j \ell(r, x, j) }}
\end{align*}
where we let $p_0 = \frac{\sum_{j = 1}^{\num} \ov c_j(x)}{\num \paren*{\E_{y | x}\bracket*{\sfL(h(x), y)} + \sum_{j = 1 }^{\num} \ov c_{j}(x)}}$ and $p_j = \frac{\E_{y | x}\bracket*{\sfL(h(x), y)} + \sum_{j' \neq j}^{\num} \ov c_{j'}(x)}{\num \paren*{\E_{y | x}\bracket*{\sfL(h(x), y)} + \sum_{j = 1 }^{\num} \ov c_{j}(x)}}$, $j = \curl*{1, \ldots, \num}$ in the last equality. 
By Lemma~\ref{lemma:aux}, we have
\begin{align*}
& \sum_{j = 0}^{\num} p_j \ell(r, x, j) - \inf_{r \in \sR} \paren*{\sum_{j = 0}^{\num} p_j \ell(r, x, j) }\\
& \geq \Gamma^{-1}\paren*{\sum_{j = 0}^{\num} p_j 1_{\rr(x) \neq j} - \inf_{r \in \sR} \paren*{\sum_{j = 0}^{\num} p_j 1_{\rr(x) \neq j}}} \\
& = \Gamma^{-1}\paren*{\max_{j \in [\num]}p_j - p_{\rr(x)}}\\
& = \Gamma^{-1}\paren*{\frac{\E_{y | x}\bracket*{\sfL(h(x), y)} - \min_{j = 1}^{\num} \ov c_j(x)}{\num \paren*{\E_{y | x}\bracket*{\sfL(h(x), y)} + \sum_{j = 1 }^{\num} \ov c_{j}(x)}}}
\end{align*}
Therefore, we obtain
\begin{align*}
& \E_{y | x}\bracket*{L_{\ell}(h, r, x, y)} - \inf_{h\in \sH,r\in \sR}\E_{y | x}\bracket*{L_{\ell}(h, r, x, y)}\\
& \geq  \num \paren*{\E_{y | x}\bracket*{\sfL(h(x), y)} + \sum_{j = 1 }^{\num} \ov c_{j}(x)}  \Gamma^{-1}\paren*{\frac{\E_{y | x}\bracket*{\sfL(h(x), y)} - \min_{j = 1}^{\num} \ov c_j(x)}{\num \paren*{\E_{y | x}\bracket*{\sfL(h(x), y)} + \sum_{j = 1 }^{\num} \ov c_{j}(x)}}} \\
& \geq
\num \paren*{\E_{y | x}\bracket*{\sfL(h(x), y)} + \sum_{j = 1 }^{\num} \ov c_{j}(x)}  \Gamma^{-1}\paren*{\frac{\E_{y | x}\bracket*{\ldef(h, r, x, y)}-\inf_{h\in \sH,r\in \sR}\E_{y | x}\bracket*{\ldef(h, r, x, y)}}{\num \paren*{\E_{y | x}\bracket*{\sfL(h(x), y)} + \sum_{j = 1 }^{\num} \ov c_{j}(x)}}}\\
& \geq \frac{1}{\beta^{\frac{1}{\alpha}}} \frac{\paren*{\E_{y | x}\bracket*{\ldef(h, r, x, y)}-\inf_{h\in \sH,r\in \sR}\E_{y | x}\bracket*{\ldef(h, r, x, y)}}^{\frac{1}{\alpha}}}{\paren*{\num(\ul + \sum_{j = 1}^{\num}\uc_j)}^{\frac{1}{\alpha} - 1}}
\end{align*}
where we use the fact that $\Gamma(t) = \beta t^{\alpha}$, $\alpha \in (0, 1]$, $\beta > 0$, $\sfL \leq \ul$ and $c_j \leq \uc_j$, $j = \curl*{1, \ldots, \num}$ in the last inequality.

\paragraph{Case III: \texorpdfstring{$\rr(x) > 0$}{III} and \texorpdfstring{$\ov c_0(x) \leq \min_{j = 1}^{\num} \ov c_j(x)$}{III}.} In this case, by \eqref{eq:cond-error} and \eqref{eq:best-cond-error},  the calibration gap of the deferral loss can be expressed as
\begin{align*}
\E_{y | x}\bracket*{\ldef(h, r, x, y)}-\inf_{h\in \sH,r\in \sR}\E_{y | x}\bracket*{\ldef(h, r, x, y)} = \ov c_{\rr(x)}(x) - \ov c_0(x).
\end{align*}
By \eqref{eq:cond-error-sur} and \eqref{eq:best-cond-error-sur}, the calibration gap of the surrogate loss can be expressed as
\begin{align*}
&\E_{y | x}\bracket*{L_{\ell}(h, r, x, y)} - \inf_{h\in \sH, r\in \sR}\E_{y | x}\bracket*{L_{\ell}(h, r, x, y)}\\
& =  \paren*{\sum_{j = 1}^{\num} \ov c_j(x)} \ell(r, x, 0)
+ \sum_{j = 1}^{\num} \paren*{\E_{y | x}\bracket*{\sfL(h(x), y)} + \sum_{j' \neq j}^{\num} \ov c_{j'}(x)} \ell(r, x, j) - \paren*{\num - 1} \E_{y | x}\bracket*{\sfL(h(x), y)}\\
&\quad - \inf_{r \in \sR}\bracket*{\paren*{\sum_{j = 1}^{\num} \ov c_j(x)} \ell(r, x, 0)
+ \sum_{j = 1}^{\num} \paren*{\ov c_0(x) + \sum_{j' \neq j}^{\num} \ov c_{j'}(x)} \ell(r, x, j)} + \paren*{\num - 1} \ov c_0(x).
\end{align*}
Using the fact that $\E_{y | x}\bracket*{\sfL(h(x), y)} \geq \inf_{h\in \sH}\E_{y | x}\bracket*{\sfL(h(x), y)} = \ov c_0(x)$, the calibration gap of the surrogate loss can be lower bounded as
\begin{align*}
& \E_{y | x}\bracket*{L_{\ell}(h, r, x, y)} - \inf_{h\in \sH,r\in \sR}\E_{y | x}\bracket*{L_{\ell}(h, r, x, y)}\\
& \geq \paren*{\sum_{j = 1}^{\num} \ov c_j(x)} \ell(r, x, 0)
+ \sum_{j = 1}^{\num} \paren*{\ov c_0(x) + \sum_{j' \neq j}^{\num} \ov c_{j'}(x)} \ell(r, x, j)\\
&\quad - \inf_{r \in \sR}\bracket*{\paren*{\sum_{j = 1}^{\num} \ov c_j(x)} \ell(r, x, 0)
+ \sum_{j = 1}^{\num} \paren*{\ov c_0(x) + \sum_{j' \neq j}^{\num} \ov c_{j'}(x)} \ell(r, x, j)}\\
& = \num \paren*{\sum_{j = 0 }^{\num} \ov c_{j}(x)} \bracket*{\sum_{j = 0}^{\num} p_j \ell(r, x, j) - \inf_{r \in \sR} \paren*{\sum_{j = 0}^{\num} p_j \ell(r, x, j) }}
\end{align*}
where we let $p_0 = \frac{\sum_{j = 1}^{\num} \ov c_j(x)}{\num \paren*{\sum_{j = 0 }^{\num} \ov c_{j}(x)}}$ and $p_j = \frac{\ov c_0(x) + \sum_{j' \neq j}^{\num} \ov c_{j'}(x)}{\num \paren*{\sum_{j = 0 }^{\num} \ov c_{j}(x)}}$, $j = \curl*{1, \ldots, \num}$ in the last equality. 
By Lemma~\ref{lemma:aux}, we have
\begin{align*}
& \sum_{j = 0}^{\num} p_j \ell(r, x, j) - \inf_{r \in \sR} \paren*{\sum_{j = 0}^{\num} p_j \ell(r, x, j) }\\
& \geq \Gamma^{-1}\paren*{\sum_{j = 0}^{\num} p_j 1_{\rr(x) \neq j} - \inf_{r \in \sR} \paren*{\sum_{j = 0}^{\num} p_j 1_{\rr(x) \neq j}}} \\
& = \Gamma^{-1}\paren*{\max_{j \in [\num]}p_j - p_{\rr(x)}}\\
& = \Gamma^{-1}\paren*{\frac{\ov c_{\rr(x)}(x) - \ov c_0(x)}{\num \paren*{\sum_{j = 0 }^{\num} \ov c_{j}(x)}}}
\end{align*}
Therefore, we obtain
\begin{align*}
& \E_{y | x}\bracket*{L_{\ell}(h, r, x, y)} - \inf_{h\in \sH,r\in \sR}\E_{y | x}\bracket*{L_{\ell}(h, r, x, y)}\\
& \geq \num \paren*{\sum_{j = 0 }^{\num} \ov c_{j}(x)} \Gamma^{-1}\paren*{\frac{\ov c_{\rr(x)}(x) - \ov c_0(x)}{\num \paren*{\sum_{j = 0 }^{\num} \ov c_{j}(x)}}} \\
& = \num \paren*{\sum_{j = 0 }^{\num} \ov c_{j}(x)} \Gamma^{-1}\paren*{\frac{\E_{y | x}\bracket*{\ldef(h, r, x, y)}-\inf_{h\in \sH,r\in \sR}\E_{y | x}\bracket*{\ldef(h, r, x, y)}}{\num \paren*{\sum_{j = 0 }^{\num} \ov c_{j}(x)}}} \\
& \geq \frac{1}{\beta^{\frac{1}{\alpha}}} \frac{\paren*{\E_{y | x}\bracket*{\ldef(h, r, x, y)}-\inf_{h\in \sH,r\in \sR}\E_{y | x}\bracket*{\ldef(h, r, x, y)}}^{\frac{1}{\alpha}}}{\paren*{\num(\ul + \sum_{j = 1}^{\num}\uc_j)}^{\frac{1}{\alpha} - 1}}
\end{align*}
where we use the fact that $\Gamma(t) = \beta t^{\alpha}$, $\alpha \in (0, 1]$, $\beta > 0$, $\sfL \leq \ul$ and $c_j \leq \uc_j$, $j = \curl*{1, \ldots, \num}$ in the last inequality.

Overall, by taking the expectation of the deferral and surrogate calibration gaps and using Jensen's inequality in each case, we obtain
\begin{equation*}
   \sE_{\ldef}(h, r) - \sE_{\ldef}^*(\sH,\sR) + \sM_{\ldef}(\sH,\sR) \leq 
    \ov \Gamma\paren*{\sE_{L_{\ell}}(h, r) -  \sE_{L_{\ell}}^*(\sH,\sR) + \sM_{L_{\ell}}(\sH,\sR)}.
\end{equation*}
where $\ov \Gamma(t) = \max\curl*{t, \paren*{\num\paren*{\ul + \sum_{j = 1}^{\num}\uc_j}}^{1 - \alpha} \beta\, t^{\alpha}}$.
\end{proof}

\section{Proof of Theorem~\ref{thm:tsr}}
\label{app:tsr}

\TwostageR*
\begin{proof}
Given a hypothesis set $\sR$, a multi-class loss function $\ell$ and a predictor $h$. For any $r \in \sR$, $x \in \sX$ and $y \in \sY$, the conditional error of $L^h_{\ell}$ and $L^h_{\rm{def}}$ can be written as
\begin{equation}
\label{eq:tsr-cond-error}
\begin{aligned}
\E_{y | x}\bracket*{L^h_{\rm{def}}(r, x, y)} & =  \E_{y | x}\bracket*{\sfL(h(x), y)} 1_{\rr(x) = 0} + \sum_{j = 1}^{\num} \E_{y | x}\bracket*{c_j(x,y)} 1_{\rr(x) = j}\\
\E_{y | x}\bracket*{L^h_{\ell}(r, x, y)} & =  \paren*{\sum_{j = 1}^{\num} \E_{y | x}\bracket*{c_j(x,y)}} \ell(r, x, 0)
+ \sum_{j = 1}^{\num} \paren*{\E_{y | x}\bracket*{\sfL(h(x), y)} + \sum_{j' \neq j}^{\num} \E_{y | x}\bracket*{c_{j'}(x,y)}} \ell(r, x, j).
\end{aligned}
\end{equation}
Let $\ov c_0(x) = \inf_{h\in \sH}\E_{y | x}\bracket*{\sfL(h(x), y)}$ and $\ov c_j(x) = \E_{y | x}\bracket*{c(x, y)}$.
Thus, the best-in class conditional error of of $L^h_{\ell}$ and $L^h_{\rm{def}}$ can be expressed as
\begin{equation}
\label{eq:tsr-best-cond-error}
\begin{aligned}
\inf_{r\in \sR}\E_{y | x}\bracket*{L^h_{\rm{def}}(r, x, y)} & =  \min_{j \in [\num]} \ov c_j(x)\\
\inf_{r\in \sR}\E_{y | x}\bracket*{L^h_{\ell}(r, x, y)} & =  \inf_{r \in \sR}\bracket*{\paren*{\sum_{j = 1}^{\num} \ov c_j(x)} \ell(r, x, 0)
+ \sum_{j = 1}^{\num} \paren*{\E_{y | x}\bracket*{\sfL(h(x), y)} + \sum_{j' \neq j}^{\num} \ov c_{j'}(x)} \ell(r, x, j)}
\end{aligned}
\end{equation}
Let $p_0 = \frac{\sum_{j = 1}^{\num} \ov c_j(x)}{\num \paren*{\E_{y | x}\bracket*{\sfL(h(x), y)} + \sum_{j = 1 }^{\num} \ov c_{j}(x)}}$ and $p_j = \frac{\E_{y | x}\bracket*{\sfL(h(x), y)} + \sum_{j' \neq j}^{\num} \ov c_{j'}(x)}{\num \paren*{\E_{y | x}\bracket*{\sfL(h(x), y)} + \sum_{j = 1 }^{\num} \ov c_{j}(x)}}$, $j = \curl*{1, \ldots, \num}$. Then, the calibration gap of $L_{\ell}^h$ can be written as 
\begin{align*}
& \E_{y | x}\bracket*{L^h_{\ell}(r, x, y)} - \inf_{r\in \sR}\E_{y | x}\bracket*{L^h_{\ell}(r, x, y)}\\
& = \num \paren*{\E_{y | x}\bracket*{\sfL(h(x), y)} + \sum_{j = 1 }^{\num} \ov c_{j}(x)}  \bracket*{\sum_{j = 0}^{\num} p_j \ell(r, x, j) - \inf_{r \in \sR} \paren*{\sum_{j = 0}^{\num} p_j \ell(r, x, j) }}
\end{align*}
By Lemma~\ref{lemma:aux}, we have
\begin{align*}
\sum_{j = 0}^{\num} p_j \ell(r, x, j) - \inf_{r \in \sR} \paren*{\sum_{j = 0}^{\num} p_j \ell(r, x, j) } 
& \geq \Gamma^{-1}\paren*{\sum_{j = 0}^{\num} p_j 1_{\rr(x) \neq j} - \inf_{r \in \sR} \paren*{\sum_{j = 0}^{\num} p_j 1_{\rr(x) \neq j}} }\\
& = \Gamma^{-1}\paren*{\max_{j \in [\num]}p_j - p_{\rr(x)}}\\
& = \Gamma^{-1}\paren*{\frac{\ov c_{\rr(x)}(x) - \min_{j \in [\num]} \ov c_j(x)}{\num \paren*{\E_{y | x}\bracket*{\sfL(h(x), y)} + \sum_{j = 1 }^{\num} \ov c_{j}(x)}}}.
\end{align*}
Therefore, we obtain
\begin{align*}
& \E_{y | x}\bracket*{L_{\ell}(r, x, y)} - \inf_{r\in \sR}\E_{y | x}\bracket*{L_{\ell}(r, x, y)}\\
& \geq \num \paren*{\E_{y | x}\bracket*{\sfL(h(x), y)} + \sum_{j = 1 }^{\num} \ov c_{j}(x)} \Gamma^{-1}\paren*{\frac{\ov c_{\rr(x)}(x) - \min_{j \in [\num]} \ov c_j(x)}{\num \paren*{\E_{y | x}\bracket*{\sfL(h(x), y)} + \sum_{j = 1 }^{\num} \ov c_{j}(x)}}}\\
& \geq \frac{1}{\beta^{\frac{1}{\alpha}}} \frac{\paren*{\E_{y | x}\bracket*{L^h_{\rm{def}}(r, x, y)}-\inf_{r\in \sR}\E_{y | x}\bracket*{L^h_{\rm{def}}(r, x, y)}}^{\frac{1}{\alpha}}}{\paren*{\num(\ul + \sum_{j = 1}^{\num}\uc_j)}^{\frac{1}{\alpha} - 1}}
\end{align*}
where we use the fact that $\Gamma(t) = \beta t^{\alpha}$, $\alpha \in (0, 1]$, $\beta > 0$, $\sfL \leq \ul$ and $c_j \leq \uc_j$, $j = \curl*{1, \ldots, \num}$ in the last inequality.
Taking the expectation on both sides and using Jensen's inequality, we obtain
\begin{equation*}
   \sE_{L^h_{\rm{def}}}(r) - \sE_{L^h_{\rm{def}}}^*(\sR) + \sM_{L^h_{\rm{def}}}(\sR) \leq 
    \ov \Gamma\paren*{\sE_{L^h_{\ell}}(r) -  \sE_{L^h_{\ell}}^*(\sR) + \sM_{L^h_{\ell}}(\sR)}.
\end{equation*}
where $\ov \Gamma(t) = \paren*{\num\paren*{\ul + \sum_{j = 1}^{\num}\uc_j}}^{1 - \alpha} \beta\, t^{\alpha}$.
\end{proof}

\section{Proof of Theorem~\ref{thm:tshr}}
\label{app:tshr}

\TwostageHR*
\begin{proof}
The conditional error of the deferral loss can be expressed as
\begin{equation*}
\begin{aligned}
\E_{y | x}\bracket*{\ldef(h, r, x, y)} =
\E_{y | x}\bracket*{\sfL(h(x), y)} 1_{\rr(x) = 0} + \sum_{j = 1}^{\num} \E_{y | x}\bracket*{c_j(x,y)} 1_{\rr(x) = j}.
\end{aligned}
\end{equation*}
Let $\ov c_0(x) = \inf_{h\in \sH}\E_{y | x}\bracket*{\sfL(h(x), y)}$ and $\ov c_j(x) = \E_{y | x}\bracket*{c(x, y)}$.
Thus, the best-in class conditional error of the deferral loss can be expressed as
\begin{equation*}
\inf_{h\in \sH, r\in \sR}\E_{y | x}\bracket*{\ldef(h, r, x, y)} =  \min_{j \in [\num]} \ov c_j(x).
\end{equation*}
Thus, by introducing the term $\min\curl*{\E_{y | x}\bracket*{\sfL(h(x), y)}, \min_{j = 1}^{\num} \ov c_j(x)}$ and subsequently subtracting it after rearranging, the conditional regret of the deferral loss $\ldef$ can be written as follows
\begin{equation}
\label{eq:tshr-cond-reg-def}
\begin{aligned}
& \E_{y | x}\bracket*{\ldef(h, r, x, y)} - \inf_{h\in \sH, r\in \sR}\E_{y | x}\bracket*{\ldef(h, r, x, y)}\\
& = \E_{y | x}\bracket*{\sfL(h(x), y)} 1_{\rr(x) = 0} + \sum_{j = 1}^{\num} \E_{y | x}\bracket*{c_j(x,y)} 1_{\rr(x) = j} - \min_{j \in [\num]} \ov c_j(x)\\ 
& =  \E_{y | x}\bracket*{\sfL(h(x), y)} 1_{\rr(x) = 0} + \sum_{j = 1}^{\num} \E_{y | x}\bracket*{c_j(x,y)} 1_{\rr(x) = j} - \min_{j = 1}^{\num} \ov c_j(x) + \paren*{\min_{j = 1}^{\num} \ov c_j(x) - \min_{j \in [\num]} \ov c_j(x)}.
\end{aligned}
\end{equation}
Note that by the property of the minimum,  the second term can be upper-bounded as 
\begin{align*}
\min_{j = 1}^{\num} \ov c_j(x) - \min_{j \in [\num]} \ov c_j(x) \leq \E_{y | x}\bracket*{\sfL(h(x), y)} - \inf_{h \in \sH}\E_{y | x}\bracket*{\sfL(h(x), y)}.
\end{align*}
Next, we will upper-bound the first term. Note that the conditional error and the best-in class conditional error of $L^h_{\ell}$ can be expressed as
\begin{equation}
\label{eq:tshr-cond-error-sur}
\begin{aligned}
L^h_{\ell}(r, x, y) & =  \paren*{\sum_{j = 1}^{\num} \ov c_j(x)} \ell(r, x, 0)
+ \sum_{j = 1}^{\num} \paren*{\E_{y | x}\bracket*{\sfL(h(x), y)} + \sum_{j' \neq j}^{\num} \ov c_{j'}(x) } \ell(r, x, j)\\
\inf_{r\in \sR}\E_{y | x}\bracket*{L^h_{\ell}(r, x, y)} & = \inf_{r \in \sR}\bracket*{\paren*{\sum_{j = 1}^{\num} \ov c_j(x)} \ell(r, x, 0)
+ \sum_{j = 1}^{\num} \paren*{\E_{y | x}\bracket*{\sfL(h(x), y)} + \sum_{j' \neq j}^{\num} \ov c_{j'}(x)} \ell(r, x, j)}
\end{aligned}
\end{equation}
Let $p_0 = \frac{\sum_{j = 1}^{\num} \ov c_j(x)}{\num \paren*{\E_{y | x}\bracket*{\sfL(h(x), y)} + \sum_{j = 1 }^{\num} \ov c_{j}(x)}}$ and $p_j = \frac{\E_{y | x}\bracket*{\sfL(h(x), y)} + \sum_{j' \neq j}^{\num} \ov c_{j'}(x)}{\num \paren*{\E_{y | x}\bracket*{\sfL(h(x), y)} + \sum_{j = 1 }^{\num} \ov c_{j}(x)}}$, $j = \curl*{1, \ldots, \num}$. Then, the first term can be rewritten as 
\begin{align*}
& \E_{y | x}\bracket*{\sfL(h(x), y)} 1_{\rr(x) = 0} + \sum_{j = 1}^{\num} \E_{y | x}\bracket*{c_j(x,y)} 1_{\rr(x) = j} - \min_{j = 1}^{\num} \ov c_j(x) \\
& = \num \paren*{\E_{y | x}\bracket*{\sfL(h(x), y)} + \sum_{j = 1 }^{\num} \ov c_{j}(x)}  \bracket*{\sum_{j = 0}^{\num} p_j 1_{\rr(x) \neq 0} - \inf_{r \in \sR} \paren*{\sum_{j = 0}^{\num} p_j 1_{\rr(x) \neq j}}}.
\end{align*}
By Lemma~\ref{lemma:aux}, we have
\begin{align*}
\sum_{j = 0}^{\num} p_j 1_{\rr(x) \neq j} - \inf_{r \in \sR} \paren*{\sum_{j = 0}^{\num} p_j 1_{\rr(x) \neq j}} 
& \leq \Gamma\paren*{\sum_{j = 0}^{\num} p_j \ell(r, x, j) - \inf_{r \in \sR} \paren*{\sum_{j = 0}^{\num} p_j \ell(r, x, j) }}\\
& = \Gamma\paren*{\frac{L^h_{\ell}(r, x, y) - \inf_{r\in \sR}\E_{y | x}\bracket*{L^h_{\ell}(r, x, y)}}{\num \paren*{\E_{y | x}\bracket*{\sfL(h(x), y)} + \sum_{j = 1 }^{\num} \ov c_{j}(x)}}}.
\end{align*}
Therefore, the first term can be upper-bounded as
\begin{align*}
&  \E_{y | x}\bracket*{\sfL(h(x), y)} 1_{\rr(x) = 0} + \sum_{j = 1}^{\num} \E_{y | x}\bracket*{c_j(x,y)} 1_{\rr(x) = j} - \min_{j = 1}^{\num} \ov c_j(x)\\
& = \num \paren*{\E_{y | x}\bracket*{\sfL(h(x), y)} + \sum_{j = 1 }^{\num} \ov c_{j}(x)}  \bracket*{\sum_{j = 0}^{\num} p_j 1_{\rr(x) \neq 0} - \inf_{r \in \sR} \paren*{\sum_{j = 0}^{\num} p_j 1_{\rr(x) \neq j}}}\\
& \leq \num \paren*{\E_{y | x}\bracket*{\sfL(h(x), y)} + \sum_{j = 1 }^{\num} \ov c_{j}(x)} \Gamma\paren*{\frac{L^h_{\ell}(r, x, y) - \inf_{r\in \sR}\E_{y | x}\bracket*{L^h_{\ell}(r, x, y)}}{\num \paren*{\E_{y | x}\bracket*{\sfL(h(x), y)} + \sum_{j = 1 }^{\num} \ov c_{j}(x)}}}\\
& \leq \paren*{\num\paren*{\ul + \sum_{j = 1}^{\num}\uc_j}}^{1 - \alpha} \beta\, \paren*{L^h_{\ell}(r, x, y) - \inf_{r\in \sR}\E_{y | x}\bracket*{L^h_{\ell}(r, x, y)}}^{\alpha}
\end{align*}
where we use the fact that $\Gamma(t) = \beta t^{\alpha}$, $\alpha \in (0, 1]$, $\beta > 0$, $\sfL \leq \ul$ and $c_j \leq \uc_j$, $j = \curl*{1, \ldots, \num}$ in the last inequality.
After upper-bounding the first term and the second term in \eqref{eq:tshr-cond-reg-def} as above, taking the expectation on both sides and using Jensen's inequality, we obtain
\begin{align*}
   \sE_{\ldef}(h, r) - \sE_{\ldef}^*(\sH,\sR) + \sM_{\ldef}(\sH,\sR) & \leq \sE_{\sfL}(h) - \sE_{\sfL}(\sH) + \sM_{\sfL}(\sH)\\
   & \quad + \ov \Gamma\paren*{\sE_{L^h_{\ell}}(r) -  \sE_{L^h_{\ell}}^*(\sR) + \sM_{L^h_{\ell}}(\sR)},
\end{align*}
where $\ov \Gamma(t) = \paren*{\num\paren*{\ul + \sum_{j = 1}^{\num}\uc_j}}^{1 - \alpha} \beta\, t^{\alpha}$.
\end{proof}

\section{Common margin-based losses and corresponding deferral surrogate losses}
\label{app:sur-binary}
\begin{table}[t]
\caption{Common margin-based losses and corresponding deferral surrogate losses.}
  \label{tab:sur-binary}
  \centering
  \resizebox{\columnwidth}{!}{
  \begin{tabular}{@{\hspace{0cm}}lll@{\hspace{0cm}}}
    \toprule
      Name & $\Phi(u)$ & Deferral surrogate loss $\ell_{\Phi}$\\
    \midrule
     Exponential & $\Phi_{\rm{exp}}(u) = e^{-u}$ & $\sfL(h(x), y) e^{r(x)} + c(x,y) e^{-r(x)}$    \\
     Logistic & $\Phi_{\rm{log}}(u) = \log\paren*{1 + e^{-u}}$ & $\sfL(h(x), y) \log\paren*{1 + e^{r(x)}} + c(x,y) \log\paren*{1 + e^{-r(x)}}$ \\
     Quadratic & $\Phi_{\rm{quad}}(u) = \max\curl*{1 - u, 0}^2$ & $\sfL(h(x), y) \Phi_{\rm{quad}}(-r(x)) + c(x,y) \Phi_{\rm{quad}}(r(x))$ \\
     Hinge & $\Phi_{\rm{hinge}}(u) = \max\curl*{1 - u, 0}$ & $\sfL(h(x), y) \Phi_{\rm{hinge}}(-r(x)) + c(x,y) \Phi_{\rm{hinge}}(r(x))$ \\
     Sigmoid & $\Phi_{\rm{sig}}(u) = 1 - \tanh(k u), k > 0$ & $\sfL(h(x), y) \Phi_{\rm{sig}}(-r(x)) + c(x,y) \Phi_{\rm{sig}}(r(x))$  \\
     $\rho$-Margin & $\Phi_{\rho}(u) = \min\curl*{1,  \max\curl*{0, 1-\frac{u}{\rho}}}, \rho>0$ & $\sfL(h(x), y) \Phi_{\rho}(-r(x)) + c(x,y) \Phi_{\rho}(r(x))$ \\ 
    \bottomrule
  \end{tabular}
  }
\end{table}

\newpage
\section{Additional experiments}
\label{app:additional_experiments}

Here, we report additional experimental results with three simple
baselines:
\begin{itemize}
\item Baseline 1: The accuracy of the expert.

\item Baseline 2: Always defer to one expert (random or not random) with probability a\%.

\item Baseline 3: Single-expert formulation using only expert 1 (or 2, or 3).

\end{itemize}

In Table~\ref{tab:additional-reg}, we report the empirical results of our
two-stage method without base cost on the Housing dataset alongside
the corresponding baselines, which further demonstrates our approach's
effectiveness. For our method, the single-expert deferral ratio is $91
\%$, the two-expert deferral rate is $8 \%$ for the first expert and
$85 \%$ for the second expert, and the three-expert deferral rate is
$4 \%$ for the first expert, $35 \%$ for the second expert, and $60
\%$ for the third expert. We use the same deferral rate for randomly
deferring to experts in Baseline 2. The error of the base model is
$22.72 \pm 7.68$. EXP represents the expert used, and system MSE
values are reported. Clearly, our method outperforms all three
baselines.

\begin{table*}[t]
 \caption{Comparison of our proposed method with three simple baselines.}
 \label{tab:additional-reg}
  \centering
  \resizebox{\textwidth}{!}{
  \begin{tabular}{@{\hspace{0cm}}lll|lll|lll|lll@{\hspace{0cm}}}
    \multicolumn{3}{c}{Baseline 1} & \multicolumn{3}{c}{Baseline 2} & \multicolumn{3}{c}{Baseline 3} & \multicolumn{3}{c}{Ours} \\
    \midrule
    EXP 1  & EXP 2 & EXP 3 & EXP 1 & EXP 1, 2 & EXP 1, 2, 3 & EXP 1 & EXP 2 & EXP 3 & EXP 1 & EXP 1, 2 & EXP 1, 2, 3\\
    \midrule
    $17.37 \pm 4.80$ & $15.07 \pm 3.03$ & $12.72 \pm 2.30$ & $17.77 \pm 5.12$ & $15.43 \pm 2.83$ & $12.92 \pm 2.45$ & $16.26 \pm 5.58$ & $15.44 \pm 2.25$ & $12.36 \pm 3.32$ & $\mathbf{16.26 \pm 5.58}$ & $\mathbf{14.82 \pm 3.60}$ & $\mathbf{12.02 \pm 1.97}$ \\
    \bottomrule
  \end{tabular}
  }
\end{table*}
\restoreatoc



\cleardoublepage
\phantomsection


 
\addtocontents{toc}{\protect\setcounter{tocdepth}{0}}

\disableatoc
\addcontentsline{toc}{chapter}{Bibliography}
\restoreatoc

\printbibliography

\addtocontents{toc}{\protect\setcounter{tocdepth}{2}}
\end{document}